\documentclass[11pt]{article}

\usepackage{fullpage}

\usepackage[T1]{fontenc}
\usepackage[utf8]{inputenc}
\usepackage{lmodern}
\usepackage{microtype}

\usepackage{amsmath,mathtools}
\usepackage{amsfonts,amssymb}
\usepackage{amsthm,thmtools}
\usepackage{bm}
\usepackage{bbm}
\usepackage{accents}
\usepackage{placeins}
\usepackage{wrapfig}
\usepackage{graphicx}
\usepackage{capt-of}
\usepackage{xcolor}
\usepackage{booktabs}
\usepackage{array}
\usepackage{tabularx}
\usepackage{threeparttable}
\usepackage{caption}
\usepackage{subcaption}
\usepackage{pifont}
\usepackage{enumitem}

\usepackage{algorithm}
\usepackage{algpseudocode}

\usepackage[round]{natbib}
\usepackage[colorlinks=true,citecolor=blue,linkcolor=red,urlcolor=magenta]{hyperref}
\usepackage{cleveref}

\crefformat{equation}{(#2#1#3)}
\Crefformat{equation}{(#2#1#3)}
\crefrangeformat{equation}{(#3#1#4)--(#5#2#6)}
\Crefrangeformat{equation}{(#3#1#4)--(#5#2#6)}
\crefmultiformat{equation}{(#2#1#3)}{ and (#2#1#3)}{, (#2#1#3)}{, and (#2#1#3)}
\Crefmultiformat{equation}{(#2#1#3)}{ and (#2#1#3)}{, (#2#1#3)}{, and (#2#1#3)}

\newtheorem{theorem}{Theorem}
\newtheorem{lemma}[theorem]{Lemma}
\newtheorem{proposition}[theorem]{Proposition}

\newtheorem{assumption}[theorem]{Assumption}
\newtheorem{remark}[theorem]{Remark}
\newtheorem{condition}[theorem]{Condition}

\crefname{assumption}{Assumption}{Assumptions}
\Crefname{assumption}{Assumption}{Assumptions}

\newcommand{\para}[1]{\paragraph{#1.}}

\newcommand{\R}{\mathbb{R}}
\newcommand{\N}{\mathbb{N}}

\newcommand{\B}{\mathbb{B}}
\newcommand{\I}{\mathbf{I}}

\newcommand{\given}{\nonscript\;\delimsize\vert\nonscript\;}

\DeclarePairedDelimiterX{\cbrp}[1]{\lbrace}{\rbrace}{#1}
\DeclarePairedDelimiterX{\brp}[1]{(}{)}{#1}
\DeclarePairedDelimiterX{\sqbrp}[1]{[}{]}{#1}
\DeclarePairedDelimiterX{\normp}[1]{\lVert}{\rVert}{#1}
\DeclarePairedDelimiterX{\absp}[1]{\lvert}{\rvert}{#1}
\DeclarePairedDelimiterX{\ipp}[1]{\langle}{\rangle}{#1}

\newcommand{\cbr}[1]{\ensuremath{\mathchoice{\cbrp*{#1}}{\cbrp{#1}}{\cbrp{#1}}{\cbrp{#1}}}}
\newcommand{\br}[1]{\ensuremath{\mathchoice{\brp*{#1}}{\brp{#1}}{\brp{#1}}{\brp{#1}}}}
\newcommand{\sqbr}[1]{\ensuremath{\mathchoice{\sqbrp*{#1}}{\sqbrp{#1}}{\sqbrp{#1}}{\sqbrp{#1}}}}
\newcommand{\norm}[1]{\ensuremath{\mathchoice{\normp*{#1}}{\normp{#1}}{\normp{#1}}{\normp{#1}}}}
\newcommand{\abs}[1]{\ensuremath{\mathchoice{\absp*{#1}}{\absp{#1}}{\absp{#1}}{\absp{#1}}}}

\DeclareMathOperator*{\argmax}{arg\,max}

\renewcommand{\P}{\mathbb{P}}
\newcommand{\E}{\mathbb{E}}

\newcommand{\ind}[1]{\ensuremath{\mathbf{1}\,\cbr{#1}}}
\DeclareMathOperator{\Bern}{Bern}

\DeclareMathOperator{\Unif}{Unif}

\renewcommand{\v}[1]{\boldsymbol{#1}}

\newcommand{\tp}{^{\mathsf{T}}}

\renewcommand{\cal}[1]{\mathcal{#1}}

\newcommand{\what}[1]{\widehat{#1}}
\newcommand{\wtilde}[1]{\widetilde{#1}}

\DeclareRobustCommand{\term}[1]{%
  \ifmmode
    A_{\text{#1}}%
  \else
    term~\ensuremath{A_{\text{#1}}}%
  \fi
}
\newcommand{\uterm}[2]{\ensuremath{\underbrace{#1}_{\term{#2}}}}

\DeclareRobustCommand{\case}[1]{\ensuremath{\text{(case~#1)}}}
\newcommand{\udef}[2]{\ensuremath{\underbrace{#1}_{\coloneqq\,#2}}}

\newcommand{\eps}{\epsilon}
\newcommand{\veps}{\varepsilon}

\DeclarePairedDelimiterX{\parens}[1]{(}{)}{#1}

\newcommand{\bigO}[1]{\ensuremath{\mathchoice{\mathcal{O}\parens*{#1}}{\mathcal{O}\parens{#1}}{\mathcal{O}\parens{#1}}{\mathcal{O}\parens{#1}}}}
\newcommand{\tildeO}[1]{\ensuremath{\mathchoice{\widetilde{\mathcal{O}}\,\parens*{#1}}{\widetilde{\mathcal{O}}\,\parens{#1}}{\widetilde{\mathcal{O}}\,\parens{#1}}{\widetilde{\mathcal{O}}\,\parens{#1}}}}

\newcommand{\SlashComment}[1]{%
  \Statex \hspace*{\algorithmicindent}%
  \textcolor{blue}{\ttfamily // #1}%
}
\algnewcommand\algorithmicinit{\textbf{initialize}}
\algnewcommand\Init{\item[\algorithmicinit]}
\algnewcommand\algorithmicinput{\textbf{input}}
\algnewcommand\Input{\item[\algorithmicinput]}

\newcommand{\note}[1]{{\color{blue}#1}}

\title{
Learning to Route and Schedule LLMs from User Retrials via Contextual Queueing Bandits
}

\author{
Seoungbin Bae$^{1}$ \quad Junyoung Son$^{2}$ \quad Dabeen Lee$^{3}$\\[0.3em]
$^{1}$Department of Industrial \& Systems Engineering, KAIST\\
$^{2}$Graduate School of Data Science, KAIST\\
$^{3}$Department of Mathematical Sciences, Seoul National University\\
\texttt{sbbae31@kaist.ac.kr, jun.son@kaist.ac.kr, dabeenl@snu.ac.kr}
}

\date{}

\begin{document}
\maketitle

\begin{abstract}
Explosive demands for LLMs often cause user queries to accumulate in server queues, requiring efficient routing (query-LLM matching) and scheduling (query prioritization)  mechanisms. Several online algorithms are being deployed, but they overlook the following two key challenges inherent to conversational LLM services: (1) unsatisfied users may retry queries, increasing the server backlog, and (2) requests for ``explicit" feedback, such as ratings, degrade user experiences.
In this paper, we develop a joint routing and scheduling algorithm that leverages ``implicit" feedback inferred from user retrial behaviors.
The key idea is to propose and study the framework of contextual queueing bandits with multinomial logit feedback (CQB-MNL). CQB-MNL models query retrials, as well as context-based learning for user preferences over LLMs.
Our algorithm, anytime CQB (\textbf{ACQB}), achieves efficient learning while maintaining queue stability by combining Thompson sampling with forced exploration at a decaying rate.
We show that ACQB simultaneously achieves a cumulative regret of $\tildeO{\sqrt{t}}$ for routing and a queue length regret of $\tildeO{t^{-1/4}}$ for any large $t$. For experiments, we refine query embeddings via contrastive learning while adopting a disjoint parameter model to learn LLM-specific parameters.
Experiments on synthetic data, offline routing datasets (\textsc{SPROUT}, \textsc{EmbedLLM}, and \textsc{RouterBench}), and real user conversation logs (\textsc{WildChat-1M}) confirm that our methods improve routing, scheduling, and queue stability against strong online and offline-trained baselines.
\end{abstract}

\section{Introduction}

Recent advances in large language models (LLMs) have revolutionized various domains, driving explosive demands for LLM-based applications \citep{achiam2023gpt,team2023gemini,dubey2024llama}.
However, running LLMs is computationally intensive and resource-demanding.
Therefore, a flood of user queries, coupled with limited server capacity, inevitably leads to significant query accumulation in system queues.
Hence, efficient query routing and scheduling is crucial to ensure high-quality user experiences and system sustainability \citep{kwon2023efficient,agrawal2024taming}.

To mitigate such system congestion, LLM routing and scheduling frameworks have emerged as necessary components.
In a typical system of multiple LLMs with varying types and levels of capabilities, performance, and costs, the system must decide which model to serve a given query (routing) and in what order to process backlogged queries (scheduling) \citep{mitzenmacher2025queueing}.
The first approaches have relied on policies learned offline \citep{ong2024routellm,feng2024graphrouter,fu2024efficient}, but these methods often fail to adapt to the dynamic and non-stationary nature of online environments. 
This limitation has motivated the adoption of online learning frameworks \citep{chiang2025llm,jitkrittum2025universal}, which learn optimal policies through sequential interactions with the environment. 

However, existing online learning approaches are often limited in practice as they overlook two key challenges inherent to conversational LLM services.
First, they fail to account for the subsequent impact of user dissatisfaction.
When users encounter unsatisfactory responses, they often retry queries.
This retrial behavior increases server backlog, exacerbating congestion.
This introduces a trade-off between exploration and system stability: exploring potentially sub-optimal models may induce user retrials that increase queue length.
%to learn optimal policies, the system must perform exploration (potentially assigning sub-optimal models), yet this very act carries the risk of triggering retrials that can increase the queue length.
Second, most algorithms rely on explicit user feedback (e.g., ratings, preference information) to update their policies.
In practice, however, such feedback is sparse, as users are often reluctant to provide it.
Furthermore, mandating explicit feedback can disrupt conversation flows and degrade user experiences.

In this paper, we address these challenges by presenting a novel formulation for the joint routing and scheduling problem via contextual queueing bandits with multinomial logit feedback (CQB-MNL).
Specifically, CQB-MNL considers a discrete-time queueing system with $N$ LLMs, where in each round a query with some context information arrives at the queue. An algorithm for CQB-MNL chooses a pending query with the highest priority (scheduling) and recommends an assortment of $K$ LLMs for response generation (routing) (See \Cref{fig:main}). Here, the case $K=1$ corresponds to the standard single-response setting, which is essentially a dynamic matching problem.
The case $K=2$ represents the pairwise comparison setting, commonly employed to collect human preferences for reinforcement learning from human feedback (RLHF), where a user receives two candidate responses and is asked to select one. In CQB-MNL, the choice behavior of a user is modeled with the multinomial logit (MNL) framework \citep{agrawal2019mnl,oh2019thompson}. Basically, the user accepts one of $K$ candidate responses, indicating satisfaction, or rejects all and retries the query, indicating dissatisfaction, and the choice decision is governed by the MNL model. This formulation enables the system to learn optimal policies based on implicit feedback inferred from user retrials.

However, an online learning algorithm for CQB-MNL requires exploration, which may involve suggesting suboptimal models, and as a result, it inherits the risk of triggering retrials, causing system congestion. Therefore, we need an algorithm that cleverly balances exploration and system stability. To address this requirement, we propose an algorithm, anytime CQB (ACQB), that achieves efficient learning of the underlying MNL model while maintaining queue stability. The main backbone of ACQB is Thompson sampling for MNL, while it enforces uniform exploration steps.

Our contributions are summarized as follows:

\begin{itemize}[itemsep=0pt, topsep=0pt, parsep=0pt, partopsep=0pt]
    \item We introduce CQB-MNL, a novel online learning framework designed to address the two critical challenges inherent to online learning for LLM routing and scheduling: system congestion due to user dissatisfaction triggering retrials and impracticality of relying on explicit feedback. 
    
    \item We propose ACQB for CQB-MNL that allows efficient learning of the unknown MNL model while guaranteeing queue stability. In each time step, Thompson sampling or uniform exploration is performed. Here, the probability of running uniform exploration decreases as time goes on.

    \item ACQB is an ``anytime" algorithm suitable for continuous LLM services. It works without prior knowledge of the time horizon or the traffic slackness parameter. These parameters are usually unavailable a priori. We establish that ACQB simultaneously achieves a cumulative regret of $\tildeO{\sqrt{t}}$ for learning the MNL model and a queue length regret of $\tildeO{t^{-1/4}}$, ensuring convergence to optimal queue length, for any large $t$. 

    \item For experiments, we employ disjoint parameterization to capture individual LLM-specific characteristics (\Cref{sec:disjoint}). Moreover, we refine query embeddings via contrastive learning. Experiments on synthetic data and offline routing datasets (\textsc{EmbedLLM}, \textsc{SPROUT}, \textsc{RouterBench}) demonstrate that our methods consistently outperform queueing bandit baselines and offline-trained routing baselines. We also test the scheduling component of our method on real user conversation logs from \textsc{WildChat-1M} by replaying timestamped conversations, where retrials are extracted from repeated or highly similar queries by the same user.
\end{itemize}

\section{Problem formulation}
\label{sec:prelim}

\begin{figure}[t] % [t]는 페이지 상단 배치, [h]는 현재 위치
    \centering
    % width=\linewidth 로 설정하면 칼럼 폭에 딱 맞게 들어갑니다.
    \includegraphics[width=0.55\linewidth]{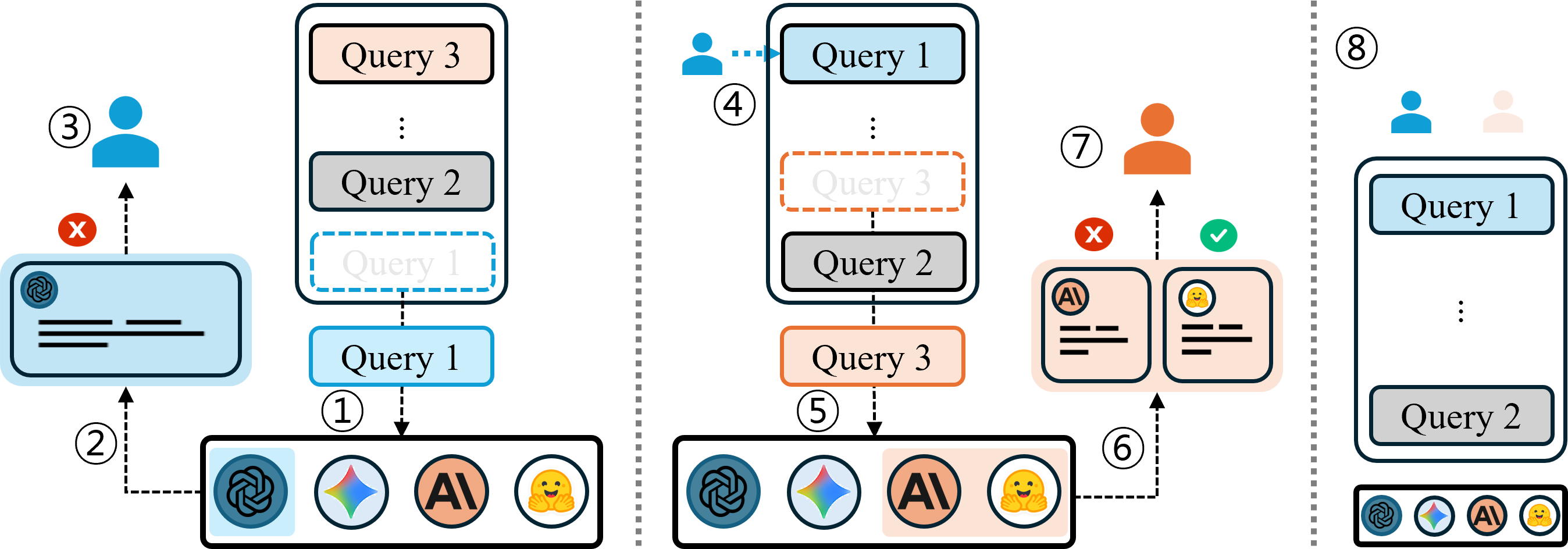}
    % \vspace{-0.2cm} % 그림과 캡션 사이 간격 조절 (필요시 사용)
    \caption{Illustration demonstrating retrial and departure dynamics. 
\textcolor{cyan}{($K=1$):} \textcircled{1} The agent schedules a query (Query 1). \textcircled{2} An assortment of size $K=1$ is assigned. \textcircled{3} The user is dissatisfied with the response, which \textcircled{4} triggers a retrial. 
\textcolor{orange}{($K=2$):} \textcircled{5} The agent schedules a query (Query 3). \textcircled{6} An assortment of size $K=2$ is assigned. \textcircled{7} The user selects one of the responses (satisfaction), and consequently, \textcircled{8} the query departs the queue.}
    \label{fig:main}
\end{figure}

In this section, we formulate the online learning of LLM routing and scheduling with users’ retrials as a \emph{contextual queueing bandit} problem. We consider a discrete-time queueing environment with a single queue and $N$ LLMs.
At each round $t$, a user query (together with its context) arrives and joins the queue. The agent then performs a joint operation: it selects a pending query to process from the queue (scheduling) and assigns an appropriate LLM (or an assortment of LLMs) to the selected query (routing).
Formally, at round $t$, the agent observes a queue state $\mathcal{X}_t \subseteq \mathcal{X} \subset \mathbb{R}^d$, given by the set of contexts of the remaining queries, and chooses a query context $x_t \in \mathcal{X}_t$. The corresponding feature vector of $x_t$ for each LLM $j \in [N]$ is denoted as $x_{t,j}$ (or $x_{-j}$ for an arbitrary context $x$). The agent then chooses an assortment of servers $S_t \in \mathcal{C}$, where $\mathcal{C} := \{S \subseteq [N] : |S| = K\}$ and $K$ is the assortment size.
For simplicity, we fix $K$ throughout the paper; in practice, one may deploy a policy where the agent primarily uses a single LLM ($K=1$) and occasionally generates two answers ($K=2$) when additional feedback is required, and the same analysis applies.

We characterize user satisfaction through follow-up behavior. If the user is satisfied---specifically, accepting the answer in the single-response setting ($K=1$) or selecting a preferred answer in the pairwise setting ($K=2$)---the corresponding query departs the queue. Conversely, if the user is dissatisfied, they exhibit retrials, which we model as the query remaining, or re-entering the queue. We describe these dynamics formally as follows: Let $Q(t) = |\mathcal{X}_t|$ denote the queue length at the beginning of round $t$. Let $A(t) \in \{0,1\}$ indicate the random arrival of a new job, and $D(t) \in \{0,1\}$ denote the random departure at time $t$. When $A(t)=1$, we denote the context of the newly arriving job as $x^{(t)}$. The queue state then evolves as:
\begin{align*}
    \cal X_{t+1} = \cal X_t \setminus\{x_t: D(t)=1\} \cup \{x^{(t)}:A(t)=1\},
\end{align*}
and the queue length evolves according to $Q(t+1) = [Q(t) + A(t) - D(t)]^+$, where $[z]^+ = \max\{0, z\}$. For technical convenience, if the queue is empty ($Q(t)=0$), the agent selects a dummy job $x_0 \in \mathbb{R}^d$, and the resulting feedback is not used for learning.

We model the user choice probability using the MNL model \citep{agrawal2019mnl,oh2019thompson}. The agent observes a binary choice vector $y_t = [y_{t0},y_{t1},\dots,y_{tK}] \in \{0,1\}^{K+1}$, where $\sum_{j\in[K]\cup\{0\}} y_{tj}=1$ and $y_{tj}=1$ if the user chooses the $j$-th LLM in the assortment $S_t$, and $0$ otherwise. We designate the outside option (no choice) as $y_{t0}=1$, representing user dissatisfaction. The probability that a user chooses item $j$ from assortment $S$ given context $x$ is defined as $p_j(x, S, \theta^*)$, where $\theta^* \in \mathbb{R}^d$ is the unknown parameter, and $p_{j}(x, S, \theta) := \frac{\exp(x_{-j}\tp \theta)}{1 + \sum_{j'\in S} \exp (x_{-j'}\tp\theta)}$ for $j\in \cal S$, and $p_{0}(x, S, \theta) := \frac{1}{1 + \sum_{j'\in S} \exp (x_{-j'}\tp\theta)}$ for the outside option ($j=0$). Therefore, a departure $D(t)=1$ occurs if the user selects any valid option in $S_t$ (i.e., avoids the outside option). Thus, $D(t)$ is a random variable with mean $R(x_t, S_t, \theta^*)$, where the success probability is given by 
\begin{align*}
R(x,S,\theta) := \sum_{j\in S} p_j(x, S, \theta) = \frac{\sum_{j\in S}\exp(x_{-j}^\top \theta)}{1 + \sum_{j'\in S} \exp(x_{-j'}^\top \theta)}.
\end{align*} 
Finally, we assume the arrival $A(t)$ is a random variable with mean $\lambda$.

We evaluate the performance of the agent using two key measures: (i) \emph{queue length regret} and (ii) \emph{cumulative regret}. Let $\pi$ denote the agent's policy and $\pi^*$ denote the optimal policy that has prior knowledge of the true parameter $\theta^*$. Given a set of remaining features $\cal Y \subseteq \cal X$, $\pi^*$ selects the query-assortment pair $(x^*, S^*)$ that maximizes the expected departure rate, i.e., $(x^*, S^*) \in \argmax_{x\in\mathcal{Y}, S\in\mathcal{C}} R(x,S,\theta^*)$. Now, we define the queue length regret at round $t$ as
\begin{align*}
R_t := \mathbb{E}[Q(t) - Q^*(t)],
\end{align*}
where $Q^*(t)$ is the queue length under the optimal policy $\pi^*$. To minimize $R_t$, the agent must learn both which query to serve and which LLM assortment to assign. Accordingly, this metric simultaneously evaluates the routing and scheduling performance. Moreover, to evaluate the routing performance in terms of departure-rate optimality, we consider cumulative regret, a standard measure in the bandit literature. Let ${S_i^*} = \argmax_{S\in\mathcal{C}} R(x_i, S, \theta^*)$ be the optimal assortment for the fixed query $x_i$ chosen by the agent. The standard cumulative regret is defined as 
\begin{align*}
    \text{Regret}_t := \sum_{i=1}^t \E\sqbr{R(x_i, {S_i^*}, \theta^*) - R(x_i, S_i, \theta^*)}.
\end{align*}
% \begin{remark}
% We analyze a stronger regret definition. Let $(x_i^*, S_i^*) = \argmax_{x\in\mathcal{X}_i, S \in\mathcal{C}} R(x, S, \theta^*)$ be the optimal pair chosen from the queue state $\mathcal{X}_i$. Define
% \begin{align*}
% \text{Regret}_t := \sum_{i=1}^t \mathbb{E}\sqbr{R(x_i^*, S_i^*, \theta^*) - R(x_i, S_i, \theta^*)}.
% \end{align*}
% Since $\text{Regret}'_t \leq \text{Regret}_t$ holds by definition, establishing a bound on $\text{Regret}_t$ suffices.
% \end{remark}
Lastly, we introduce the assumptions as follows:
\begin{assumption} \label{ass:basic}
$\|x_{-j}\|_2 \leq 1$ for all $x\in\cal X,j\in[N]$. Also, $\|\theta^*\|\leq 1$.
\end{assumption}

\begin{assumption} \label{ass:constants}
    There exist $\kappa >0$ such that for all $x\in\cal X, j\in[N]$, $\inf_{S\in\cal C,\theta\in\R^d} p_j(x,S,\theta) p_0(x,S,\theta) \geq 1/\kappa$.
\end{assumption}

\begin{assumption} \label{ass:iid}
The features of newly arriving jobs are assumed to be independently and identically distributed (i.i.d.) from an unknown distribution $\cal D$. Moreover, there exists $\sigma_0 > 0$ such that $\lambda_{\min}(\E_{x\sim\mathcal D}[\frac{1}{N}\sum_{j\in[N]} x_{-j} x_{-j}^\top]) \geq \sigma_0^2$.
\end{assumption}

\begin{assumption} \label{ass:slackness}
There exists some traffic slackness $\eps>0$ such that for each $x\in\cal X$, there exists an corresponding assortment $S^*(x)\in\cal C$ with $R(x,S^*(x),\theta^*)-\lambda\geq \epsilon$.
\end{assumption}

\Cref{ass:basic} states that the norms of the feature vector $x_{-j}$ and the unknown parameter $\theta^*$ are bounded. \Cref{ass:constants} introduces problem-dependent parameters that control the local behavior of $\dot\mu(\cdot)$. \Cref{ass:iid} imposes a regularity assumption on the underlying distribution. \Cref{ass:slackness} specifies a traffic slack condition to guarantee stability, which is standard in the queueing bandit literature \citep{krishnasamy2016regret,bae2026queuelengthregretbounds}.

\section{Algorithm for contextual queueing bandits with multinomial logit feedback}

\begin{algorithm} [t] 
\caption{ACQB} \label{alg:1}
\begin{algorithmic} [1] 
    \Init design matrix $V_0=\lambda_0\I$, assortment size $K$, sample size $M$, combination set $\cal C=\{S\subset [N]: |S|=K\}$, combination counter $c=0$, exploration parameter $\eta(t)$, confidence radius $\alpha_t$
    \For{$t=1,\dots$}
        \If{$A(t-1)=1$ \textbf{and} $E(t-1)=1$}  \Comment{\note{on arrivals, $\eta(t)$-exploration}}
            \State Set $x_t \leftarrow x^{(t-1)}$, $S_t \leftarrow \cal C[c+1]$
            \State $c \leftarrow c+1 \pmod{|\cal C|}$
        \Else
            \State Sample $\{\wtilde\theta_{t-1}^{(i)}\}_{i=1}^{M} \sim \cal N(\what\theta_{t-1}, \alpha_{t-1}^2 V_{t-1}^{-1})$ 
            \State Set $x_t, S_t \leftarrow \argmax_{x\in\cal X_t, S\in\cal C} \widetilde R(x, S)$ 
        \EndIf
        \State Assign $x_t$ to $S_t$, receive $y_t=(y_{t0},y_{t1}, \dots,y_{tK})$
        \State Update $\what \theta_t$ as in \Cref{eq:mle}
        \State $V_{t} \leftarrow V_{t-1} + \sum_{j\in S_t} x_{tj} x_{tj}\tp$ 
        \State Sample $E(t)\sim \Bern(\eta(t))$
    \EndFor
\end{algorithmic}
\end{algorithm}

In this section, we introduce our proposed algorithm, ACQB, which is illustrated in \Cref{alg:1}. The algorithm consists of two branches: (i) random exploration and (ii) a Thompson sampling-based optimistic rule. Let $E(t)\sim \Bern(\eta(t))$ be the random variable that indicates whether we run random exploration in round $t+1$, where $\eta(t) = \min\{1,c_1(t+1)^{-1/2}\}$ is the exploration parameter for some absolute constant $c_1>0$. Then, in each round $t$, if there is a new job arrival ($A(t-1)=1$), we perform random exploration with probability $\eta(t-1)$ (i.e., when $E(t-1)=1$) by selecting the newly arriving job $x^{(t-1)}$ and choosing the LLM assortment in a round-robin manner. If this event (i.e., $A(t-1)=1$ and $E(t-1)=1$) does not occur, we choose the query-assortment pair according to the Thompson sampling-based optimistic rule \citep{oh2019thompson}: (Line~6) sample $\{\wtilde\theta^{(i)}_{t-1}\}_{i=1}^M$ from the Gaussian distribution $N(\what\theta_{t-1}, \alpha_{t-1}^2 V_{t-1}^{-1})$, where $M=\lceil 1 - \frac{\log (K)}{\log(1-1/(4\sqrt{e\pi}))} \rceil$ and $\alpha_{t-1}$ is a confidence radius defined as
\begin{align*}
\alpha_l = \frac{\kappa}{2} \sqrt{d \log \br{1 + lK / (d\lambda_0)} + 4 \log l} + \kappa \sqrt{\lambda_0},    
\end{align*}
where $\lambda_0>0$ is a regularization parameter.
(Line~7) We then choose the query and LLM assortment pair with the largest optimistic departure rate estimate $\wtilde R(x, S)$, where
\begin{align*}
    \wtilde u_{tj}(x) := \max_i x_{-j}\tp \wtilde\theta_{t-1}^{(i)}, \quad \wtilde R(x, S) := \sum_{k\in S} \frac{\exp(\wtilde u_{tk}(x))}{1 + \sum_{j\in S} \exp(\wtilde u_{tj}(x))}.
\end{align*}
After assigning the query $x_t$ to the LLM assortment $S_t$, we observe the reward vector $y_t$. Finally, we update the maximum likelihood estimator (MLE) $\what\theta_t$ by minimizing the regularized cross-entropy loss
\begin{align}
    \cal L_t(\theta) = \frac{\lambda_0}{2} \|\theta\|_2^2 -\sum_{i=1}^t \sum_{j\in S_i\cup\{0\}} y_{ij} \log p_{j}(x_i, S_i,\theta) . \label{eq:mle}
\end{align}
Now we introduce our main results, which provide the queue length regret and cumulative regret bound of \Cref{alg:1}. The proof sketches are provided in \Cref{sec:regret}, and the full proofs are deferred to \Cref{sec:regret_formal}.
\begin{theorem} \label{thm:regret1}
For any large $t$ (\Cref{cond:t}), we have
\begin{align*}
    R_t = \bigO{\frac{d^5t^{-1/4}\log^5(t)}{\sigma_0^9 \eps^5} + \frac{d^{11/2}t^{-1}\log^5(t)}{\sigma_0^{12}\eps^5}}.
\end{align*}
\end{theorem}

\begin{theorem} \label{thm:regret2}
    For any $t\geq 1$, we have
    \begin{align*}
        \text{Regret}_t = \wtilde{\cal O}(d^{3/2} \sqrt{t}).
    \end{align*}
\end{theorem}

\para{Comparison with previous works}

Closely related works in queueing bandits literature include \citet{krishnasamy2016regret,kim2024queueing,bae2026queuelengthregretbounds}.
First, \citet{krishnasamy2016regret} pioneered the analysis of queue length regret. 
However, their multi-armed bandit (MAB)-based framework ignores query contexts and assumes fixed departure rates, making it unsuitable for our setting.
Later, \citet{kim2024queueing} address contexts but restrict it to a single fixed type per queue, unlike our arbitrary setting, and they lack an analysis of queue length regret.
The most comparable study is Algorithm 1 of \citet{bae2026queuelengthregretbounds}, which accommodates arbitrary contexts and analyzes queue length regret. Notably, it yields a tighter leading term of $\cal O(d^{3/2}T^{-1/4}\log^{3/2}(T) \sigma_0^{-5}\eps^{-3})$ compared to the $\cal O(d^5T^{-1/4}\log^5(T)\sigma_0^{-9} \eps^{-5})$ term of our \Cref{alg:1}.
The heavier dependence on $d$ mainly comes from the Thompson-sampling analysis and the anytime exploration design, rather than from the MNL feedback model itself. If we replace Thompson sampling with a UCB rule in our framework, the queue length regret improves to $\wtilde{\cal O}(d^{5/2}T^{-1/4})$, and the cumulative regret improves to $\wtilde{\cal O}(d\sqrt{T})$.
In contrast, their approach relies on an exploration schedule dependent on the horizon $T$ and slackness $\epsilon$, which fails to satisfy the anytime property, making it unsuitable for continuous LLM serving (see \Cref{remark:tau}). Finally, their model considers only the logistic function, requiring non-trivial extension to the MNL model.

% Algorithm 1 of \citet{bae2026queuelengthregretbounds} consists of two phases. For the first $\tau$ rounds, it runs a pure-exploration phase in which, whenever a new job arrives, it always performs random exploration. Afterward, it applies an $\eta$-exploration policy: if a new job arrives, it performs random exploration with probability $\eta=T^{-1/2}$, where $T$ is the known horizon length; otherwise, it selects an optimistic pair by maximizing the UCB term. This algorithm has two limitations. (i) It requires prior knowledge of the horizon $T$ to set the exploration probability $\eta$ and the length of the exploration phase $\tau$, which is typically unavailable in online services that must run continuously. (ii) Moreover, to set the pure-exploration length $\tau$, it also requires prior knowledge of the slackness parameter $\eps$, since $\tau$ is inversely proportional to $\eps$ as $\tau=\bigO{d\log(T)/\eps^2)}$, which can be unrealistic in some situations.

% In contrast, our \Cref{alg:1} bypasses both limitations by adopting a decaying exploration rate $\eta(t)$. As a result, the algorithm is anytime and does not require prior knowledge of $\eps$, which is important for deployment. Correspondingly, we carefully choose $\tau(t)$, which can be viewed as an exploration length but appears only in the analysis, and show that this choice suffices to meet the condition required of $\tau$ (in terms of $T$ and $\eps$) in Algorithm~1 of \citet{bae2026queuelengthregretbounds}.

\section{Algorithmic extensions}

We introduce two algorithmic adaptations designed to address the challenges of real-world deployment and to enhance performance.

\subsection{Capturing LLM heterogeneity via disjoint parameterization}\label{sec:disjoint}

Many prior works \citep{chiang2025llm,shirkavand2025cost} use a joint feature vector $x_{t,j}$ with a shared parameter vector $\theta^*$.
Although query embeddings $x_t$ can be obtained from pre-trained embedding models \citep{reimers2019sentence,sentence-transformers-allminilm-2021}, constructing $x_{t,j}$ requires additional feature engineering for each LLM.
This can make it unclear whether the gain comes from the routing algorithm or from the feature construction.

Therefore, for the experiments, we adopt a disjoint parameter model to bypass the reliance on manual feature engineering. In this setting, the query context $x_t$ is shared, but each server $j$ is governed by a unique unknown parameter $\theta_j^*$. This allows the agent to learn the unique characteristics of heterogeneous LLMs through online learning, thereby removing the confounding effects of feature engineering and enabling a fair assessment of algorithmic performance.
Accordingly, define the collection of parameters as $\Theta^*=[\theta_1^*, \theta_2^*, \dots, \theta_N^*]\in \R^{d\times N}$. When a user with context $x$ is presented with an assortment $S$, the probability $p_j(x, S, \Theta^*)$ of selecting item $j$ and the corresponding departure rate $R(x, S, \Theta^*)$ are defined as $p_{j}(x, S, \Theta^*) := \frac{\exp(x^\top \theta_j^*)}{1 + \sum_{j'\in S} \exp (x^\top\theta_{j'}^*)}$ for $j\in S$, $p_{0}(x, S, \Theta^*) := \frac{1}{1 + \sum_{j'\in S} \exp (x^\top\theta_{j'}^*)}$ for $j=0$, and $R(x,S,\Theta^*) := \sum_{j\in S} p_j(x, S, \Theta^*)$.
Under this disjoint model, the agent maintains separate statistics for each model $j$: a design matrix $V_{t,j}$, a maximum likelihood estimator $\hat{\theta}_{t,j}$, and a confidence radius $\alpha_{t,j}$. The detailed algorithm is outlined in \Cref{ssec:disjoint}.
\begin{remark}
    Notice that this disjoint parameterization is a specific instance of the shared parameter setting used in our theoretical analysis. Let $e_j \in \mathbb{R}^N$ be the standard basis vector with a $1$ at the $j$-th position and $0$ elsewhere. Then, for all $j \in [N]$, we can write $x\tp \theta_j^* =(x \otimes e_j)\tp \text{vec}(\Theta^*)$ where $\otimes$ denotes the Kronecker product and $\text{vec}(\Theta^*)$ represents the vectorization of a matrix. By viewing $(x \otimes e_j) \in \mathbb{R}^{dN}$ as a feature vector and $\text{vec}(\Theta^*) \in \mathbb{R}^{dN}$ as the shared parameter, our theoretical analysis remains valid.
\end{remark}

\subsection{Utility-aligned query embeddings via contrastive learning}
\label{sec:EQCL}
Standard query encoders are typically trained to align representations with linguistic semantics \citep{reimers2019sentence}, resulting in an embedding space where proximity reflects categorical or paraphrase-level similarity \citep{chiang2025llm}. 
However, effective multi-LLM routing requires embeddings to capture utility alignment rather than semantic closeness. 
That is, two queries should be considered similar if they exhibit similar routing utilities (i.e., performance-cost trade-offs) across varying LLMs, even if they are semantically distinct \citep{chen2024routerdc}. 
To address this, we propose ACQB-CL (\Cref{alg:2,alg:ucl}), which realigns the query representation space. Specifically, given a raw prompt $\xi$, we freeze the backbone encoder $E(\cdot)$ and train a two-layer MLP projection head $B(\cdot;\theta)$ to obtain a utility-aligned representation $z = B(E(\xi);\theta)\in\R^{d'}$. 
The projection head is optimized via an InfoNCE loss \citep{oord2018representation} on an offline dataset to cluster queries with similar utilities while separating divergent ones.
Comprehensive details regarding offline data construction, pair selection, and the precise loss formulation are provided in \Cref{ssec:ucl}.

% \subsection{Enhancing Query Embeddings via Contrastive Learning}
% \label{sec:EQCL}
% Standard query encoders are trained to align representations with linguistic semantics~\citep{reimers2019sentence}, yielding an embedding geometry that groups prompts by categorical or paraphrase level similarity \citep{chiang2025llm}. In multi-LLM routing, however, similarity should be defined by LLMs' performance alignment. Two queries are considered similar if they achieve high performance on the same models regardless of semantic content \citep{chen2024routerdc}. To address this, we propose ACQB-CL (\Cref{alg:2}), which realigns the query representation space. With a raw prompt $\xi$, we freeze the query encoder $E(\cdot)$ and train a two-layer MLP projection head $B(\cdot;\theta):\mathbb{R}^{d_1}\!\to\!\mathbb{R}^{d_2}$ to obtain utility aligned representation, $z_{(i)} = B(E(\xi_{(i)});\theta)$, using offline dataset $\xi_{(i)}\in\xi$. The head is optimized with an InfoNCE loss \citep{oord2018representation} that pulls positive query pairs closer and pushes negative pairs farther apart, where positives and negatives are defined based on utility similarity. Full details of offline data construction, pair selection, the exact loss, and the algorithm are deferred to \Cref{alg:ucl}.

\section{Experiments} \label{sec:experiments}

\newcommand{\best}[1]{\textbf{#1}}

\begin{figure*}[!t]
    \centering
    
    % --- 설정: 너비 변수 (전체 가로 폭을 2등분) ---
    \newcommand{\groupwidth}{0.49\textwidth} 
    
    % ---------------- Header (K=1, K=2) ----------------
    \begin{minipage}{\groupwidth}
        \centering \scriptsize $K=1$
    \end{minipage}
    \hfill
    \begin{minipage}{\groupwidth}
        \centering \scriptsize $K=2$
    \end{minipage}

    \vspace{0.1cm} % 헤더와 그림 사이 간격

    % ---------------- Main Content ----------------
    
    % --- K=1 Group ---
    \begin{minipage}{\groupwidth}
        \centering
        \begin{subfigure}[c]{0.49\linewidth}
            \includegraphics[width=\linewidth]{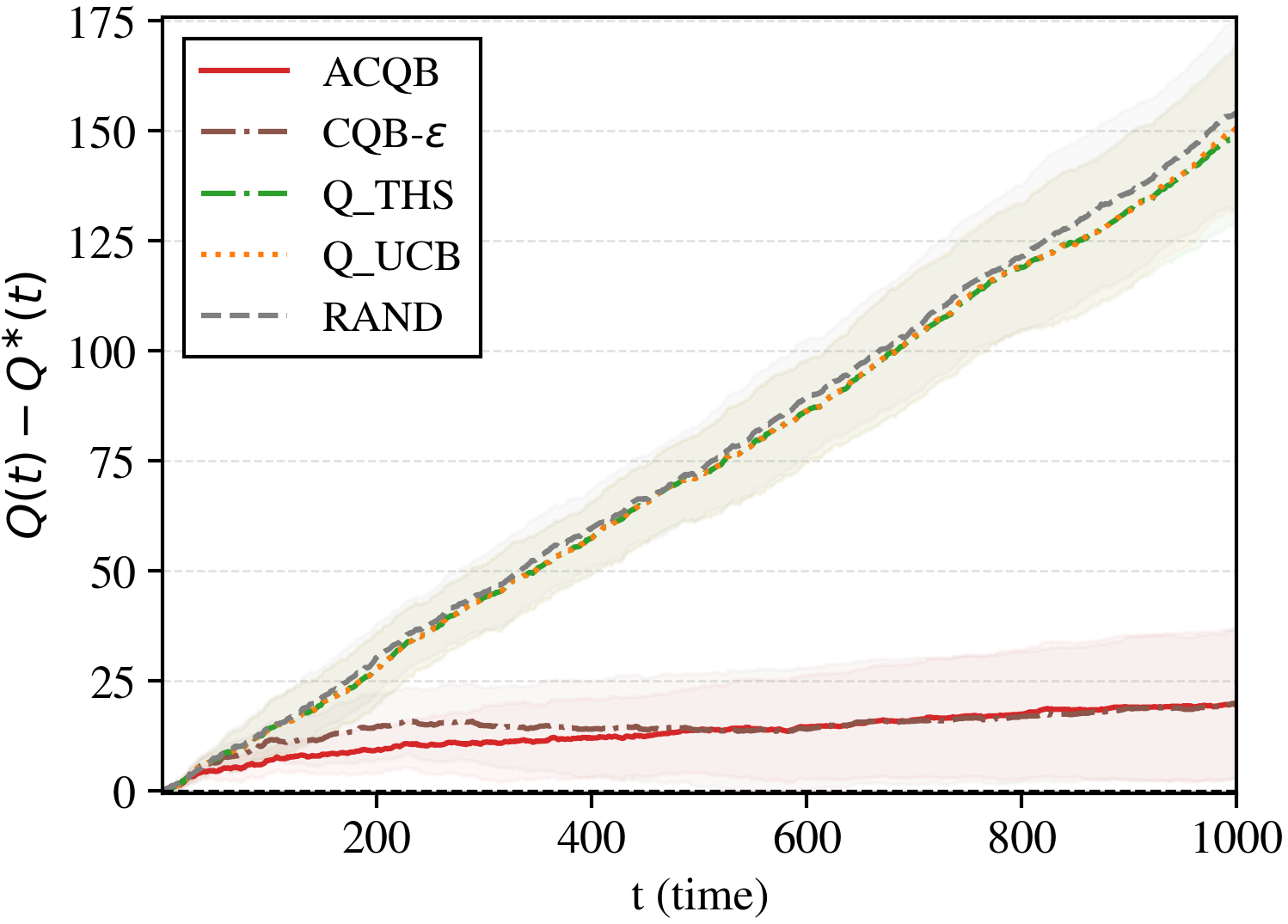}
            \label{fig:sim-k1-q}
        \end{subfigure}
        \hfill
        \begin{subfigure}[c]{0.49\linewidth}
            \includegraphics[width=\linewidth]{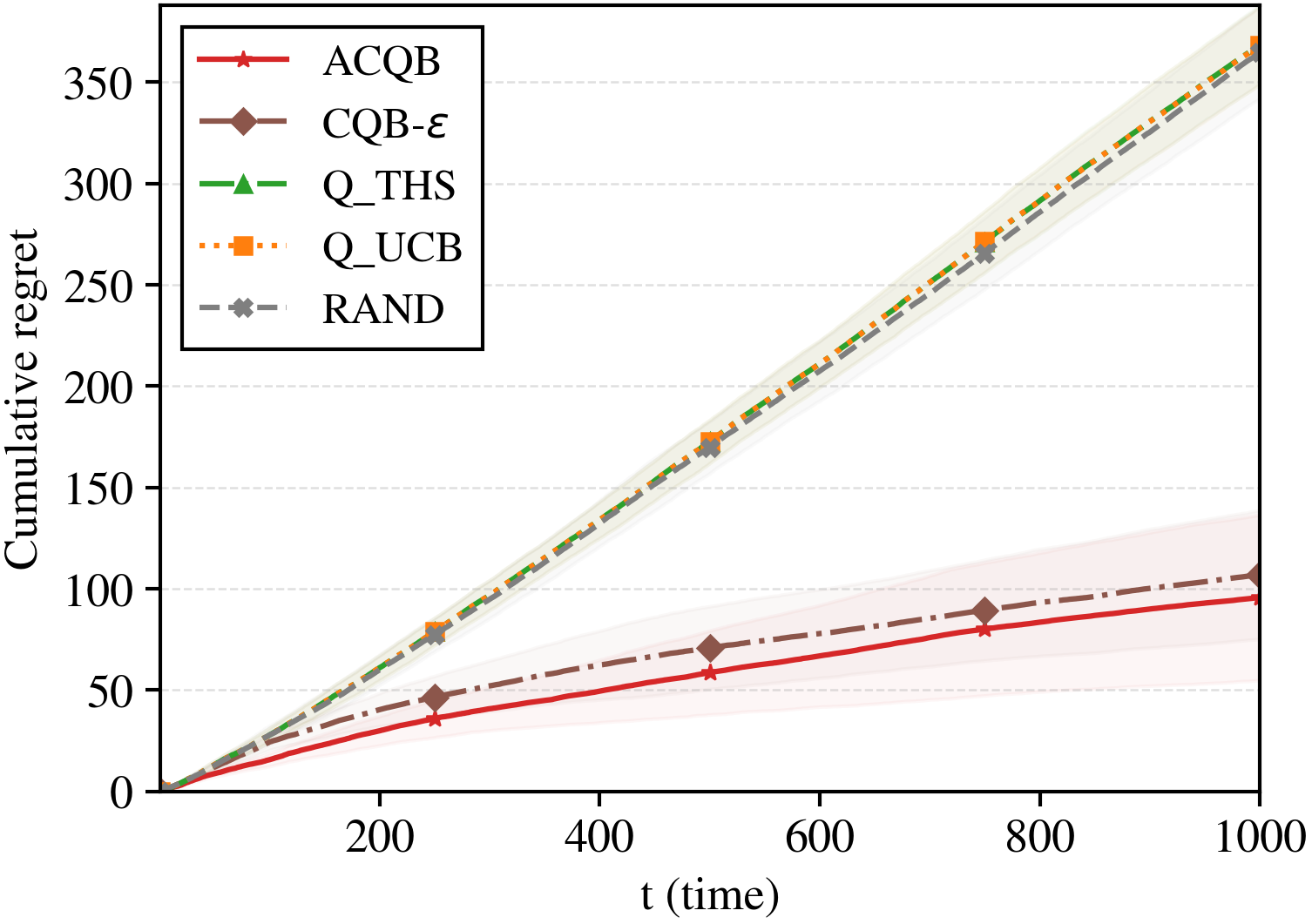}
            \label{fig:sim-k1-r}
        \end{subfigure}
    \end{minipage}
    \hfill
    % --- K=2 Group ---
    \begin{minipage}{\groupwidth}
        \centering
        \begin{subfigure}[c]{0.49\linewidth}
            \includegraphics[width=\linewidth]{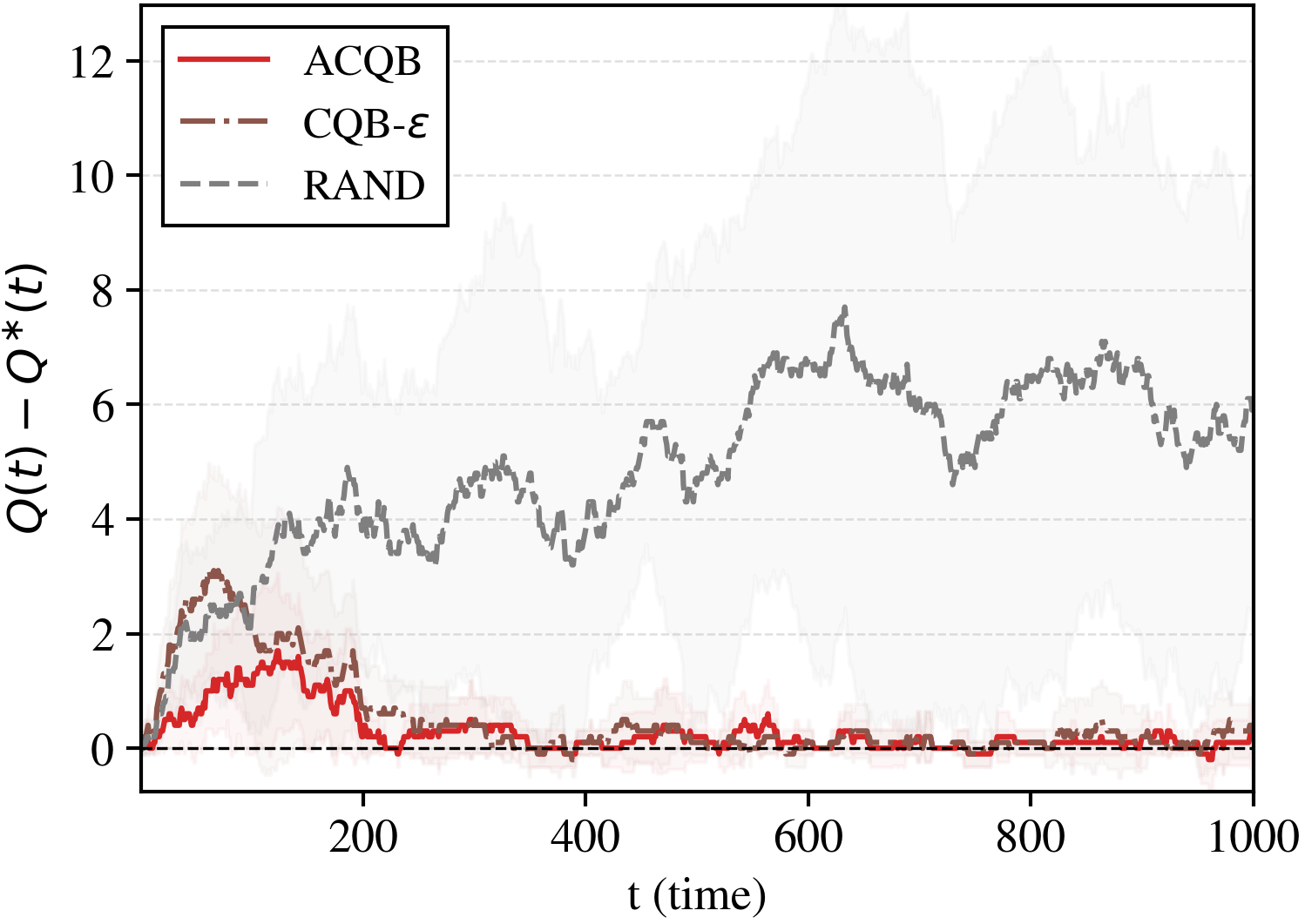}
            \label{fig:sim-k2-q}
        \end{subfigure}
        \hfill
        \begin{subfigure}[c]{0.49\linewidth}
            \includegraphics[width=\linewidth]{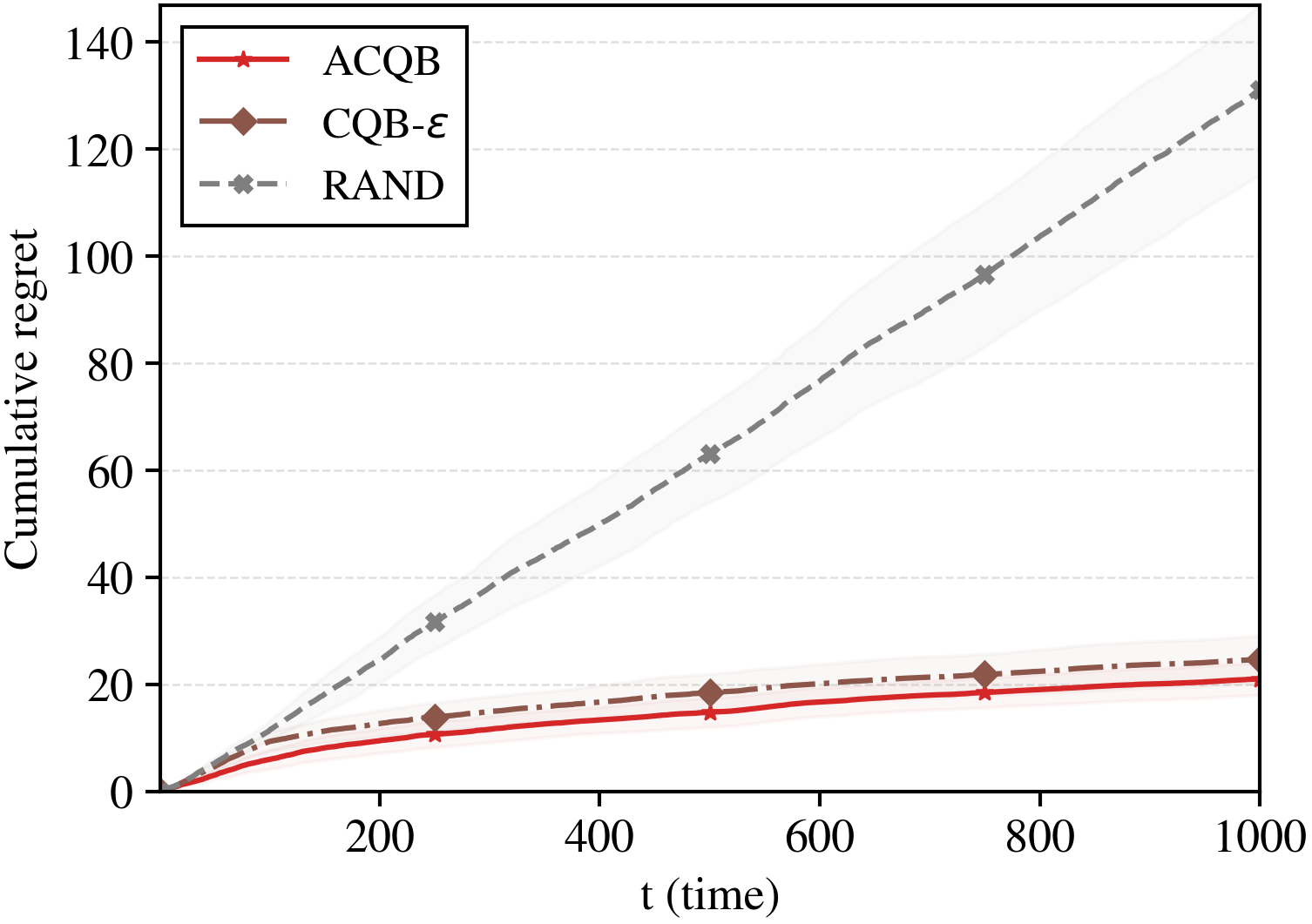}
            \label{fig:sim-k2-r}
        \end{subfigure}
    \end{minipage}
    \vspace{-0.5cm}
    \caption{Queue length and cumulative regret on synthetic data with $\lambda=0.7$, $\epsilon=0.03$, and  $N=5$.}
    \label{fig:sim_1}
\vspace{0.1cm}
    \centering
    
    % --- 설정: 너비 변수 (전체 가로 폭을 2등분) ---
    \renewcommand{\groupwidth}{0.49\textwidth} 
    
    % ---------------- Header (K=1, K=2) ----------------
    \begin{minipage}{\groupwidth}
        \centering \scriptsize $K=1$
    \end{minipage}
    \hfill
    \begin{minipage}{\groupwidth}
        \centering \scriptsize $K=2$
    \end{minipage}

    \vspace{0.1cm} % 헤더와 그림 사이 간격

    % ---------------- Main Content ----------------
    
    % --- K=1 Group ---
    \begin{minipage}{\groupwidth}
        \centering
        \begin{subfigure}[c]{0.49\linewidth}
            \includegraphics[width=\linewidth]{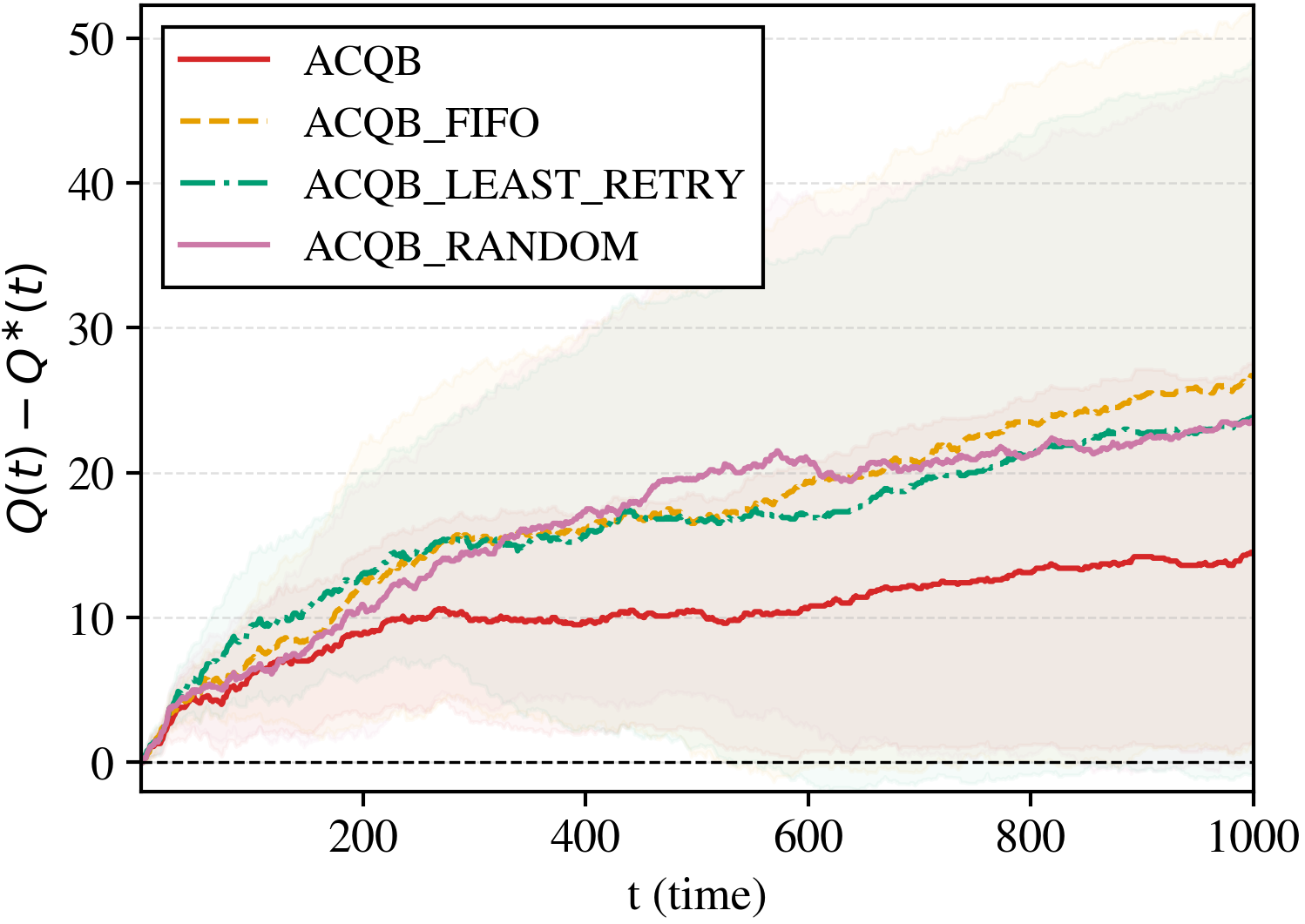}
            \label{fig:sim-sch-k1-q}
        \end{subfigure}
        \hfill
        \begin{subfigure}[c]{0.49\linewidth}
            \includegraphics[width=\linewidth]{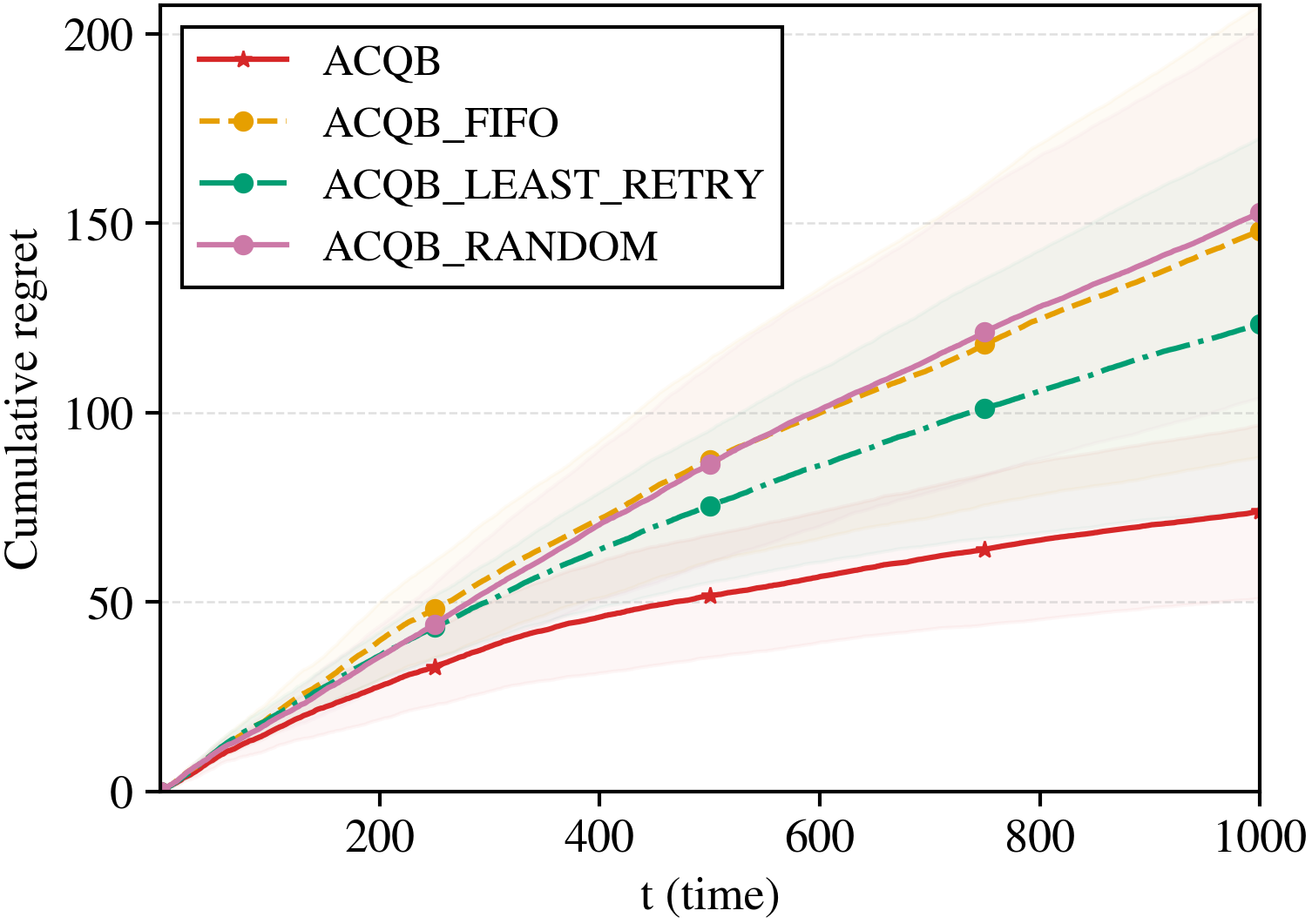}
            \label{fig:sim-sch-k1-r}
        \end{subfigure}
    \end{minipage}
    \hfill
    % --- K=2 Group ---
    \begin{minipage}{\groupwidth}
        \centering
        \begin{subfigure}[c]{0.49\linewidth}
            \includegraphics[width=\linewidth]{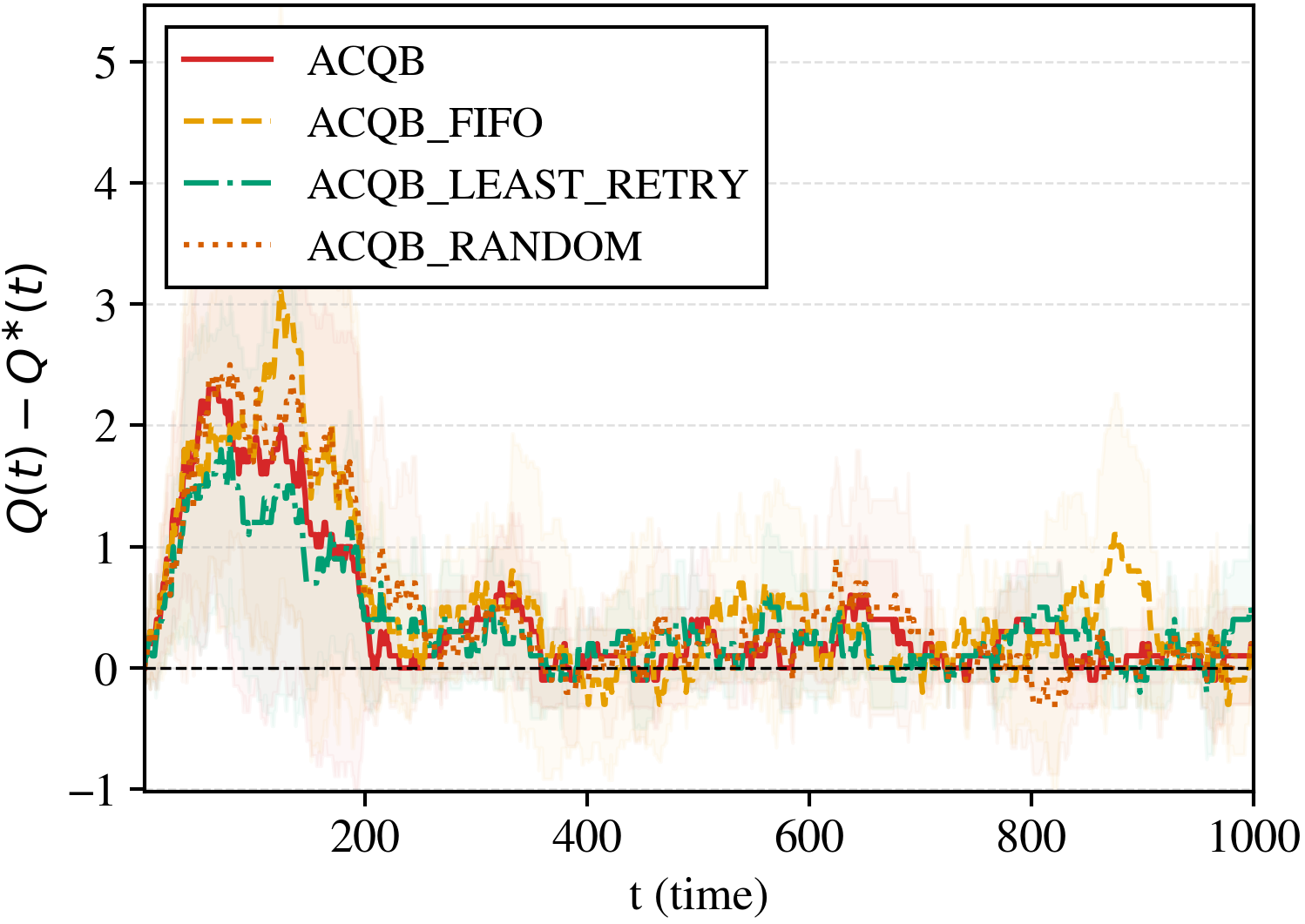}
            \label{fig:sim-sch-k2-q}
        \end{subfigure}
        \hfill
        \begin{subfigure}[c]{0.49\linewidth}
            \includegraphics[width=\linewidth]{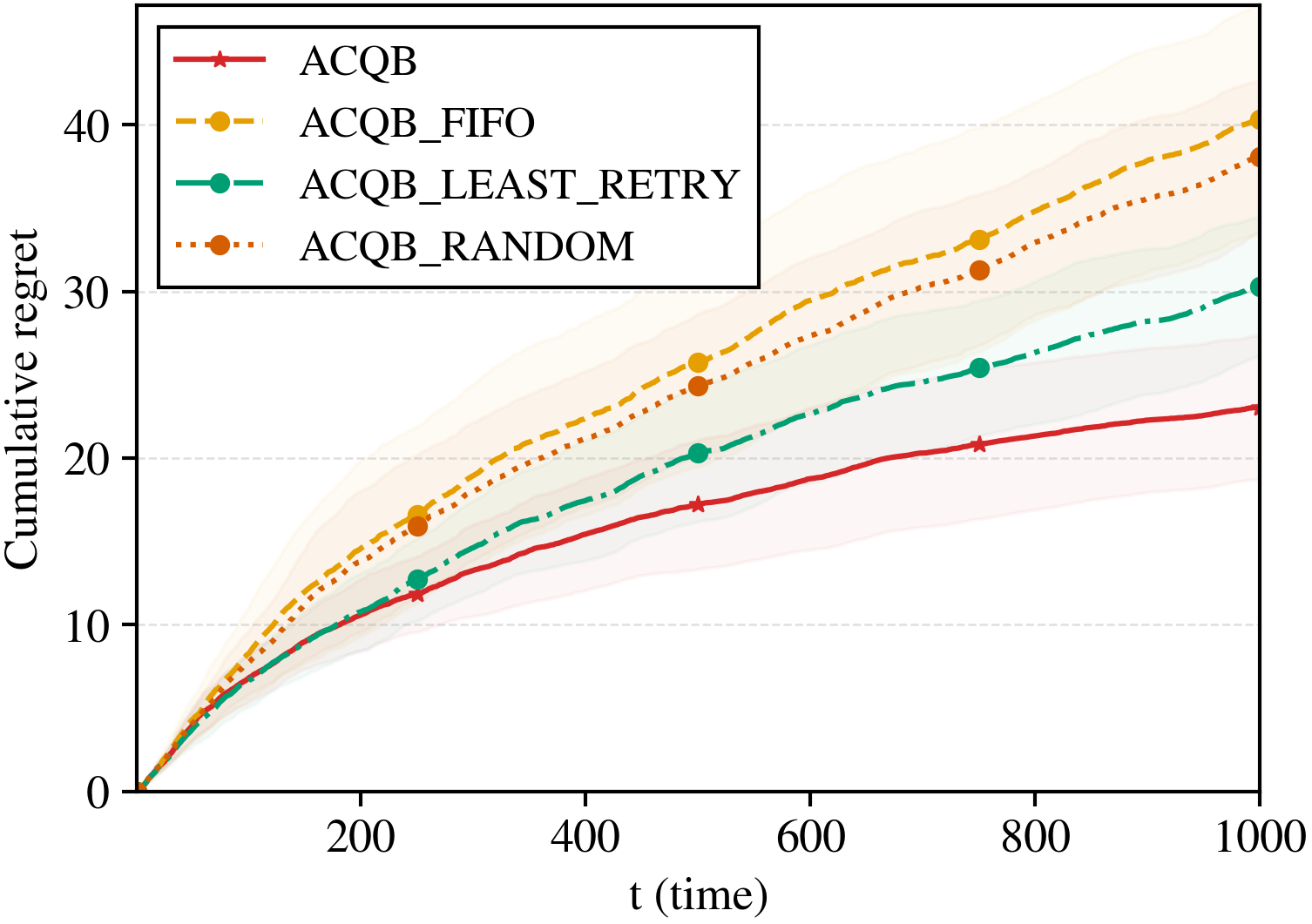}
            \label{fig:sim-sch-k2-r}
        \end{subfigure}
    \end{minipage}
    \vspace{-0.5cm}
    \caption{Ablation on ACQB scheduling on synthetic data with $\lambda=0.7$, $\epsilon=0.03$, and  $N=5$.}
    \label{fig:sim_sch}

\vspace{0.12cm}
\begin{minipage}{\textwidth}
\centering
\captionof{table}{Throughput (T.P.) and wall-clock time (W.C., seconds per round, excluding LLM inference) on synthetic data at $T=500$ and $T=1000$.}
\label{tab:throughput_wallclock_t500_t1000}
\small
\setlength{\tabcolsep}{3.2pt}
\renewcommand{\arraystretch}{1.12}
\resizebox{\linewidth}{!}{%
\begin{tabular}{llcccccccc}
\toprule
Time & Metric
& ACQB & ACQB-FIFO & ACQB-RR & ACQB-RAND & CQB-$\v\veps$ & Q-ThS & Q-UCB & RAND \\
\midrule

$T=500$ & T.P.
& \best{0.672}
& 0.660
& 0.660
& 0.655
& 0.658
& 0.551
& 0.551
& 0.549 \\

& W.C
& 0.0130
& 0.0135
& 0.0151
& 0.0125
& 0.0147
& 0.0000
& 0.0000
& 0.0001 \\

\midrule

$T=1000$ & T.P.
& \best{0.680}
& 0.667
& 0.670
& 0.670
& 0.667
& 0.545
& 0.544
& 0.540 \\

& W.C
& 0.0126
& 0.0114
& 0.0129
& 0.0121
& 0.0142
& 0.0001
& 0.0001
& 0.0001 \\

\bottomrule
\end{tabular}%
}
\end{minipage}
\vspace{-1.0em}
\end{figure*}

In this section, we empirically evaluate the performance of our proposed algorithms.

\subsection{Experiments on synthetic data} 
\label{ssec:sync}

\para{Experimental setup}
We generate random instances for $K=\{1,2\}$ by setting the arrival rates to $\lambda=0.7$ with the parameters $d=5$, $\eps=0.03$, $N=5$, and $T=1{,}000$. The context vectors $x\in\R^d$ and disjoint parameters $\theta_j^*\in\R^d$ for $j\in[N]$ are sampled from $\Unif(-1,1)$, ensuring that the contexts satisfy the slackness condition, i.e., $\max_{S\in\cal C} R(x,S,\Theta^*) \geq \lambda + \eps$. We perform 10 independent runs and report the average queue length regret and average cumulative regret with $\pm 1$ standard deviation.
We benchmark ACQB against five baseline algorithms for $K=1$: (i) the Optimal policy; (ii) the Random policy, which selects a random query and assortment; (iii--iv) two MAB-based queueing bandit algorithms from \citep{krishnasamy2021learning}; and (v) the contextual queueing bandit algorithm (Algorithm 1 from \citet{bae2026queuelengthregretbounds}). To the best of our knowledge, aside from these queueing bandit-based approaches, no other existing baselines are capable of learning from user retrials (implicit feedback).
For the $K=2$ setting, since baselines (iii) and (iv) are incompatible with the MNL model, we compare ACQB only with (i), (ii), and the MNL-adapted version of (v). Detailed descriptions of each baseline policy are provided in \Cref{sec:exp:baselines}.

\para{Performance comparison}
As shown in \Cref{fig:sim_1}, ACQB has lower queue length and cumulative regret values than the baselines. To isolate the scheduling rule, we also fix the routing rule of ACQB and replace only the scheduling policy with FIFO, round-robin over the least-served jobs (RR), or random scheduling. \Cref{fig:sim_sch} shows that ACQB scheduling has lower queue length regret than these variants.
\Cref{tab:throughput_wallclock_t500_t1000} reports throughput and wall-clock time. ACQB achieves higher throughput at both $T=500$ and $T=1{,}000$, with wall-clock time comparable to the other contextual policies. This overhead is small relative to LLM inference time: a response time can be approximated by $\mathrm{TTFT}+(\text{number of output tokens})/(\text{tokens/sec})$, and a public benchmark reports TTFT $2.84$s and throughput $30$ tokens/sec for GPT-4 Turbo.\footnote{\url{https://benchlm.ai/llm-speed}} The near-zero W.C. entries for Q-ThS, Q-UCB, and RAND are rounded values with very small but nonzero overheads. Details of the scheduling variants and additional experiments with varying $N$ and $\eps$ are provided in \Cref{sec:exp:baselines,sec:exp:sim}.

\para{Exploration variants}
ACQB triggers random exploration only when a new query arrives, which preserves i.i.d. arriving contexts in the eigenvalue-growth argument. In \Cref{sec:exp:exploration-variants}, we compare this design with heuristic variants that trigger exploration independently of arrivals and select the exploration query by FIFO, LIFO, round-robin over the least-served
queries, or random scheduling.

\subsection{Experiments on offline routing datasets}
\label{ssec:offline-routing}
\begin{figure*}[t]
    \centering
    
    % --- 설정: 너비 변수 ---
    \newcommand{\lbwidth}{0.02\textwidth} % 세로 라벨 너비
    \newcommand{\blockwidth}{0.46\textwidth} % 그림 블록 너비 (이미지 2개 포함)
    
    % ---------------- Header (K=1, K=2) ----------------
    % 1. Left Spacing
    \begin{minipage}{\lbwidth} \centering \phantom{L} \end{minipage}%
    \hfill%
    % 2. K=1 Header
    \begin{minipage}{\blockwidth} \centering \scriptsize $K=1$ \end{minipage}%
    \hfill%
    % 3. Middle Spacing
    \begin{minipage}{\lbwidth} \centering \phantom{L} \end{minipage}%
    \hfill%
    % 4. K=2 Header
    \begin{minipage}{\blockwidth} \centering \scriptsize $K=2$ \end{minipage}%
    
    \vspace{0.1cm} % 헤더와 첫 줄 사이 간격

    % ================= ROW 1: lambda=0.55 =================
    \begin{minipage}[c]{\lbwidth}
        \centering \rotatebox{90}{\scriptsize $\lambda=0.7$}
    \end{minipage}%
    \hfill%
    \begin{minipage}[c]{\blockwidth}
        \centering
        \begin{subfigure}[c]{0.49\linewidth}
            \includegraphics[width=\linewidth]{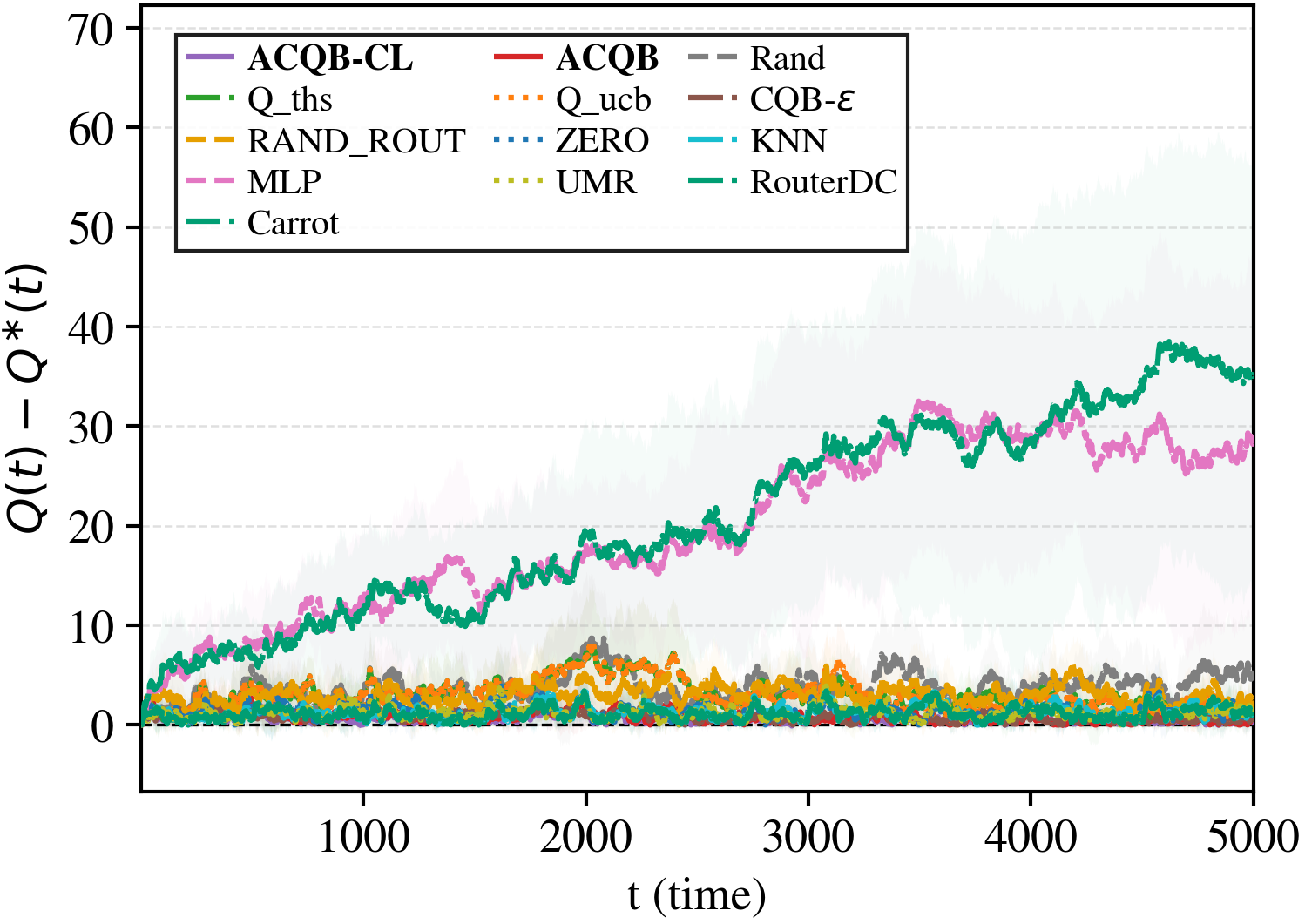}
        \end{subfigure}\hfill
        \begin{subfigure}[c]{0.49\linewidth}
            \includegraphics[width=\linewidth]{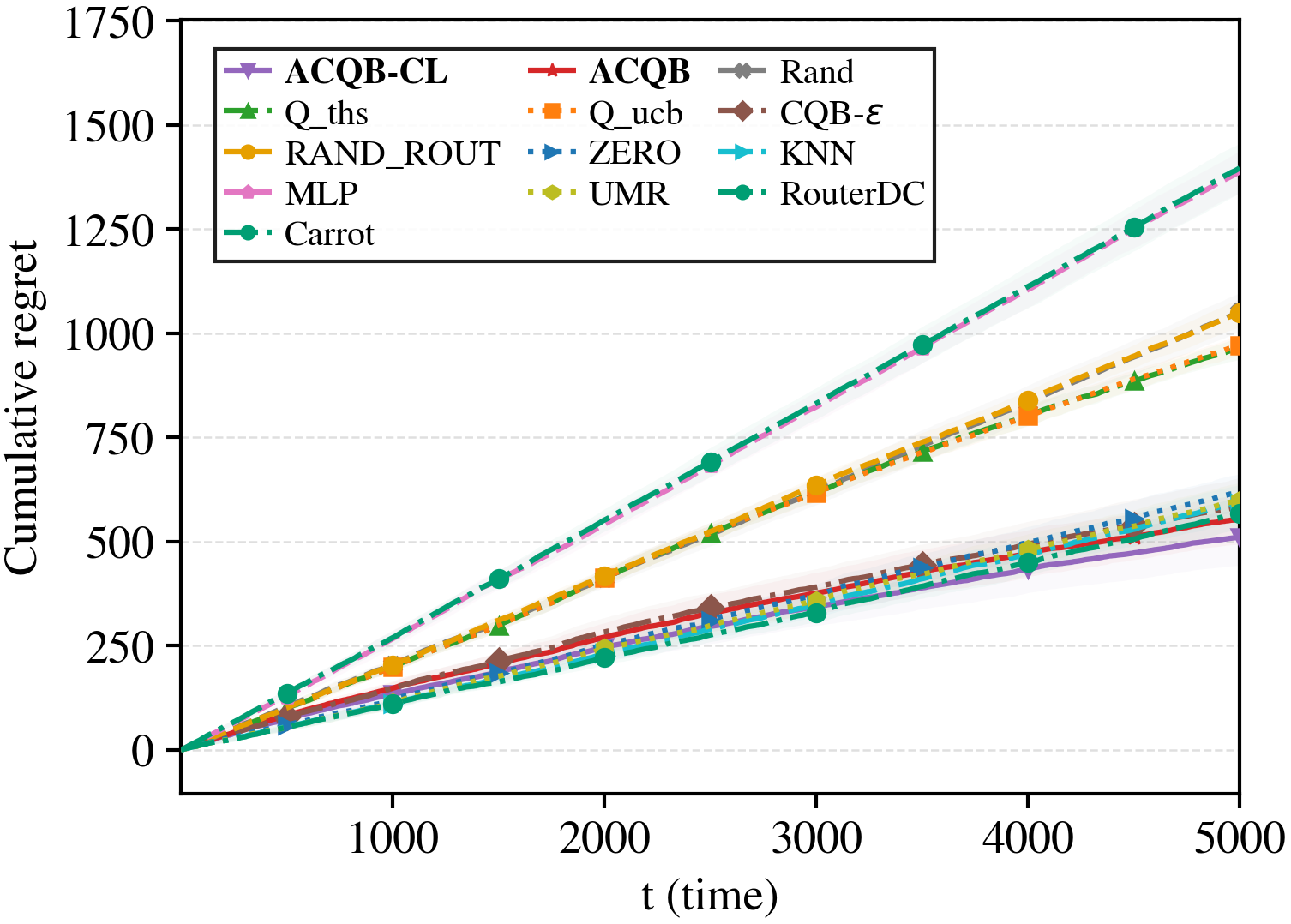}
        \end{subfigure}
    \end{minipage}%
    \hfill%
    \begin{minipage}[c]{\lbwidth} % 중앙 라벨 (필요 없다면 삭제 후 \blockwidth 조정)
        \centering \rotatebox{90}{\scriptsize $\lambda=0.75$}
    \end{minipage}%
    \hfill%
    \begin{minipage}[c]{\blockwidth}
        \centering
        \begin{subfigure}[c]{0.49\linewidth}
            \includegraphics[width=\linewidth]{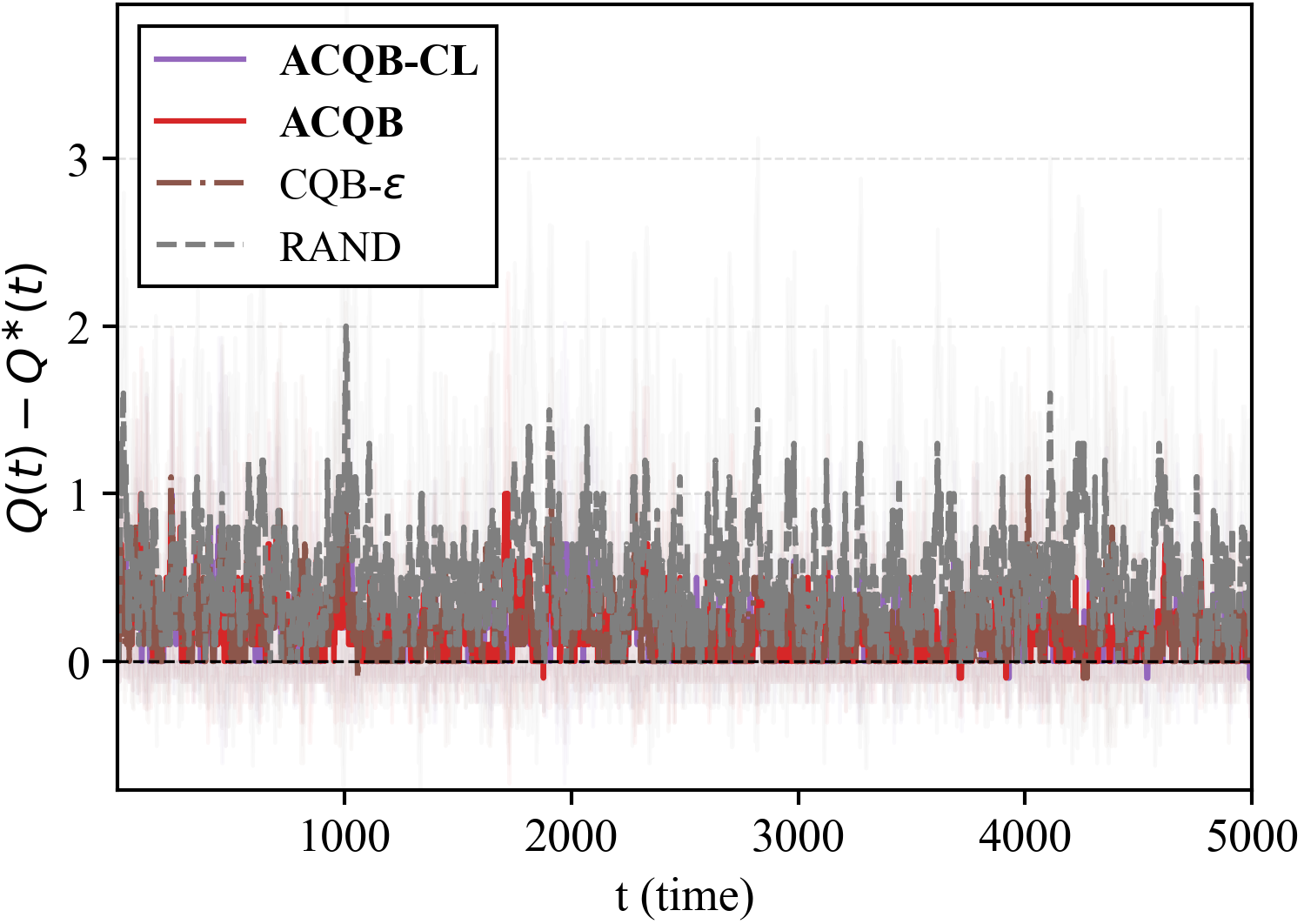}
        \end{subfigure}\hfill
        \begin{subfigure}[c]{0.49\linewidth}
            \includegraphics[width=\linewidth]{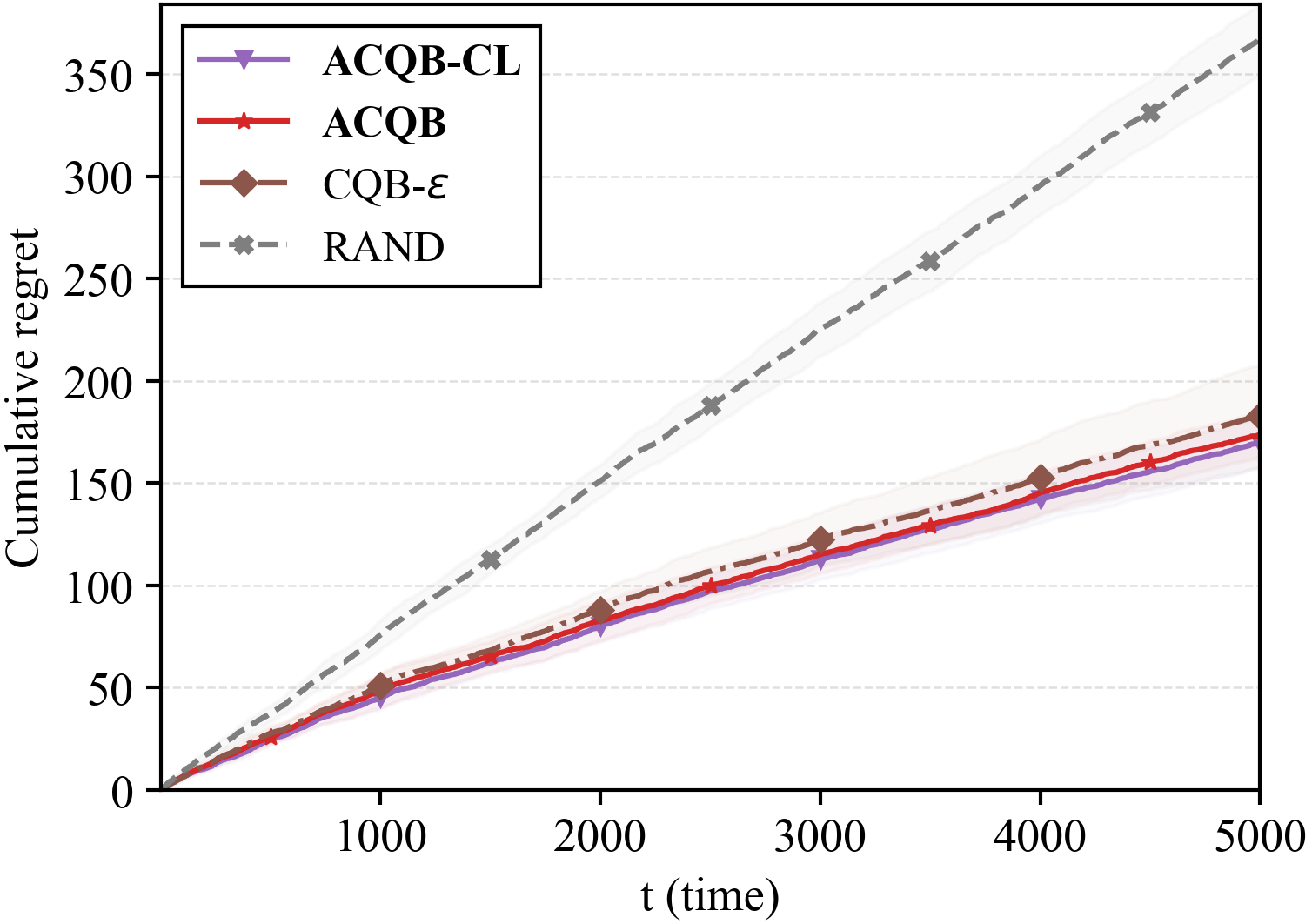}
        \end{subfigure}
    \end{minipage}%

    \vspace{0.15cm} % 행 간 간격

    % ================= ROW 2: lambda=0.65 (예시) =================
    \begin{minipage}[c]{\lbwidth}
        \centering \rotatebox{90}{\scriptsize $\lambda=0.8$}
    \end{minipage}%
    \hfill%
    \begin{minipage}[c]{\blockwidth}
        \centering
        \begin{subfigure}[c]{0.49\linewidth}
            \includegraphics[width=\linewidth]{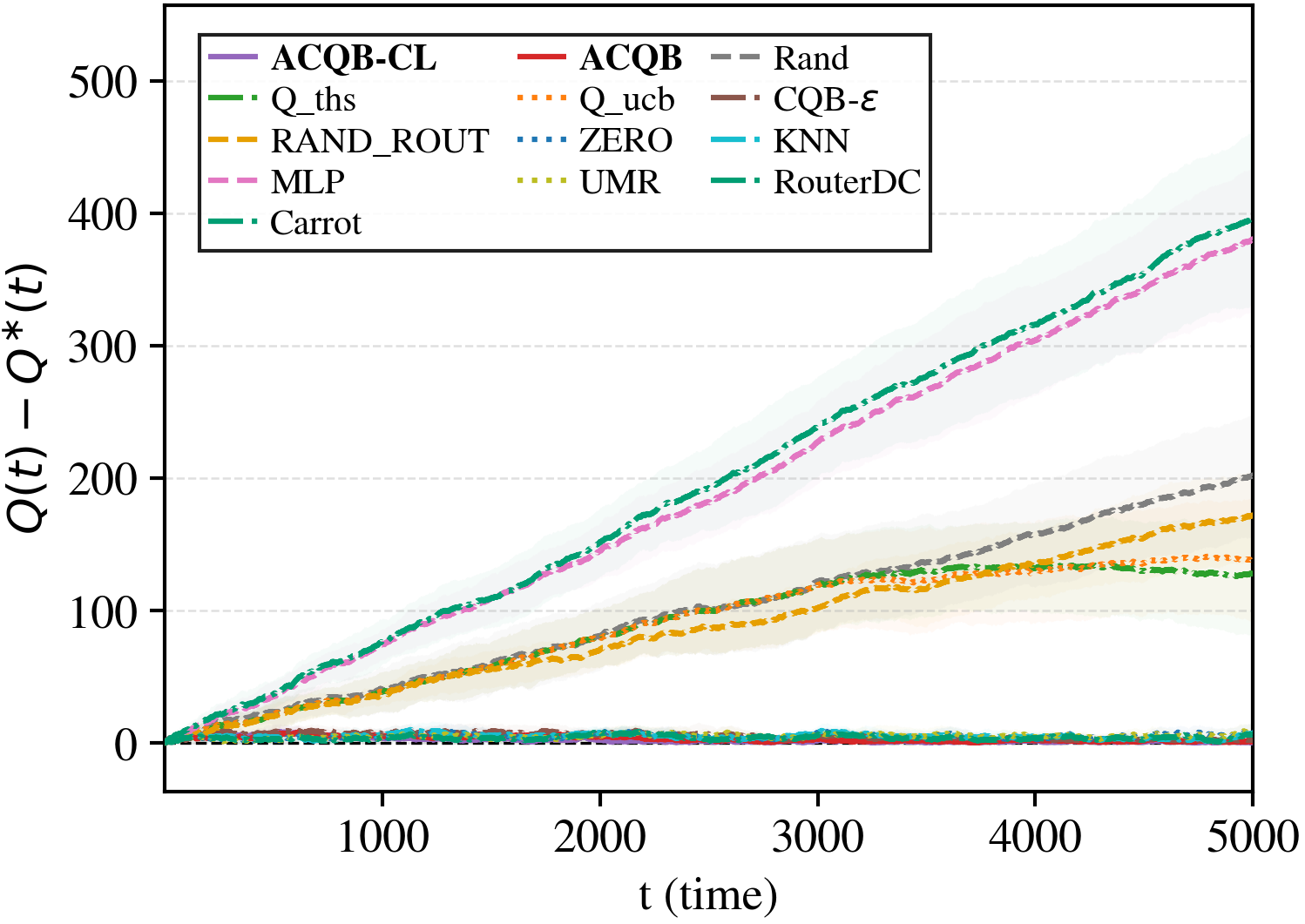}
        \end{subfigure}\hfill
        \begin{subfigure}[c]{0.49\linewidth}
            \includegraphics[width=\linewidth]{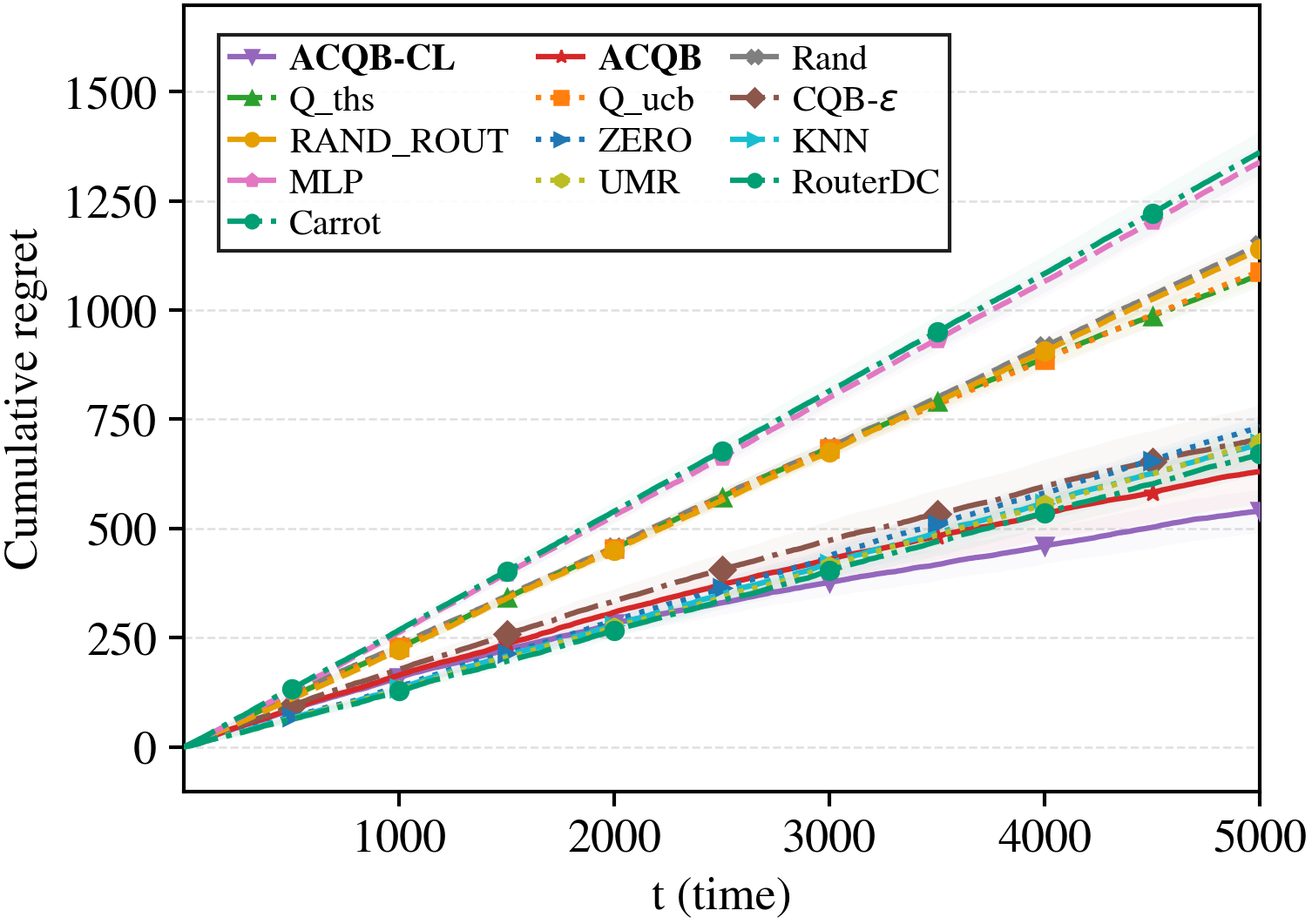}
        \end{subfigure}
    \end{minipage}%
    \hfill%
    \begin{minipage}[c]{\lbwidth}
        \centering \rotatebox{90}{\scriptsize $\lambda=0.85$}
    \end{minipage}%
    \hfill%
    \begin{minipage}[c]{\blockwidth}
        \centering
        \begin{subfigure}[c]{0.49\linewidth}
            \includegraphics[width=\linewidth]{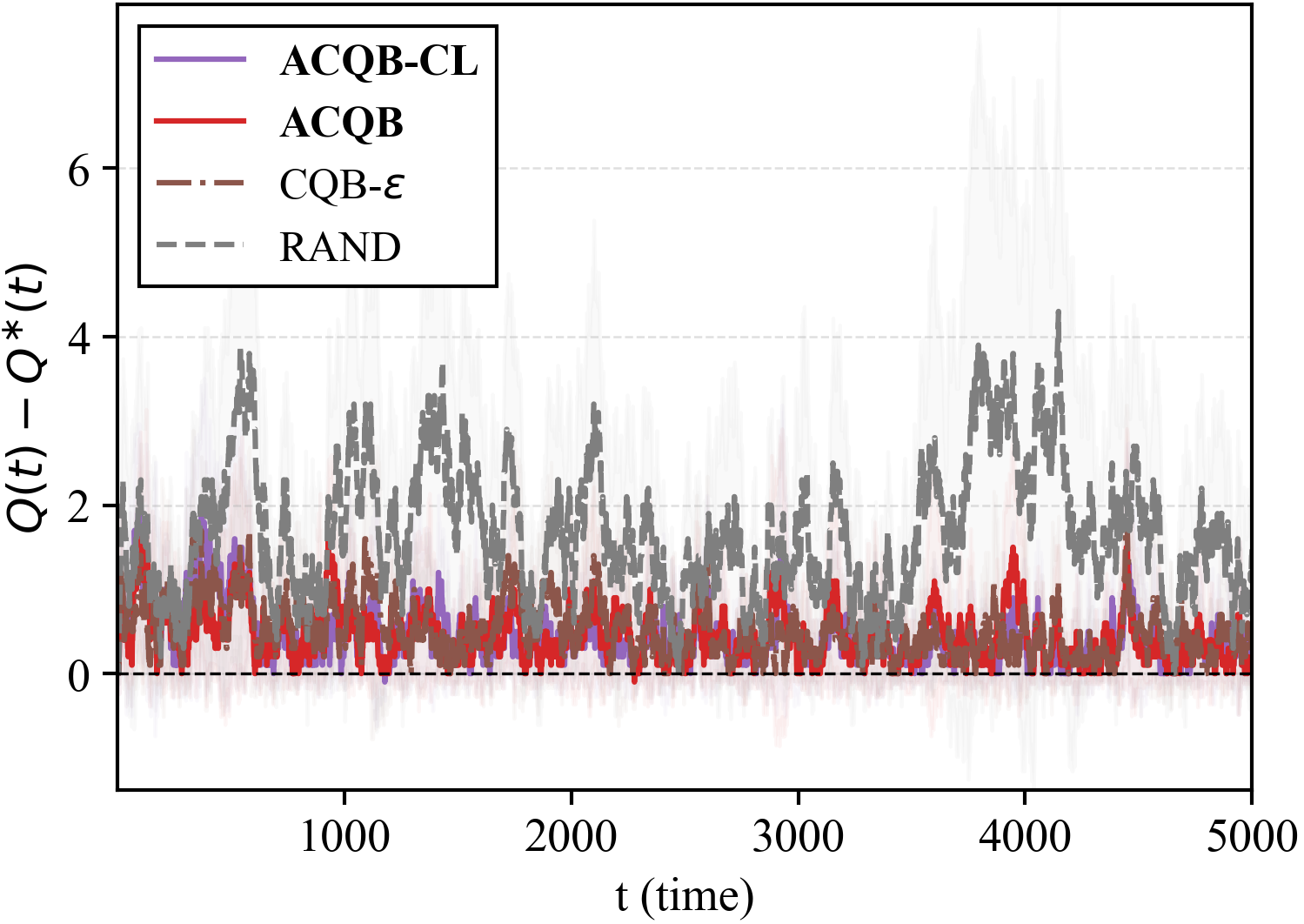}
        \end{subfigure}\hfill
        \begin{subfigure}[c]{0.49\linewidth}
            \includegraphics[width=\linewidth]{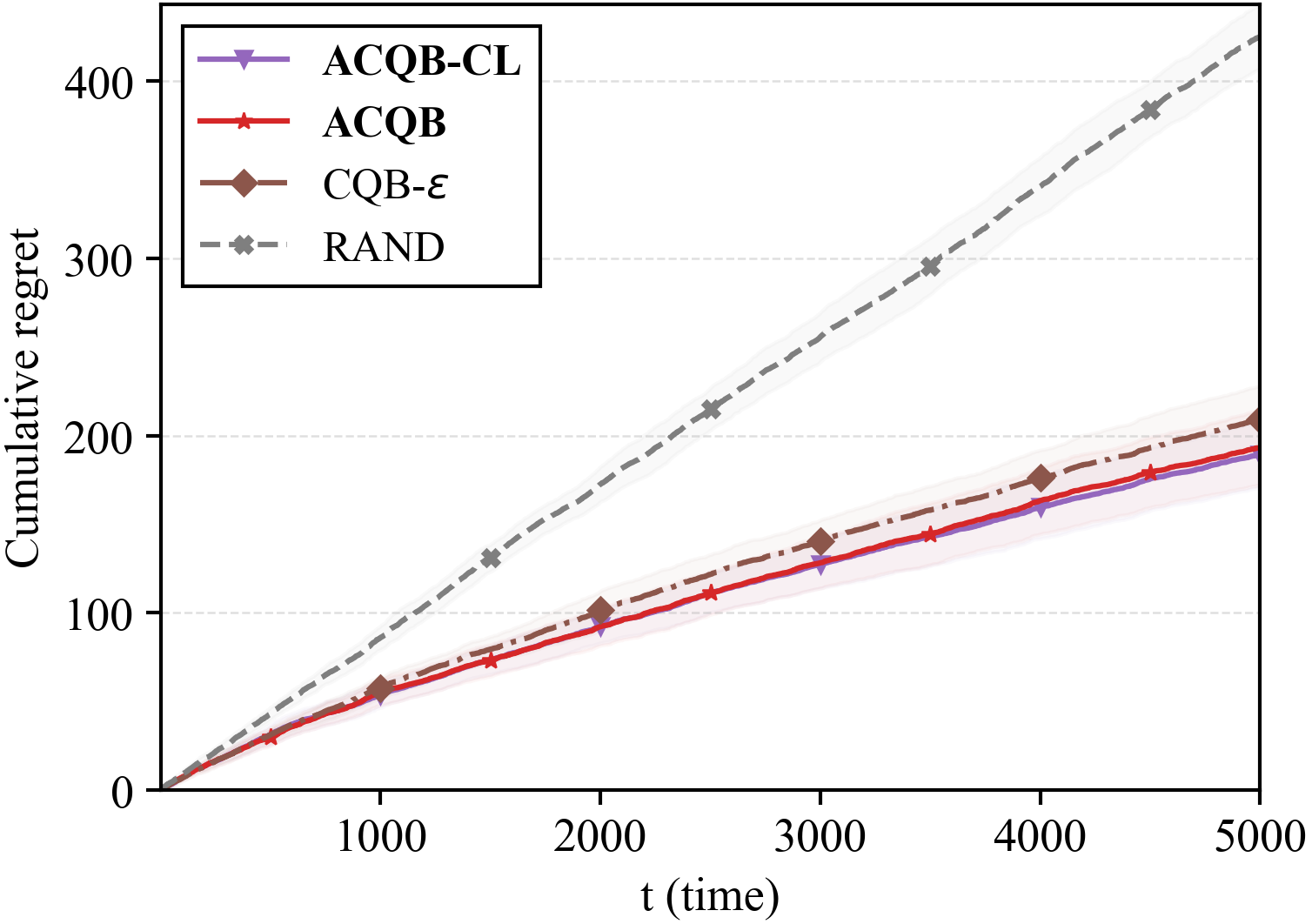}
        \end{subfigure}
    \end{minipage}%

    \vspace{0.15cm} % 행 간 간격

    % ================= ROW 3: lambda=0.75 (예시) =================
    \begin{minipage}[c]{\lbwidth}
        \centering \rotatebox{90}{\scriptsize $\lambda=0.9$}
    \end{minipage}%
    \hfill%
    \begin{minipage}[c]{\blockwidth}
        \centering
        \begin{subfigure}[c]{0.49\linewidth}
            \includegraphics[width=\linewidth]{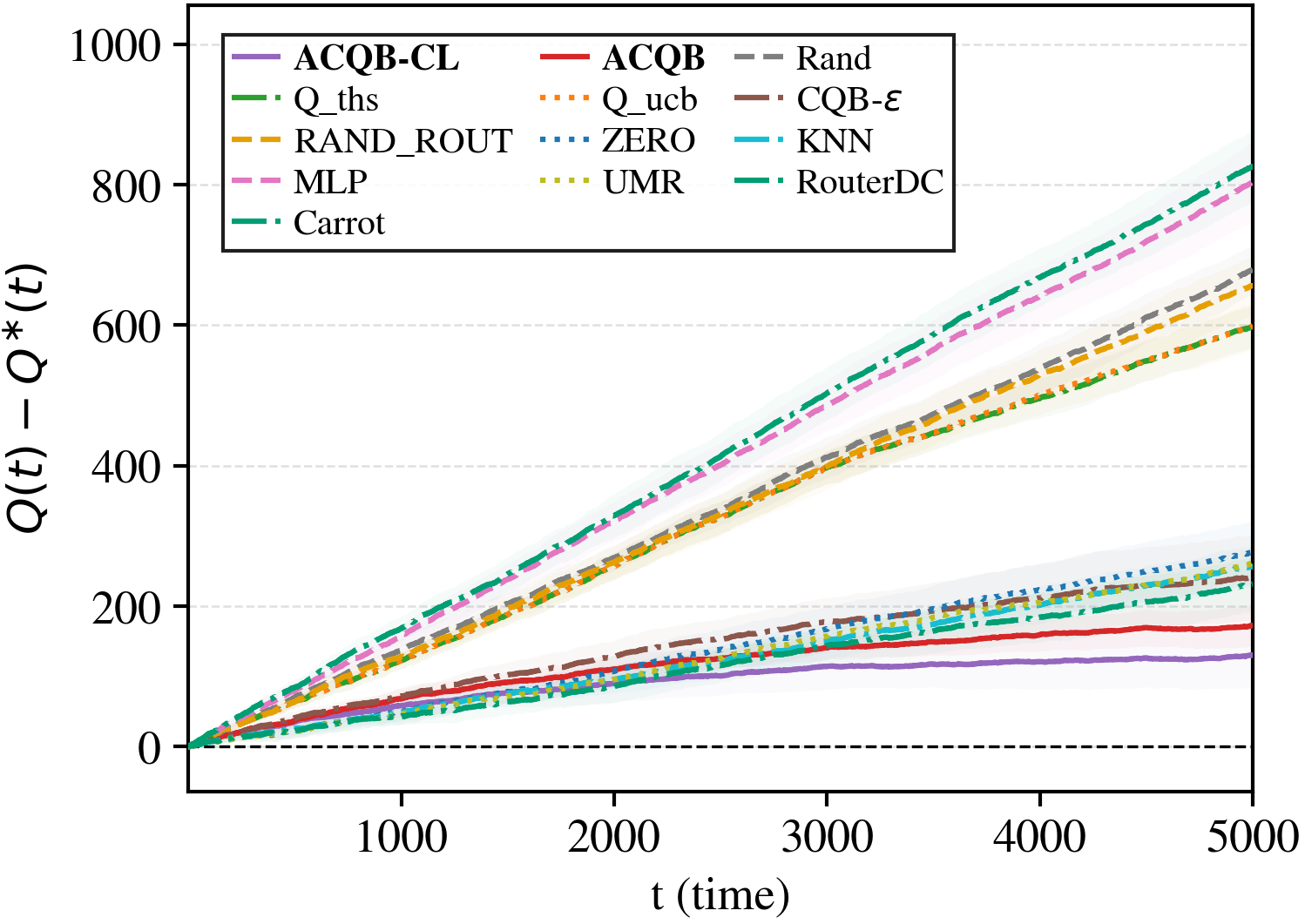}
        \end{subfigure}\hfill
        \begin{subfigure}[c]{0.49\linewidth}
            \includegraphics[width=\linewidth]{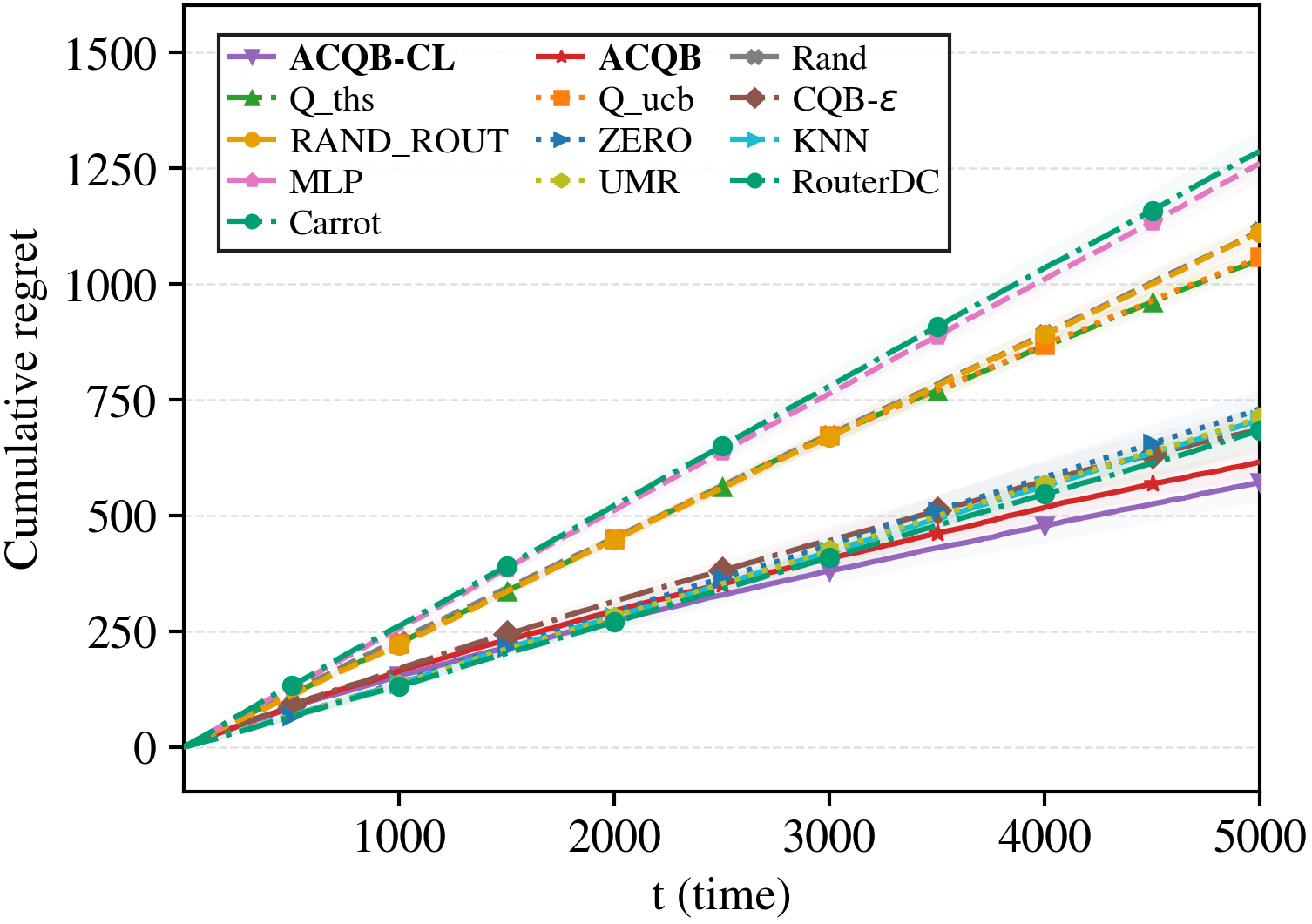}
        \end{subfigure}
    \end{minipage}%
    \hfill%
    \begin{minipage}[c]{\lbwidth}
        \centering \rotatebox{90}{\scriptsize $\lambda=0.95$}
    \end{minipage}%
    \hfill%
    \begin{minipage}[c]{\blockwidth}
        \centering
        \begin{subfigure}[c]{0.49\linewidth}
            \includegraphics[width=\linewidth]{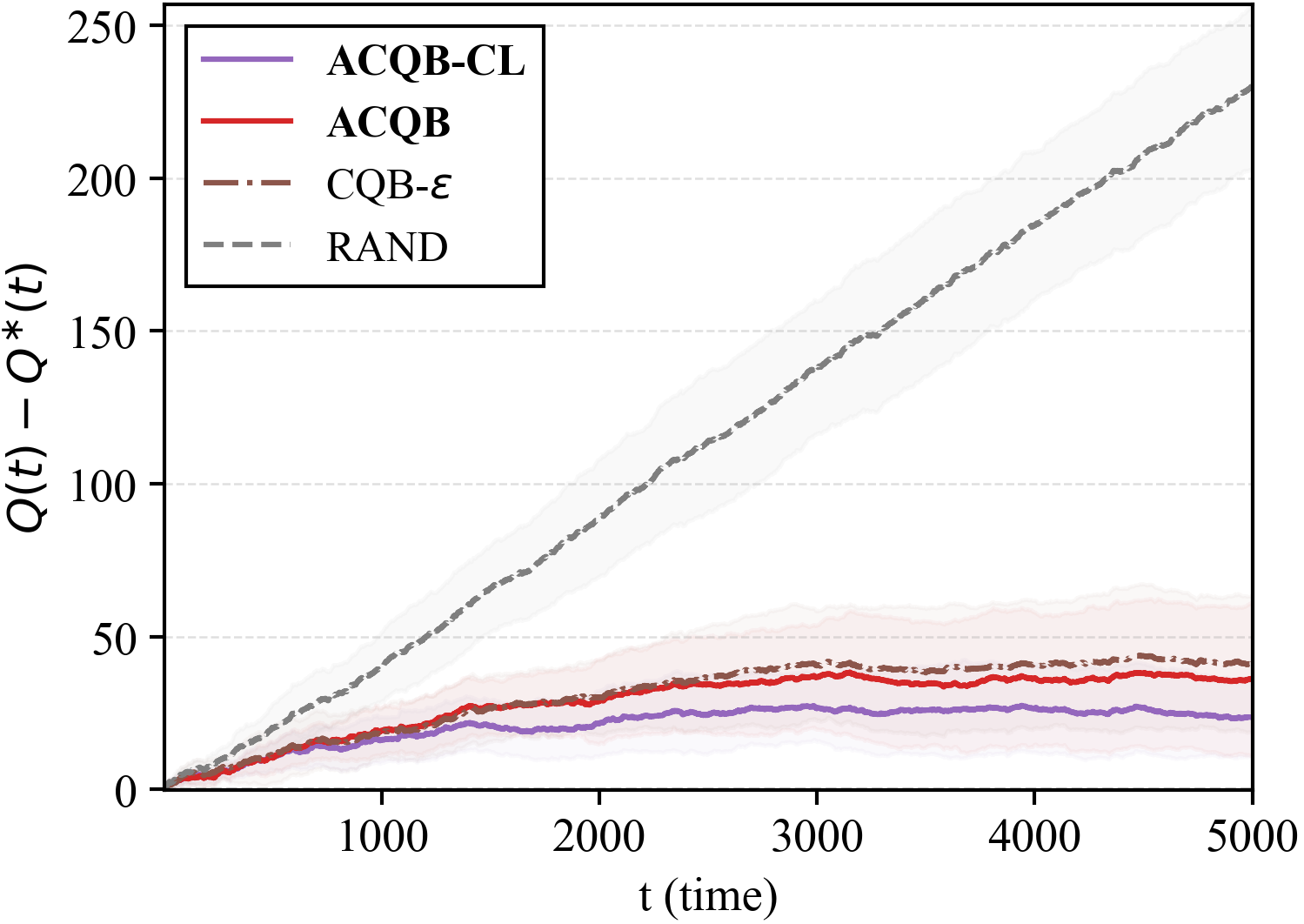}
        \end{subfigure}\hfill
        \begin{subfigure}[c]{0.49\linewidth}
            \includegraphics[width=\linewidth]{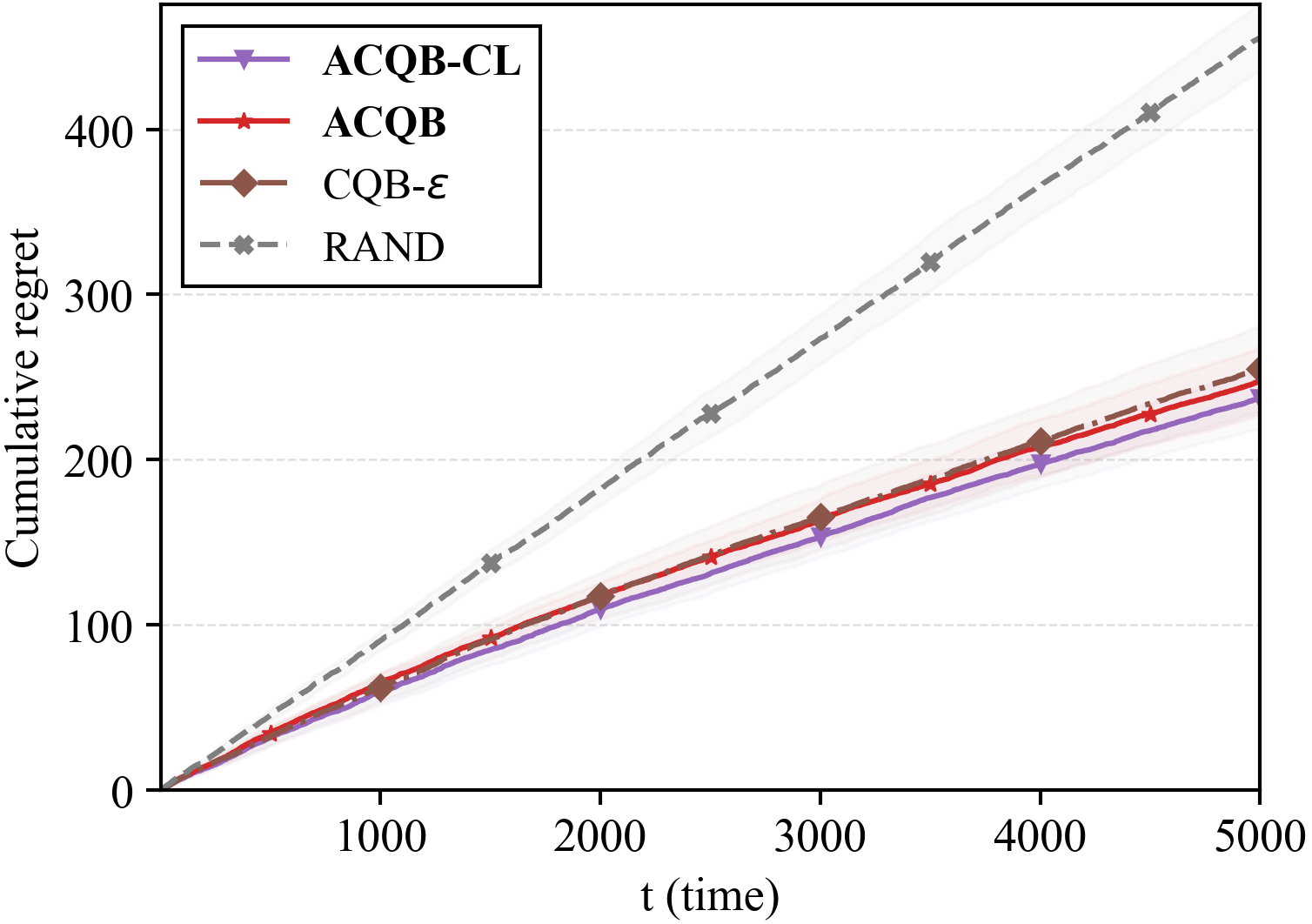}
        \end{subfigure}
    \end{minipage}%

    \vspace{-0.1cm}
    % \caption{Queue length and cumulative regret on \textsc{SPROUT-o3mini} dataset across various arrival rates $\lambda$.}
    \caption{Queue length and cumulative regret on the \textsc{SPROUT-o3mini} dataset across various arrival rates $\lambda$.}
    \label{fig:main_results}
\vspace{-0.8em}
\end{figure*}
We next evaluate ACQB and ACQB-CL through simulations on offline LLM routing datasets. We use three datasets: \textsc{SPROUT-o3mini} (14 LLMs)~\citep{somerstep2025carrot}, \textsc{EmbedLLM} (112 LLMs)~\citep{zhuang2024embedllm}, and \textsc{RouterBench} (11 LLMs)~\citep{hu2024routerbench}.
All datasets include benchmark-specific correctness scores (e.g., exact-match or graded accuracy) and inference costs. For \textsc{RouterBench}, we use the provided cost metadata, while for \textsc{EmbedLLM} and \textsc{SPROUT-o3mini}, we construct the cost map using reports from \citep{shirkavand2025cost, somerstep2025carrot}.
For the query embedding, we adopt the \texttt{sentence-transformers/all-MiniLM-L6-v2} ~\citep{sentence-transformers-allminilm-2021} as the backbone query encoder ($d=384$). In the online phase, we run the contextual bandit simulation over an online prompt pool of size 1{,}000 for a horizon of $T=5{,}000$.

\newcommand{\second}[1]{\underline{#1}}

\begin{table*}[t]
\centering
\caption{Mean ($\mu$), standard deviation ($\sigma$) of queue length gap, and throughput (T.P.) at time steps 2500 and 5000 for $K=1$ (arrival setting: $\lambda_{\mathrm{Sprout}}=0.8$,  $\lambda_{\mathrm{EmbedLLM}}=0.4$,
$\lambda_{\mathrm{RouterBench}}=0.6$).}
\label{tab:k1-lam5-qgap-throughput}
\scriptsize
\setlength{\tabcolsep}{2.4pt}
\renewcommand{\arraystretch}{1.15}
\resizebox{\textwidth}{!}{%
\begin{tabular}{lllrrrrrrrrrrrrr}
\toprule
Dataset & Time & Stat.
& ACQB-CL & ACQB & RAND & Q-ThS & Q-UCB & CQB-$\v\veps$ & RAND\_ROUT & ZERO & KNN & MLP & UMR & RouterDC & CARROT \\
\midrule

Sprout & 2500 & $\mu$
& \best{1.2} & 3.3 & 101.4 & 99.9 & 99.9 & 6.5 & 87.8 & 3.9 & 2.3 & 181.6 & 5.6 & 192.4 & \second{2.2} \\
& & $\sigma$
& 1.7 & 3.2 & 48.5 & 46.6 & 46.6 & 10.0 & 23.7 & 3.4 & 2.7 & 38.4 & 6.5 & 45.6 & 2.5 \\
& & T.P.
& \best{0.796} & 0.795 & 0.756 & 0.756 & 0.756 & 0.794 & 0.761 & 0.795 & 0.795 & 0.724 & 0.794 & 0.719 & \second{0.795} \\

& 5000 & $\mu$
& \best{0.3} & \second{1.1} & 202.0 & 127.8 & 138.1 & 2.4 & 171.8 & 8.4 & 6.7 & 380.3 & 7.9 & 397.2 & 6.5 \\
& & $\sigma$
& 0.5 & 1.7 & 66.6 & 66.6 & 67.2 & 3.7 & 48.0 & 9.4 & 5.0 & 76.9 & 9.9 & 93.0 & 5.4 \\
& & T.P.
& \best{0.797} & \second{0.797} & 0.757 & 0.772 & 0.770 & 0.797 & 0.763 & 0.796 & 0.796 & 0.721 & 0.796 & 0.718 & 0.796 \\
\midrule

EmbedLLM & 2500 & $\mu$
& \best{55.0} & \second{59.4} & 188.3 & 179.5 & 179.5 & 70.5 & 168.0 & 255.5 & 181.4 & 328.7 & 182.4 & 295.7 & 181.4 \\
& & $\sigma$
& 20.3 & 11.3 & 27.0 & 22.9 & 22.9 & 28.2 & 25.5 & 30.0 & 36.4 & 35.8 & 30.5 & 34.2 & 36.4 \\
& & T.P.
& \best{0.375} & \second{0.374} & 0.322 & 0.326 & 0.326 & 0.369 & 0.330 & 0.295 & 0.325 & 0.266 & 0.324 & 0.279 & 0.325 \\

& 5000 & $\mu$
& \best{54.6} & \second{60.5} & 364.4 & 358.5 & 358.5 & 63.4 & 333.5 & 506.0 & 360.6 & 653.8 & 355.1 & 586.0 & 360.6 \\
& & $\sigma$
& 20.8 & 21.4 & 38.4 & 40.0 & 40.0 & 43.3 & 45.4 & 32.7 & 50.2 & 71.0 & 54.7 & 69.9 & 50.2 \\
& & T.P.
& \best{0.388} & \second{0.387} & 0.326 & 0.327 & 0.327 & 0.386 & 0.332 & 0.297 & 0.326 & 0.268 & 0.328 & 0.281 & 0.326 \\

\midrule

RouterBench & 2500 & $\mu$
& \best{6.9} & \second{10.3} & 160.3 & 157.3 & 156.2 & 29.9 & 155.7 & 70.5 & 68.4 & 342.8 & 60.0 & 287.5 & 62.9 \\
& & $\sigma$
& 6.7 & 10.3 & 59.2 & 62.9 & 60.6 & 35.1 & 26.7 & 30.6 & 42.6 & 40.8 & 33.3 & 41.3 & 43.1 \\
& & T.P.
& \best{0.591} & \second{0.590} & 0.530 & 0.531 & 0.532 & 0.582 & 0.532 & 0.566 & 0.567 & 0.457 & 0.570 & 0.479 & 0.569 \\

& 5000 & $\mu$
& \best{3.7} & \second{3.8} & 330.3 & 206.3 & 220.6 & 22.4 & 313.3 & 123.7 & 118.0 & 683.1 & 110.0 & 586.4 & 114.8 \\
& & $\sigma$
& 3.4 & 3.3 & 84.2 & 85.9 & 86.2 & 36.7 & 56.1 & 77.1 & 90.9 & 84.7 & 86.1 & 81.2 & 93.0 \\
& & T.P.
& \best{0.595} & \second{0.595} & 0.529 & 0.554 & 0.551 & 0.591 & 0.533 & 0.571 & 0.572 & 0.459 & 0.573 & 0.478 & 0.572 \\

\bottomrule
\end{tabular}%
}
\vspace{-1.0em}
\end{table*}

\para{Modeling departure and choice probabilities} \label{sec:exp:mnl}

For every LLM $j$ and raw prompt $\xi$, we denote the provided performance score and normalized inference cost as $\mathrm{perf}_{j}(\xi)\in[0,1]$ and $\mathrm{cost}_{j}(\xi)\in[0,1]$, respectively.
The performance score measures the response quality of model $j$ on prompt $\xi$; for example, a binary score gives an outcome in $\{0,1\}$, while a graded score represents the expected quality level.
To obtain a cost-aware departure probability, we first define the raw utility as $u_j^{\mathrm{raw}}(\xi) = \mathrm{perf}_{j}(\xi) - \rho \cdot \mathrm{cost}_{j}(\xi)$.
Since $u_j^{\mathrm{raw}}(\xi)$ need not lie in a valid probability range, we apply min-max normalization over models and then rescale it to a clipped interval $[r_{\mathrm{lo}},r_{\mathrm{hi}}]\subset(0,1)$, where $r_{\mathrm{lo}}=0.1$ and $r_{\mathrm{hi}}=0.99$ in our experiments.
We use the resulting $u_j(\xi)\in(0,1)$ as the departure probability, i.e., a job departs the queue when the served response is accepted, and convert it to the corresponding MNL choice probability as described in \Cref{sec:exp:details}.

\label{sec:exp:protocol}

\para{Performance comparison}
\label{sec:exp:comp}
We benchmark ACQB and ACQB-CL under varying load conditions. We set the arrival rates to $\lambda \in \{0.7, 0.8, 0.9\}$ for $K=1$ and $\lambda \in \{0.75, 0.85, 0.95\}$ for $K=2$. We also compare with offline-trained routing baselines combined with FIFO scheduling: CARROT~\citep{somerstep2025carrot}, KNN~\citep{hu2024routerbench}, MLP~\citep{hu2024routerbench}, ROUTERDC~\citep{chen2024routerdc}, UMR~\citep{jitkrittum2025universal}, and ZERO~\citep{hu2024routerbench}. Details of these baselines are provided in \Cref{sec:exp:baselines}.
\Cref{fig:main_results} shows the queue length and cumulative regrets on \textsc{SPROUT-o3mini}. ACQB and ACQB-CL have lower queue length regret than the baselines across arrival rates, and ACQB-CL also has lower cumulative regret than ACQB in most settings. \Cref{tab:k1-lam5-qgap-throughput} reports the queue-length gap and throughput on all three offline routing datasets. ACQB-CL has the lowest mean queue-length gap in all reported cases, and ACQB has the second-lowest gap in most cases. Additional results are provided in \Cref{ssec:exp:real-result}.

\subsection{Experiments on a real LLM conversation dataset}
\label{ssec:wildchat}

\begin{wrapfigure}{r}{0.46\columnwidth}
\vspace{-0.8em}
\centering

\includegraphics[width=\linewidth]{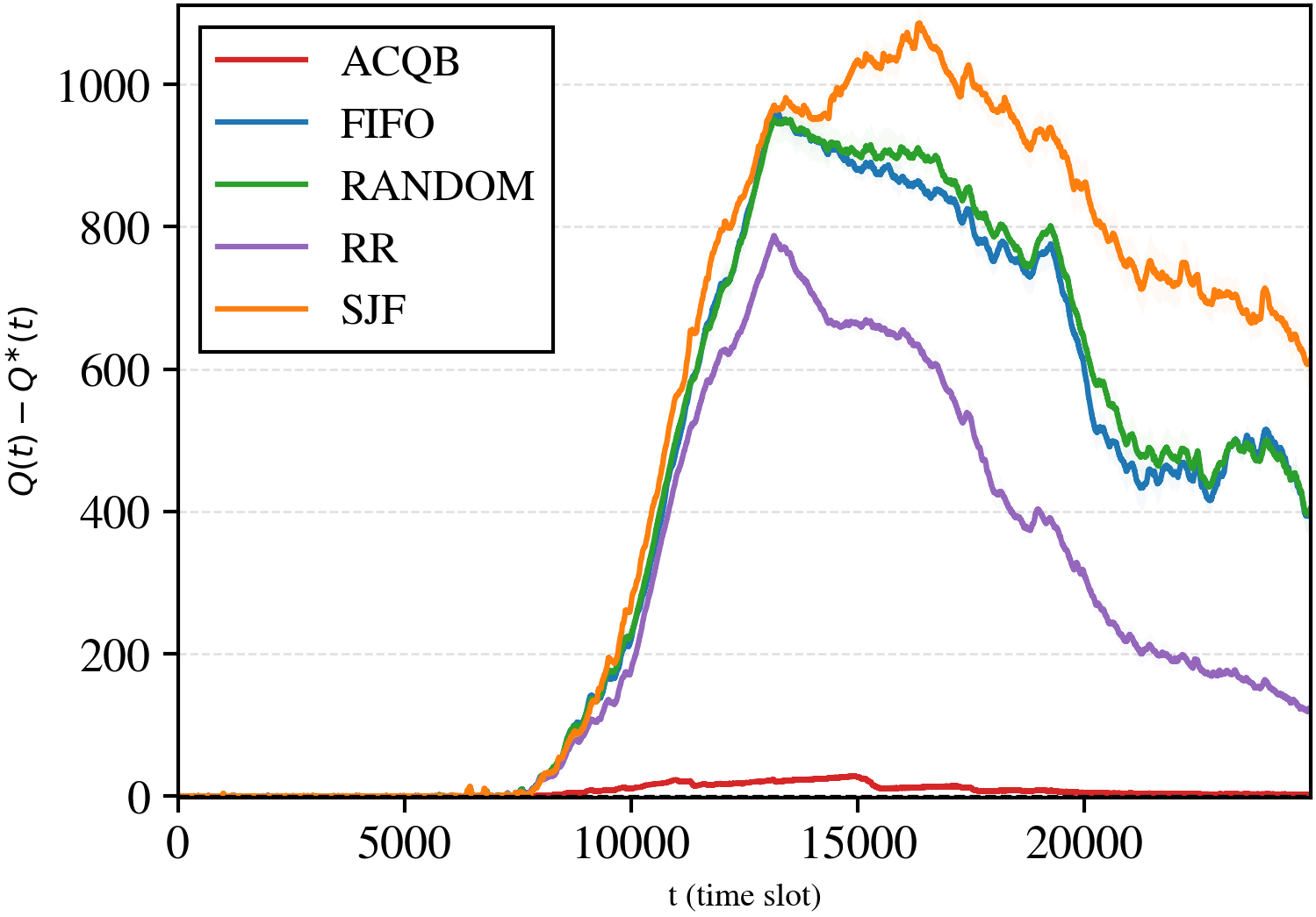}
\caption{Scheduling evaluation on \textsc{WildChat-1M}.}
\label{fig:real_world_setting}

\vspace{0.35em}

\captionof{table}{Scheduling ablation results on \textsc{WildChat-1M}.}
\label{tab:wildchat-qgap-throughput}
\tiny
\setlength{\tabcolsep}{1.7pt}
\renewcommand{\arraystretch}{1.05}
\resizebox{\linewidth}{!}{%
\begin{tabular}{llrrrrr}
\toprule
Time & Stat. & ACQB & FIFO & RR & RANDOM & SJF \\
\midrule
12500 & $\mu$
& \best{18.6} & 800.8 & \second{666.4} & 790.6 & 845.0 \\
& $\sigma$
& 2.5 & 19.3 & 16.4 & 28.6 & 16.2 \\
& T.P.
& \best{0.520} & 0.457 & \second{0.468} & 0.458 & 0.454 \\
\midrule
25000 & $\mu$
& \best{2.4} & 405.4 & \second{124.6} & 405.0 & 614.6 \\
& $\sigma$
& 0.5 & 22.9 & 9.3 & 29.1 & 31.4 \\
& T.P.
& \best{0.584} & 0.568 & \second{0.579} & 0.568 & 0.559 \\
\bottomrule
\end{tabular}%
}

\vspace{-1.4em}
\end{wrapfigure}

We additionally test the scheduling component on real user conversation logs from \textsc{WildChat-1M}~\citep{zhao2024wildchat}. Since publicly available logs from real multi-LLM services are unavailable, we use the conversations served by \texttt{gpt-4-1106-preview} and hold routing fixed.

Retrials are not explicitly labeled in \textsc{WildChat-1M}. We infer retrial chains by sorting conversations by inferred user arrival time and linking two consecutive unresolved conversations from the same hashed IP when their first prompts are identical or have cosine similarity at least $0.995$ under \texttt{paraphrase-multilingual-mpnet-base-v2} embeddings. The user arrival time is estimated by subtracting the response time from the logged response timestamp.

We replay the inferred chains in a discrete-time single-server queue with arrival rate $0.7$. A served job departs if it is the last node in its chain, and otherwise its next retrial appears in the next time slot. Since most chains have no retrial, we randomly subsample chains of length one and keep all chains of length at least two. We compare ACQB with FIFO, round-robin scheduling over the least-served queries (RR), random scheduling, and shortest-job-first (SJF) over five runs. Further implementation details are given in \Cref{sec:exp:wildchat-details}. As shown in \Cref{fig:real_world_setting} and \Cref{tab:wildchat-qgap-throughput}, ACQB has lower queue length regret and higher throughput than the other scheduling rules. At $T=25{,}000$, the mean queue length gap is $2.4$ for ACQB and $124.6$ for RR.

\section{Conclusion}

In this work, we introduce contextual queueing bandits with MNL feedback (CQB-MNL), an online learning framework for joint LLM routing and scheduling from user retrials.
We propose anytime CQB (ACQB), which achieves sublinear cumulative regret and decaying queue length regret without prior knowledge of the time horizon or the traffic slackness parameter.
To make the method practical for heterogeneous LLMs, we use disjoint parameterization and refine query embeddings through contrastive learning.
Experiments on synthetic data, offline routing datasets, and real user conversation logs show that ACQB improves routing, scheduling, and queue stability against online and offline-trained baselines.
Our feedback model focuses on retrials captured through the MNL outside option; extending the model to richer retrial mechanisms involving latency, interface design, reformulation effort, and user patience is an important direction for future work.

\bibliographystyle{plainnat}
\bibliography{example_paper}

\newpage
\appendix
\onecolumn

\section{Related work}

\para{LLM routing and scheduling}
LLM routing selects models to balance quality and cost, while LLM scheduling determines the service order of pending requests to minimize latency. Although substantial research has relied on supervised learning with offline datasets, such as \citet{ong2024routellm,feng2024graphrouter,lu2024routing,somerstep2025carrot} for routing and \citet{fu2024efficient,shahout2024don,yang2025justitia} for scheduling, recent works have emerged that utilize online learning to address the high costs of offline annotation and the non-stationarity of real-world environments \citep{mohammadshahi2024routoo,wu2025efficient,chiang2025llm,jitkrittum2025universal,jaillet2025online,ao2025optimizing}.
However, these methods neglect the interactive nature of LLMs, either ignoring congestion caused by retrials or disrupting user experience by requiring explicit feedback.

\para{Queueing bandits}
The queueing bandit framework, originally designed to address routing problems in systems with unknown service rates, was introduced by \citet{krishnasamy2016regret}.
This seminal work has inspired a significant body of research establishing queue stability---ensuring that expected queue lengths remain bounded---and deriving queue length regret bounds across various settings \citep{liang2018minimizing,stahlbuhk2021learning,choudhury2021job,krishnasamy2021learning,wijewardena2025bandit}.
More recently, this framework has been extended to the contextual setting, incorporating job-specific features and adopting logistic models for departure dynamics \citep{kim2024queueing,bae2026queuelengthregretbounds}.
However, the approach of \citet{kim2024queueing} is limited by a fixed context assumption. While \citet{bae2026queuelengthregretbounds} accommodate heterogeneous contexts, their reliance on the time horizon and traffic slackness parameter hinders practical deployment.

\para{Logistic Bandits}
Starting with the seminal work on generalized linear bandits \citep{filippi2010parametric}, the (multinomial) logistic bandit problem has been extensively studied due to its effectiveness in modeling discrete user choices and categorical feedback \citep{li2017provably,agrawal2019mnl,oh2019thompson,jun2021improved,abeille2021instance,bae2025neural,lee2025improved,zhang2025generalized}.
In this work, we adopt the MNL model for queue dynamics \citep{kim2024queueing,bae2026queuelengthregretbounds} and directly apply these analytical tools to bound the departure rate gap between the optimal policy and our proposed policy.

\section{Experiment details}

\subsection{Algorithm adaptation for disjoint parameterization} \label{ssec:disjoint}

In this section, we present the adaptation of \Cref{alg:1} to the disjoint parameterization setting, as outlined in \Cref{alg:2}. Recall that while the query context $x_t$ is shared, each server $j$ is governed by a distinct unknown parameter vector $\theta_j^*$. Accordingly, we define the collection of parameters as $\Theta^*=[\theta_1^*, \theta_2^*, \dots, \theta_N^*]\in \R^{d\times N}$. When a user with context $x$ is presented with an assortment $S$, the probability $p_j(x, S, \Theta^*)$ of selecting item $j$ and the corresponding departure rate $R(x, S, \Theta^*)$ are defined as 
\begin{align*}
    &p_{j}(x, S, \Theta^*) := \frac{\exp(x^\top \theta_j^*)}{1 + \sum_{j'\in S} \exp (x^\top\theta_{j'}^*)} \quad \text{for } j\in S,  \\
    &p_{0}(x, S, \Theta^*) := \frac{1}{1 + \sum_{j'\in S} \exp (x^\top\theta_{j'}^*)} \quad \text{for } j=0, \\
    &R(x,S,\Theta^*) := \sum_{j\in S} \frac{\exp(x^\top \theta_j^*)}{1 + \sum_{j'\in S} \exp (x^\top\theta_{j'}^*)} = \sum_{j\in S} p_j(x, S, \Theta^*).
\end{align*}

Under this disjoint model, the agent maintains separate statistics for each server $j$: a design matrix $V_{t,j}$, a maximum likelihood estimator $\what{\theta}_{t,j}$, and a confidence radius $\alpha_{t,j}$.
For the design matrices, we initialize $\{V_{t,j}\}_{j=1}^N = \lambda_0 \I$. In each round, we update only the matrices corresponding to the servers included in the assortment $S_t$, such that $V_{t,j}\leftarrow V_{t-1,j} + x_t x_t\tp$ for all $j\in S_t$ (Line 15).
Similarly, the confidence radius is updated only for $j\in S_t$ (Line 14) as follows:
\begin{align}
    \alpha_{t,j} \leftarrow \frac{\kappa}{2} \sqrt{d \log \br{1 + \frac{K \sum_{i=1}^t \ind{j\in S_i}}{d\lambda_0}} + 4 \log \br{\sum_{i=1}^t \ind{j\in S_i}}} + \kappa \sqrt{\lambda_0}, \label{eq:disjoint_alpha}
\end{align}
where we replace the time horizon $t$ in the original $\alpha_t$ definition with the effective number of times server $j$ has been selected, i.e., $\sum_{i=1}^t \ind{j\in S_i}$.

Regarding the maximum likelihood estimators, we define the regularized negative log-likelihood function as
\begin{align*}
    \cal L_t(\Theta) := \frac{\lambda_0}{2} \sum_{j=1}^N \|\theta_j\|_2^2 - \sum_{i=1}^t \sum_{j\in S_i \cup \{0\}} y_{ij} \log p_j(x,S_i, \Theta).
\end{align*}
Then, we update the estimators $\what\theta_{t,j}$ for the relevant servers $j\in S_t$ (Line 14) by solving the first order stationary point $\nabla_{\theta_j} \cal L_t(\Theta) = 0$, which is given by
\begin{align}
      0 = \lambda_0 \theta_j - \sum_{i=1}^t \ind{j\in S_i} (y_{ij} - p_j(x,S_i,\Theta)). \label{eq:disjoint_mle}
\end{align}

Finally, utilizing these separate statistics, the agent selects a query and an assortment as follows: If $A(t-1)=1$ and $E(t-1)=1$, we proceed with the random exploration round (Lines 4--5). Otherwise, for each $j\in[N]$, we sample $\{\wtilde\theta_{t-1,j}^{(i)}\}_{i=1}^M$ as shown in Line 8. We then compute the optimistic departure rate estimate as
\begin{align*}
    &\wtilde u_{tj}(x) = \max_i x\tp \wtilde\theta_{t-1,j}^{(i)}, \quad \wtilde R(x, S) = \sum_{k\in S} \frac{\exp(\wtilde u_{tk}(x))}{1 + \sum_{j\in S} \exp(\wtilde u_{tj}(x))},
\end{align*}
and select the pair maximizing this value (Line 10).

\begin{algorithm} [t] 
\caption{ACQB/ACQB-CL with Disjoint Parameterization} \label{alg:2}
\begin{algorithmic} [1] 
    \Init design matrix $\{V_{0,j}\}_{j=1}^N=\lambda_0\I$, assortment size $K$, sample size $M$, combination set $\cal C=\{S\subset [N]: |S|=K\}$, combination counter $c=0$, exploration parameter $\eta(t)$, confidence radius $\{\alpha_{t,j}\}_{j=1}^N$, projection head $B(\cdot;\theta)$ 
    \State Define context map $\phi(x) = 
\begin{cases} 
    x & \text{if ACQB} \\ 
    B(x; \theta) & \text{if ACQB-CL (with $\theta$ from \Cref{alg:ucl}}) 
\end{cases}$

    \For{$t=1,\dots$}
        \If{$A(t-1)=1$ \textbf{and} $E(t-1)=1$}  
        \SlashComment{on arrivals, $\eta(t)$-exploration}
            \State Set $x_t \leftarrow x^{(t-1)}$, $S_t \leftarrow \cal C[c+1]$
            \State $c \leftarrow c+1 \pmod{|\cal C|}$
        \Else
            \For{$j\in[N]$}
                \State Sample $\{\wtilde\theta_{t-1,j}^{(i)}\}_{i=1}^{M} \sim \cal N(\what\theta_{t-1,j}, \alpha_{t-1,j}^2 V_{t-1,j}^{-1})$ 
            \EndFor
            \State Set $x_{t}, S_{t} \leftarrow \argmax_{x\in\cal X_t, S\in\cal C} \widetilde R(\phi(x), S)$ 
        \EndIf
        \State Assign $x_t$ to $S_t$, receive $y_t=(y_{t0},y_{t1}, \dots,y_{tK})$
        \For{$j\in[N]$}
            \If{$j\in S_t$}
                \State Update $\what \theta_{t,j}$ as in \Cref{eq:disjoint_mle}, $\alpha_{t,j}$ as in \Cref{eq:disjoint_alpha}
                \State $V_{t,j} \leftarrow V_{t-1,j} + \phi(x_t) \phi(x_t)\tp$ 
            \Else
                \State $\what \theta_{t,j} \leftarrow \what \theta_{t-1,j}$, $\alpha_{t,j}\leftarrow \alpha_{t-1,j}$, $V_{t,j} \leftarrow V_{t-1,j}$
            \EndIf
        \EndFor
        \State Sample $E(t)\sim \Bern(\eta(t))$
    \EndFor
\end{algorithmic}
\end{algorithm}

\subsection{Utility-aligned query embeddings via contrastive learning} \label{ssec:ucl}

\begin{algorithm}[htbp]
\caption{Utility-Based Contrastive Learning}
\label{alg:ucl}
\begin{algorithmic}[1]
\Init offline dataset $\mathcal D_{\mathrm{off}}=\{(\xi_{(i)},u_{(i)})\}_{i=1}^{n}$, frozen encoder $E(\cdot)$, projection head $B(\cdot;\theta)$, thresholds $(\iota_{\mathrm{_{pos}}},{\iota_\mathrm{{neg}}})$, temperature $\tau_{temp}$, negative cap $K_{\mathrm{neg}}$, learning rate $\eta$
\For{$\text{epoch}=1,\dots,e$}
    \State $x_{(i)} \leftarrow E(\xi_{(i)})$,\quad $z_{(i)}\leftarrow B(x_{(i)};\theta)$,\quad $\bar{u}_{(i)} \leftarrow u_{(i)} - \mathrm{mean}(u_{(i)})$\ for all $i$
    \State $c_{(i,j)}\leftarrow \cos(\bar u_{(i)},\bar u_{(j)})$ \ for all $i\neq j$
    \For{$i=1,\dots,n$}
        \State Set $P(i)=\{j\neq i:\ c_{(i,j)}>\iota_{\mathrm{pos}}\}$,\quad $N(i)=\{j\neq i:\ c_{(i,j)}<\iota_{\mathrm{neg}}\}$
        \State Pick $j^+\in\arg\max_{j\in P(i)} c_{(i,j)}$
        \State Pick $\mathcal N_{(i)}\subseteq N(i)$ with $|\mathcal N_{(i)}|\le K_{\mathrm{neg}}$
        \State Set $\ell_{(i)}(\theta)$ as in \Cref{eq:infonce}
    \EndFor
    \State $\theta \leftarrow \theta - \eta \nabla_\theta \sum_{(i):\,|P((i))|>0}\ell_{(i)}(\theta)$ 
\EndFor
\State \Return $\theta$
\end{algorithmic}
\end{algorithm}

This section details the utility-based contrastive learning framework for \textsc{ACQB-CL}, extending the overview in \Cref{sec:EQCL}. The goal of this offline phase is to train the projection head $B(\cdot;\theta)$ appended to the frozen encoder $E(\cdot)$. Leveraging an offline dataset containing utility values (i.e., performance and cost) for all servers, we aim to structure the embedding space such that queries with similar routing utilities are clustered together, while those with dissimilar utilities are separated.

We begin the process by constructing a balanced offline dataset from the representation learning split. Specifically, we employ a model-balanced sampling strategy. For each available prompt $\xi$, we compute the utility vector $\{u_j(\xi)\}_{j=1}^N$ across all models and identify the optimal model that yields the highest utility.
We then group prompts according to their optimal model index. To ensure diversity, we uniformly sample a fixed number of prompts (e.g., 5 prompts in our experiments for \textsc{Sprout} data) from each group.
This procedure ensures that the offline training set is composed of samples selected evenly across all candidate LLMs.

\para{Offline Training Algorithm}
After sampling, we form the offline dataset as $\mathcal D_{\mathrm{off}}=\{(\xi_{(i)},u_{(i)})\}_{i=1}^{n}$, where $\xi_{(i)}$ denotes the selected prompt and $u_{(i)}$ represents its utility vector across all models.
We use the frozen encoder outputs $x_{(i)}=E(\xi_{(i)})$ and the utility vectors $u_{(i)}$ to train the two-layer MLP projection head $B(\cdot;\theta)$. \Cref{alg:ucl} summarizes the procedure.

At the beginning of each epoch, we compute the projected representation $z_{(i)}=B(x_{(i)};\theta)$ for each offline prompt. To determine similarity targets, we first mean-center each utility vector across models:
\[
    \bar u_{(i)} \leftarrow u_{(i)}-\mathrm{mean}(u_{(i)}).
\]
We then compute the pairwise utility similarity $c_{(i,j)}$ for all $i\neq j$ using the cosine similarity of these centered vectors:
\[
    c_{(i,j)}\leftarrow \cos(\bar u_{(i)},\bar u_{(j)}):=\frac{\bar u_{(i)}^\top \bar u_{(j)}}{\|\bar u_{(i)}\|_2\,\|\bar u_{(j)}\|_2}.
\]
This metric is used for selecting positive and negative pairs. Given a pair of thresholds $(\iota_{\mathrm{pos}}, \iota_{\mathrm{neg}})$, we define the candidate positive set $P(i)$ and negative set $N(i)$ as
\[
    P(i)=\{j\neq i:\ c_{(i,j)}>\iota_{\mathrm{pos}}\}, \quad N(i)=\{j\neq i:\ c_{(i,j)}<\iota_{\mathrm{neg}}\}.
\]
Any candidate $j$ with similarity falling between the thresholds (i.e., $\iota_{\mathrm{neg}} \leq c_{(i,j)} \leq \iota_{\mathrm{pos}}$) is ignored.
From the positive set $P(i)$, we select the single best positive example $j^+$ that maximizes utility similarity, i.e. $j^+\in\arg\max_{j\in P(i)} c_{(i,j)}$.
For the negative set, we construct a subset $\mathcal N_{(i)}\subseteq N(i)$ to limit computational cost. Specifically, we select up to $K_{\mathrm{neg}}$ hard negatives (those with the smallest similarity values $c_{(i,j)}$). If $P(i)$ is empty or no negatives remain, we skip the update for query $i$.

Using the selected positive $j^+$ and negative set $\mathcal N_{(i)}$, we compute the representation similarity $\mathrm{sim}(z_{(i)},z_{(j)})=z_{(i)}^\top z_{(j)}$ and define the InfoNCE loss with temperature $\tau_{\mathrm{temp}}$ as
\begin{align} \label{eq:infonce}
    \ell_{(i)}(\theta)= -\log \dfrac{\exp(\mathrm{sim}(z_{(i)},z_{(j^+)})/\tau_{\mathrm{temp}})}{\exp(\mathrm{sim}(z_{(i)},z_{(j^+)})/\tau_{\mathrm{temp}})+\sum_{j\in\mathcal N_{(i)}}\exp(\mathrm{sim}(z_{(i)},z_{(j)})/\tau_{\mathrm{temp}})}.
\end{align}
Finally, we update the projection head parameters $\theta$ via gradient descent over all valid queries:
\[
    \theta \leftarrow \theta - \eta \nabla_\theta \sum_{i:\,|P(i)|>0}\ell_{(i)}(\theta).
\]
After offline pretraining, we freeze the projection head $B(\cdot;\theta)$ and proceed to the online learning stage described in \Cref{alg:2}.

\paragraph{Experimental Settings.}
We use $\tau_{\mathrm{temp}}=0.07$ and a negative cap $K_{\mathrm{neg}}=64$. The projection head is a two-layer MLP with an output dimension of $384$, matching the encoder dimension. We set distinct threshold pairs $(\iota_{\mathrm{pos}},\iota_{\mathrm{neg}})$ for each dataset: $(0.6,0.3)$ for \textsc{SPROUT-o3mini}, $(0.4,0.1)$ for \textsc{EmbedLLM}, and $(0.5,0.3)$ for \textsc{RouterBench}. For contrastive pretraining of $B$ in an offline setting, we use 5 prompts per model for \textsc{SPROUT-o3mini} and train for 10 epochs; for \textsc{EmbedLLM}, we use 10 prompts per model (30 prompts per epoch) and train for 50 epochs.

\subsection{Baseline algorithms}
\label{sec:exp:baselines}

We compare our proposed ACQB and ACQB-CL algorithms against the following baselines:

\begin{enumerate}
    \item[(i)] \textbf{Optimal Policy:} 
    The optimal policy has full access to the true set of disjoint parameters $\Theta^*=\{\theta_j^*\}_{j=1}^N$. In every round $t$, given the set of backlogged queries $\cal X_t^*$, it selects the query-assortment pair $(x^*, S^*)$ that maximizes the expected departure rate:
    \begin{align*}
        (x^*, S^*) \in \argmax_{x\in\mathcal{X}_t^*, S\in\mathcal{C}} R(x,S,\Theta^*).
    \end{align*}
    
    \item[(ii)] \textbf{Random Policy (RAND):} 
    In every round $t$, the random policy selects a query $x_t$ uniformly at random from the current backlog and chooses an assortment $S_t$ uniformly at random from the feasible set $\cal C$.
    \begin{itemize}
        \item \textbf{Random Router Policy (RAND\_ROUT):} In every round $t$, the random router policy selects a query $x_t$ that arrived first among the current backlog and chooses an assortment $S_t$ uniformly at random from the feasible set $\cal C$.
    \end{itemize}
    
    \item[(iii)] \textbf{Q-UCB} \citep[Algorithm 1]{krishnasamy2021learning}: 
    This is a MAB-based algorithm designed for the single-item setting ($K=1$). Since this algorithm treats all queries as identical ignoring their contexts, it employs a first-in-first-out (FIFO) scheduling rule, simply selecting the oldest job in the queue as $x_t$. The algorithm explores with probability $\min\{1, 3N(\log^2t)/t\}$. For the assortment selection, it chooses the item that maximizes the upper confidence bound (UCB):
    \begin{align*}
        S_t = \argmax_{S\in\cal C, |S|=1} \left( \what\mu_{S}(t) + \sqrt{\frac{\log^2 t}{2 T_S(t-1)}} \right),
    \end{align*}
    where $\what\mu_S(t) = \sum_{i=1}^{t-1} \ind{S_i=S} y_i / T_S(t-1)$ is the empirical mean reward and  $T_S(t-1) = \sum_{i=1}^{t-1} \ind{S_i = S}$is the number of times assortment $S$ has been played up to time $t-1$.

    \item[(iv)] \textbf{Q-ThS} \citep[Algorithm 2]{krishnasamy2021learning}: 
    Similar to Q-UCB, this is an MAB-based approach for $K=1$ using FIFO scheduling. It explores with the same probability schedule as Q-UCB. For every assortment selection $S_t$, it samples a departure rate estimate from a Beta posterior:
    \begin{align*}
        \wtilde r_t(x_t, S) \sim \text{Beta}\Big(\what \mu_S(t) T_S(t-1) + 1, (1-\what\mu_S(t)) T_S(t-1) + 1\Big).
    \end{align*}
    It then selects the assortment $S_t$ that maximizes this sampled value, i.e. $S_t:=\argmax_{K=1,S\in\cal C} \wtilde r_t(x_t, S)$, where we use the same definitions of $\what\mu_S(t)$ and $T_S(t-1)$ as Q-UCB.

    \item[(v)] \textbf{CQB-$\v\veps$} (Adapted from \citet{bae2026queuelengthregretbounds}): 
    This baseline is an adaptation of the contextual queueing bandit algorithm from \citet{bae2026queuelengthregretbounds}. While the original algorithm was developed as a UCB-based approach under logistic feedback, we modified it to utilize Thompson Sampling with MNL feedback. This adaptation ensures a fair and consistent comparison with our proposed framework under the same environmental assumptions.

    Let $\tau$ denote the length of the pure-exploration phase. If $A(t-1)=1$ and the agent is either in the exploration phase ($t \leq \tau$) or triggers a random exploration with probability $T^{-1/2}$, it selects $x_t\leftarrow x^{(t-1)}$ (the most recent job) and chooses $S_t$ in a round-robin manner. Otherwise, it follows the optimistic rule described in ACQB (Line 6--9 of \Cref{alg:2}).

    Notice that the theoretical requirement for the exploration length of CQB-$\v\veps$, $\tau=\mathcal{O}(\frac{d\log(T)}{\sigma_0^4 \epsilon^2})$, is often prohibitively large for practical experiments. Therefore, we adopt distinct practical heuristics depending on the experimental setting.
For synthetic data experiments, following previous work \citep{bae2026queuelengthregretbounds}, we set a fixed exploration length of $\tau = T/10$.
For offline routing datasets, we determine $\tau$ by aligning it with the stabilization point of our exploration rate $\eta(t)$. Specifically, we set $\tau = \min \{t : c_1(t+1)^{-1/2} \leq 1\}$, ensuring that the pure-exploration phase ends when ACQB's exploration probability drops below 1.    
    % Notice that the theoretical requirement for the exploration length of CQB-$\v\veps$, $\tau=\bigO{\frac{d\log(T)}{\sigma_0^4 \eps^2}}$, is often prohibitively large for practical experiments. Therefore, we adopt a practical heuristic by aligning $\tau$ with the stabilization point of our exploration rate $\eta(t)$. Specifically, we set $\tau = \min \{t : c_1(t+1)^{-1/2} \leq 1\}$, ensuring that the pure-exploration phase ends when ACQB's exploration probability drops below 1.
\end{enumerate}

\para{Scheduling variants}
To isolate the effect of scheduling from routing, we evaluate variants that keep the ACQB routing rule fixed while replacing only the query-selection rule. In the synthetic experiments, all variants use the same exploration schedule, arrival process, MNL feedback generation, and Thompson-sampling routing rule as ACQB. The difference appears only in exploitation rounds: FIFO selects the oldest query, round-robin selects uniformly among the least-served queries, and random scheduling selects a query uniformly from the current queue. After the query is selected, each variant chooses the assortment by the same optimistic ACQB routing rule for that query.

\para{Offline-trained routing baselines}
For simulations on offline routing datasets, we further compare against offline-trained routing baselines combined with FIFO scheduling. We include CARROT~\citep{somerstep2025carrot}, KNN~\citep{hu2024routerbench}, MLP~\citep{hu2024routerbench}, ROUTERDC~\citep{chen2024routerdc}, UMR~\citep{jitkrittum2025universal}, and ZERO~\citep{hu2024routerbench}. For each run, the baselines are trained on the offline split excluding all evaluation prompts and then evaluated on the same online prompt pool as ACQB. They use the same query embeddings and utility target $\mathrm{perf}_j(\xi)-\rho\mathrm{cost}_j(\xi)$: ZERO chooses the best fixed model from the training split; KNN and MLP predict per-model performance from the query embedding and subtract the model cost; CARROT predicts both performance and cost; UMR clusters query embeddings and routes by cluster-level error and cost; and ROUTERDC trains a neural router with query and model embeddings. Since these routers are not designed for queueing systems with retrial feedback, we combine each router with FIFO scheduling and use the same MNL departure-probability construction as in \Cref{sec:exp:details}.

\subsection{\textsc{WildChat-1M} experiment details}
\label{sec:exp:wildchat-details}

For the \textsc{WildChat-1M} experiment, we use conversations served by \texttt{gpt-4-1106-preview}. Since the logged timestamp is the assistant response timestamp, we estimate the user arrival time by subtracting $2.84+\mathrm{tokens}/30$ seconds from the timestamp of the first assistant response. We then sort conversations by the inferred arrival time and construct retrial chains separately for each hashed IP. A conversation is linked to the previous unresolved conversation from the same hashed IP if the first user prompts are identical or their normalized \texttt{paraphrase-multilingual-mpnet-base-v2} embeddings have cosine similarity at least $0.995$, and if the inferred arrival time is after the previous assistant timestamp. Since most conversations do not have a retrial, we randomly keep a subset of singleton chains so that the singleton-chain fraction is $0.8$, while keeping all non-singleton chains. For the queue replay, root conversations are assigned to discrete time bins according to the inferred arrival times with arrival rate $0.7$ and service capacity one. When a served node is not terminal in its chain, the next node enters the queue in the next time slot; otherwise, the job departs. ACQB uses 32-dimensional projected prompt embeddings with the chain state variables. For SJF, we set the job length to the number of tokens in the first assistant response.

\subsection{Modeling departure and choice probabilities for offline routing dataset experiments}
\label{sec:exp:details}

This section details how we transform raw utility scores into departure and choice probabilities, complementing the experimental setup in \Cref{ssec:offline-routing}.
Recall that for a raw prompt $\xi$ and model $j$, we define the (raw) utility as $u_j^{\text{raw}}(\xi) = \mathrm{perf}_j(\xi) - \rho \cdot \mathrm{cost}_j(\xi)$, where we use $\rho=5$ in our experiments. 
Since this value may fall outside the feasible probability range $[0,1]$ depending on $\rho$, we apply a two-step transformation. 
First, we apply min-max normalization over the models to obtain a value in $[0,1]$:
\begin{align*}
    u_j^{\text{norm}}(\xi) = \frac{u_j^{\text{raw}}(\xi)-\min_k u_k^{\text{raw}}(\xi)}{\max_k u_k^{\text{raw}}(\xi)-\min_k u_k^{\text{raw}}(\xi)}.
\end{align*}
Next, to prevent extreme probabilities (0 or 1) that can cause numerical instability, we linearly rescale this normalized value to a bounded range $[r_{\mathrm{lo}}, r_{\mathrm{hi}}] \subset (0,1)$. We define the final departure probability $u_j(\xi)$ as:
\begin{align*}
    u_j(\xi) := r_{\mathrm{lo}} + (r_{\mathrm{hi}}-r_{\mathrm{lo}}) \cdot u_j^{\text{norm}}(\xi),
\end{align*}
where we set $r_{\mathrm{lo}}=0.1$ and $r_{\mathrm{hi}}=0.99$ in our experiments.

We define the true departure probability for prompt $\xi$ on model $j$ as this transformed value $u_j(\xi)$. 
To be consistent with our MNL choice framework, let $x=E(\xi)$ denote the query context. We construct the latent reward $x\tp \theta_j^*$ such that its logistic function matches the departure probability:
\begin{align*}
    u_j(\xi) = \frac{\exp(x\tp\theta_j^*)}{1 + \exp(x\tp\theta_j^*)} \quad \Longrightarrow \quad \exp(x\tp\theta_j^*) = \frac{u_j(\xi)}{1 - u_j(\xi)}.
\end{align*}
Substituting this relationship into the MNL choice model for an assortment $S$, we derive the final choice probabilities entirely in terms of the derived utilities:
\begin{align*}
    p(j \mid x, S) = \frac{\exp(x\tp\theta_j^*)}{1 + \sum_{k\in S} \exp(x\tp\theta_k^*)} = \frac{\frac{u_j(\xi)}{1 - u_j(\xi)}}{1 + \sum_{k\in S} \frac{u_k(\xi)}{1 - u_k(\xi)}}.
\end{align*}

\section{Additional experimental results}
\label{sec:exp:result}
This section provides additional experiments omitted from the main text. We first report synthetic experiments with varying $N$ and $\eps$, then show the exploration-scheduling ablation, and finally provide additional simulations on offline routing datasets.
All experiments were conducted on a server equipped with an AMD EPYC 9354 32-Core Processor, 251 GiB of RAM, and one NVIDIA RTX A6000 GPU.

\subsection{Additional synthetic data results} 
\label{sec:exp:sim}
We use the same setup as \Cref{ssec:sync} and vary one parameter at a time. \Cref{fig:sim_results2} changes the number of LLMs $N$, while \Cref{fig:sim_result3} changes the slackness parameter $\eps$.

\begin{figure*}[htbp]
    \centering
    
    % --- 설정: 너비 변수 수정 ---
    % 중앙 라벨이 사라졌으므로 blockwidth를 0.46 -> 0.48로 늘려 공간 활용
    \newcommand{\lbwidth}{0.02\textwidth}  % 세로 라벨 너비
    \newcommand{\blockwidth}{0.48\textwidth} % 그림 블록 너비 (확장됨)
    
    % ---------------- Header (K=1, K=2) ----------------
    % 1. Left Spacing (첫 번째 세로 라벨 위의 빈 공간)
    \begin{minipage}{\lbwidth} \centering \phantom{L} \end{minipage}%
    \hfill%
    % 2. K=1 Header (첫 번째 그림 블록 중앙)
    \begin{minipage}{\blockwidth} \centering \scriptsize $K=1$ \end{minipage}%
    \hfill%
    % 3. K=2 Header (두 번째 그림 블록 중앙) - 중간 라벨 공간 없음
    \begin{minipage}{\blockwidth} \centering \scriptsize $K=2$ \end{minipage}%
    
    \vspace{0.1cm} % 헤더와 첫 줄 사이 간격

    % ================= ROW 1: lambda=0.55 =================
    % 1. Row Label
    \begin{minipage}[c]{\lbwidth}
        \centering \rotatebox{90}{\scriptsize $N=3$}
    \end{minipage}%
    \hfill%
    % 2. K=1 Block (Images)
    \begin{minipage}[c]{\blockwidth}
        \centering
        \begin{subfigure}[c]{0.49\linewidth}
            \includegraphics[width=\linewidth]{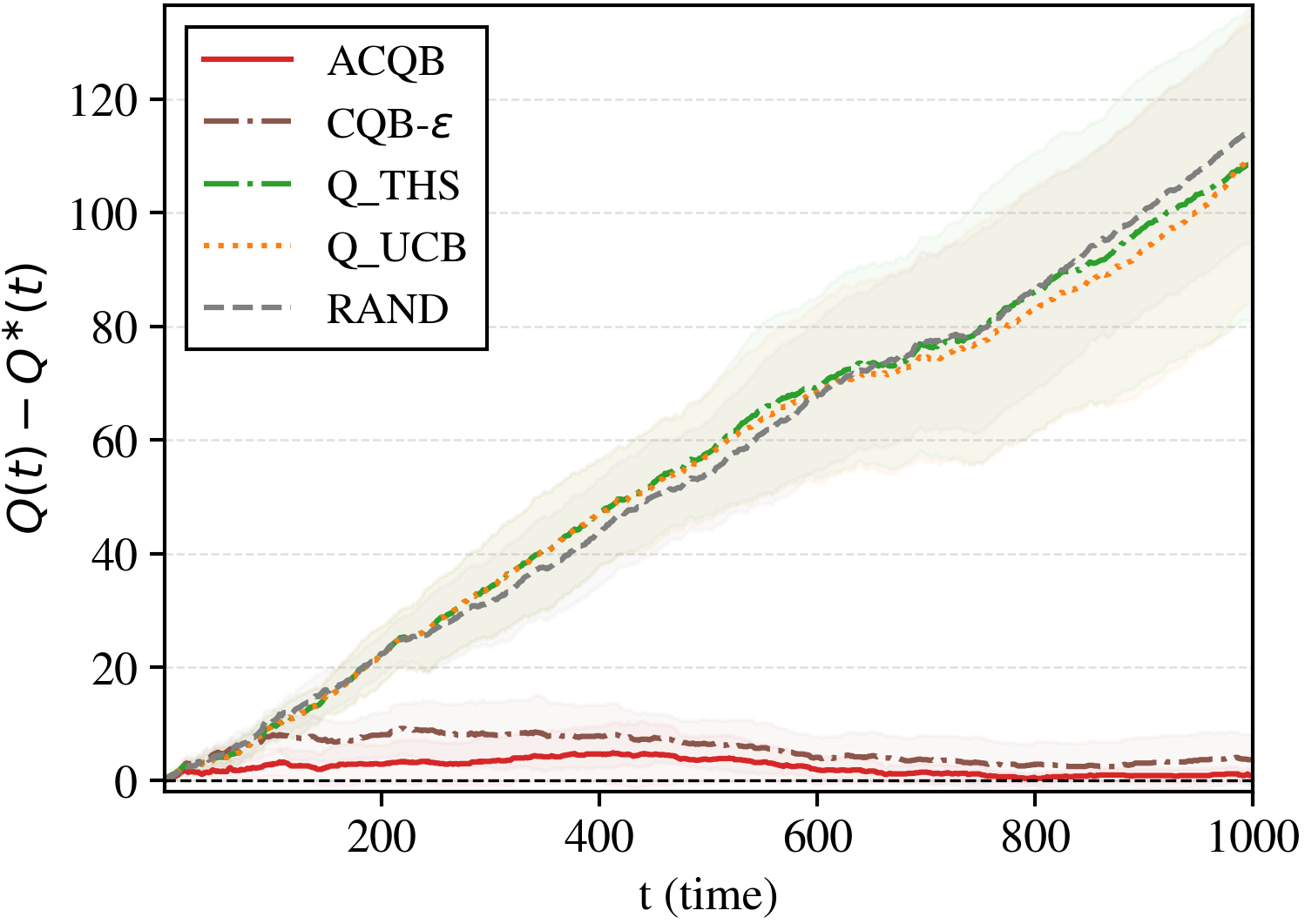}
        \end{subfigure}\hfill
        \begin{subfigure}[c]{0.49\linewidth}
            \includegraphics[width=\linewidth]{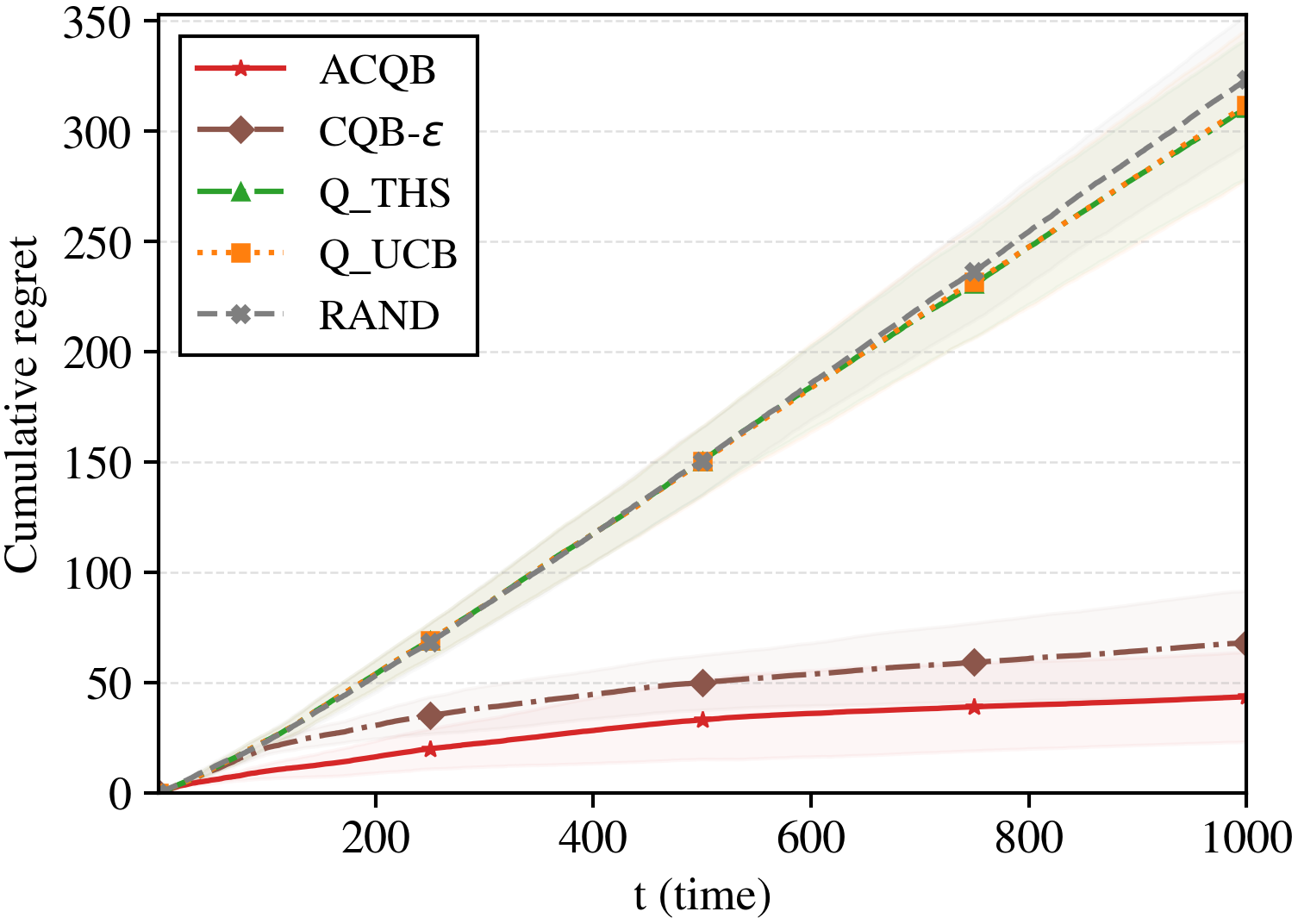}
        \end{subfigure}
    \end{minipage}%
    \hfill%
    % 3. K=2 Block (Images) - 중간 라벨 제거됨, 바로 이어짐
    \begin{minipage}[c]{\blockwidth}
        \centering
        \begin{subfigure}[c]{0.49\linewidth}
            \includegraphics[width=\linewidth]{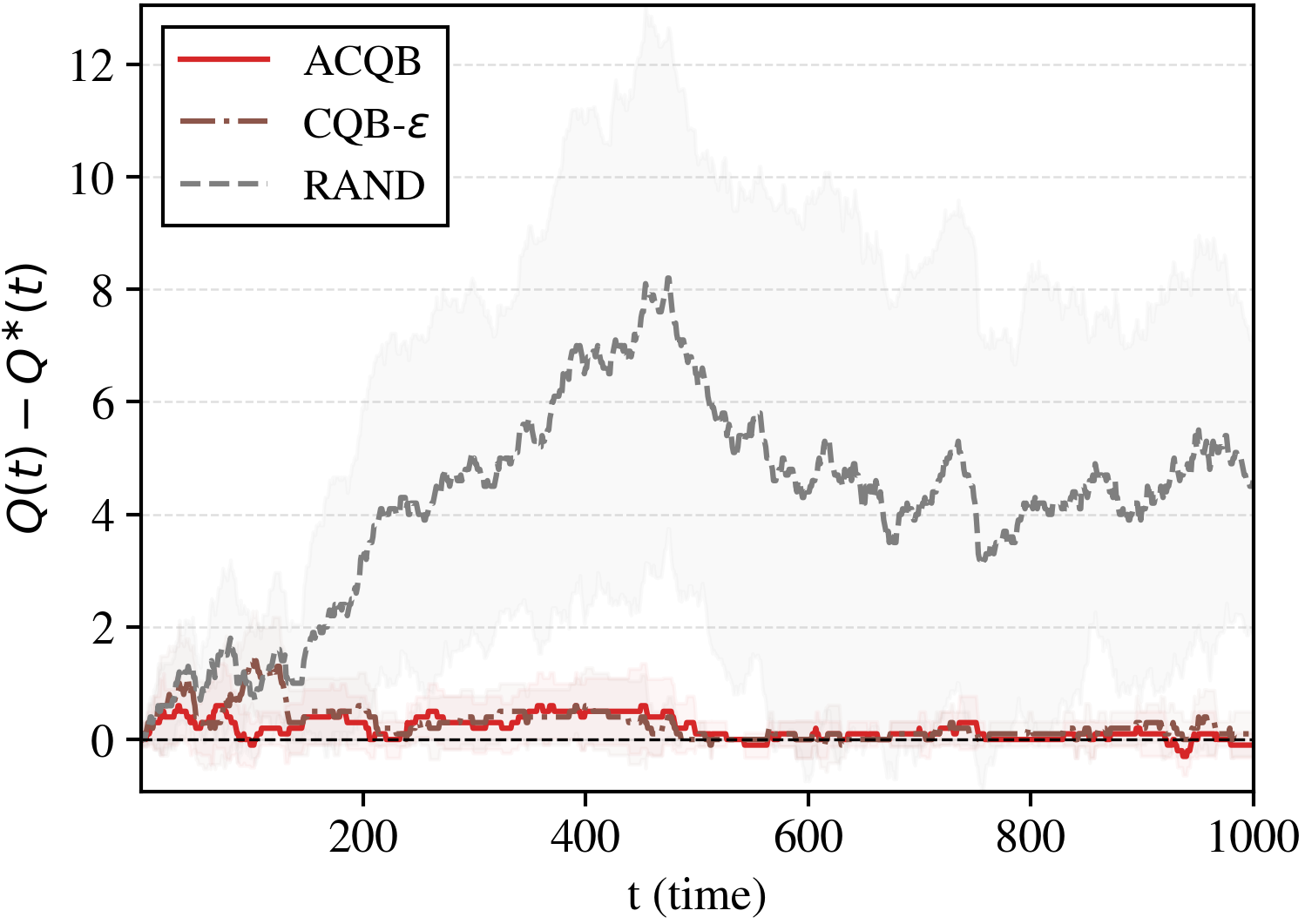}
        \end{subfigure}\hfill
        \begin{subfigure}[c]{0.49\linewidth}
            \includegraphics[width=\linewidth]{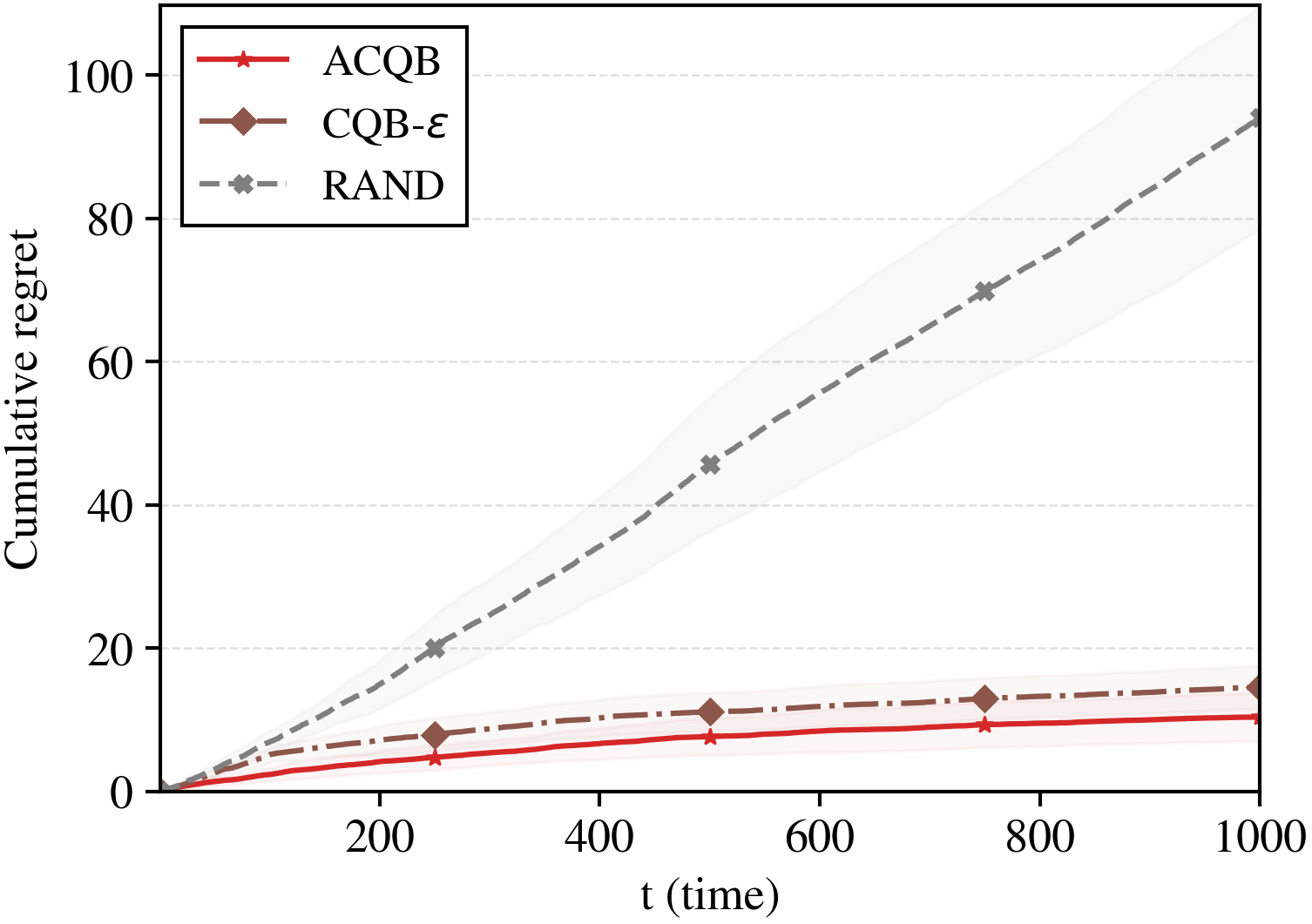}
        \end{subfigure}
    \end{minipage}%

    \vspace{0.15cm} % 행 간 간격

    % ================= ROW 2: lambda=0.65 =================
    \begin{minipage}[c]{\lbwidth}
        \centering \rotatebox{90}{\scriptsize $N=5$}
    \end{minipage}%
    \hfill%
    \begin{minipage}[c]{\blockwidth}
        \centering
        \begin{subfigure}[c]{0.49\linewidth}
            \includegraphics[width=\linewidth]{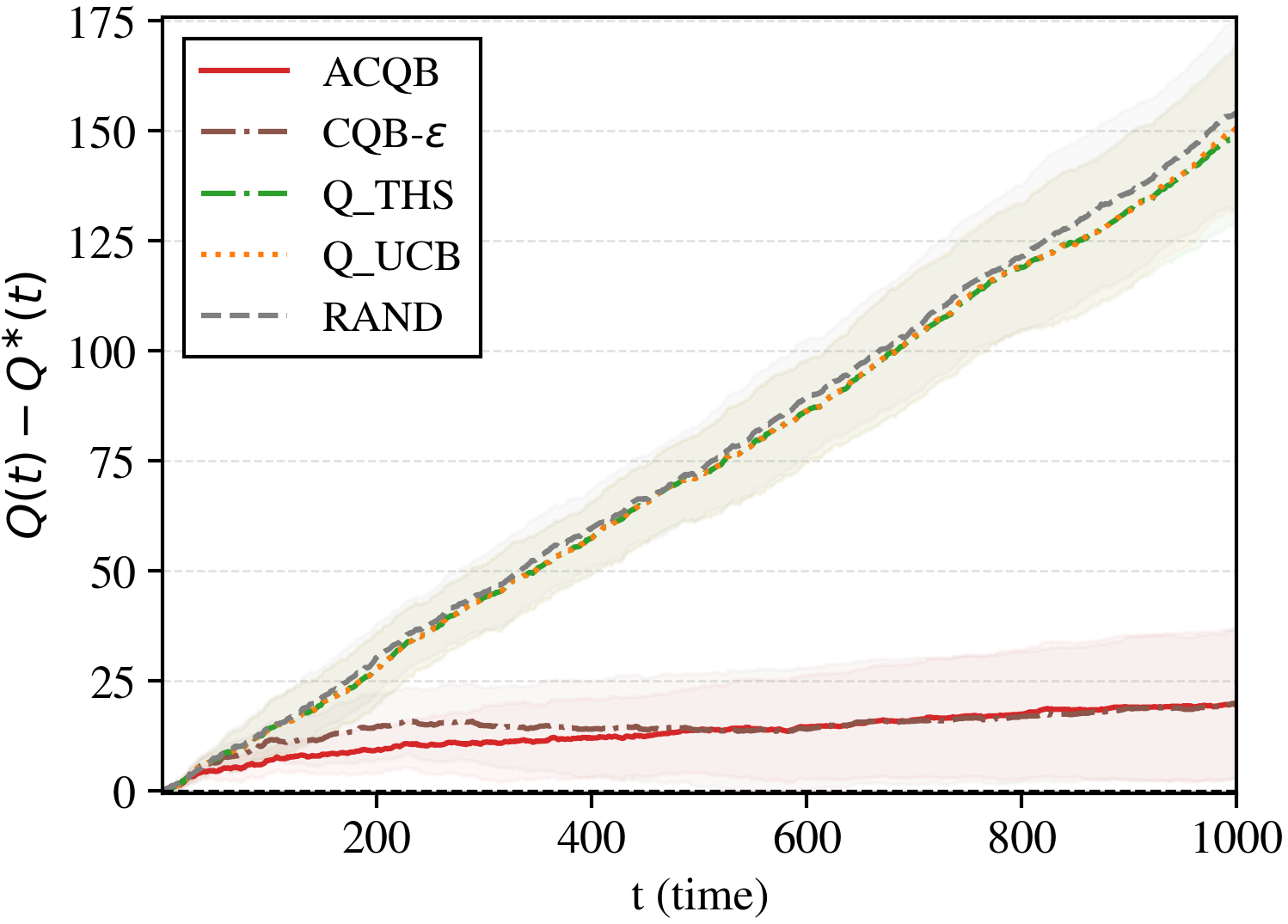}
        \end{subfigure}\hfill
        \begin{subfigure}[c]{0.49\linewidth}
            \includegraphics[width=\linewidth]{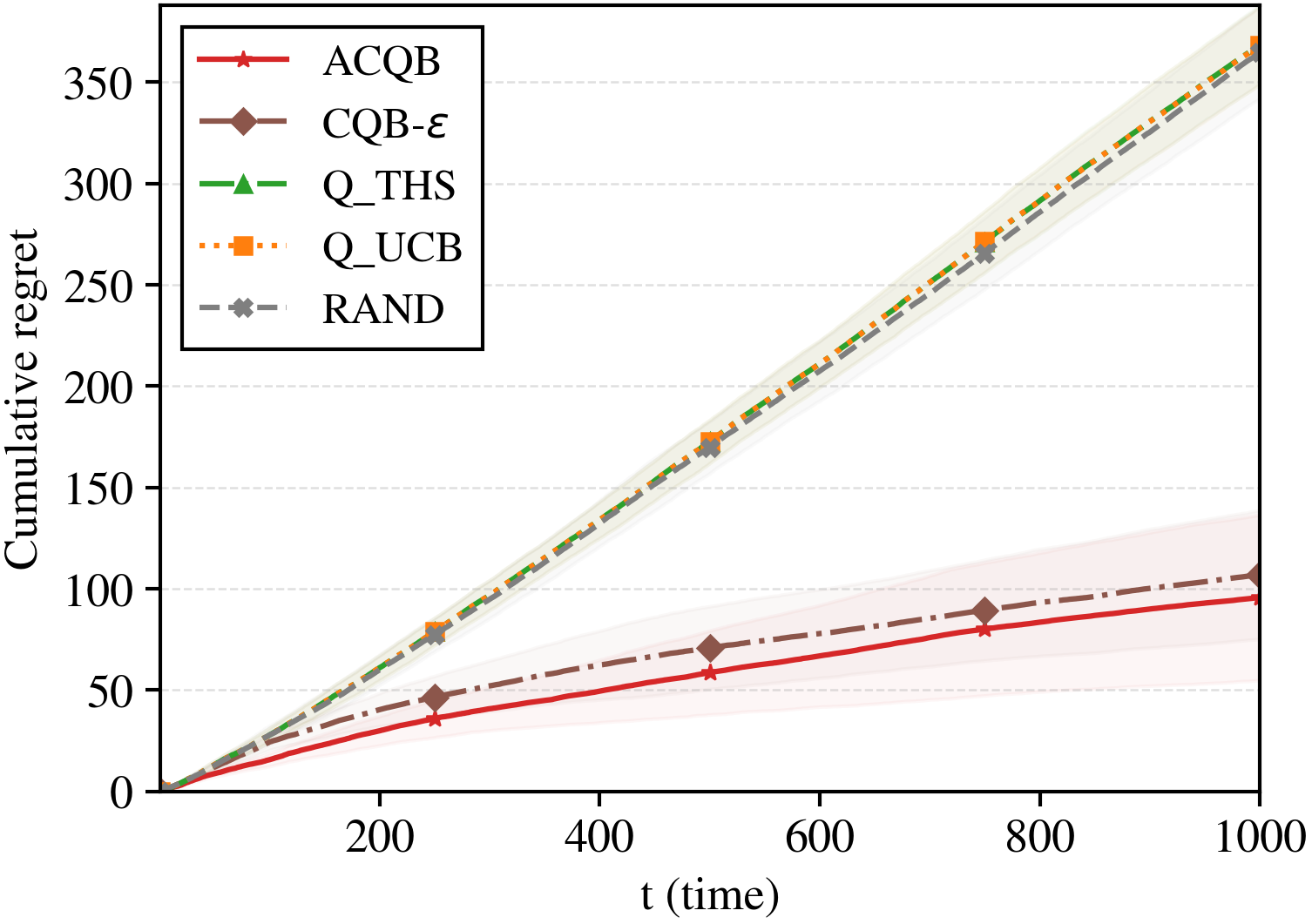}
        \end{subfigure}
    \end{minipage}%
    \hfill%
    \begin{minipage}[c]{\blockwidth}
        \centering
        \begin{subfigure}[c]{0.49\linewidth}
            \includegraphics[width=\linewidth]{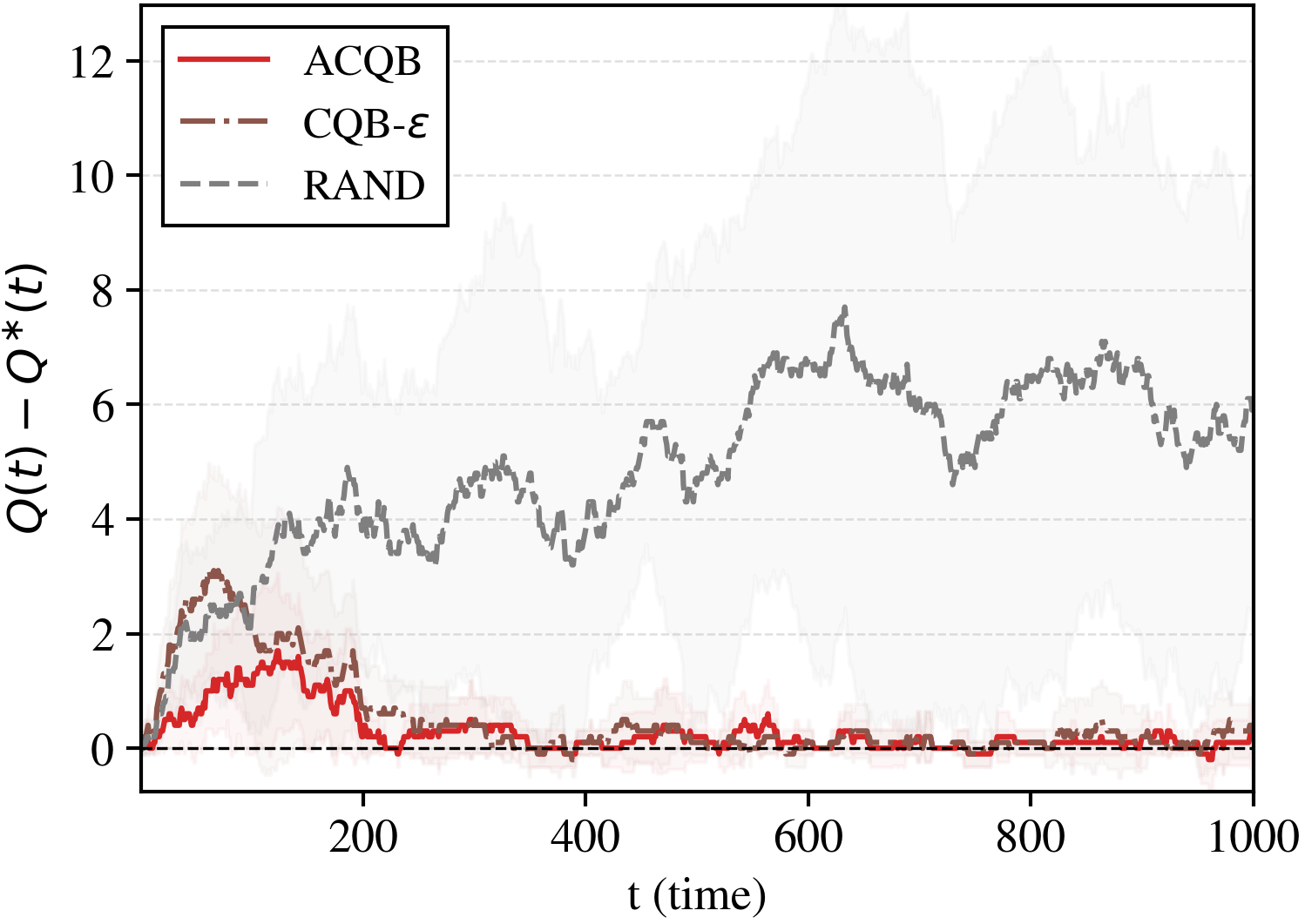}
        \end{subfigure}\hfill
        \begin{subfigure}[c]{0.49\linewidth}
            \includegraphics[width=\linewidth]{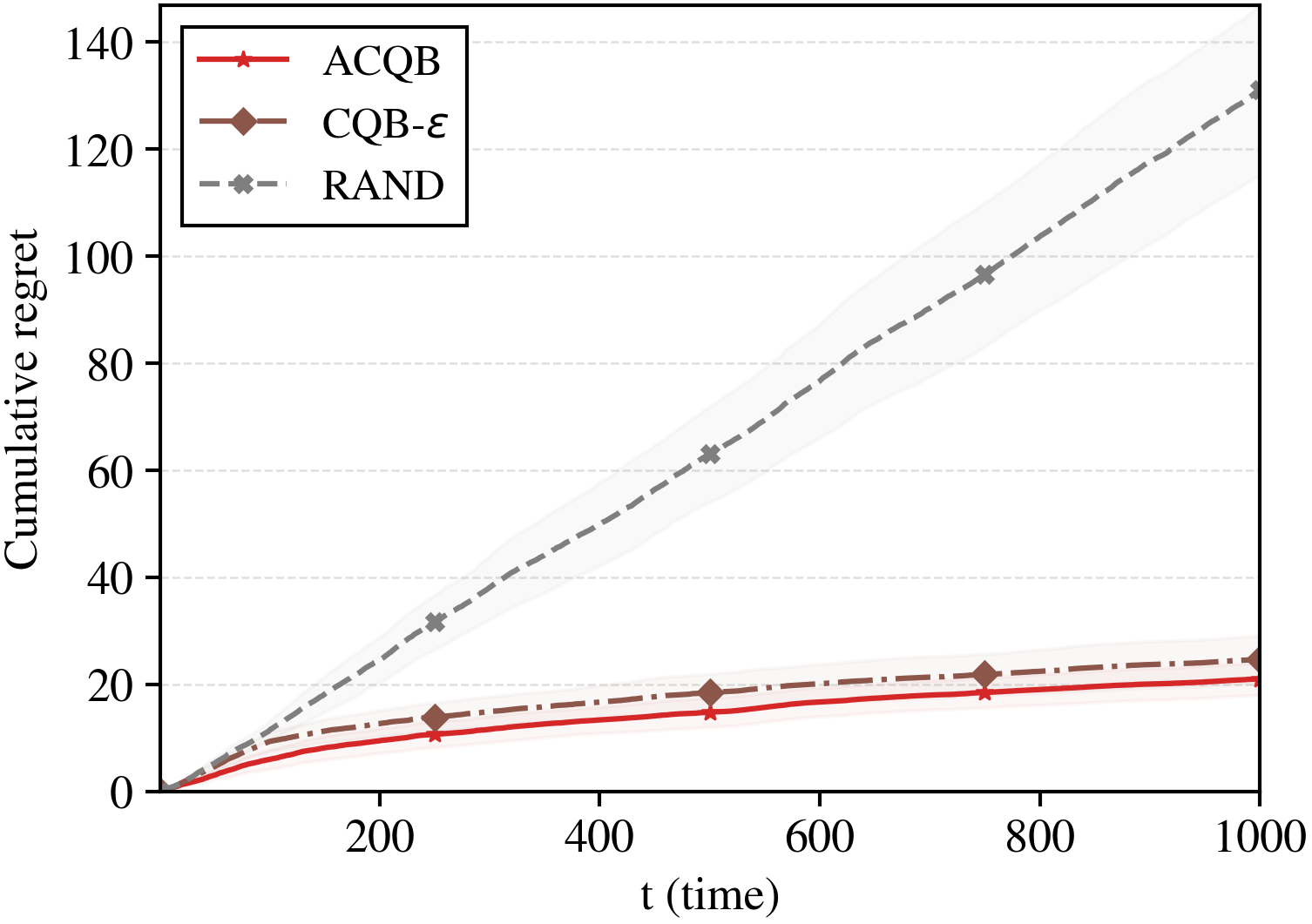}
        \end{subfigure}
    \end{minipage}%

    \vspace{0.15cm} % 행 간 간격

    % ================= ROW 3: lambda=0.75 =================
    \begin{minipage}[c]{\lbwidth}
        \centering \rotatebox{90}{\scriptsize $N=10$}
    \end{minipage}%
    \hfill%
    \begin{minipage}[c]{\blockwidth}
        \centering
        \begin{subfigure}[c]{0.49\linewidth}
            \includegraphics[width=\linewidth]{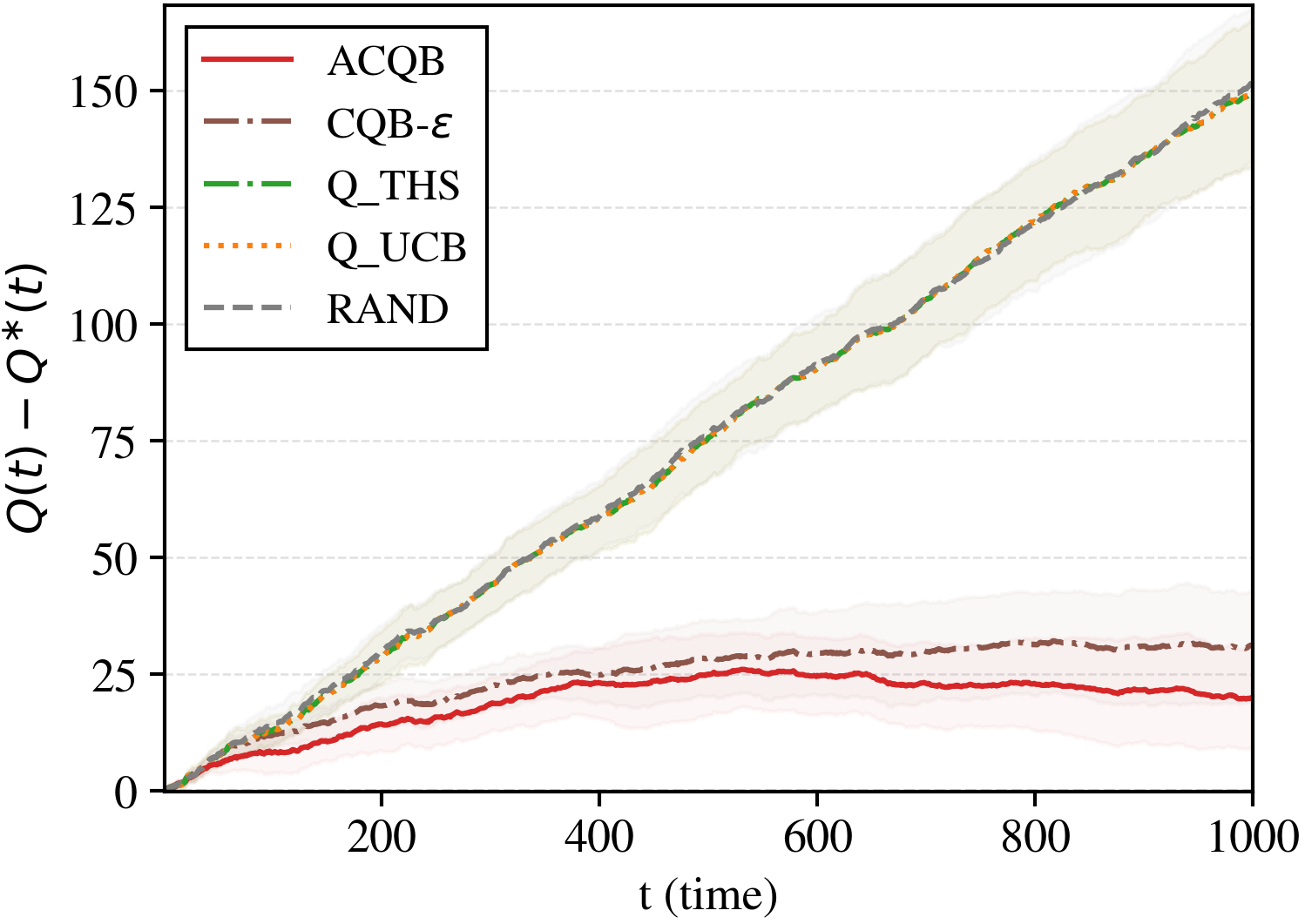}
        \end{subfigure}\hfill
        \begin{subfigure}[c]{0.49\linewidth}
            \includegraphics[width=\linewidth]{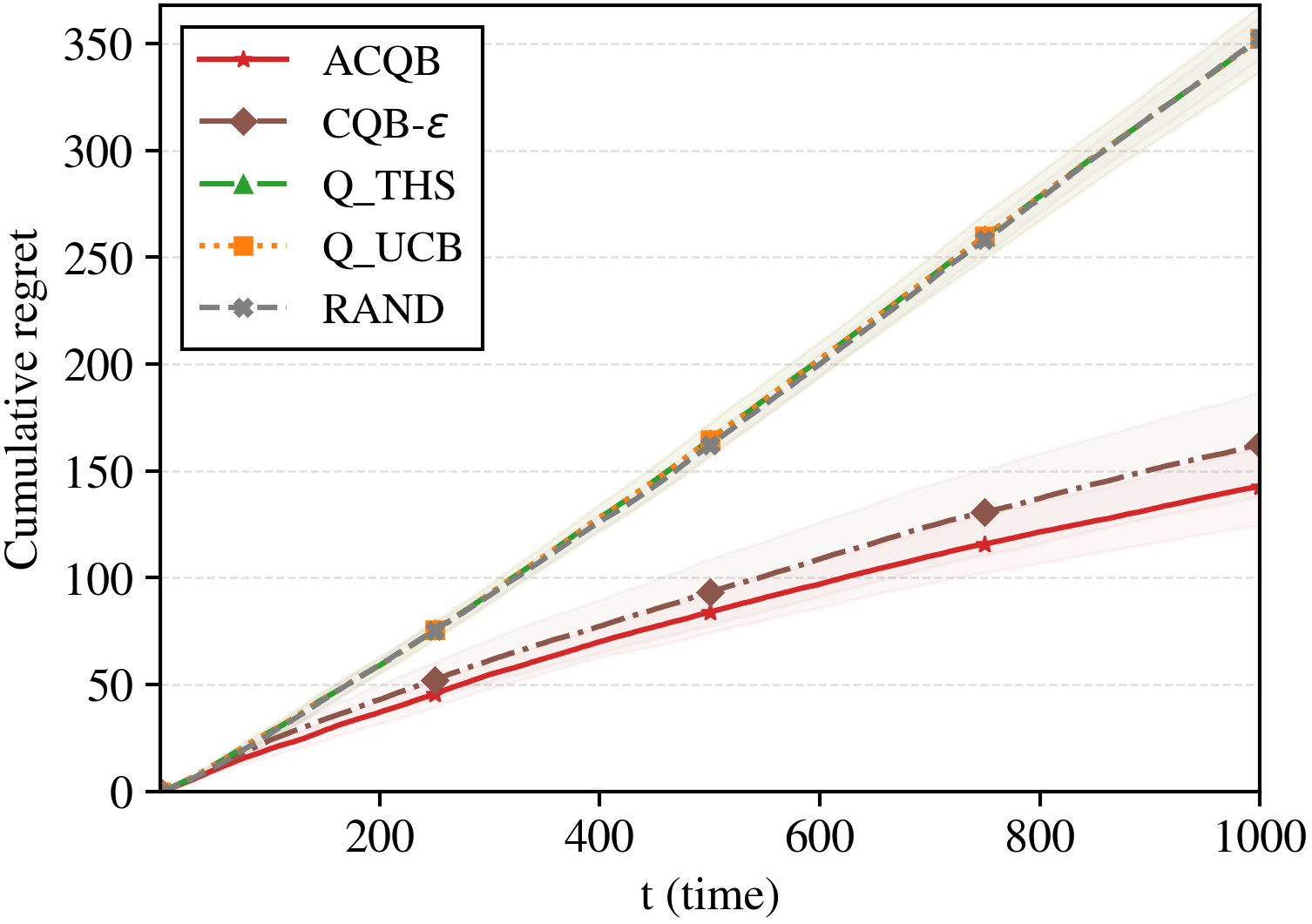}
        \end{subfigure}
    \end{minipage}%
    \hfill%
    \begin{minipage}[c]{\blockwidth}
        \centering
        \begin{subfigure}[c]{0.49\linewidth}
            \includegraphics[width=\linewidth]{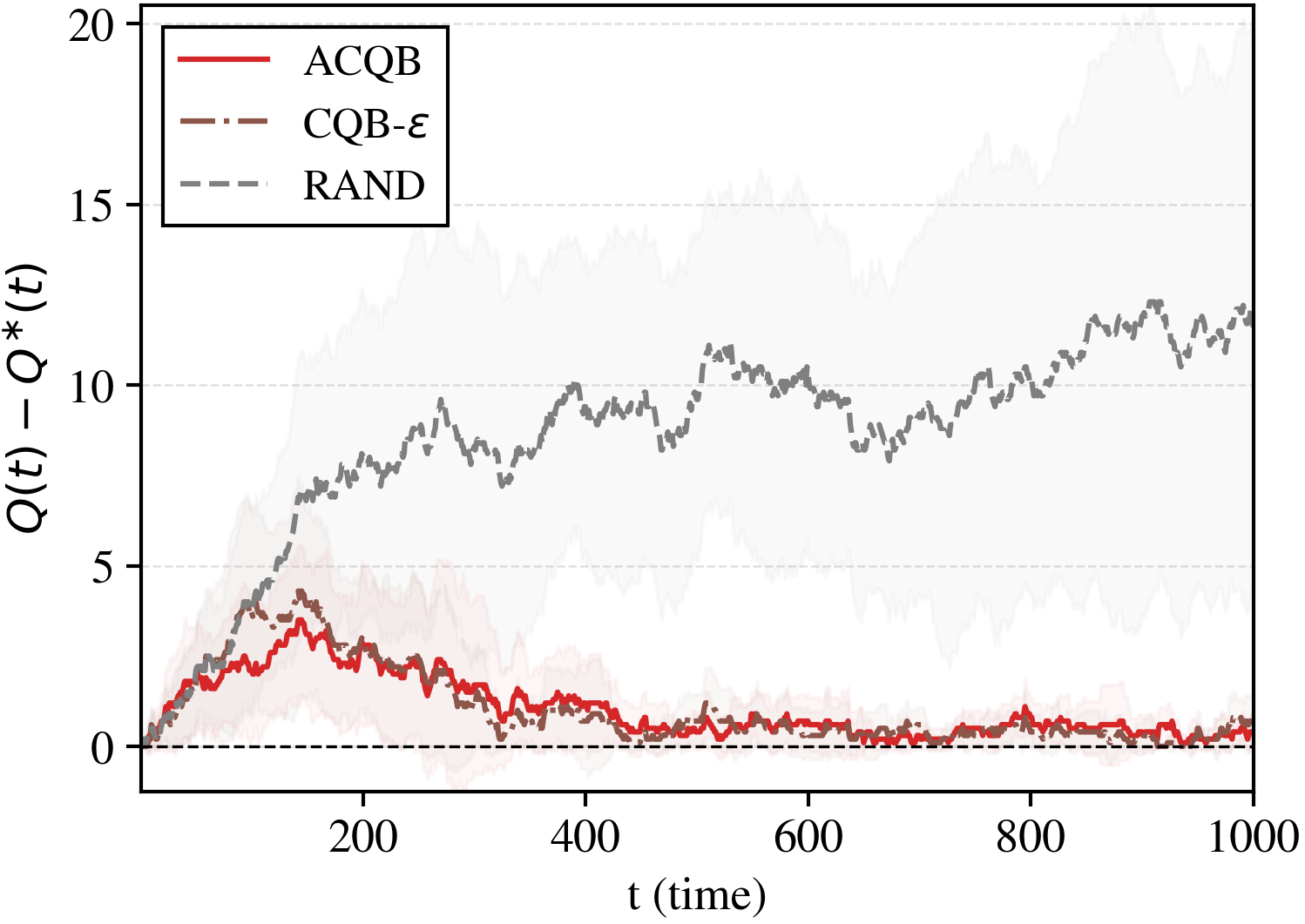}
        \end{subfigure}\hfill
        \begin{subfigure}[c]{0.49\linewidth}
            \includegraphics[width=\linewidth]{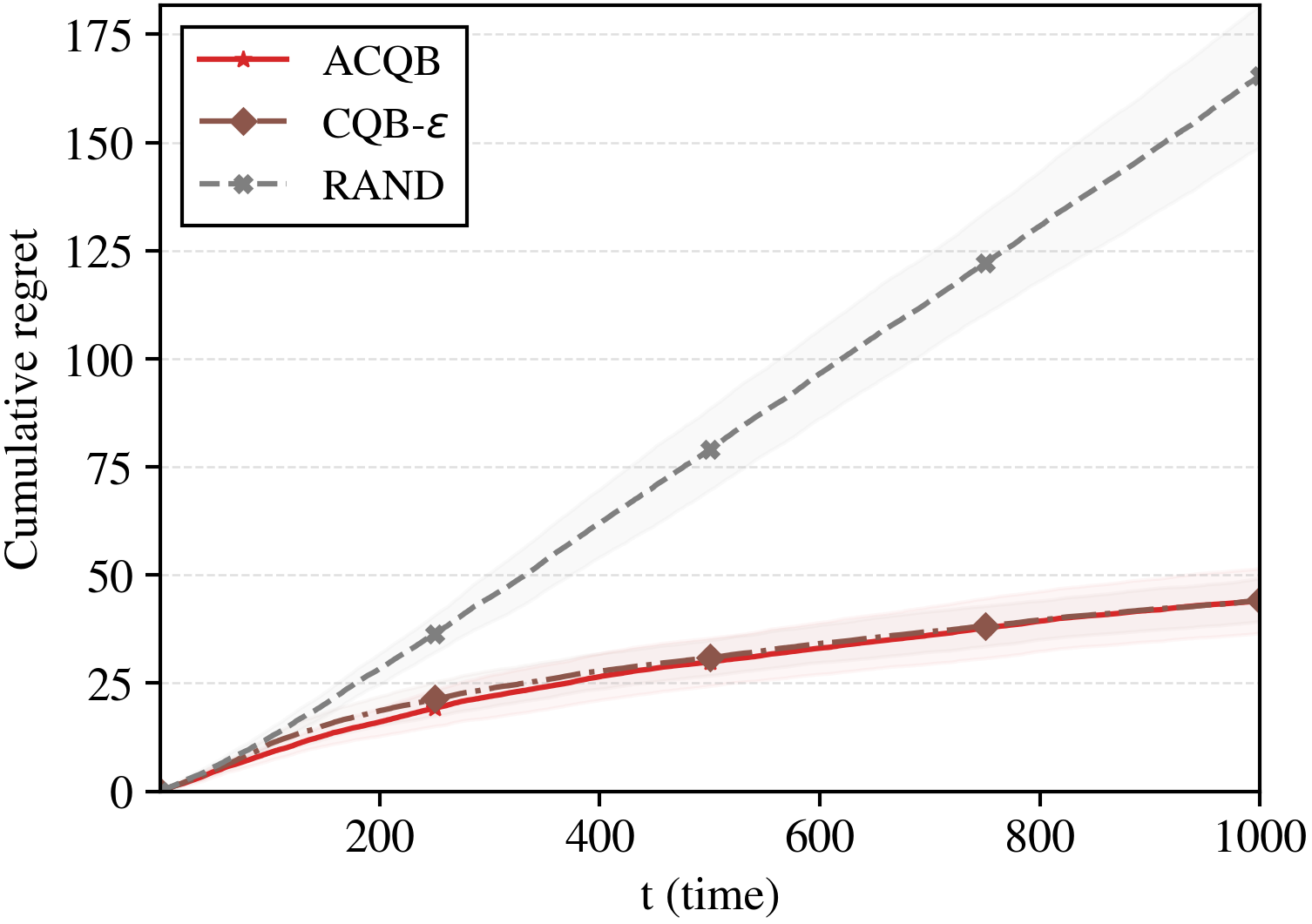}
        \end{subfigure}
    \end{minipage}%

    \vspace{-0.1cm}
    \caption{Queue length and cumulative regret on synthetic data with $\lambda=0.7$, $\epsilon=0.03$, and varying $N \in \{3,5,10\}$.}
    \label{fig:sim_results2}
\end{figure*}

\begin{figure*}[htbp]
    \centering
    
    % --- 설정: 너비 변수 수정 ---
    % 중앙 라벨이 사라졌으므로 blockwidth를 0.46 -> 0.48로 늘려 공간 활용
    \newcommand{\lbwidth}{0.02\textwidth}  % 세로 라벨 너비
    \newcommand{\blockwidth}{0.48\textwidth} % 그림 블록 너비 (확장됨)
    
    % ---------------- Header (K=1, K=2) ----------------
    % 1. Left Spacing (첫 번째 세로 라벨 위의 빈 공간)
    \begin{minipage}{\lbwidth} \centering \phantom{L} \end{minipage}%
    \hfill%
    % 2. K=1 Header (첫 번째 그림 블록 중앙)
    \begin{minipage}{\blockwidth} \centering \scriptsize $K=1$ \end{minipage}%
    \hfill%
    % 3. K=2 Header (두 번째 그림 블록 중앙) - 중간 라벨 공간 없음
    \begin{minipage}{\blockwidth} \centering \scriptsize $K=2$ \end{minipage}%
    
    % \vspace{0.1cm} % 헤더와 첫 줄 사이 간격

    % ================= ROW 1: lambda=0.55 =================
    % 1. Row Label
    \begin{minipage}[c]{\lbwidth}
        \centering \rotatebox{90}{\scriptsize $\eps=0.05$}
    \end{minipage}%
    \hfill%
    % 2. K=1 Block (Images)
    \begin{minipage}[c]{\blockwidth}
        \centering
        \begin{subfigure}[c]{0.49\linewidth}
            \includegraphics[width=\linewidth]{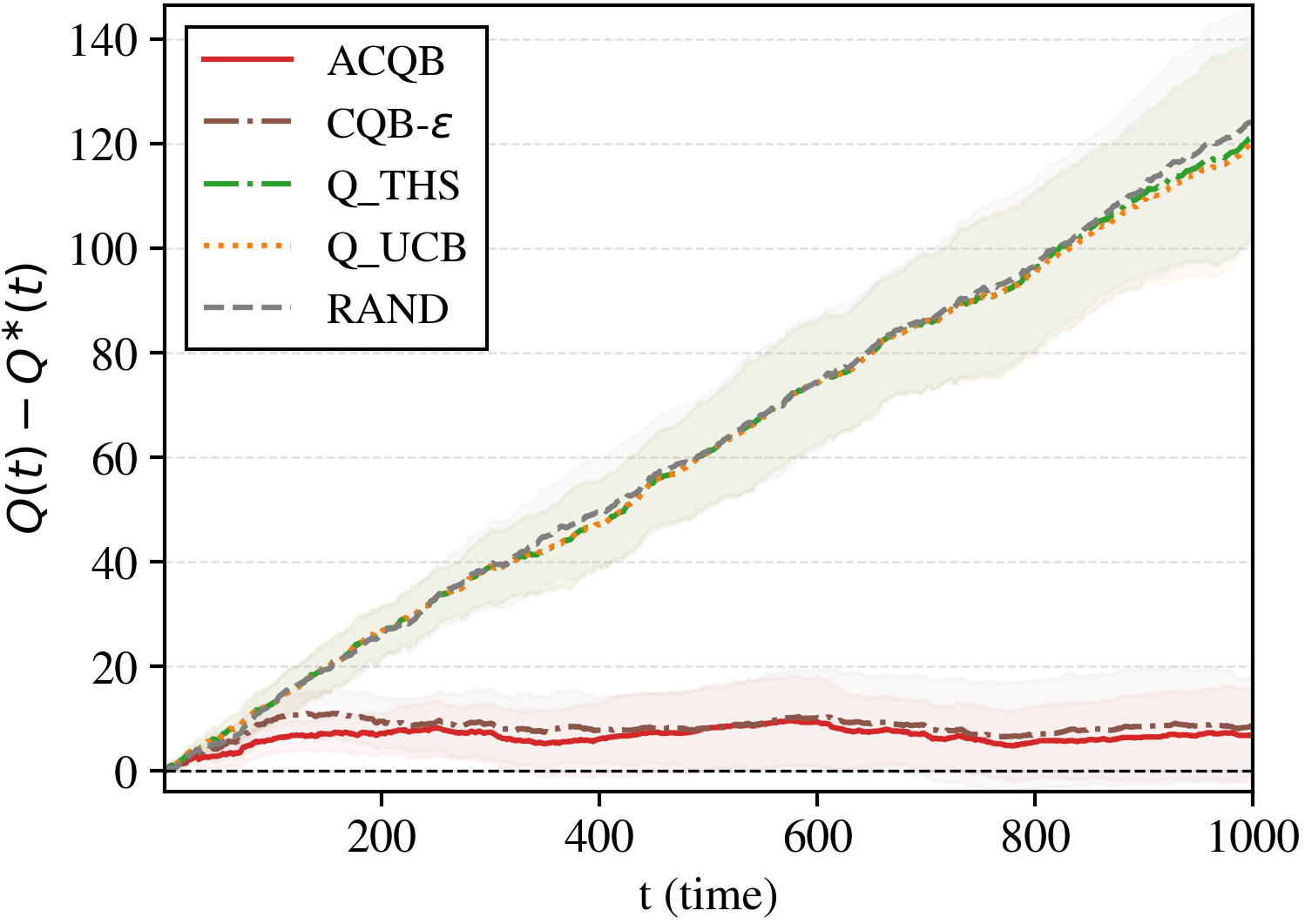}
        \end{subfigure}\hfill
        \begin{subfigure}[c]{0.49\linewidth}
            \includegraphics[width=\linewidth]{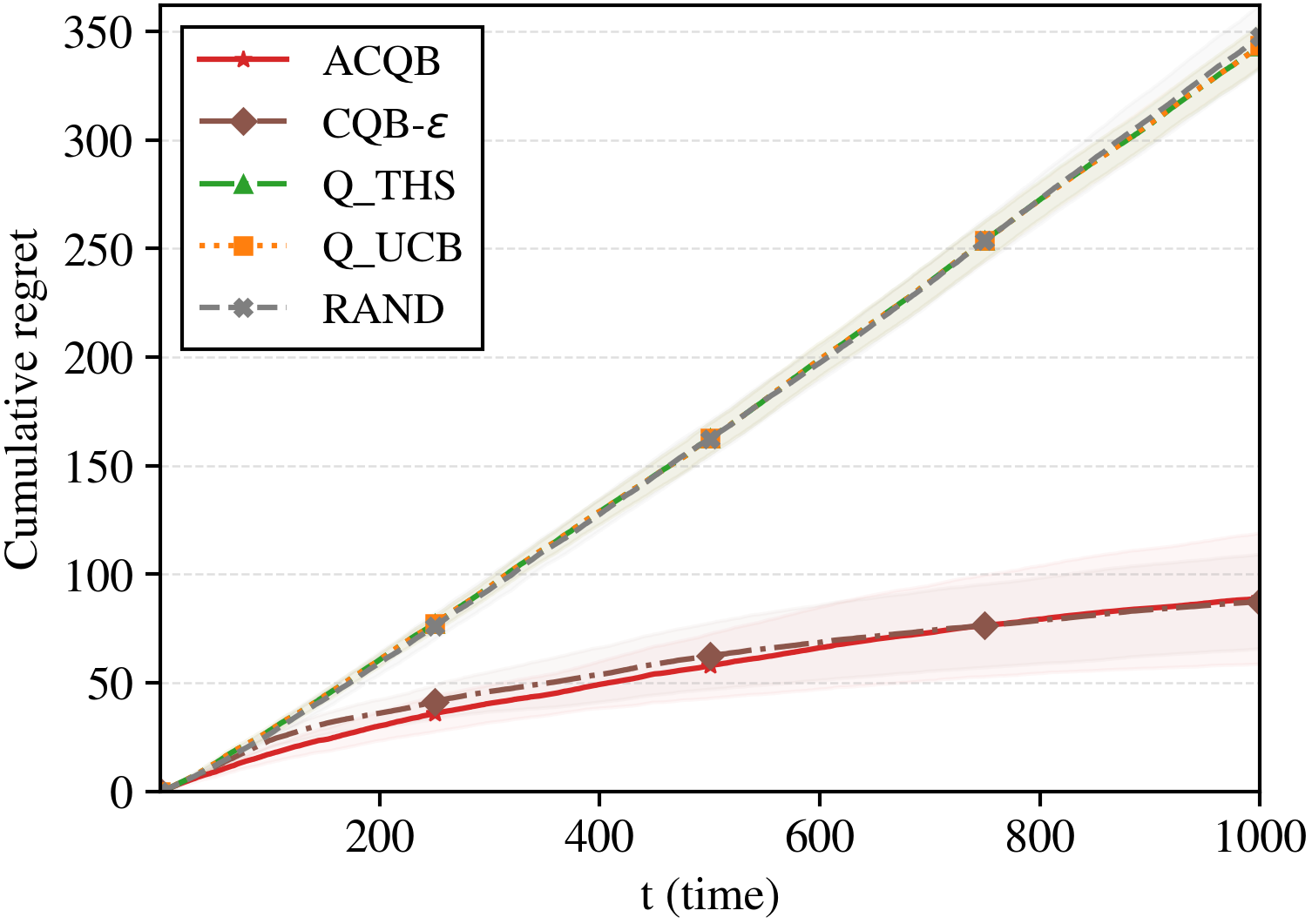}
        \end{subfigure}
    \end{minipage}%
    \hfill%
    % 3. K=2 Block (Images) - 중간 라벨 제거됨, 바로 이어짐
    \begin{minipage}[c]{\blockwidth}
        \centering
        \begin{subfigure}[c]{0.49\linewidth}
            \includegraphics[width=\linewidth]{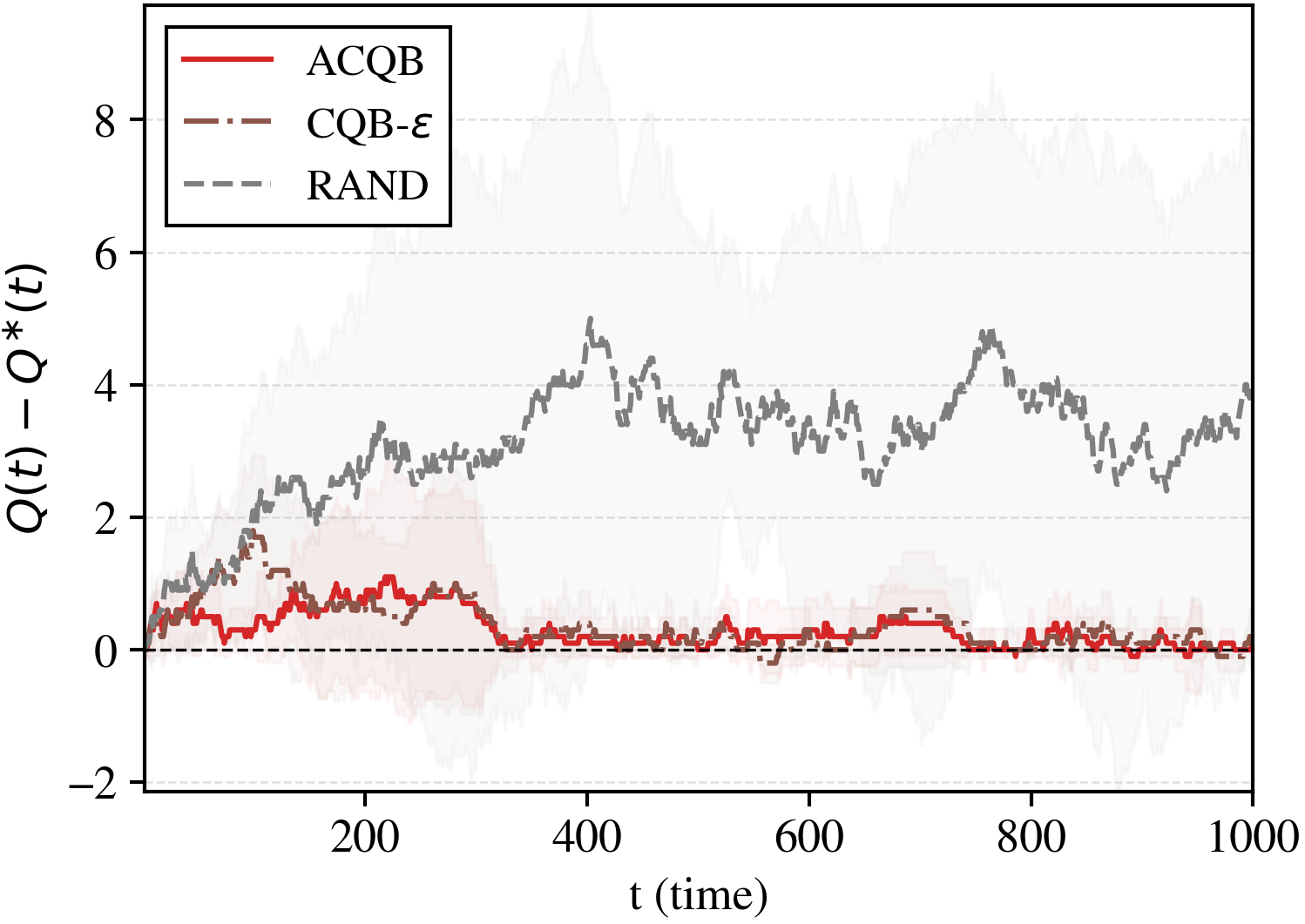}
        \end{subfigure}\hfill
        \begin{subfigure}[c]{0.49\linewidth}
            \includegraphics[width=\linewidth]{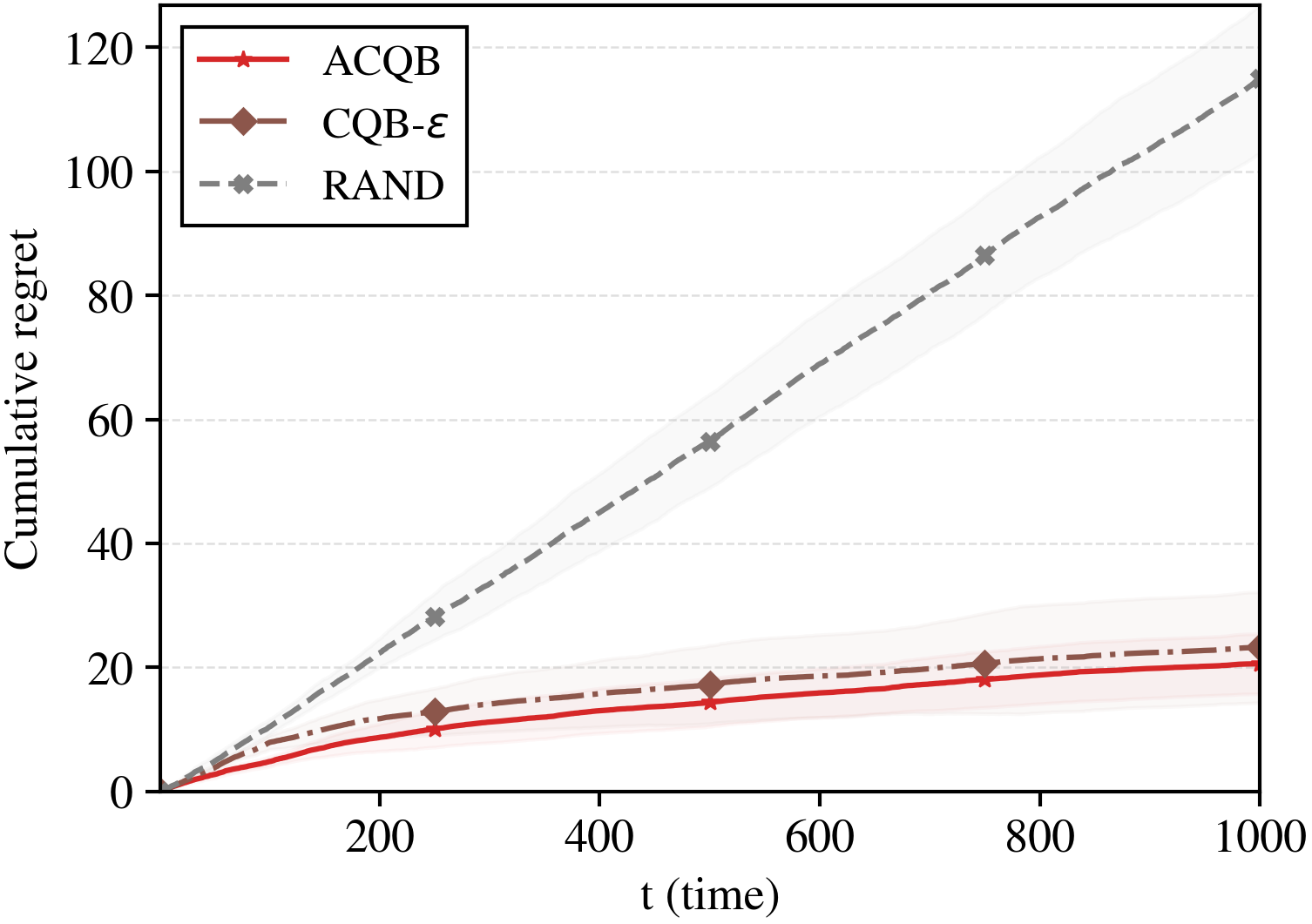}
        \end{subfigure}
    \end{minipage}%

    \vspace{0.15cm} % 행 간 간격

    % ================= ROW 2: lambda=0.65 =================
    \begin{minipage}[c]{\lbwidth}
        \centering \rotatebox{90}{\scriptsize $\eps=0.03$}
    \end{minipage}%
    \hfill%
    \begin{minipage}[c]{\blockwidth}
        \centering
        \begin{subfigure}[c]{0.49\linewidth}
            \includegraphics[width=\linewidth]{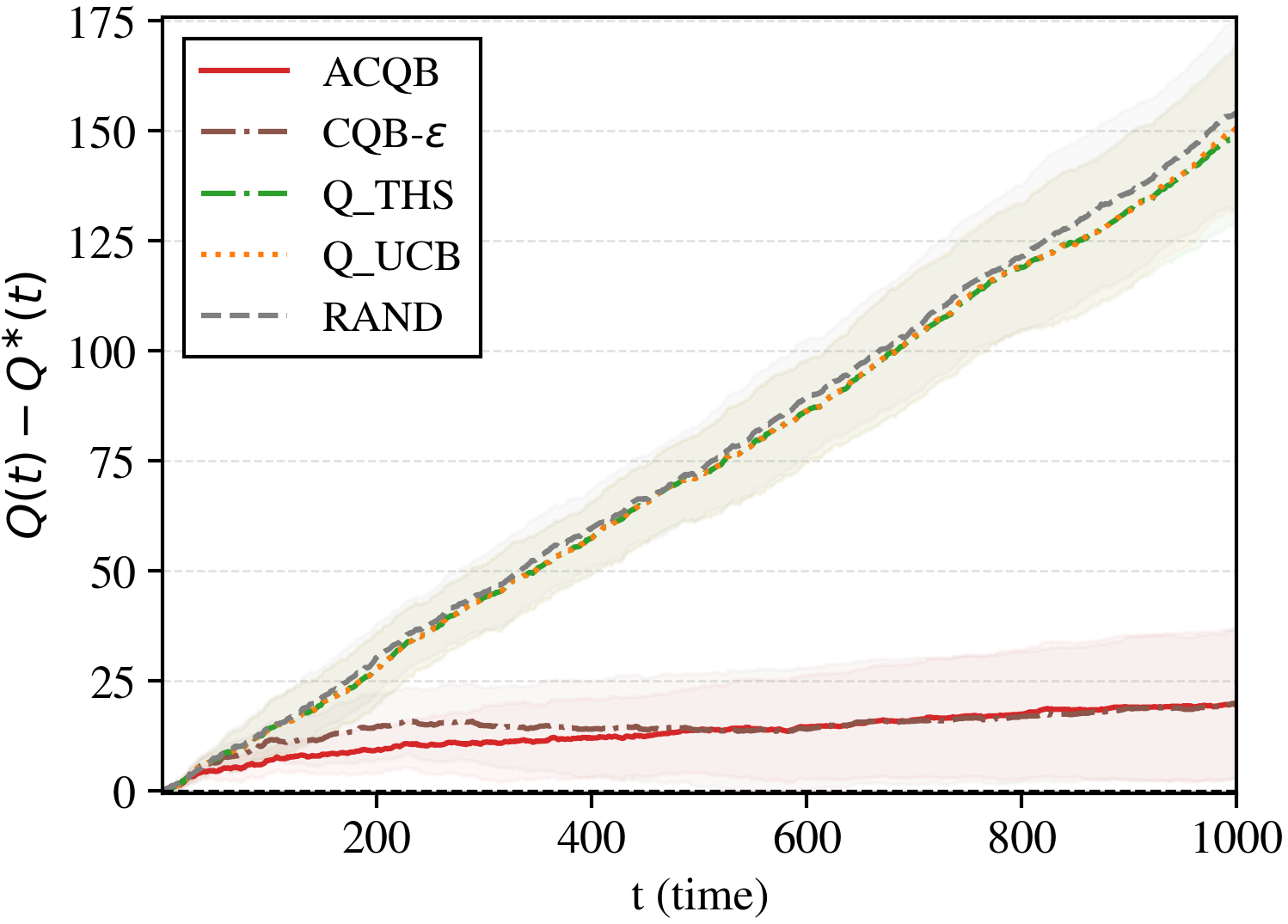}
        \end{subfigure}\hfill
        \begin{subfigure}[c]{0.49\linewidth}
            \includegraphics[width=\linewidth]{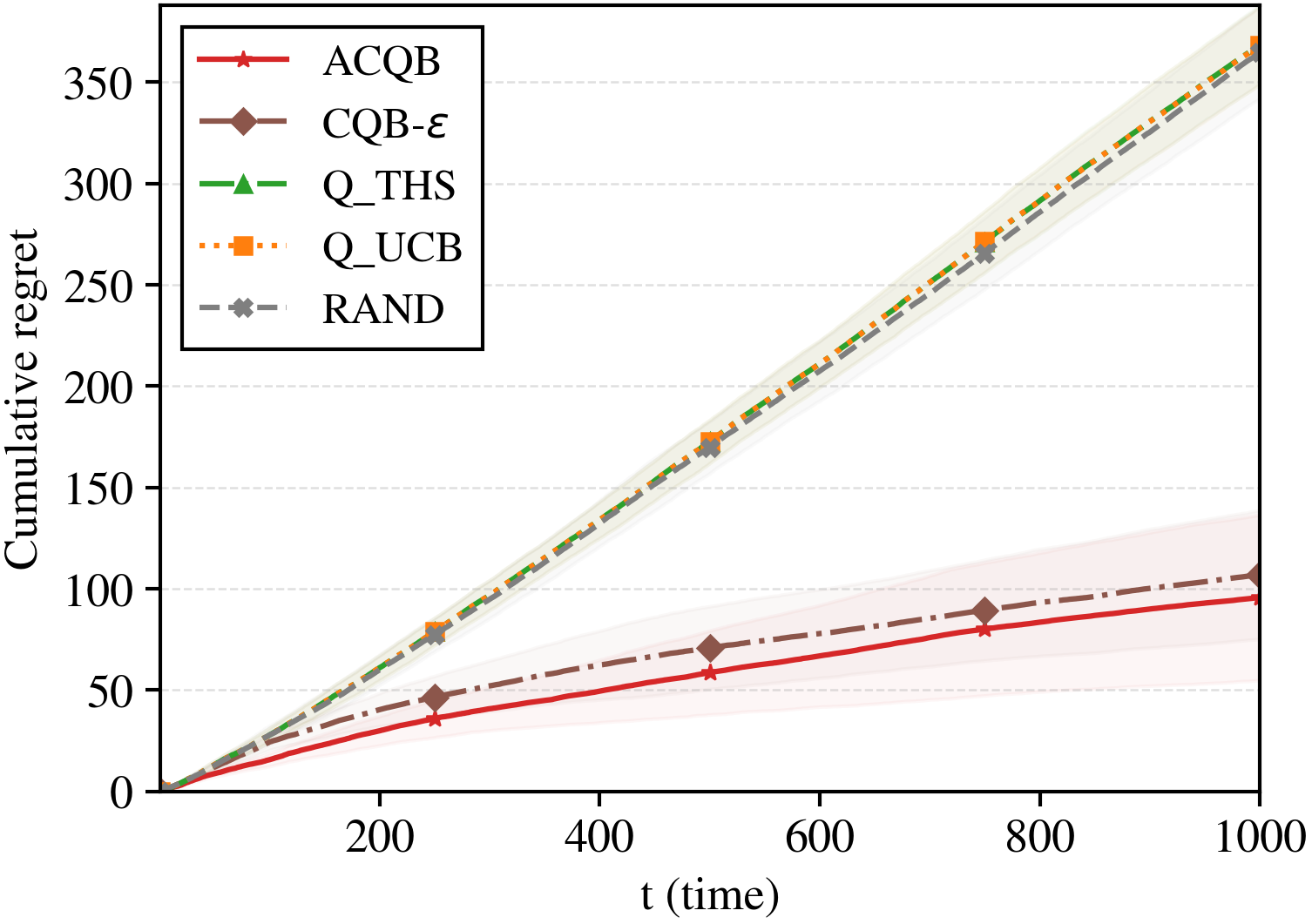}
        \end{subfigure}
    \end{minipage}%
    \hfill%
    \begin{minipage}[c]{\blockwidth}
        \centering
        \begin{subfigure}[c]{0.49\linewidth}
            \includegraphics[width=\linewidth]{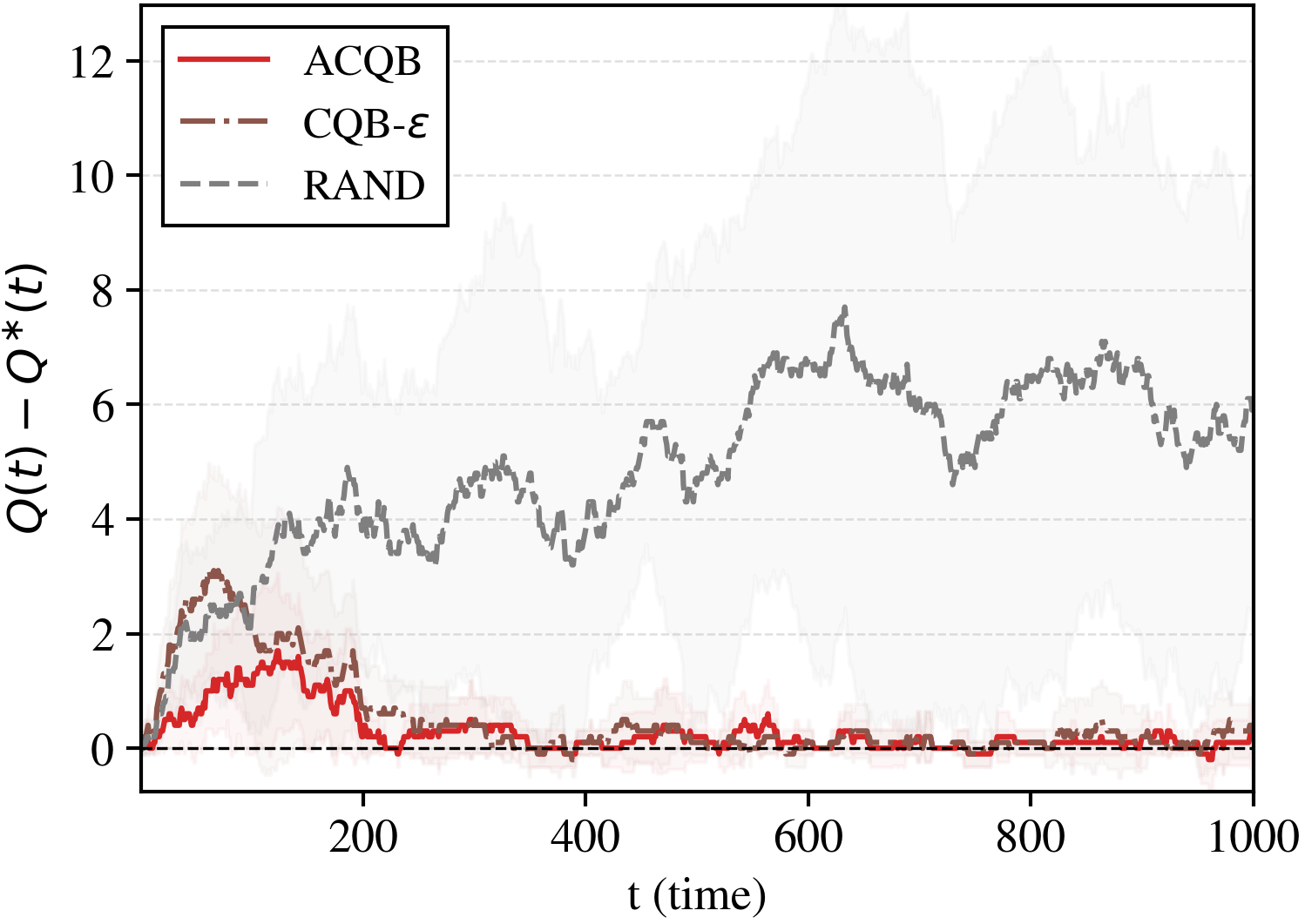}
        \end{subfigure}\hfill
        \begin{subfigure}[c]{0.49\linewidth}
            \includegraphics[width=\linewidth]{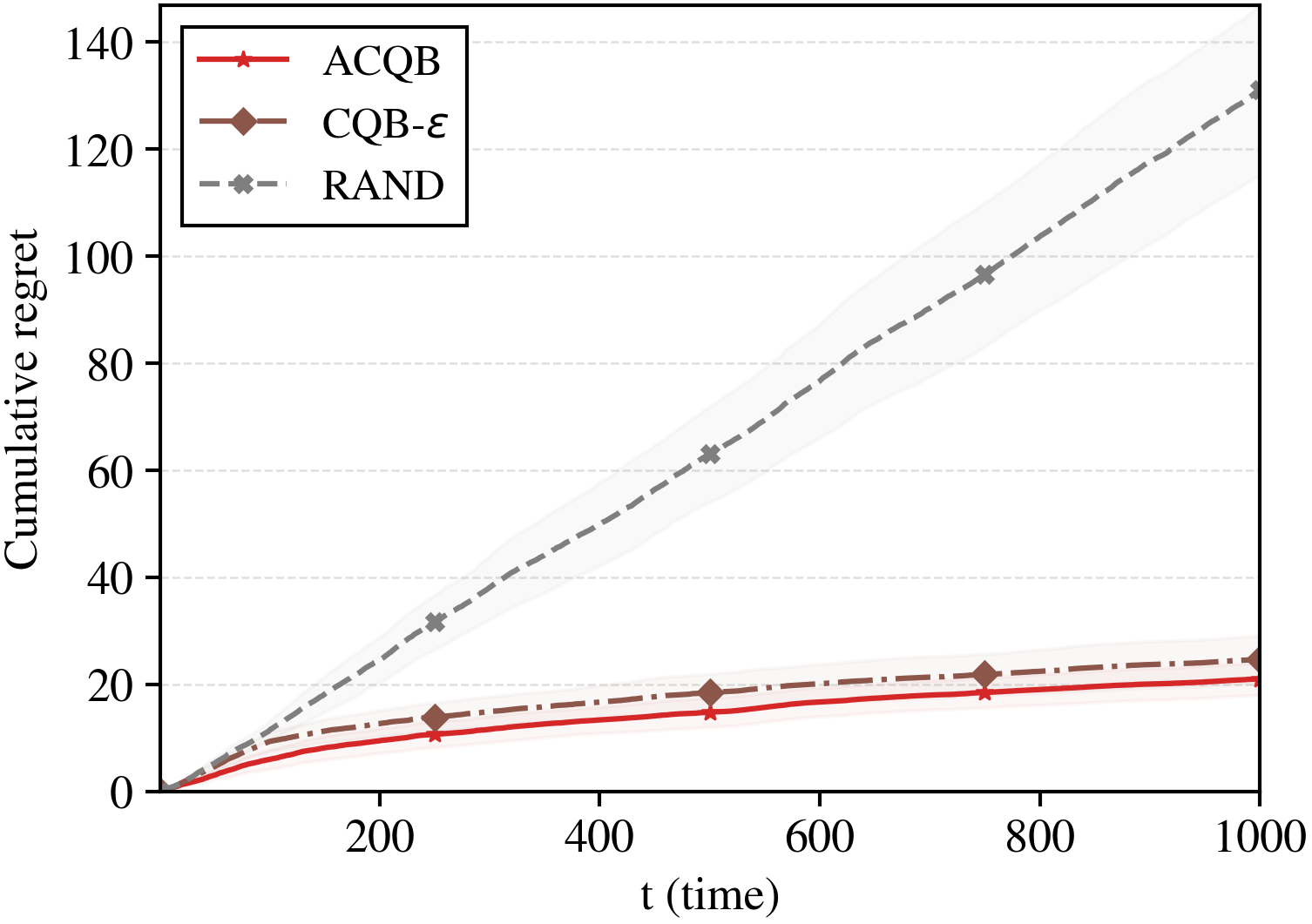}
        \end{subfigure}
    \end{minipage}%

    \vspace{0.15cm} % 행 간 간격

    % ================= ROW 3: lambda=0.75 =================
    \begin{minipage}[c]{\lbwidth}
        \centering \rotatebox{90}{\scriptsize $\eps=0.01$}
    \end{minipage}%
    \hfill%
    \begin{minipage}[c]{\blockwidth}
        \centering
        \begin{subfigure}[c]{0.49\linewidth}
            \includegraphics[width=\linewidth]{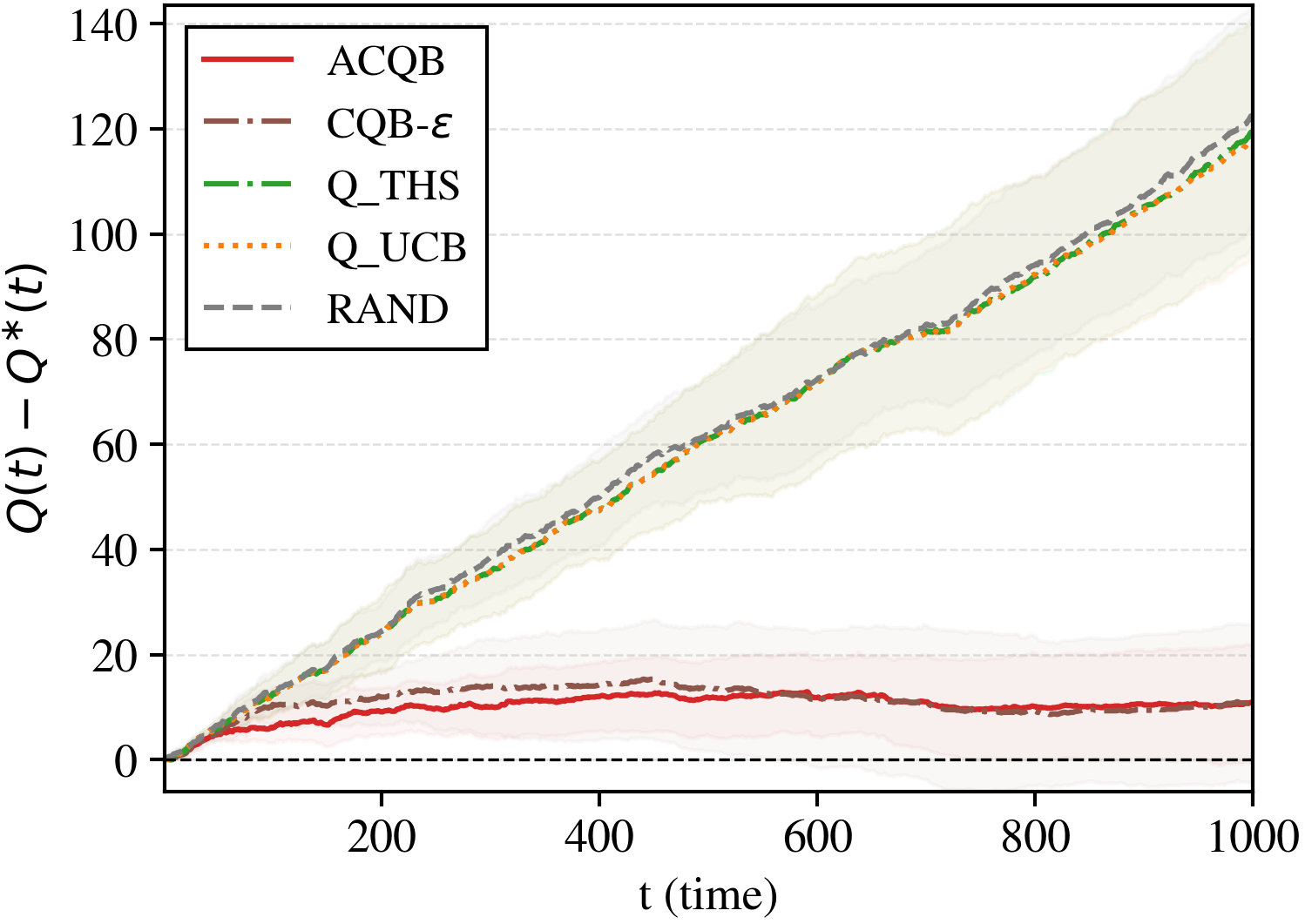}
        \end{subfigure}\hfill
        \begin{subfigure}[c]{0.49\linewidth}
            \includegraphics[width=\linewidth]{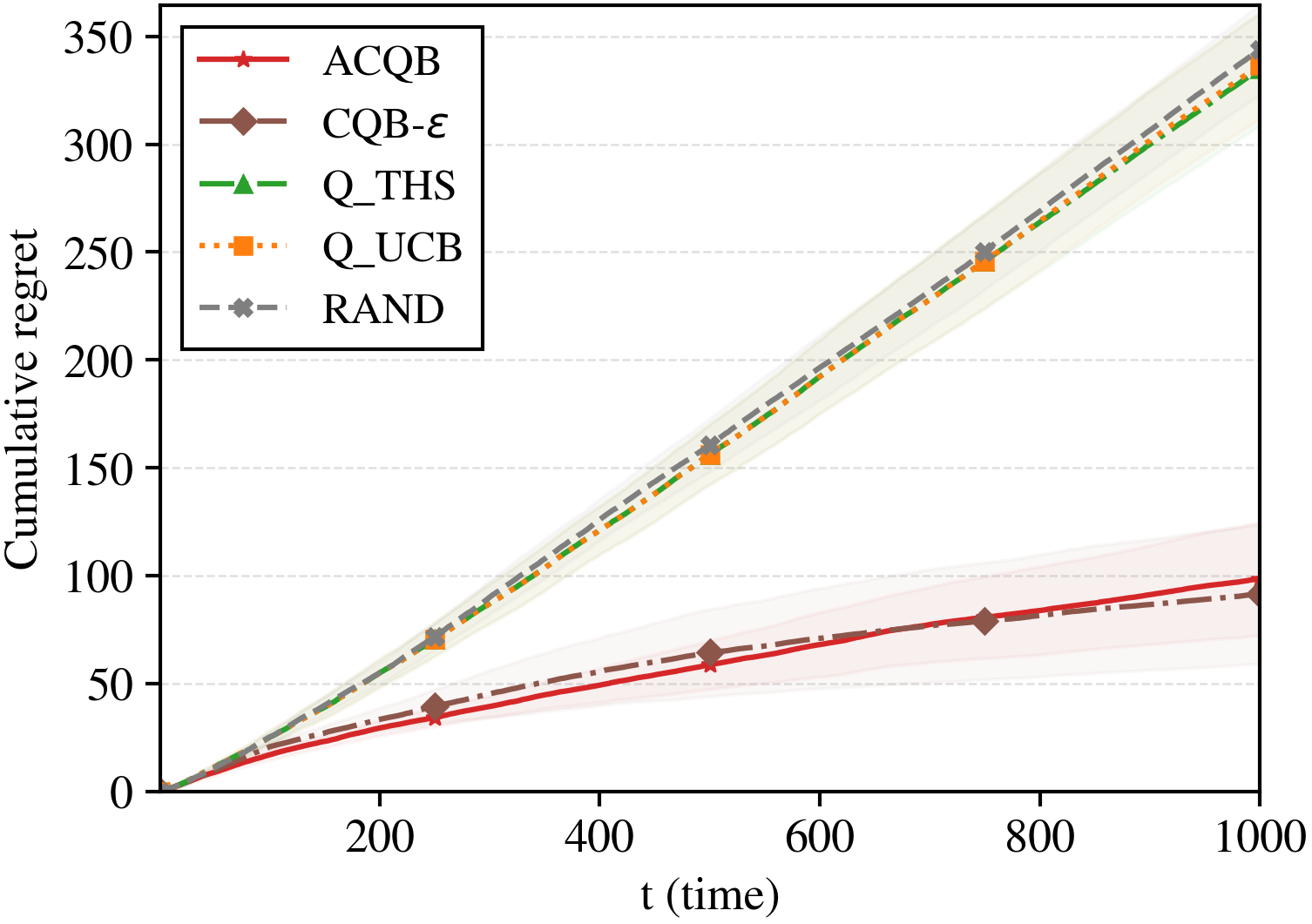}
        \end{subfigure}
    \end{minipage}%
    \hfill%
    \begin{minipage}[c]{\blockwidth}
        \centering
        \begin{subfigure}[c]{0.49\linewidth}
            \includegraphics[width=\linewidth]{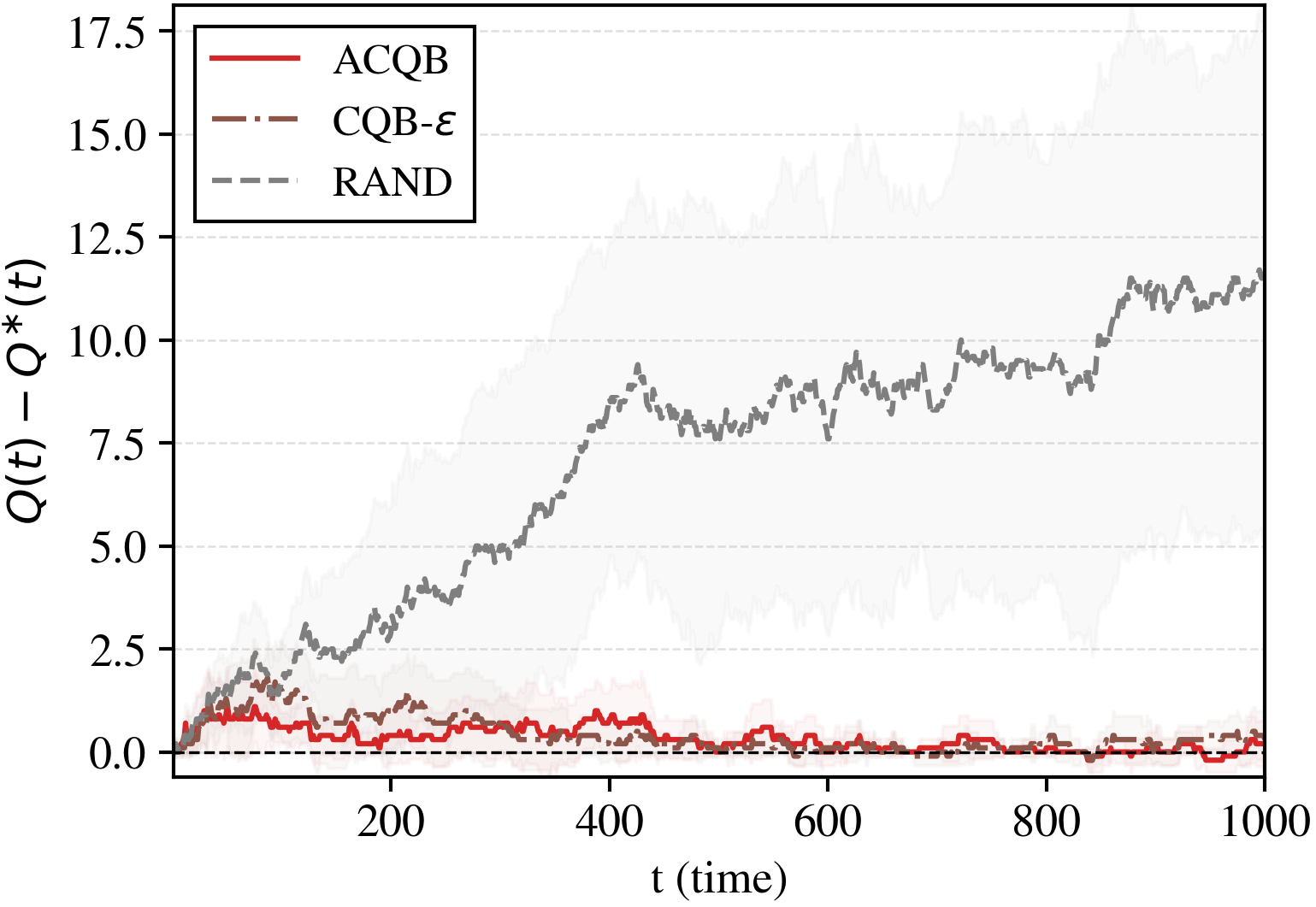}
        \end{subfigure}\hfill
        \begin{subfigure}[c]{0.49\linewidth}
            \includegraphics[width=\linewidth]{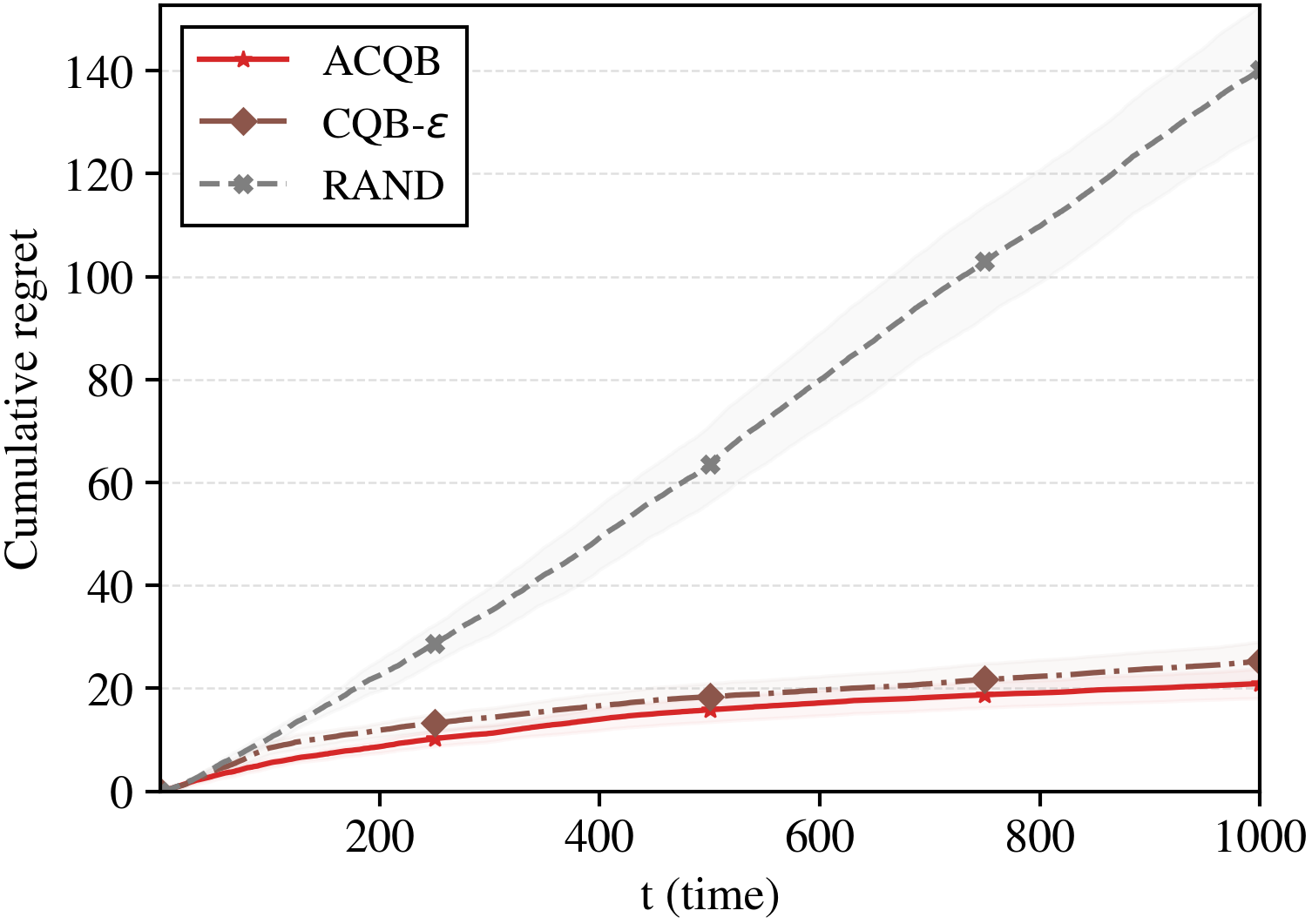}
        \end{subfigure}
    \end{minipage}%

    \vspace{-0.1cm}
    \caption{Queue length and cumulative regret on synthetic data with $\lambda=0.7$, $N=5$, and varying $\eps\in\{0.05,0.03,0.01\}$.}
    \label{fig:sim_result3}
\end{figure*}
\FloatBarrier

\subsection{Additional exploration scheduling results}
\label{sec:exp:exploration-variants}
We compare ACQB with heuristic exploration variants that trigger exploration independently of arrivals. Once exploration is triggered, the exploration query is selected by FIFO, round-robin over the least-served jobs, or random scheduling, while the routing rule is kept fixed. The results in \Cref{fig:sim_exp_n5} show that arrival-based exploration remains stable across both $K=1$ and $K=2$.

\begin{figure*}[htbp]
    \centering
    
    \newcommand{\lbwidth}{0.02\textwidth}
    \newcommand{\blockwidth}{0.48\textwidth}
    
    % ---------------- Header (K=1, K=2) ----------------
    \begin{minipage}{\lbwidth}
        \centering \phantom{L}
    \end{minipage}%
    \hfill%
    \begin{minipage}{\blockwidth}
        \centering \scriptsize $K=1$
    \end{minipage}%
    \hfill%
    \begin{minipage}{\blockwidth}
        \centering \scriptsize $K=2$
    \end{minipage}%
    
    \vspace{0.1cm}

    % ================= ROW: N=5 =================
    \begin{minipage}[c]{\lbwidth}
        \centering \rotatebox{90}{\scriptsize $N=5$}
    \end{minipage}%
    \hfill%
    \begin{minipage}[c]{\blockwidth}
        \centering
        \begin{subfigure}[c]{0.49\linewidth}
            \includegraphics[width=\linewidth]{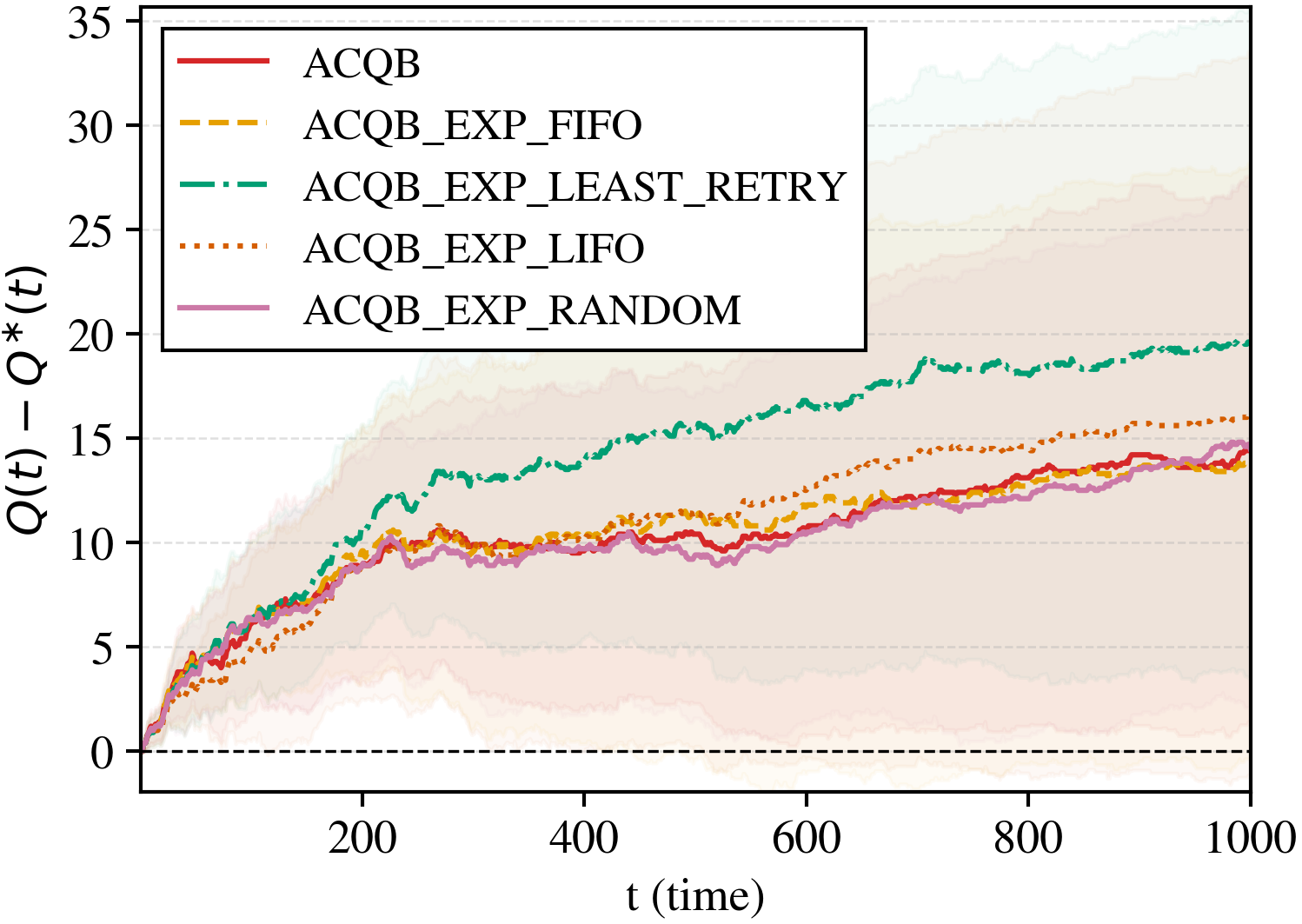}
        \end{subfigure}\hfill
        \begin{subfigure}[c]{0.49\linewidth}
            \includegraphics[width=\linewidth]{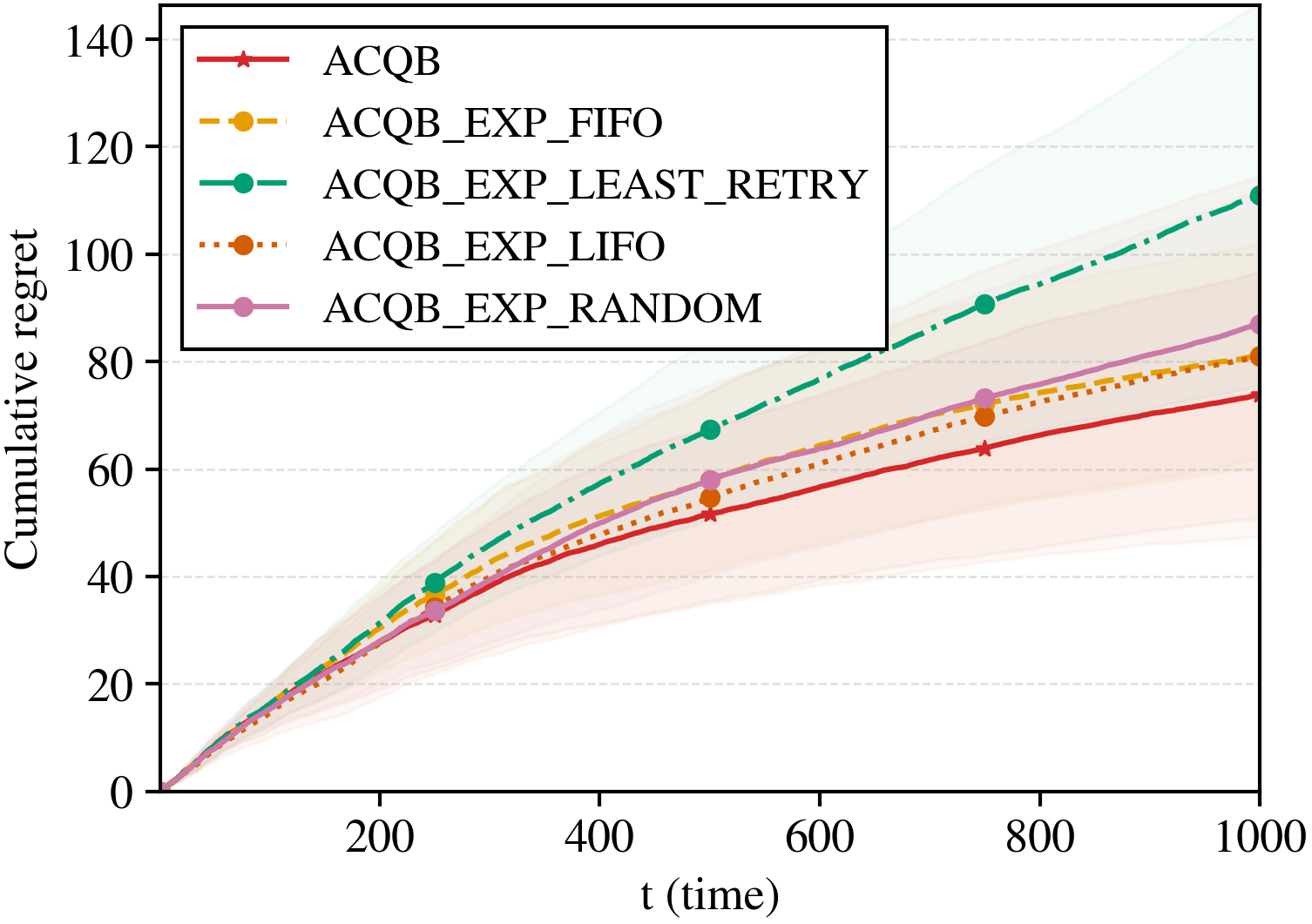}
        \end{subfigure}
    \end{minipage}%
    \hfill%
    \begin{minipage}[c]{\blockwidth}
        \centering
        \begin{subfigure}[c]{0.49\linewidth}
            \includegraphics[width=\linewidth]{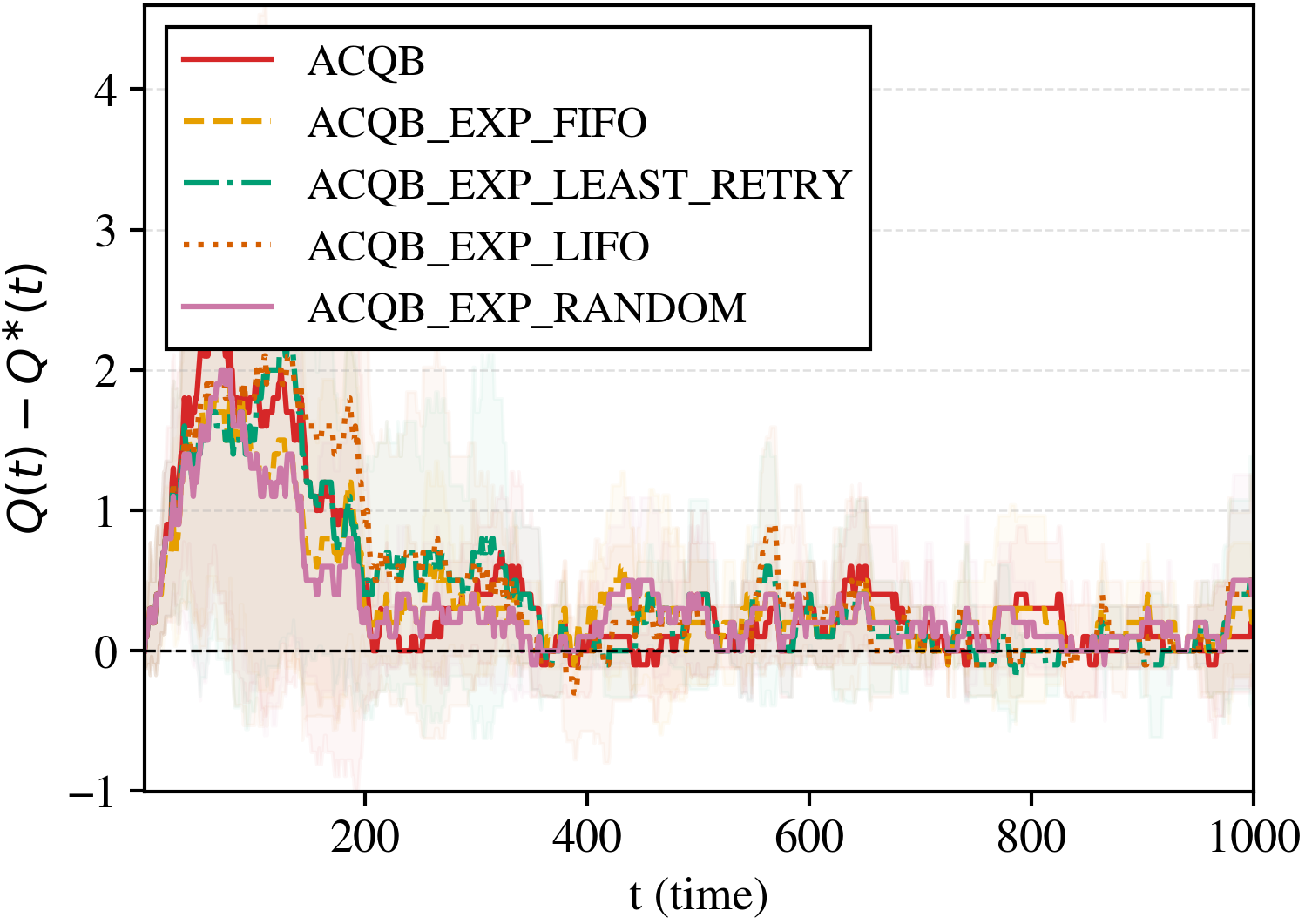}
        \end{subfigure}\hfill
        \begin{subfigure}[c]{0.49\linewidth}
            \includegraphics[width=\linewidth]{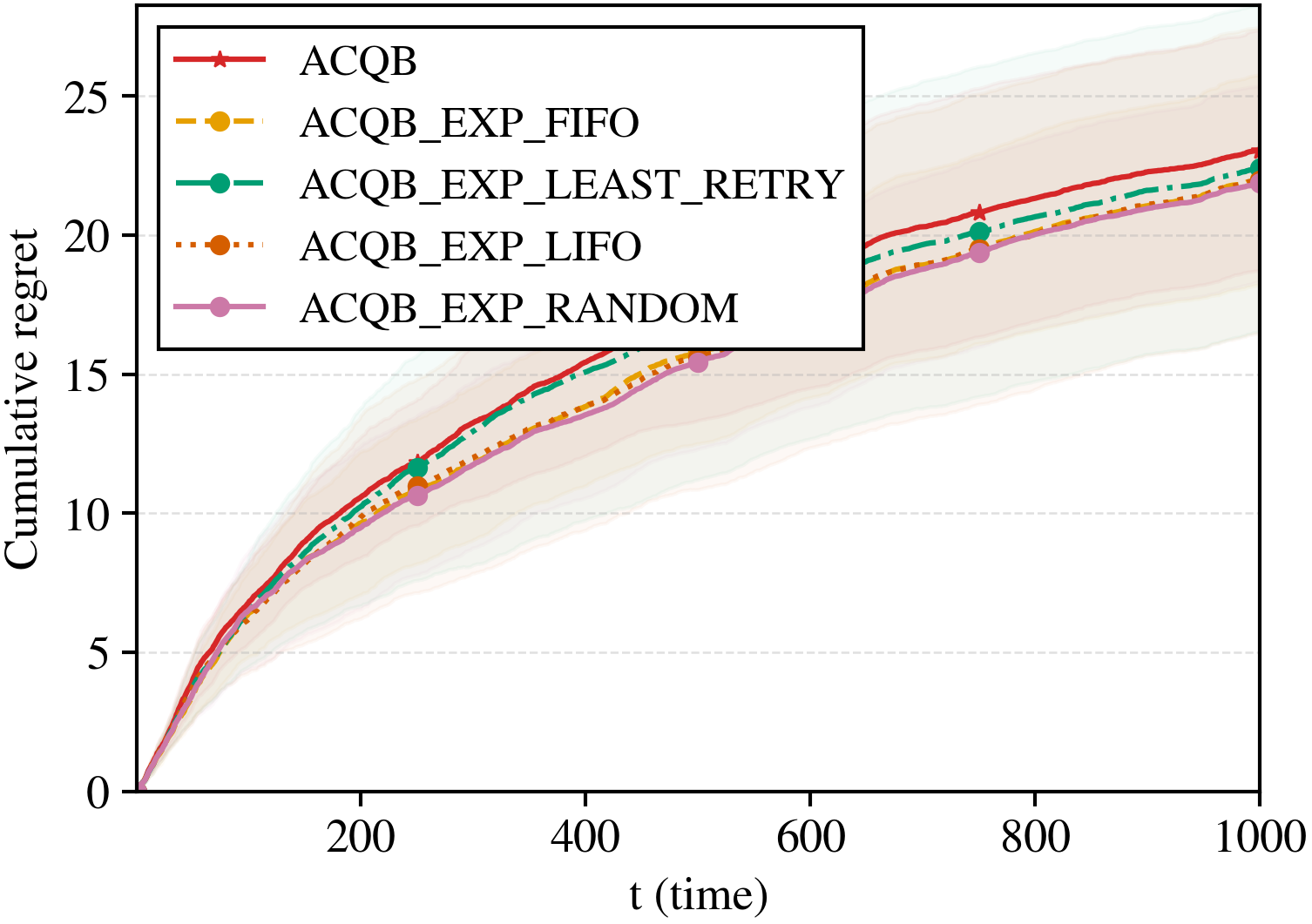}
        \end{subfigure}
    \end{minipage}%

    \vspace{-0.1cm}
    \caption{Queue length and cumulative regret for exploration-scheduling variants on synthetic data with $N=5$, $\eps=0.03$, and $K\in\{1,2\}$.}
    \label{fig:sim_exp_n5}
\end{figure*}
\FloatBarrier

\subsection{Additional offline routing dataset results}

\label{ssec:exp:real-result}
We provide additional results on \textsc{EmbedLLM} and \textsc{RouterBench}. The figures report queue length and cumulative regret across arrival rates, and the tables report the queue length gap and throughput at intermediate and final time steps.

\begin{figure*}[htbp]
    \centering
    
    % --- 설정: 너비 변수 ---
    \newcommand{\lbwidth}{0.02\textwidth} % 세로 라벨 너비
    \newcommand{\blockwidth}{0.46\textwidth} % 그림 블록 너비 (이미지 2개 포함)
    
    % ---------------- Header (K=1, K=2) ----------------
    % 1. Left Spacing
    \begin{minipage}{\lbwidth} \centering \phantom{L} \end{minipage}%
    \hfill%
    % 2. K=1 Header
    \begin{minipage}{\blockwidth} \centering \scriptsize $K=1$ \end{minipage}%
    \hfill%
    % 3. Middle Spacing
    \begin{minipage}{\lbwidth} \centering \phantom{L} \end{minipage}%
    \hfill%
    % 4. K=2 Header
    \begin{minipage}{\blockwidth} \centering \scriptsize $K=2$ \end{minipage}%
    
    \vspace{0.1cm} % 헤더와 첫 줄 사이 간격

    % ================= ROW 1: lambda=0.55 =================
    \begin{minipage}[c]{\lbwidth}
        \centering \rotatebox{90}{\scriptsize $\lambda=0.3$}
    \end{minipage}%
    \hfill%
    \begin{minipage}[c]{\blockwidth}
        \centering
        \begin{subfigure}[c]{0.49\linewidth}
            \includegraphics[width=\linewidth]{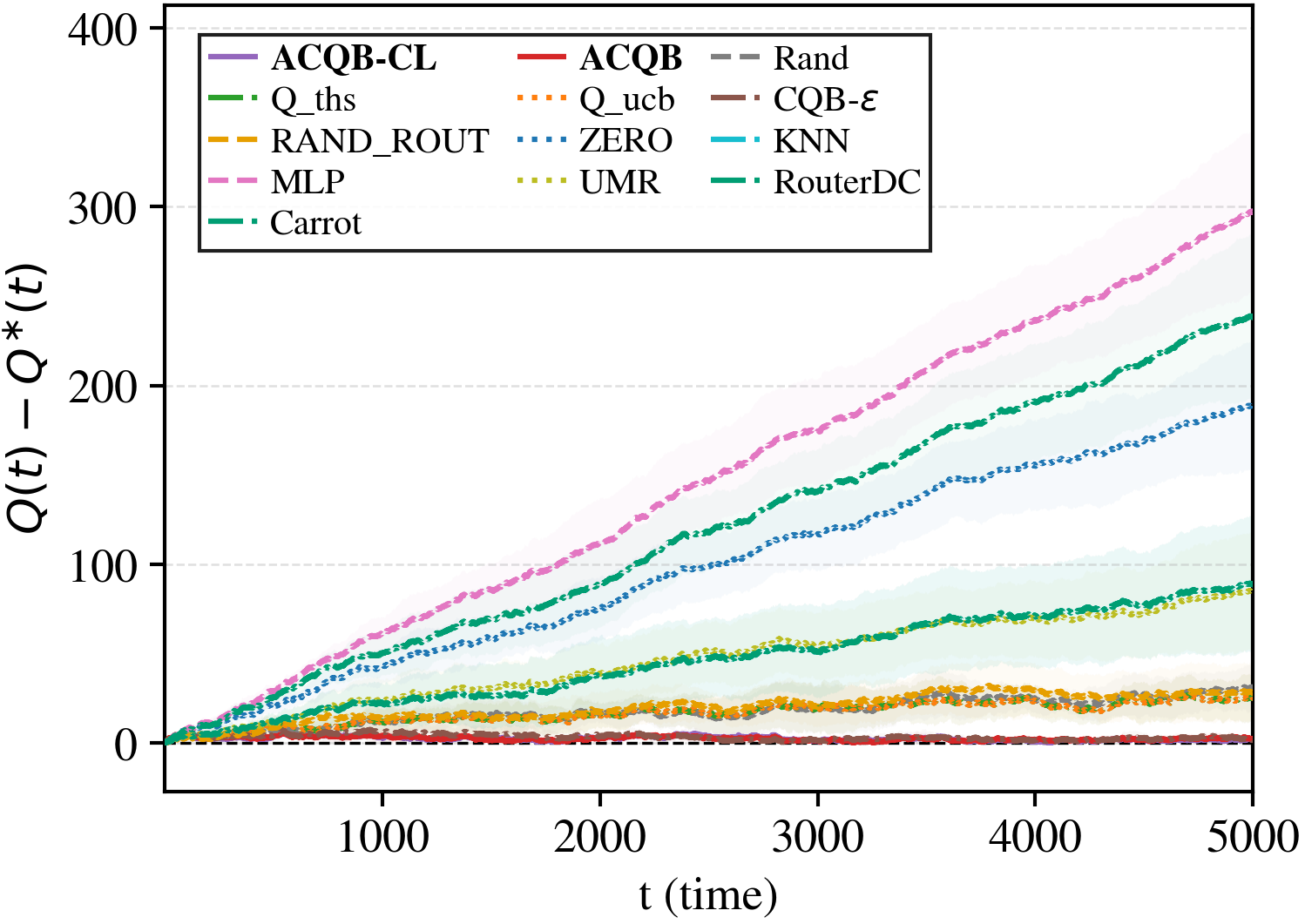}
        \end{subfigure}\hfill
        \begin{subfigure}[c]{0.49\linewidth}
            \includegraphics[width=\linewidth]{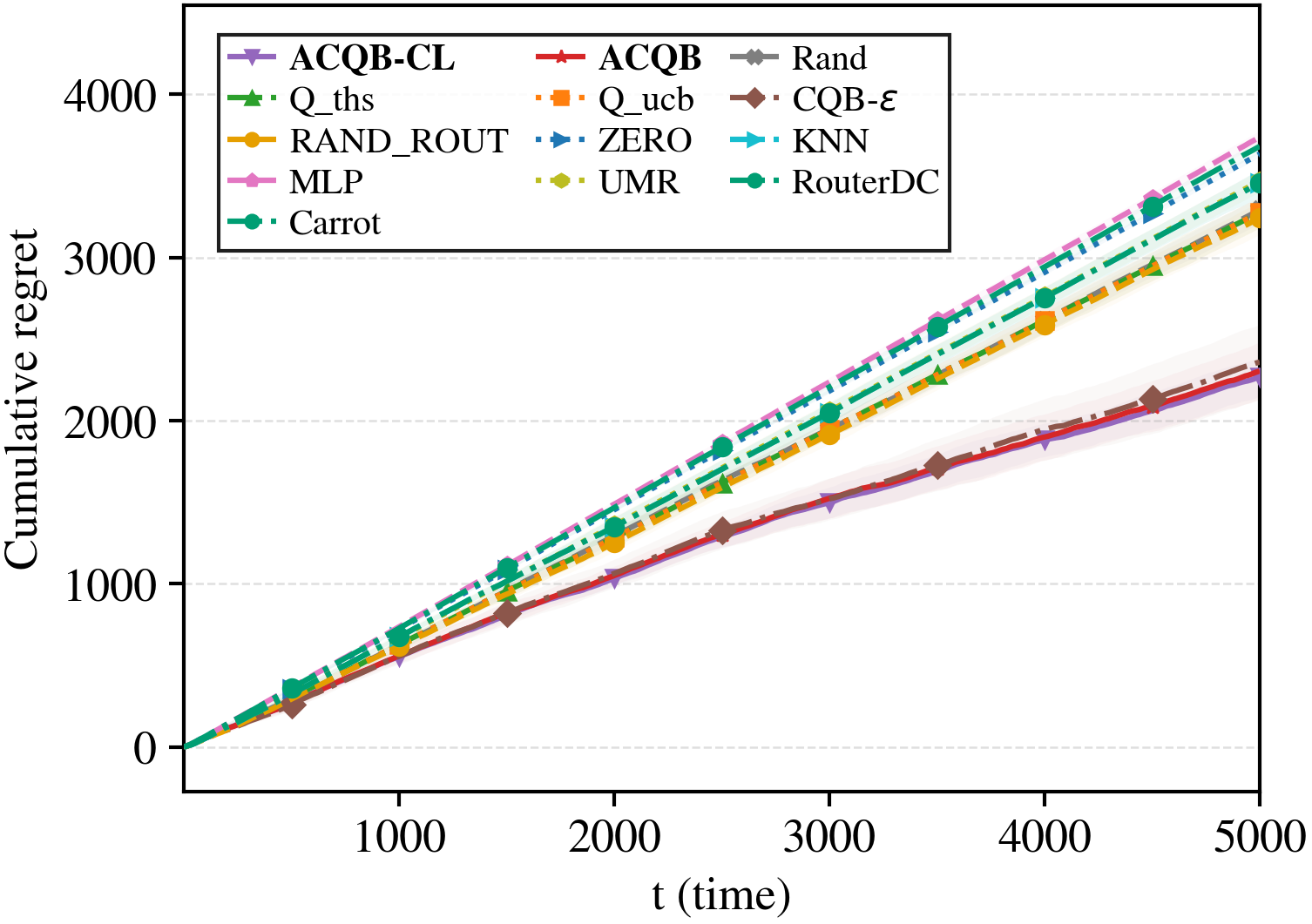}
        \end{subfigure}
    \end{minipage}%
    \hfill%
    \begin{minipage}[c]{\lbwidth} % 중앙 라벨 (필요 없다면 삭제 후 \blockwidth 조정)
        \centering \rotatebox{90}{\scriptsize $\lambda=0.45$}
    \end{minipage}%
    \hfill%
    \begin{minipage}[c]{\blockwidth}
        \centering
        \begin{subfigure}[c]{0.49\linewidth}
            \includegraphics[width=\linewidth]{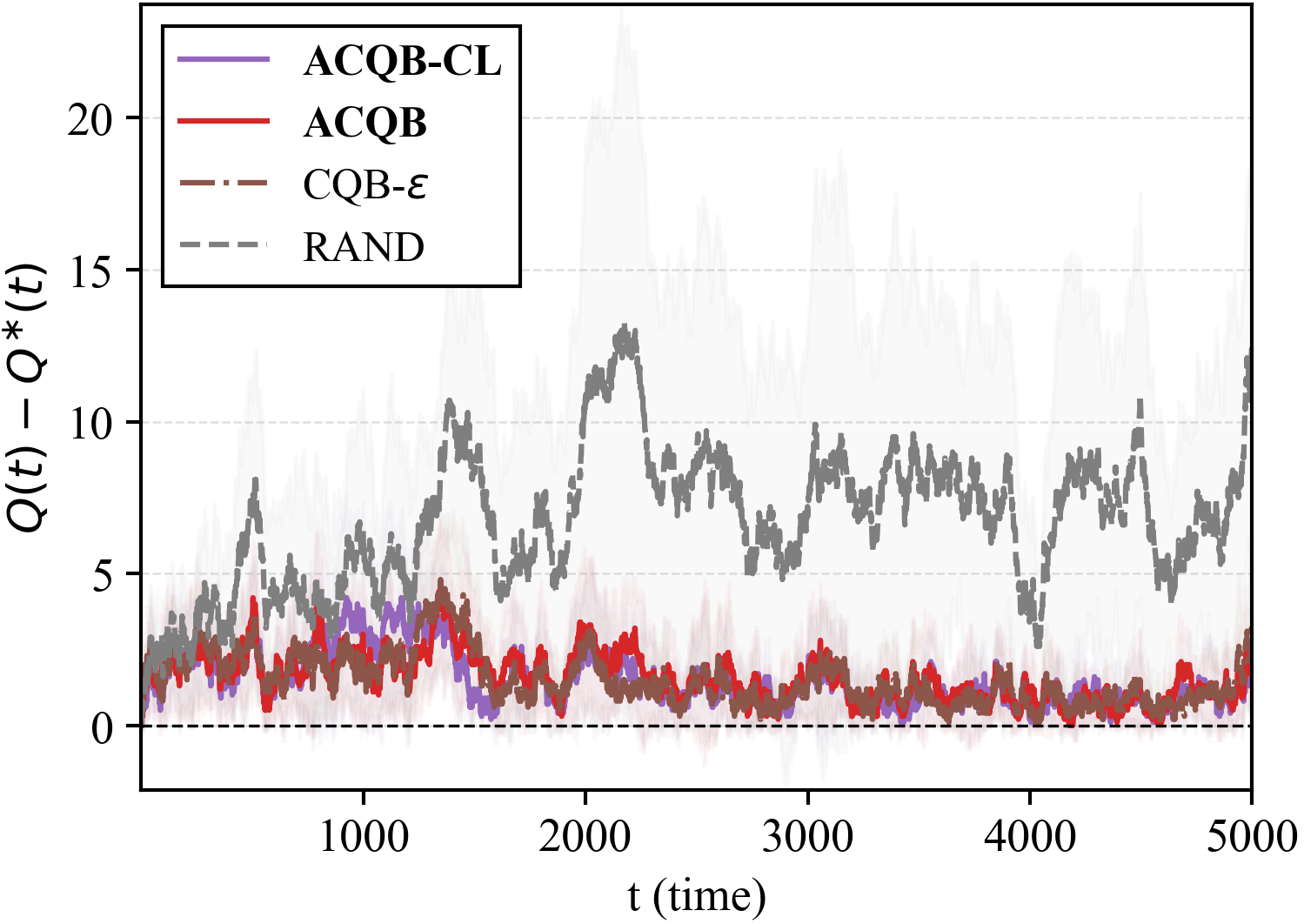}
        \end{subfigure}\hfill
        \begin{subfigure}[c]{0.49\linewidth}
            \includegraphics[width=\linewidth]{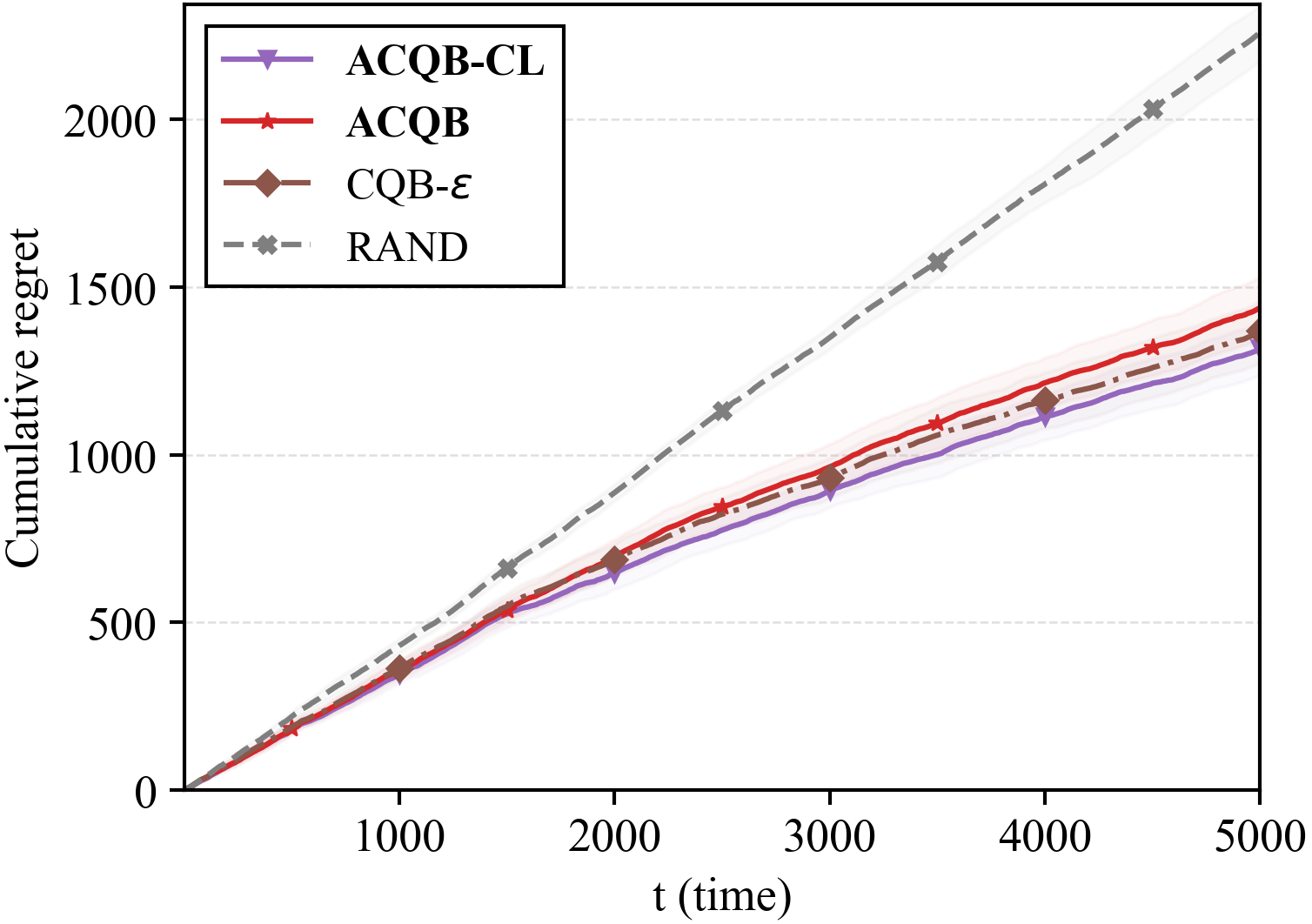}
        \end{subfigure}
    \end{minipage}%

    \vspace{0.15cm} % 행 간 간격

    % ================= ROW 2: lambda=0.65 (예시) =================
    \begin{minipage}[c]{\lbwidth}
        \centering \rotatebox{90}{\scriptsize $\lambda=0.4$}
    \end{minipage}%
    \hfill%
    \begin{minipage}[c]{\blockwidth}
        \centering
        \begin{subfigure}[c]{0.49\linewidth}
            \includegraphics[width=\linewidth]{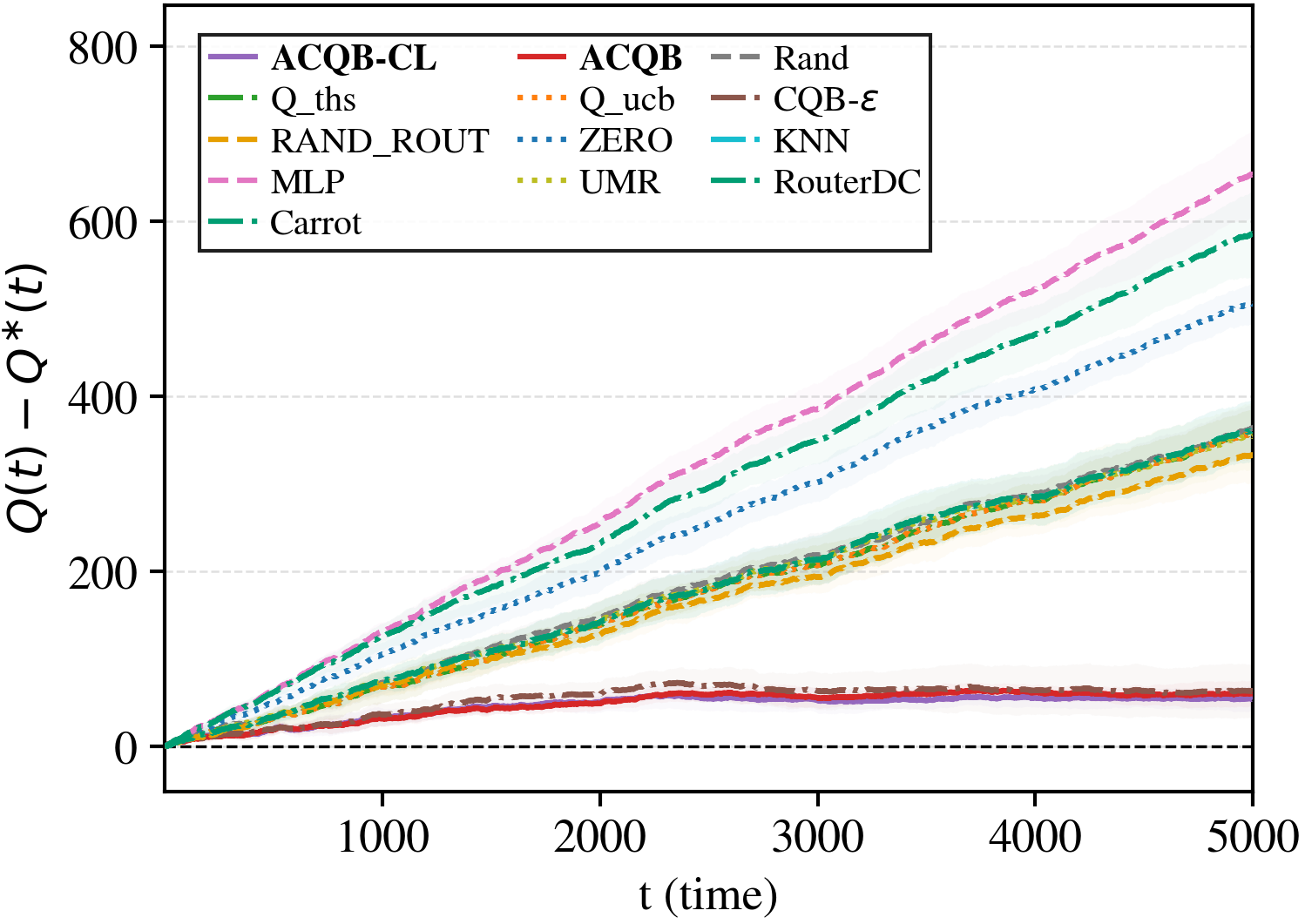}
        \end{subfigure}\hfill
        \begin{subfigure}[c]{0.49\linewidth}
            \includegraphics[width=\linewidth]{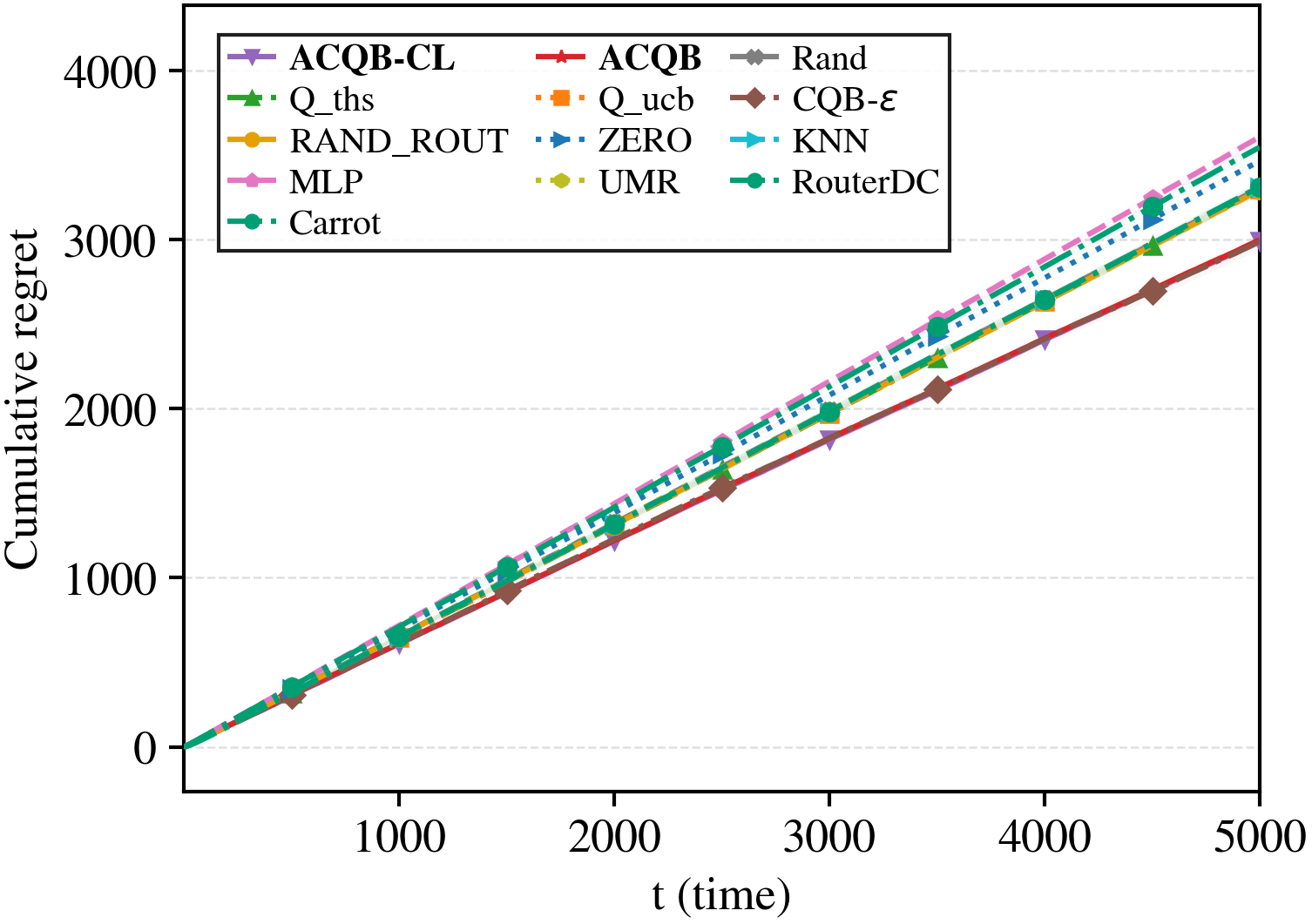}
        \end{subfigure}
    \end{minipage}%
    \hfill%
    \begin{minipage}[c]{\lbwidth}
        \centering \rotatebox{90}{\scriptsize $\lambda=0.55$}
    \end{minipage}%
    \hfill%
    \begin{minipage}[c]{\blockwidth}
        \centering
        \begin{subfigure}[c]{0.49\linewidth}
            \includegraphics[width=\linewidth]{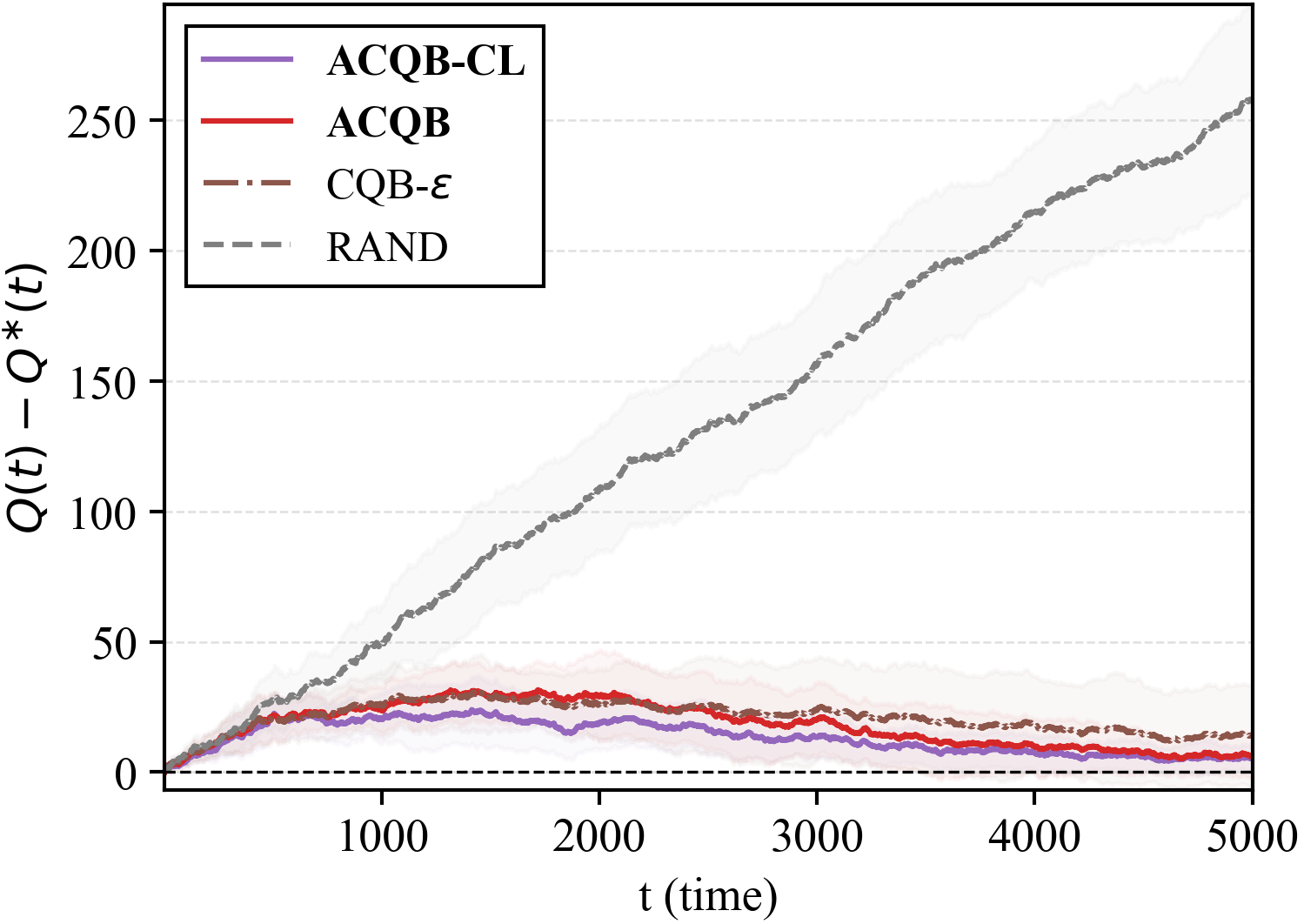}
        \end{subfigure}\hfill
        \begin{subfigure}[c]{0.49\linewidth}
            \includegraphics[width=\linewidth]{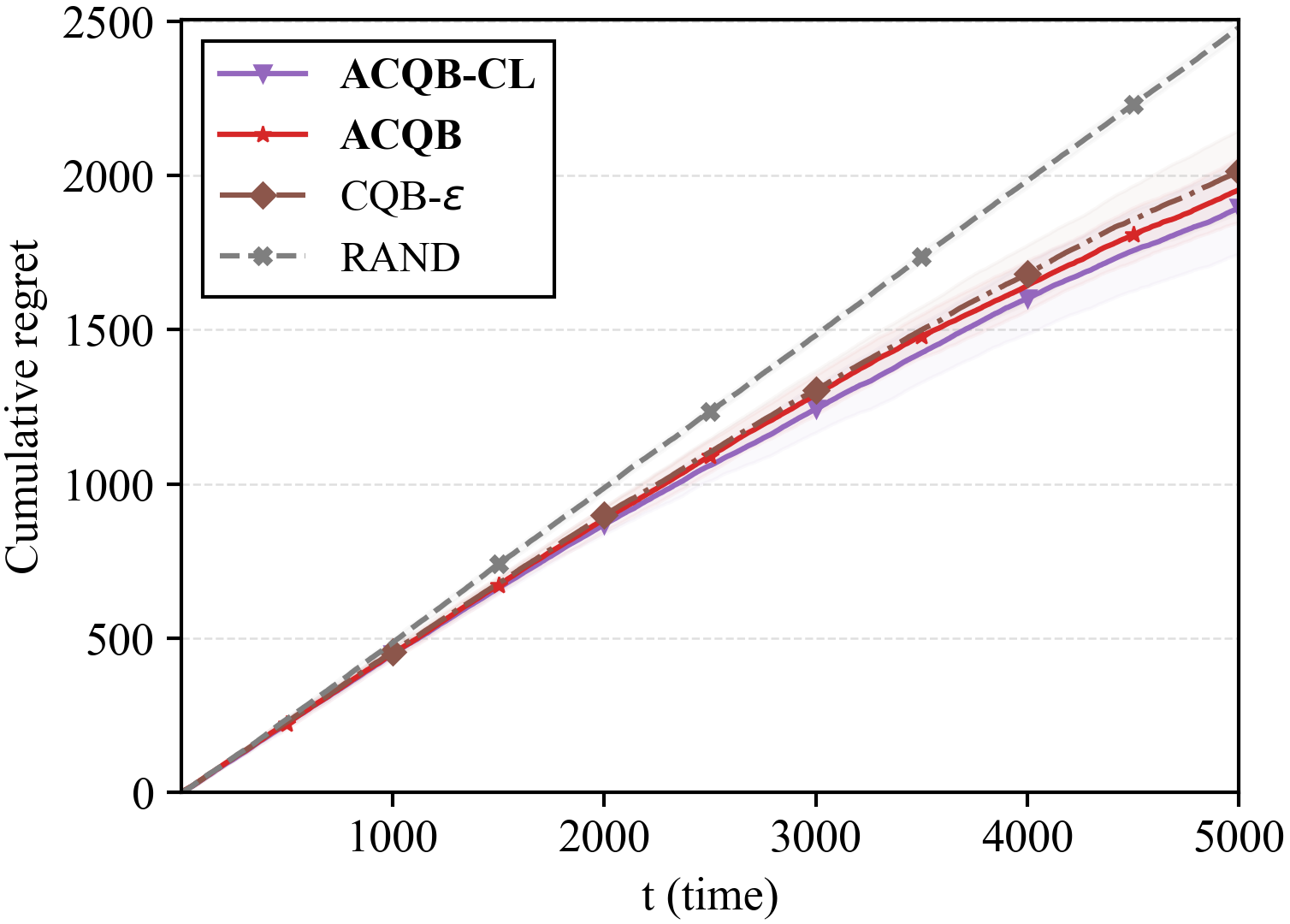}
        \end{subfigure}
    \end{minipage}%

    \vspace{0.15cm} % 행 간 간격

    % ================= ROW 3: lambda=0.75 (예시) =================
    \begin{minipage}[c]{\lbwidth}
        \centering \rotatebox{90}{\scriptsize $\lambda=0.5$}
    \end{minipage}%
    \hfill%
    \begin{minipage}[c]{\blockwidth}
        \centering
        \begin{subfigure}[c]{0.49\linewidth}
            \includegraphics[width=\linewidth]{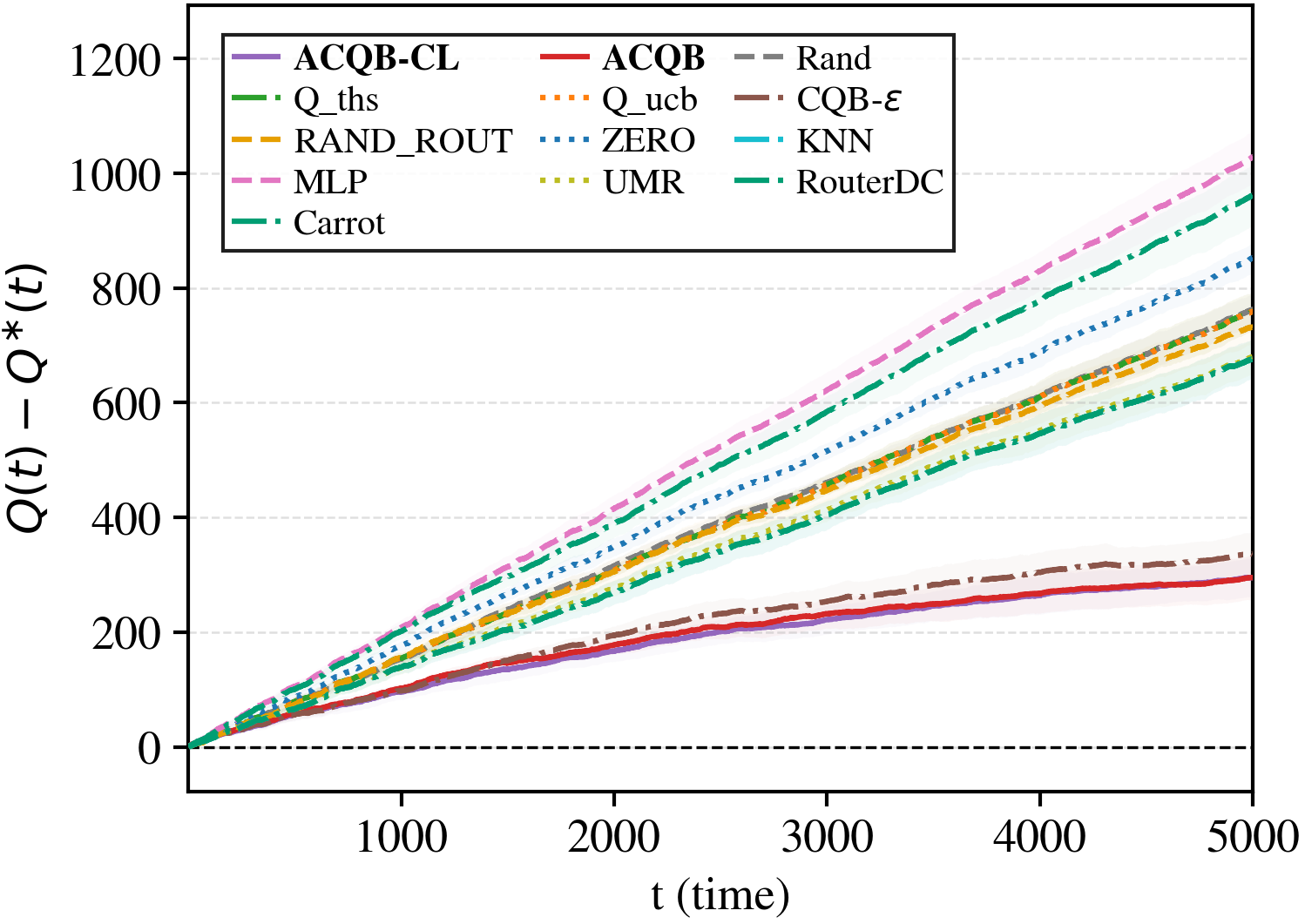}
        \end{subfigure}\hfill
        \begin{subfigure}[c]{0.49\linewidth}
            \includegraphics[width=\linewidth]{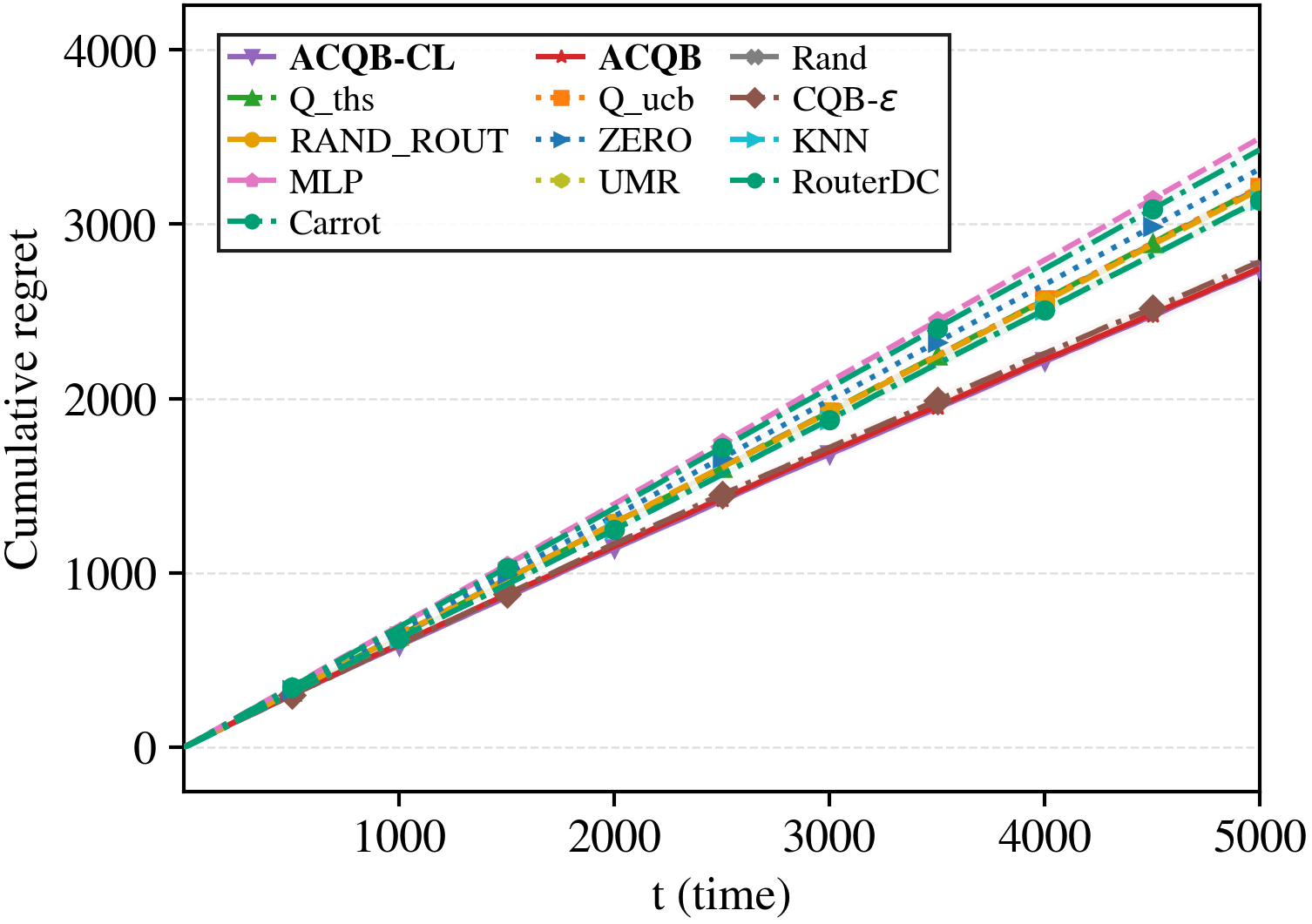}
        \end{subfigure}
    \end{minipage}%
    \hfill%
    \begin{minipage}[c]{\lbwidth}
        \centering \rotatebox{90}{\scriptsize $\lambda=0.65$}
    \end{minipage}%
    \hfill%
    \begin{minipage}[c]{\blockwidth}
        \centering
        \begin{subfigure}[c]{0.49\linewidth}
            \includegraphics[width=\linewidth]{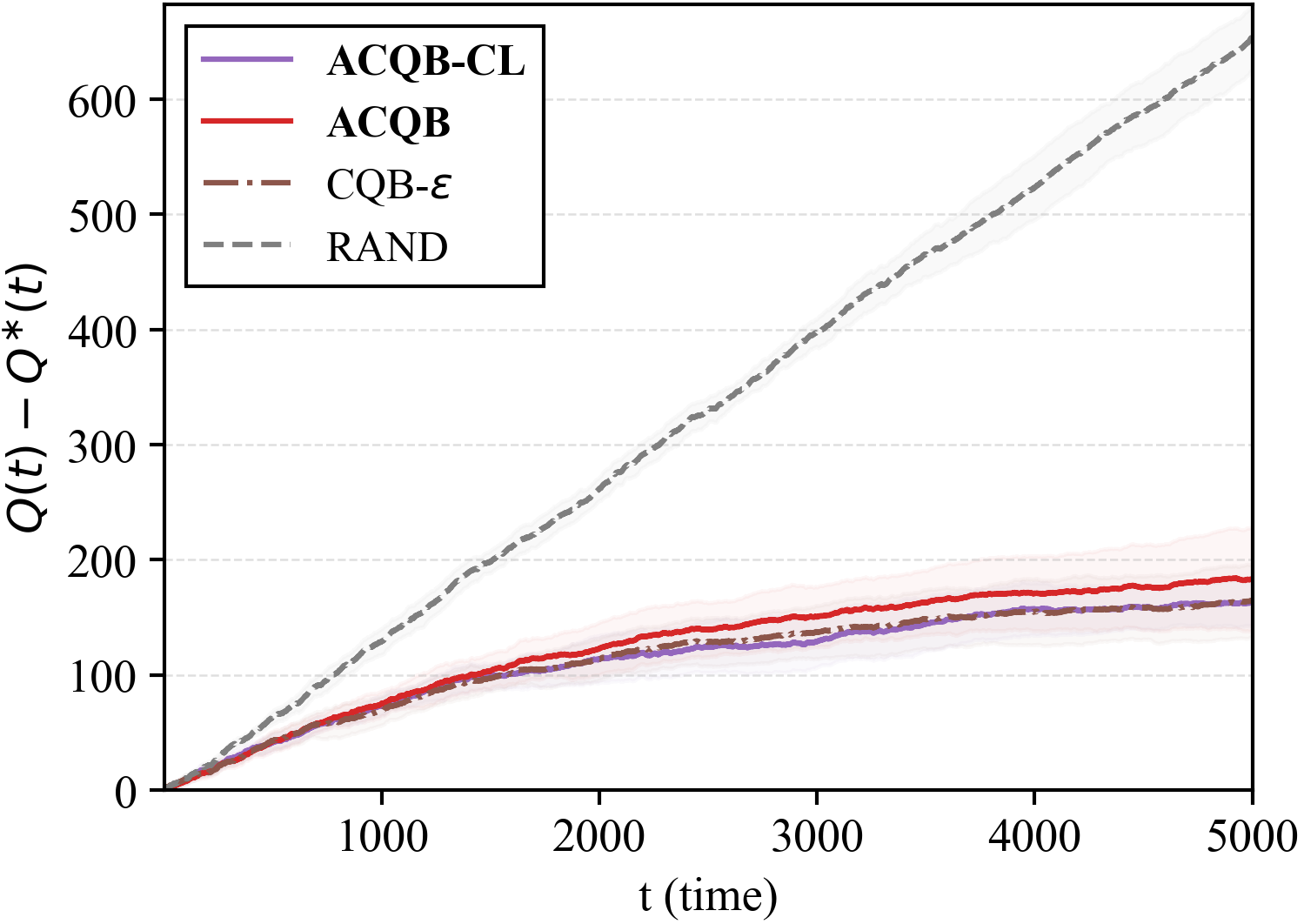}
        \end{subfigure}\hfill
        \begin{subfigure}[c]{0.49\linewidth}
            \includegraphics[width=\linewidth]{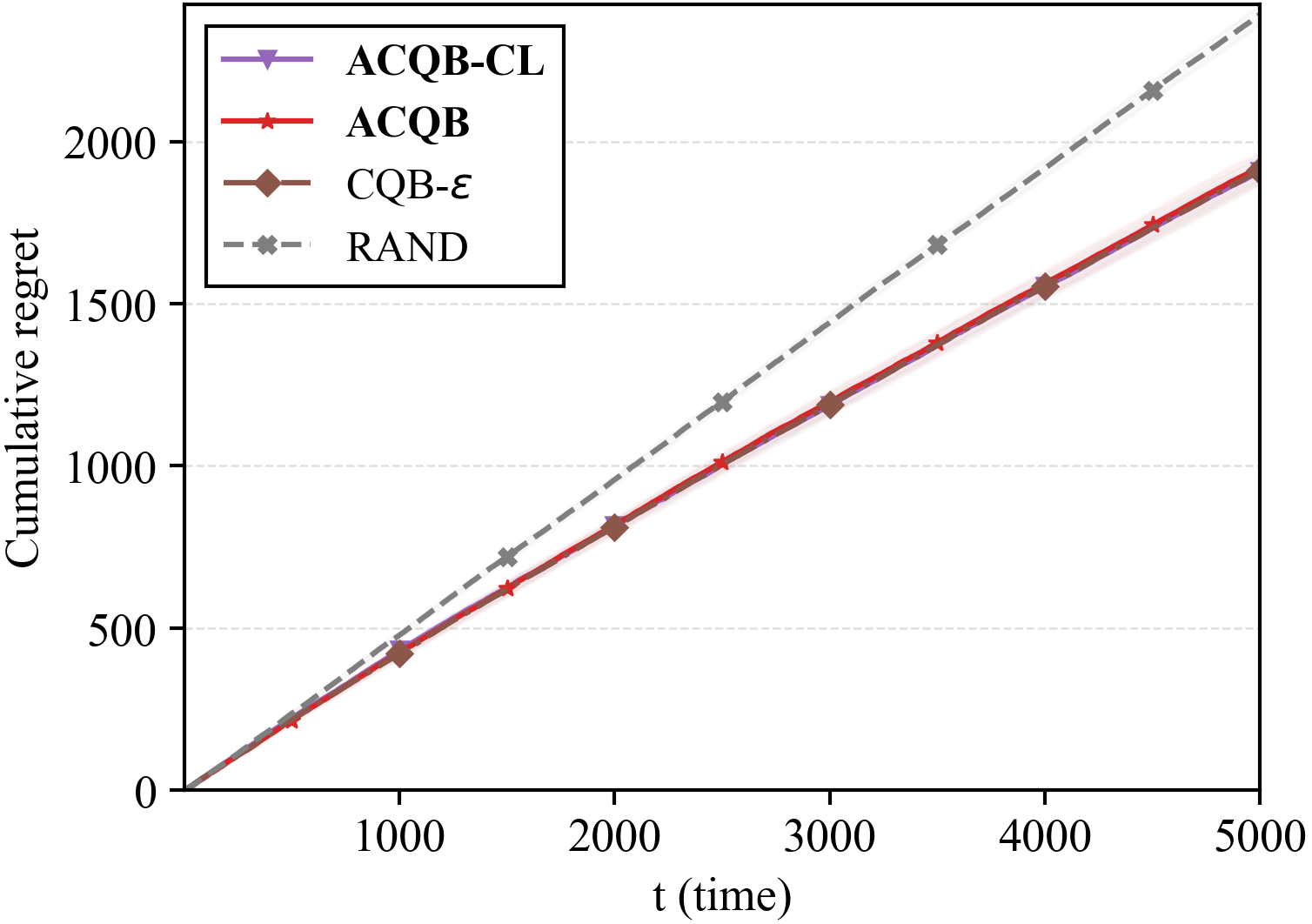}
        \end{subfigure}
    \end{minipage}%

    \vspace{-0.1cm}
    \caption{Queue length and cumulative regret on the \textsc{EmbedLLM} dataset.}
    \label{fig:data:emb}
    
\end{figure*}

\begin{figure*}[htbp]
    \centering
    
    % --- 설정: 너비 변수 ---
    \newcommand{\lbwidth}{0.02\textwidth} % 세로 라벨 너비
    \newcommand{\blockwidth}{0.46\textwidth} % 그림 블록 너비 (이미지 2개 포함)
    
    % ---------------- Header (K=1, K=2) ----------------
    % 1. Left Spacing
    \begin{minipage}{\lbwidth} \centering \phantom{L} \end{minipage}%
    \hfill%
    % 2. K=1 Header
    \begin{minipage}{\blockwidth} \centering \scriptsize $K=1$ \end{minipage}%
    \hfill%
    % 3. Middle Spacing
    \begin{minipage}{\lbwidth} \centering \phantom{L} \end{minipage}%
    \hfill%
    % 4. K=2 Header
    \begin{minipage}{\blockwidth} \centering \scriptsize $K=2$ \end{minipage}%
    
    \vspace{0.1cm} % 헤더와 첫 줄 사이 간격

    % ================= ROW 1: lambda=0.55 =================
    \begin{minipage}[c]{\lbwidth}
        \centering \rotatebox{90}{\scriptsize $\lambda=0.5$}
    \end{minipage}%
    \hfill%
    \begin{minipage}[c]{\blockwidth}
        \centering
        \begin{subfigure}[c]{0.49\linewidth}
            \includegraphics[width=\linewidth]{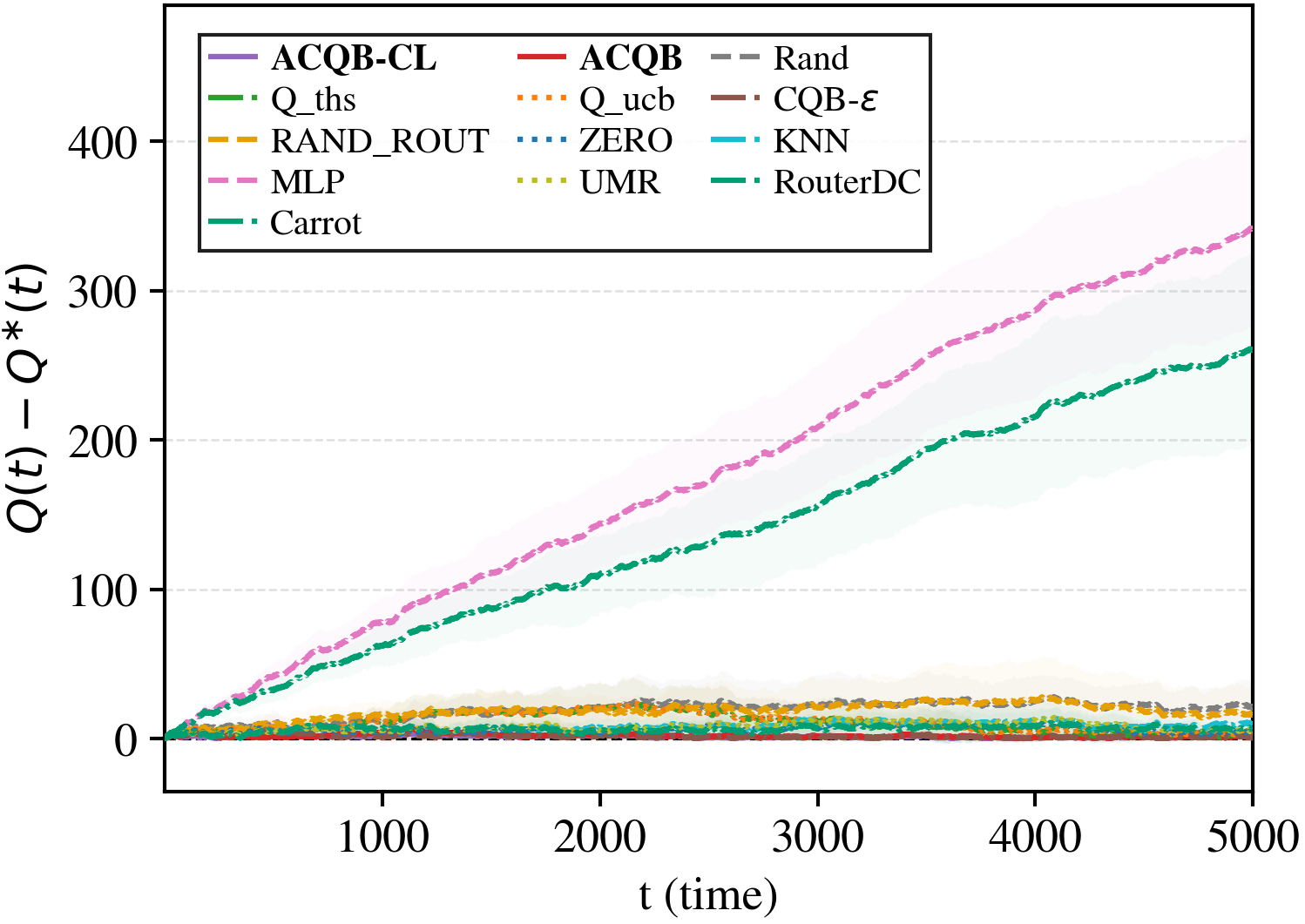}
        \end{subfigure}\hfill
        \begin{subfigure}[c]{0.49\linewidth}
            \includegraphics[width=\linewidth]{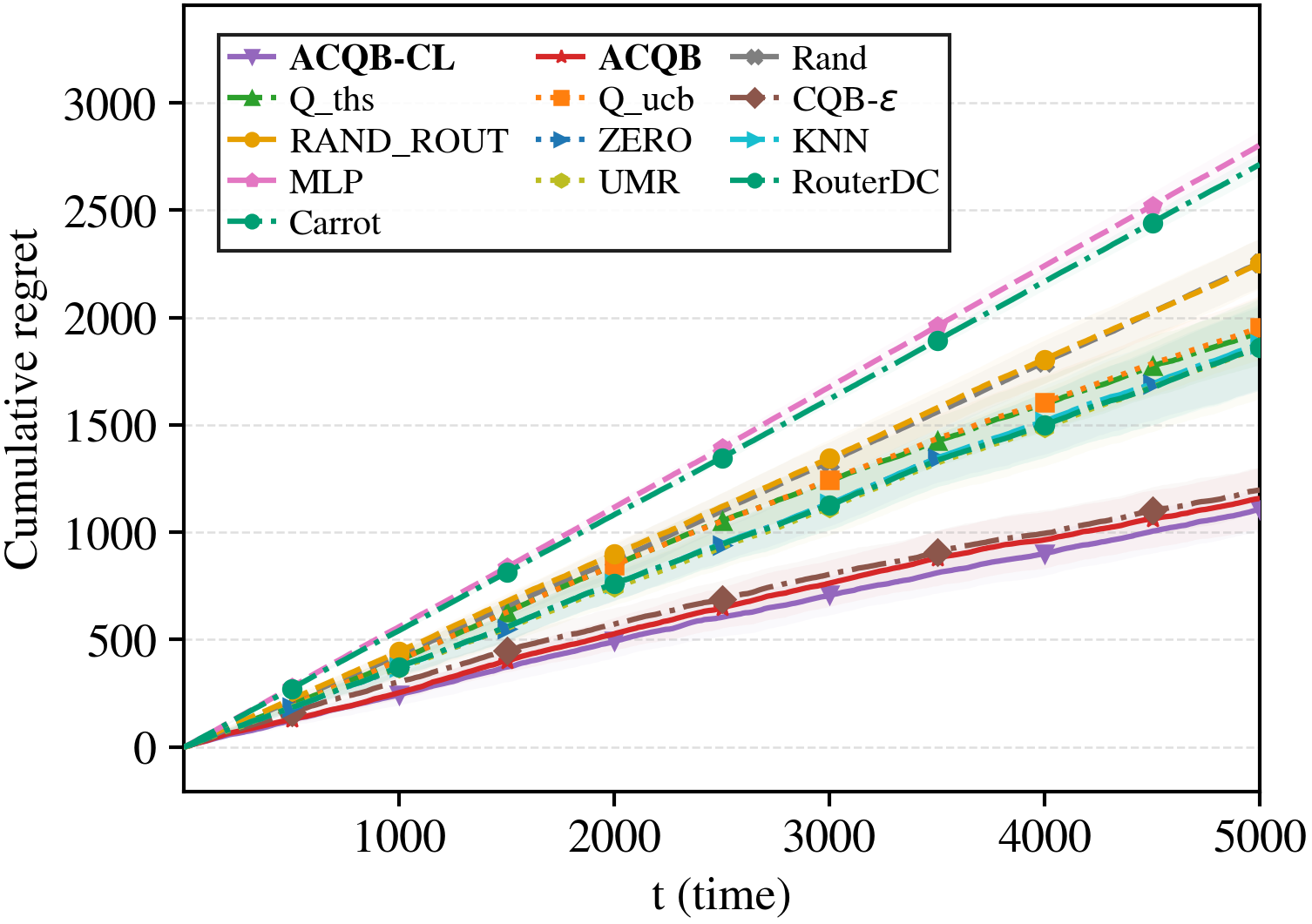}
        \end{subfigure}
    \end{minipage}%
    \hfill%
    \begin{minipage}[c]{\lbwidth} % 중앙 라벨 (필요 없다면 삭제 후 \blockwidth 조정)
        \centering \rotatebox{90}{\scriptsize $\lambda=0.65$}
    \end{minipage}%
    \hfill%
    \begin{minipage}[c]{\blockwidth}
        \centering
        \begin{subfigure}[c]{0.49\linewidth}
            \includegraphics[width=\linewidth]{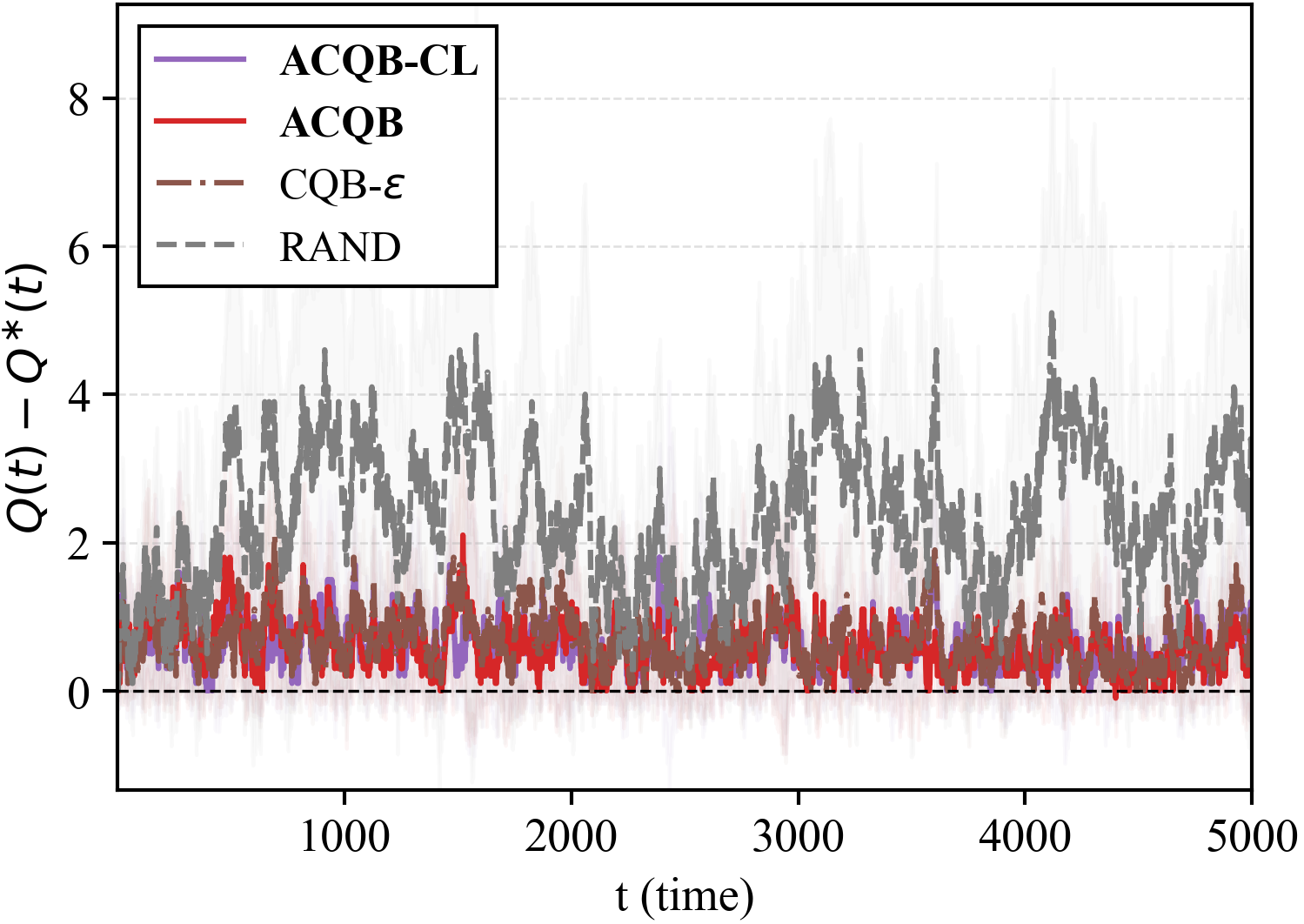}
        \end{subfigure}\hfill
        \begin{subfigure}[c]{0.49\linewidth}
            \includegraphics[width=\linewidth]{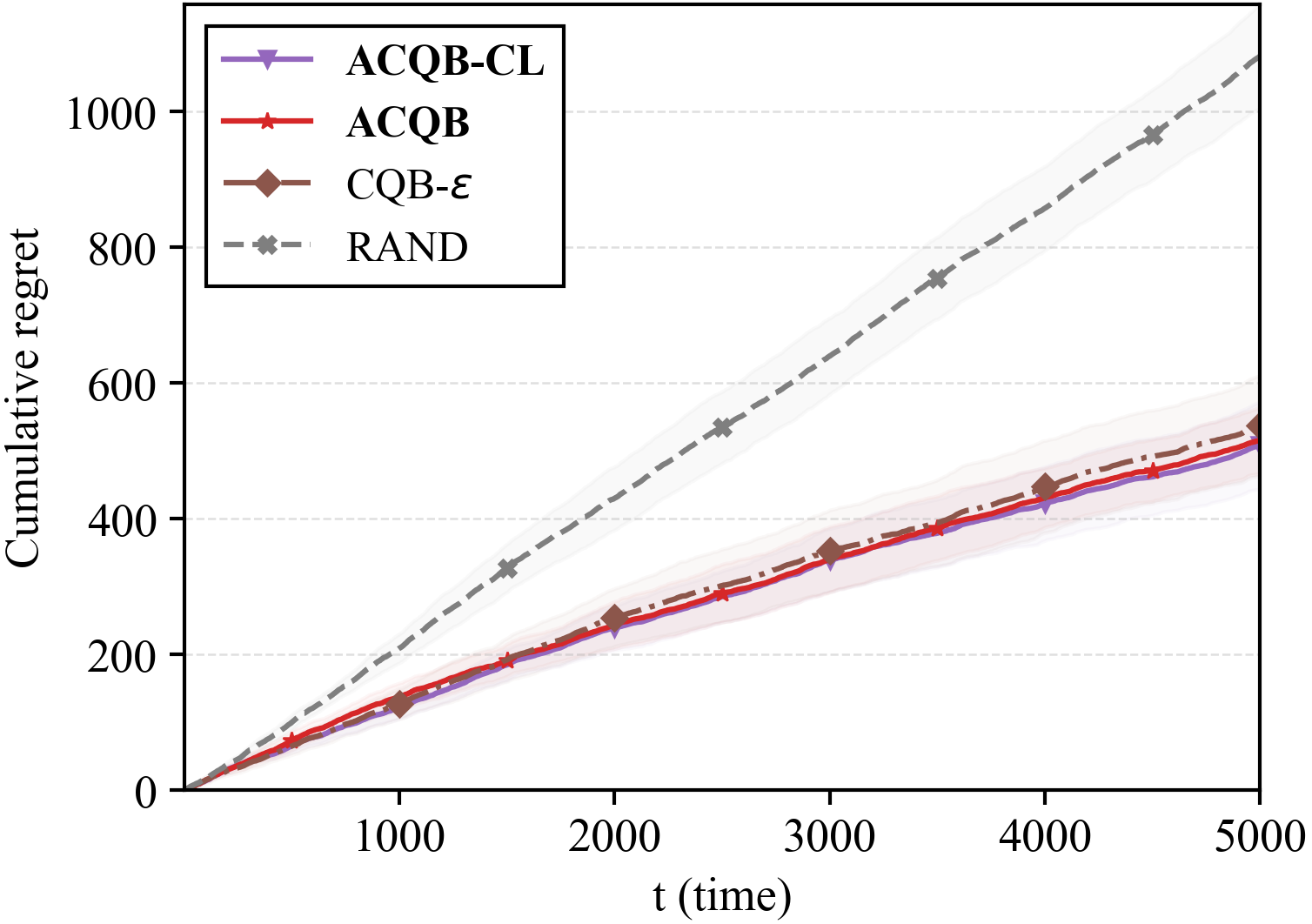}
        \end{subfigure}
    \end{minipage}%

    \vspace{0.15cm} % 행 간 간격

    % ================= ROW 2: lambda=0.65 (예시) =================
    \begin{minipage}[c]{\lbwidth}
        \centering \rotatebox{90}{\scriptsize $\lambda=0.6$}
    \end{minipage}%
    \hfill%
    \begin{minipage}[c]{\blockwidth}
        \centering
        \begin{subfigure}[c]{0.49\linewidth}
            \includegraphics[width=\linewidth]{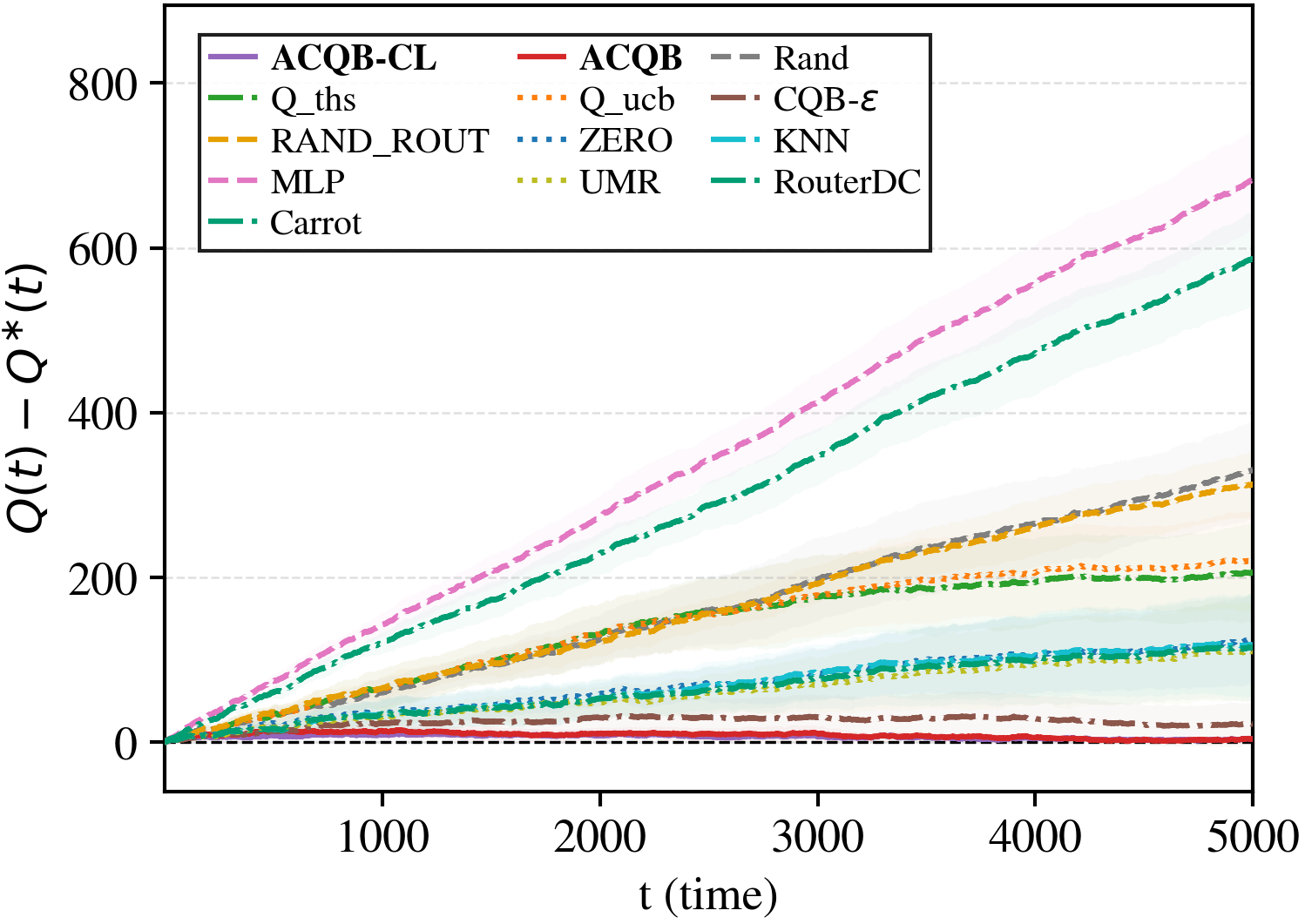}
        \end{subfigure}\hfill
        \begin{subfigure}[c]{0.49\linewidth}
            \includegraphics[width=\linewidth]{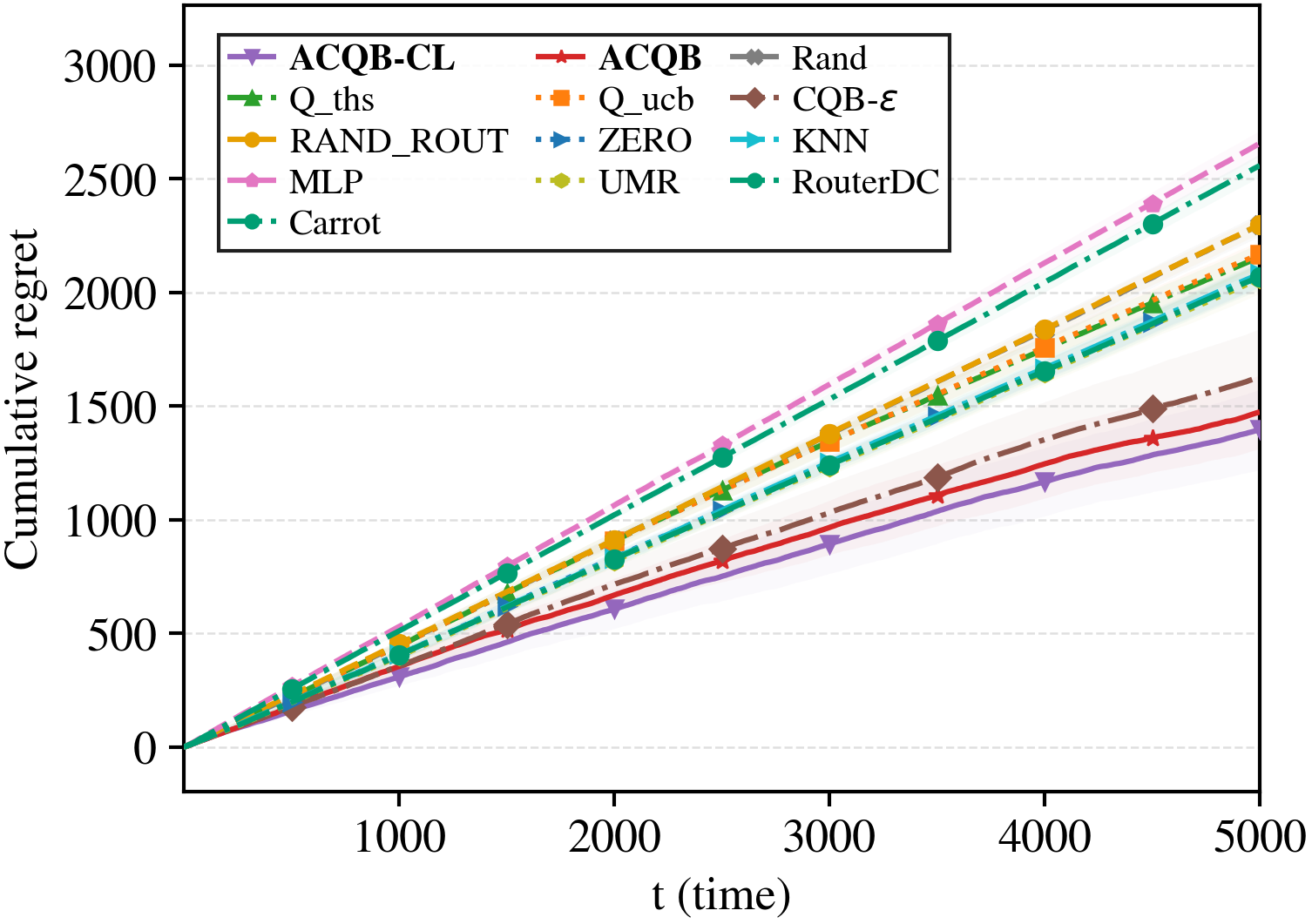}
        \end{subfigure}
    \end{minipage}%
    \hfill%
    \begin{minipage}[c]{\lbwidth}
        \centering \rotatebox{90}{\scriptsize $\lambda=0.75$}
    \end{minipage}%
    \hfill%
    \begin{minipage}[c]{\blockwidth}
        \centering
        \begin{subfigure}[c]{0.49\linewidth}
            \includegraphics[width=\linewidth]{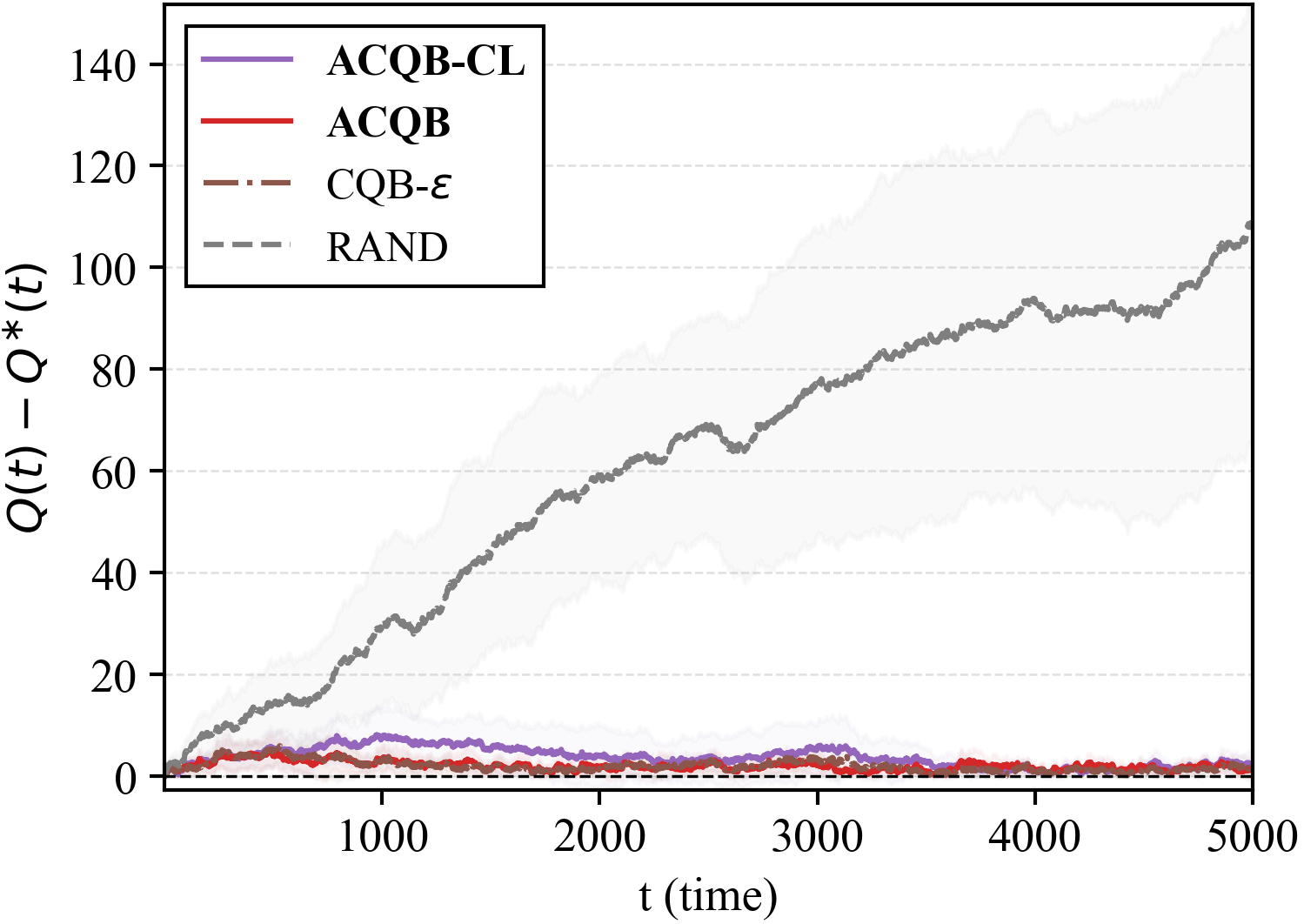}
        \end{subfigure}\hfill
        \begin{subfigure}[c]{0.49\linewidth}
            \includegraphics[width=\linewidth]{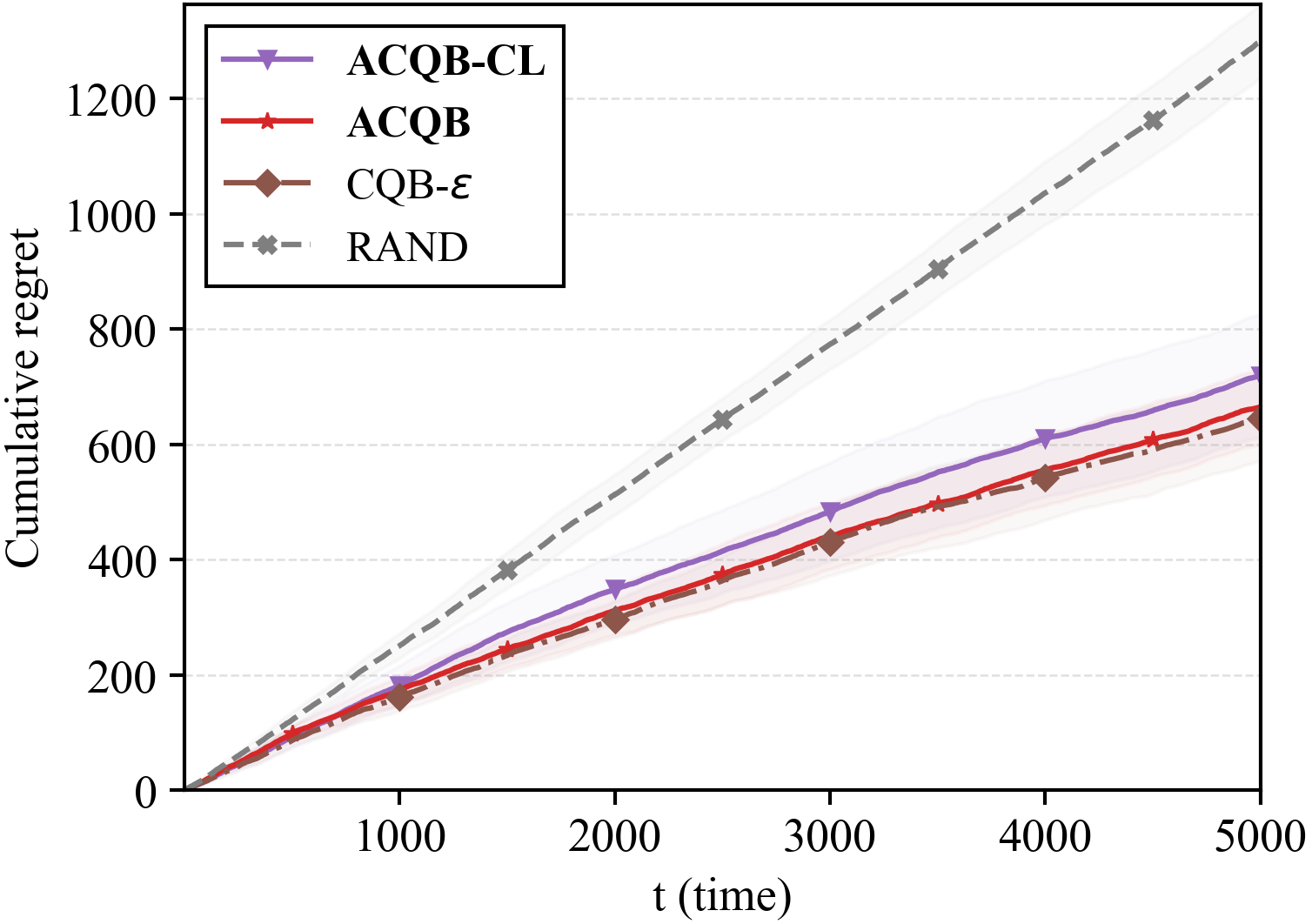}
        \end{subfigure}
    \end{minipage}%

    \vspace{0.15cm} % 행 간 간격

    % ================= ROW 3: lambda=0.75 (예시) =================
    \begin{minipage}[c]{\lbwidth}
        \centering \rotatebox{90}{\scriptsize $\lambda=0.7$}
    \end{minipage}%
    \hfill%
    \begin{minipage}[c]{\blockwidth}
        \centering
        \begin{subfigure}[c]{0.49\linewidth}
            \includegraphics[width=\linewidth]{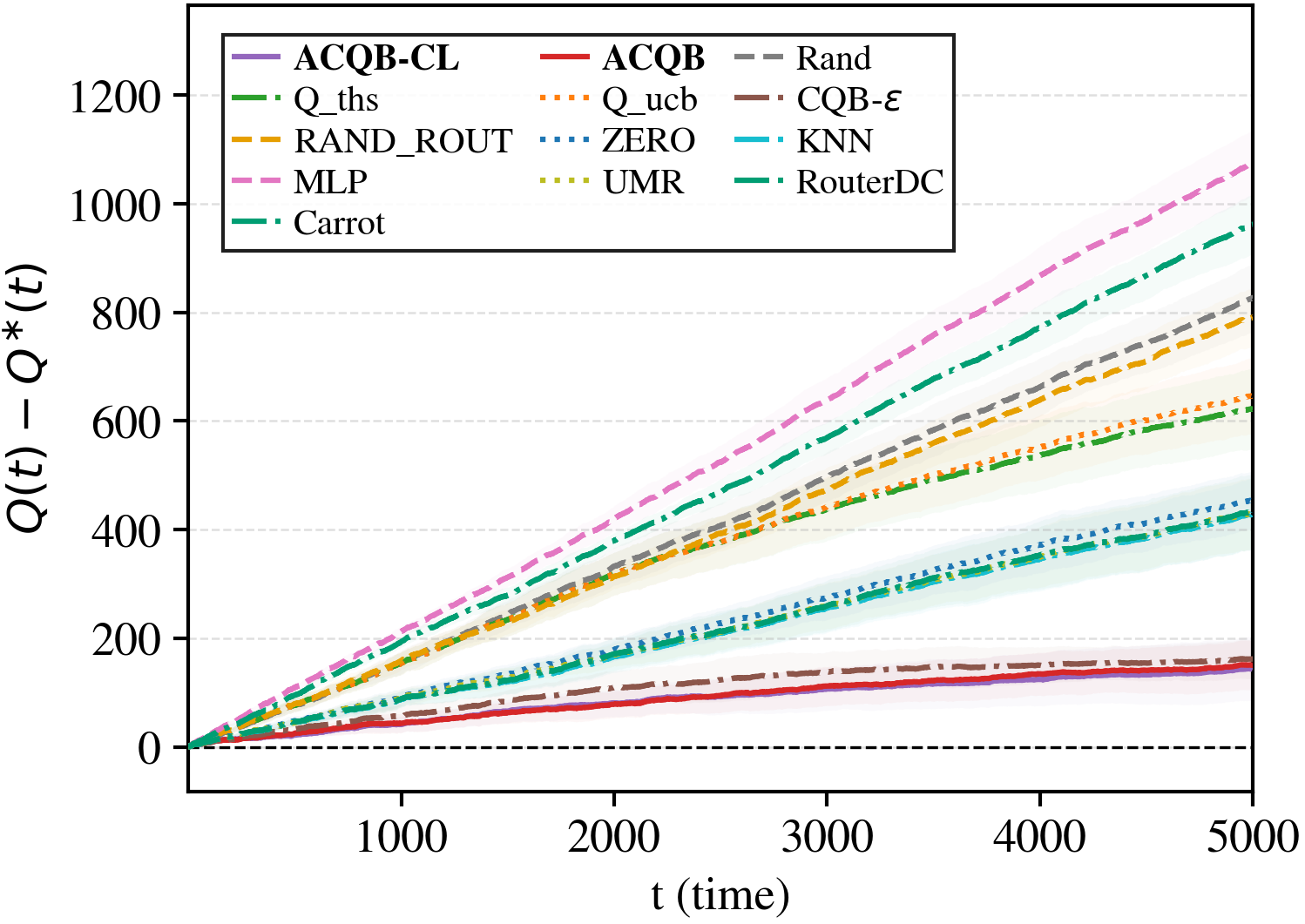}
        \end{subfigure}\hfill
        \begin{subfigure}[c]{0.49\linewidth}
            \includegraphics[width=\linewidth]{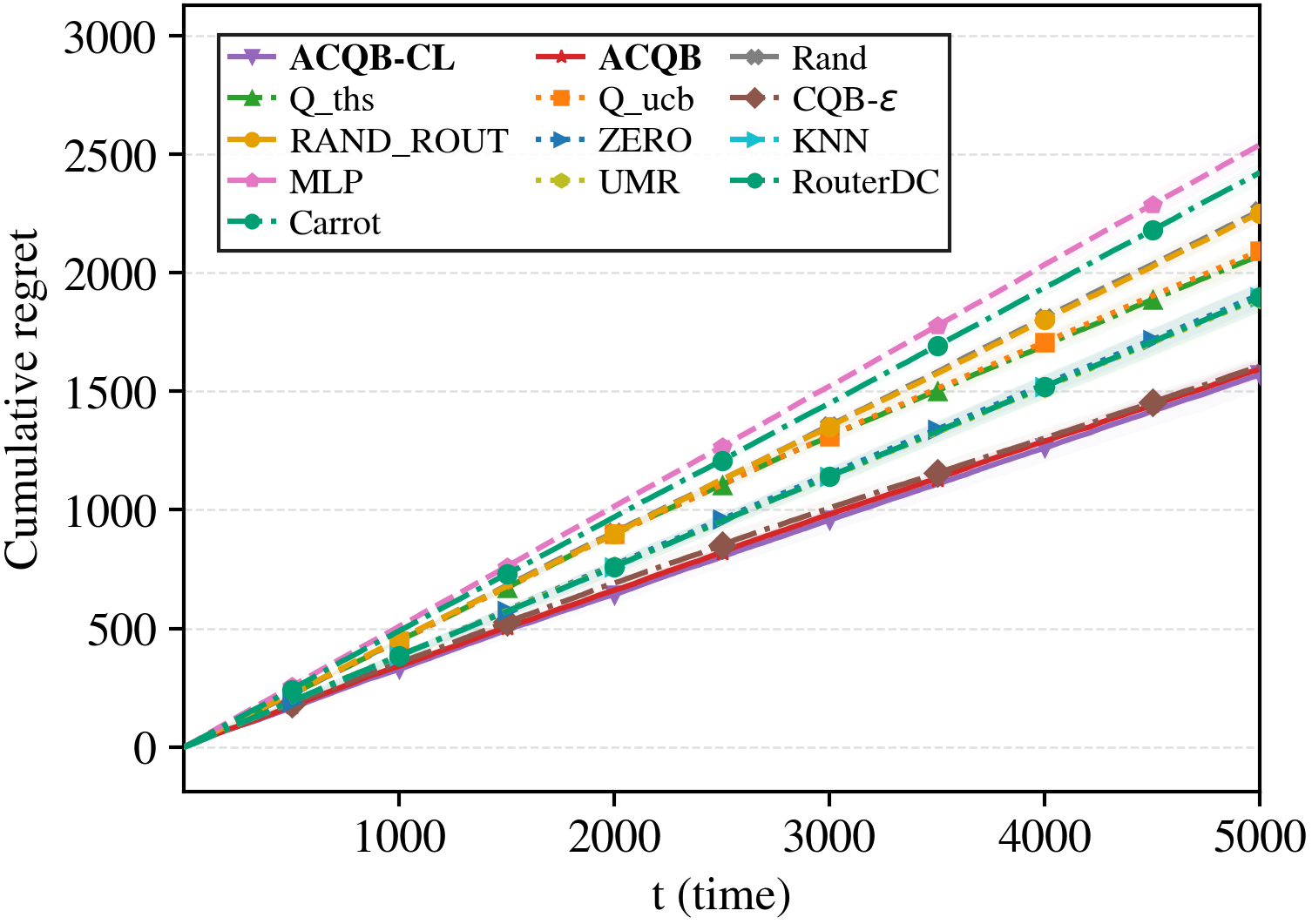}
        \end{subfigure}
    \end{minipage}%
    \hfill%
    \begin{minipage}[c]{\lbwidth}
        \centering \rotatebox{90}{\scriptsize $\lambda=0.85$}
    \end{minipage}%
    \hfill%
    \begin{minipage}[c]{\blockwidth}
        \centering
        \begin{subfigure}[c]{0.49\linewidth}
            \includegraphics[width=\linewidth]{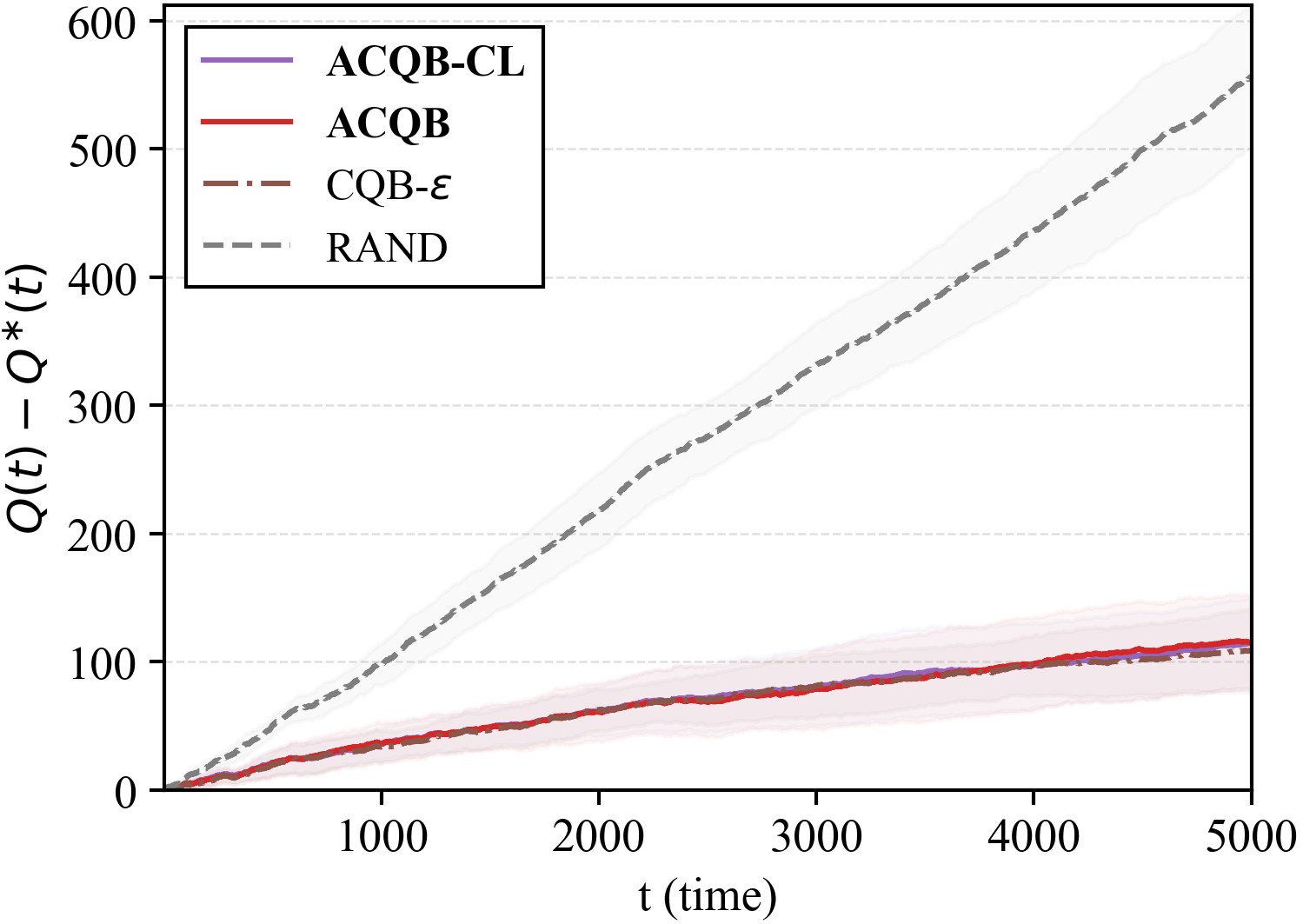}
        \end{subfigure}\hfill
        \begin{subfigure}[c]{0.49\linewidth}
            \includegraphics[width=\linewidth]{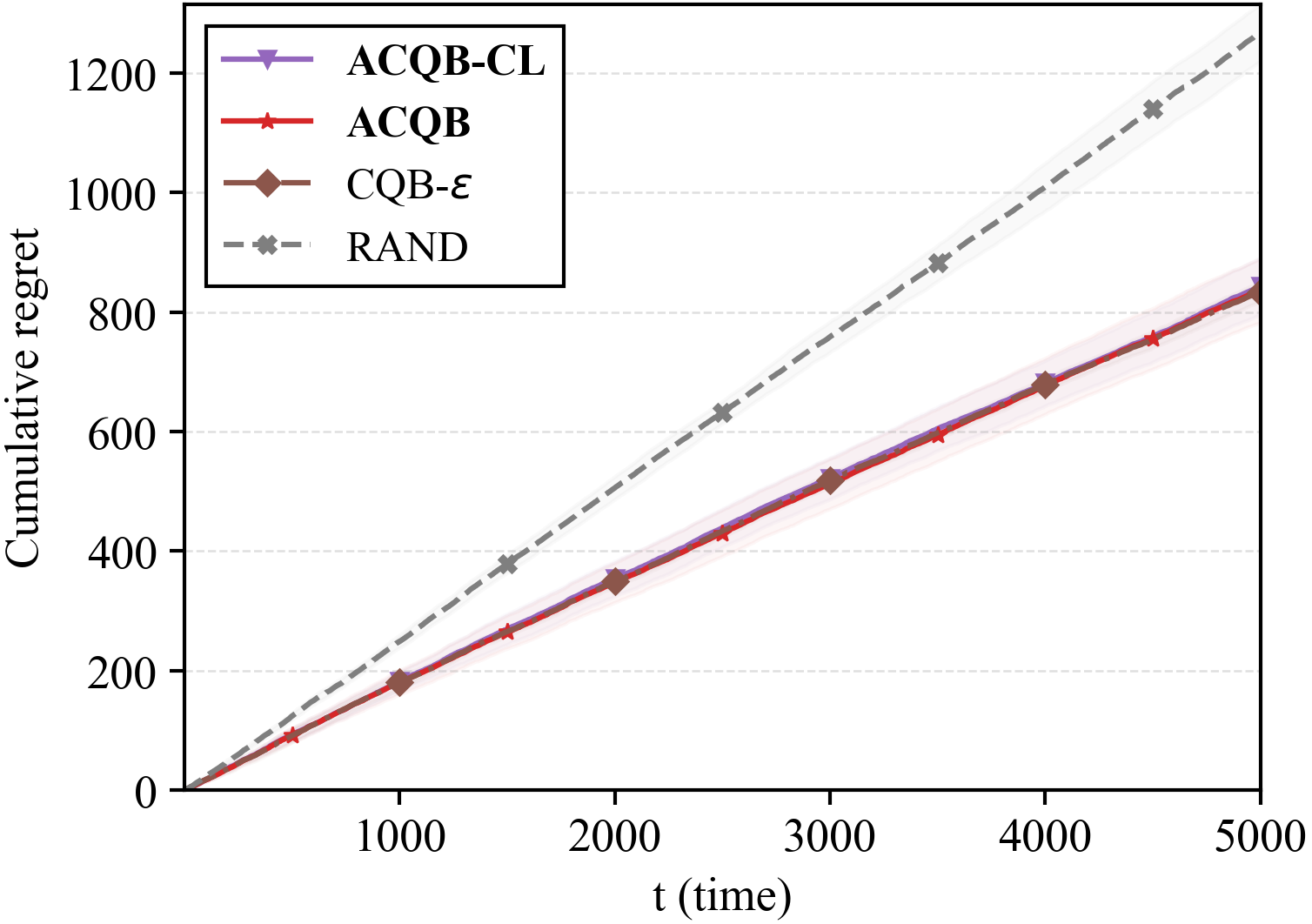}
        \end{subfigure}
    \end{minipage}%

    \vspace{-0.1cm}
    \caption{Queue length and cumulative regret on the \textsc{RouterBench} dataset.}
    \label{fig:data:rout}
    
\end{figure*}
\FloatBarrier

\begin{table*}[t]
\centering
\caption{Mean ($\mu$), standard deviation ($\sigma$) of queue length gap, and throughput (T.P.) at time steps 2500 and 5000 for $K=1$ (low-arrival setting: $\lambda_{\mathrm{Sprout}}=0.7$, $\lambda_{\mathrm{RouterBench}}=0.5$, $\lambda_{\mathrm{EmbedLLM}}=0.3$).}
\label{tab:k1-low-lam5-qgap-throughput}
\scriptsize
\setlength{\tabcolsep}{2.4pt}
\renewcommand{\arraystretch}{1.15}
\resizebox{\textwidth}{!}{%
\begin{tabular}{lllrrrrrrrrrrrrr}
\toprule
Dataset & Time & Stat.
& ACQB-CL & ACQB & RAND & Q-ThS & Q-UCB & CQB-$\v\veps$ & RAND\_ROUT & ZERO & KNN & MLP & UMR & RouterDC & CARROT \\
\midrule
Sprout & 2500 & $\mu$
& \second{0.4} & \best{0.2} & 2.0 & 2.7 & 2.7 & 0.7 & 3.2 & 1.7 & 1.4 & 19.0 & 1.1 & 19.0 & 1.1 \\
& & $\sigma$
& 1.0 & 0.4 & 1.8 & 3.0 & 3.0 & 1.1 & 4.3 & 2.2 & 2.3 & 16.1 & 1.7 & 17.7 & 1.3 \\
& & T.P.
& \second{0.696} & \best{0.696} & 0.695 & 0.695 & 0.695 & 0.696 & 0.695 & 0.695 & 0.696 & 0.689 & 0.696 & 0.689 & 0.696 \\

& 5000 & $\mu$
& \best{0.3} & 0.6 & 5.0 & 0.6 & 1.3 & 0.7 & 1.6 & \second{0.5} & 0.6 & 28.3 & 1.0 & 34.1 & 0.9 \\
& & $\sigma$
& 0.5 & 1.3 & 7.3 & 0.8 & 1.3 & 0.8 & 1.3 & 1.6 & 1.1 & 25.1 & 1.9 & 30.4 & 1.7 \\
& & T.P.
& \best{0.699} & 0.698 & 0.698 & 0.698 & 0.698 & 0.698 & 0.698 & \second{0.699} & 0.698 & 0.693 & 0.698 & 0.692 & 0.698 \\
\midrule

EmbedLLM & 2500 & $\mu$
& 3.9 & \second{3.4} & 14.9 & 16.1 & 16.1 & \best{2.6} & 20.3 & 98.7 & 47.3 & 146.6 & 50.6 & 117.6 & 47.3 \\
& & $\sigma$
& 4.4 & 3.8 & 14.2 & 14.4 & 14.4 & 3.0 & 12.1 & 25.9 & 31.1 & 36.3 & 23.9 & 29.8 & 31.1 \\
& & T.P.
& 0.297 & \second{0.297} & 0.292 & 0.292 & 0.292 & \best{0.297} & 0.290 & 0.259 & 0.280 & 0.240 & 0.278 & 0.251 & 0.280 \\

& 5000 & $\mu$
& \second{2.4} & 2.5 & 30.7 & 25.2 & 25.2 & \best{1.8} & 28.4 & 189.0 & 88.9 & 297.0 & 84.5 & 238.3 & 88.9 \\
& & $\sigma$
& 2.2 & 2.3 & 18.5 & 18.6 & 18.6 & 2.5 & 22.4 & 51.1 & 54.1 & 64.3 & 46.3 & 65.8 & 54.1 \\
& & T.P.
& \second{0.300} & 0.300 & 0.295 & 0.296 & 0.296 & \best{0.300} & 0.295 & 0.263 & 0.283 & 0.241 & 0.284 & 0.253 & 0.283 \\
\midrule

RouterBench & 2500 & $\mu$
& \best{0.8} & \second{1.2} & 23.1 & 17.3 & 17.2 & 2.1 & 21.7 & 8.3 & 8.7 & 172.5 & 8.9 & 130.5 & 6.4 \\
& & $\sigma$
& 1.1 & 1.4 & 28.2 & 24.8 & 24.9 & 2.2 & 23.1 & 9.3 & 12.8 & 50.9 & 11.1 & 50.5 & 9.3 \\
& & T.P.
& \best{0.501} & \second{0.501} & 0.492 & 0.494 & 0.494 & 0.500 & 0.492 & 0.498 & 0.498 & 0.432 & 0.498 & 0.449 & 0.499 \\

& 5000 & $\mu$
& \best{0.7} & \second{1.2} & 20.5 & 4.0 & 3.2 & 1.5 & 15.2 & 4.2 & 9.1 & 341.5 & 4.5 & 260.2 & 5.9 \\
& & $\sigma$
& 1.3 & 1.4 & 24.4 & 3.7 & 4.5 & 2.3 & 26.4 & 6.0 & 14.1 & 93.2 & 6.9 & 91.8 & 8.1 \\
& & T.P.
& \best{0.498} & \second{0.498} & 0.494 & 0.497 & 0.498 & 0.498 & 0.495 & 0.497 & 0.496 & 0.430 & 0.497 & 0.446 & 0.497 \\

\bottomrule
\end{tabular}%
}
\end{table*}

\FloatBarrier

\begin{table*}[t]
\centering
\caption{Mean ($\mu$), standard deviation ($\sigma$) of queue length gap, and throughput (T.P.) at time steps 2500 and 5000 for $K=1$ (high-arrival setting: $\lambda_{\mathrm{Sprout}}=0.9$, $\lambda_{\mathrm{RouterBench}}=0.7$, $\lambda_{\mathrm{EmbedLLM}}=0.5$).}
\label{tab:k1-high-lam5-qgap-throughput}
\scriptsize
\setlength{\tabcolsep}{2.4pt}
\renewcommand{\arraystretch}{1.15}
\resizebox{\textwidth}{!}{%
\begin{tabular}{lllrrrrrrrrrrrrr}
\toprule
Dataset & Time & Stat.
& ACQB-CL & ACQB & RAND & Q-ThS & Q-UCB & CQB-$\v\veps$ & RAND\_ROUT & ZERO & KNN & MLP & UMR & RouterDC & CARROT \\
\midrule
Sprout & 2500 & $\mu$
& \best{100.6} & 125.4 & 337.9 & 326.2 & 326.2 & 153.2 & 330.6 & 138.1 & 121.8 & 400.9 & 125.9 & 412.8 & \second{115.6} \\
& & $\sigma$
& 39.3 & 20.8 & 31.1 & 32.6 & 32.6 & 40.0 & 43.2 & 41.2 & 33.5 & 45.2 & 16.8 & 48.4 & 32.4 \\
& & T.P.
& \best{0.856} & 0.846 & 0.761 & 0.766 & 0.766 & 0.835 & 0.765 & 0.842 & 0.848 & 0.736 & 0.846 & 0.732 & \second{0.851} \\

& 5000 & $\mu$
& \best{130.3} & \second{172.6} & 679.1 & 598.4 & 599.1 & 241.4 & 656.4 & 276.5 & 256.3 & 803.1 & 261.1 & 826.2 & 231.3 \\
& & $\sigma$
& 64.4 & 40.5 & 49.6 & 42.9 & 46.6 & 85.1 & 57.9 & 63.3 & 42.5 & 82.3 & 31.2 & 72.6 & 54.3 \\
& & T.P.
& \best{0.873} & \second{0.864} & 0.763 & 0.779 & 0.779 & 0.851 & 0.768 & 0.843 & 0.848 & 0.738 & 0.847 & 0.734 & 0.853 \\
\midrule

EmbedLLM & 2500 & $\mu$
& \best{201.0} & \second{208.7} & 393.1 & 386.5 & 386.5 & 232.3 & 381.0 & 436.1 & 340.4 & 523.1 & 350.2 & 492.2 & 340.4 \\
& & $\sigma$
& 30.1 & 29.1 & 19.9 & 21.5 & 21.5 & 35.4 & 37.1 & 24.1 & 35.8 & 36.1 & 21.9 & 35.2 & 35.8 \\
& & T.P.
& \best{0.420} & \second{0.417} & 0.343 & 0.346 & 0.346 & 0.407 & 0.348 & 0.326 & 0.365 & 0.292 & 0.361 & 0.304 & 0.365 \\

& 5000 & $\mu$
& \best{294.8} & \second{295.3} & 762.4 & 759.0 & 759.0 & 336.1 & 731.9 & 852.0 & 676.5 & 1028.3 & 680.1 & 961.5 & 676.5 \\
& & $\sigma$
& 48.7 & 55.7 & 38.9 & 47.1 & 47.1 & 54.7 & 42.7 & 39.1 & 49.4 & 65.2 & 42.5 & 76.9 & 49.4 \\
& & T.P.
& \best{0.439} & \second{0.439} & 0.346 & 0.346 & 0.346 & 0.431 & 0.352 & 0.328 & 0.363 & 0.293 & 0.362 & 0.306 & 0.363 \\
\midrule

RouterBench & 2500 & $\mu$
& \best{92.3} & \second{94.3} & 407.2 & 376.4 & 377.0 & 123.1 & 393.0 & 227.4 & 209.3 & 524.1 & 211.4 & 469.8 & 212.8 \\
& & $\sigma$
& 47.8 & 45.2 & 59.7 & 69.6 & 68.5 & 58.2 & 46.4 & 41.7 & 53.2 & 44.5 & 45.9 & 33.3 & 53.4 \\
& & T.P.
& \best{0.659} & \second{0.658} & 0.533 & 0.545 & 0.545 & 0.646 & 0.539 & 0.605 & 0.612 & 0.486 & 0.611 & 0.508 & 0.611 \\

& 5000 & $\mu$
& \best{143.1} & \second{150.5} & 826.9 & 622.1 & 645.8 & 161.9 & 791.1 & 453.5 & 427.8 & 1075.0 & 429.1 & 962.0 & 431.9 \\
& & $\sigma$
& 83.5 & 64.5 & 89.4 & 105.3 & 99.0 & 51.6 & 76.6 & 75.0 & 92.7 & 85.5 & 93.5 & 77.8 & 99.4 \\
& & T.P.
& \best{0.670} & \second{0.668} & 0.533 & 0.574 & 0.569 & 0.666 & 0.540 & 0.608 & 0.613 & 0.483 & 0.613 & 0.506 & 0.612 \\

\bottomrule
\end{tabular}%
}
\end{table*}

\begin{table*}[t]
\centering
\caption{Mean ($\mu$), standard deviation ($\sigma$) of queue length gap, and throughput (T.P.) at time steps 2500 and 5000 for $K=2$ (low-arrival setting: $\lambda_{\mathrm{Sprout}}=0.75$, $\lambda_{\mathrm{RouterBench}}=0.65$, $\lambda_{\mathrm{EmbedLLM}}=0.45$).}
\label{tab:k2-low-lam5-qgap-throughput}
\scriptsize
\setlength{\tabcolsep}{2.4pt}
\renewcommand{\arraystretch}{1.15}
\resizebox{\textwidth}{!}{%
\begin{tabular}{lllrrrrrrrrrrrrr}
\toprule
Dataset & Time & Stat.
& ACQB-CL & ACQB & RAND & Q-ThS & Q-UCB & CQB-$\v\veps$ & RAND\_ROUT & ZERO & KNN & MLP & UMR & RouterDC & CARROT \\
\midrule
Sprout & 2500 & $\mu$
& \second{0.1} & 0.1 & 0.3 & - & - & \best{0.0} & - & - & - & - & - & - & - \\
& & $\sigma$
& 0.3 & 0.3 & 0.8 & - & - & 0.0 & - & - & - & - & - & - & - \\
& & T.P.
& \second{0.745} & 0.745 & 0.745 & - & - & \best{0.745} & - & - & - & - & - & - & - \\

& 5000 & $\mu$
& \best{0.1} & 0.3 & 0.3 & - & - & \second{0.1} & - & - & - & - & - & - & - \\
& & $\sigma$
& 0.3 & 0.7 & 0.7 & - & - & 0.3 & - & - & - & - & - & - & - \\
& & T.P.
& \best{0.747} & 0.747 & 0.747 & - & - & \second{0.747} & - & - & - & - & - & - & - \\
\midrule

EmbedLLM & 2500 & $\mu$
& 1.5 & \second{1.3} & 8.6 & - & - & \best{1.2} & - & - & - & - & - & - & - \\
& & $\sigma$
& 2.0 & 2.5 & 9.2 & - & - & 1.8 & - & - & - & - & - & - & - \\
& & T.P.
& 0.450 & \second{0.450} & 0.447 & - & - & \best{0.450} & - & - & - & - & - & - & - \\

& 5000 & $\mu$
& \best{1.6} & \second{2.3} & 12.3 & - & - & 2.9 & - & - & - & - & - & - & - \\
& & $\sigma$
& 1.3 & 2.5 & 10.4 & - & - & 2.8 & - & - & - & - & - & - & - \\
& & T.P.
& \best{0.450} & \second{0.449} & 0.447 & - & - & 0.449 & - & - & - & - & - & - & - \\
\midrule

RouterBench & 2500 & $\mu$
& 0.6 & \best{0.4} & 1.5 & - & - & \second{0.4} & - & - & - & - & - & - & - \\
& & $\sigma$
& 0.8 & 1.0 & 1.5 & - & - & 0.8 & - & - & - & - & - & - & - \\
& & T.P.
& 0.645 & \best{0.646} & 0.645 & - & - & \second{0.646} & - & - & - & - & - & - & - \\

& 5000 & $\mu$
& \best{0.1} & \second{0.5} & 2.9 & - & - & 0.9 & - & - & - & - & - & - & - \\
& & $\sigma$
& 0.3 & 0.7 & 4.3 & - & - & 1.2 & - & - & - & - & - & - & - \\
& & T.P.
& \best{0.647} & \second{0.646} & 0.646 & - & - & 0.646 & - & - & - & - & - & - & - \\

\bottomrule
\end{tabular}%
}
\end{table*}

\begin{table*}[t]
\centering
\caption{Mean ($\mu$), standard deviation ($\sigma$) of queue length gap, and throughput (T.P.) at time steps 2500 and 5000 for $K=2$ (middle-arrival setting: $\lambda_{\mathrm{Sprout}}=0.85$, $\lambda_{\mathrm{RouterBench}}=0.75$, $\lambda_{\mathrm{EmbedLLM}}=0.55$).}
\label{tab:k2-mid-lam5-qgap-throughput}
\scriptsize
\setlength{\tabcolsep}{2.4pt}
\renewcommand{\arraystretch}{1.15}
\resizebox{\textwidth}{!}{%
\begin{tabular}{lllrrrrrrrrrrrrr}
\toprule
Dataset & Time & Stat.
& ACQB-CL & ACQB & RAND & Q-ThS & Q-UCB & CQB-$\v\veps$ & RAND\_ROUT & ZERO & KNN & MLP & UMR & RouterDC & CARROT \\
\midrule
Sprout & 2500 & $\mu$
& \second{0.1} & \best{0.0} & 1.0 & - & - & 0.2 & - & - & - & - & - & - & - \\
& & $\sigma$
& 0.3 & 0.0 & 1.2 & - & - & 0.4 & - & - & - & - & - & - & - \\
& & T.P.
& \second{0.850} & \best{0.850} & 0.849 & - & - & 0.849 & - & - & - & - & - & - & - \\

& 5000 & $\mu$
& \best{0.1} & \second{0.1} & 1.5 & - & - & 0.3 & - & - & - & - & - & - & - \\
& & $\sigma$
& 0.3 & 0.3 & 1.7 & - & - & 0.7 & - & - & - & - & - & - & - \\
& & T.P.
& \second{0.850} & \best{0.850} & 0.849 & - & - & 0.850 & - & - & - & - & - & - & - \\
\midrule

EmbedLLM & 2500 & $\mu$
& \best{15.9} & \second{23.4} & 132.7 & - & - & 25.3 & - & - & - & - & - & - & - \\
& & $\sigma$
& 14.0 & 21.1 & 41.4 & - & - & 26.2 & - & - & - & - & - & - & - \\
& & T.P.
& \best{0.542} & \second{0.539} & 0.495 & - & - & 0.538 & - & - & - & - & - & - & - \\

& 5000 & $\mu$
& \best{5.7} & \second{6.4} & 257.6 & - & - & 14.2 & - & - & - & - & - & - & - \\
& & $\sigma$
& 7.2 & 12.3 & 52.3 & - & - & 27.4 & - & - & - & - & - & - & - \\
& & T.P.
& \best{0.546} & \second{0.546} & 0.496 & - & - & 0.544 & - & - & - & - & - & - & - \\
\midrule

RouterBench & 2500 & $\mu$
& 3.5 & \best{2.4} & 68.3 & - & - & \second{2.8} & - & - & - & - & - & - & - \\
& & $\sigma$
& 6.3 & 2.2 & 31.6 & - & - & 2.6 & - & - & - & - & - & - & - \\
& & T.P.
& 0.744 & \best{0.744} & 0.718 & - & - & \second{0.744} & - & - & - & - & - & - & - \\

& 5000 & $\mu$
& 2.2 & \best{1.2} & 108.4 & - & - & \second{1.5} & - & - & - & - & - & - & - \\
& & $\sigma$
& 3.6 & 2.6 & 61.8 & - & - & 3.4 & - & - & - & - & - & - & - \\
& & T.P.
& 0.746 & \best{0.747} & 0.725 & - & - & \second{0.746} & - & - & - & - & - & - & - \\

\bottomrule
\end{tabular}%
}
\end{table*}

\begin{table*}[t]
\centering
\caption{Mean ($\mu$), standard deviation ($\sigma$) of queue length gap, and throughput (T.P.) at time steps 2500 and 5000 for $K=2$ (high-arrival setting: $\lambda_{\mathrm{Sprout}}=0.95$, $\lambda_{\mathrm{RouterBench}}=0.85$, $\lambda_{\mathrm{EmbedLLM}}=0.65$).}
\label{tab:k2-high-lam5-qgap-throughput}
\scriptsize
\setlength{\tabcolsep}{2.4pt}
\renewcommand{\arraystretch}{1.15}
\resizebox{\textwidth}{!}{%
\begin{tabular}{lllrrrrrrrrrrrrr}
\toprule
Dataset & Time & Stat.
& ACQB-CL & ACQB & RAND & Q-ThS & Q-UCB & CQB-$\v\veps$ & RAND\_ROUT & ZERO & KNN & MLP & UMR & RouterDC & CARROT \\
\midrule
Sprout & 2500 & $\mu$
& \best{24.4} & \second{34.4} & 114.0 & - & - & 36.4 & - & - & - & - & - & - & - \\
& & $\sigma$
& 18.2 & 22.6 & 32.7 & - & - & 24.4 & - & - & - & - & - & - & - \\
& & T.P.
& \best{0.940} & \second{0.936} & 0.904 & - & - & 0.935 & - & - & - & - & - & - & - \\

& 5000 & $\mu$
& \best{23.6} & \second{36.2} & 230.2 & - & - & 41.3 & - & - & - & - & - & - & - \\
& & $\sigma$
& 18.5 & 35.7 & 38.0 & - & - & 32.0 & - & - & - & - & - & - & - \\
& & T.P.
& \best{0.946} & \second{0.943} & 0.904 & - & - & 0.942 & - & - & - & - & - & - & - \\
\midrule

EmbedLLM & 2500 & $\mu$
& \best{121.3} & 139.3 & 330.6 & - & - & \second{128.8} & - & - & - & - & - & - & - \\
& & $\sigma$
& 29.2 & 33.6 & 15.1 & - & - & 29.5 & - & - & - & - & - & - & - \\
& & T.P.
& \best{0.597} & 0.590 & 0.513 & - & - & \second{0.594} & - & - & - & - & - & - & - \\

& 5000 & $\mu$
& \best{158.0} & 183.2 & 654.1 & - & - & \second{164.4} & - & - & - & - & - & - & - \\
& & $\sigma$
& 27.8 & 64.2 & 40.3 & - & - & 45.1 & - & - & - & - & - & - & - \\
& & T.P.
& \best{0.615} & 0.610 & 0.516 & - & - & \second{0.614} & - & - & - & - & - & - & - \\
\midrule

RouterBench & 2500 & $\mu$
& \second{68.9} & \best{68.4} & 275.4 & - & - & 70.2 & - & - & - & - & - & - & - \\
& & $\sigma$
& 32.7 & 39.4 & 38.5 & - & - & 28.5 & - & - & - & - & - & - & - \\
& & T.P.
& \second{0.821} & \best{0.822} & 0.739 & - & - & 0.821 & - & - & - & - & - & - & - \\

& 5000 & $\mu$
& \second{113.4} & 115.2 & 557.1 & - & - & \best{109.0} & - & - & - & - & - & - & - \\
& & $\sigma$
& 45.1 & 51.5 & 78.7 & - & - & 45.9 & - & - & - & - & - & - & - \\
& & T.P.
& \second{0.826} & 0.826 & 0.738 & - & - & \best{0.827} & - & - & - & - & - & - & - \\

\bottomrule
\end{tabular}%
}
\end{table*}

\FloatBarrier

\newpage

\section{Regret analysis} \label{sec:regret}

In this section, we provide a proof sketch for \Cref{thm:regret1,thm:regret2}. The full proofs are deferred to \Cref{sec:regret_formal}.

\subsection{Proof sketch of \cref{thm:regret1}}

\para{Queue state misalignment} 
We start by addressing \emph{queue state misalignment}, a phenomenon unique to contextual queueing bandits with arbitrary contexts. To analyze the queue length difference $Q(t) - Q^*(t)$, we need to upper bound the expected departure rate gap $\mathbb{E}[D^*(i) - D(i)]$ for each round $i < t$.
If the queue states (i.e., the sets of pending queries) were identical under both policies at round $i$, we could control this gap following an optimistic rule from the standard bandit techniques---specifically, by selecting the query-assortment pair with the highest optimistic estimate of the departure rate (as in Line 7 of \Cref{alg:1}).
However, a suboptimal decision at round $i$ affects the queue state for the subsequent round. Consequently, the queue state at round $i+1$ under our policy diverges from that under the optimal policy. 
We call this phenomenon queue state misalignment. This misalignment invalidates standard bandit analysis because following the optimistic rule over the current queue does not guarantee optimism; the optimal policy's queue may contain a superior query that is currently unavailable to the learner.

\para{Coupling process, and policy switching queues} 
To address such misalignment, we employ a coupling argument and define policy switching queues. 
We construct a collection of $t$ coupled queueing processes $\{Q_i\}_{i=0}^{t-1}$, where each process $Q_i$ follows our policy $\pi$ up to round $i$ and then switches to the optimal policy $\pi^*$ from rounds $i+1$ to $t-1$. Denote the corresponding queue lengths by $\{Q(i,t)\}_{i=0}^{t-1}$. These $t$ processes are coupled through shared randomness. Specifically, all processes experience identical query arrivals. Furthermore, if the same query and the LLM assortment are selected across processes at any given round, the realized feedback (i.e., departure or retrial) is identical. A formal definition is provided in \Cref{ssec:1}.

\para{Queue length regret decomposition}
Utilizing the coupling construction, we decompose the queue length regret via the following telescoping sum: Define $\psi(i,t) := Q(i,t) - Q(i-1,t)$. Then
\begin{align*}
    R_t = \E[Q(t) - Q^*(t)] = \E[Q(t-1,t) - Q(0,t)] = \sum_{i=1}^{t-1} \E[\psi(i,t)].
\end{align*}
Consider the term $\psi(i,t)$. Since both processes $Q_i$ and $Q_{i-1}$ follow the same policy (our algorithm) up to round $i-1$, they share the same trajectory up to round $i-1$. Consequently, at the beginning of round $i$, the two processes $Q_i$ and $Q_{i-1}$ observe the same queue state, which implies that the queue state misalignment does not exist at round $i$. This alignment enables us to upper bound the expected departure rate gap at round $i$ using the optimistic rule of our algorithm. Building on this, we can decompose $\E[\psi(i,t)]$ as follows: From the result of \Cref{lem:regret_decomp}, for all $i\in[0,t-1]$,
\begin{align*}
    \E[\psi(i,t)] \quad \leq \udef{\sqrt{\E\sqbr{\br{R(x_i^*, S_i^*,\theta^*) - R(x_i,S_i,\theta^*)}^2}}}{e_i} \udef{\sqrt{\E\sqbr{\wtilde \psi(i,t)}}}{p_i}
\end{align*}
where $(x_i^*,S_i^*) = \argmax_{x\in\cal X_i, S\in\cal C} R(x,S,\theta^*)$, is the optimal choice under current queue state. The term $\wtilde \psi (i,t)$ is defined as $\mathbb{E}[\psi(i,t)\mid \mathcal{F}_i^+, \v D(i,i)=0, \v D(i-1,i)=1]$, where the conditioning represents the history $\mathcal{F}_i^+$ combined with the specific divergence event where the process $Q_{i-1}$ successfully departure a query ($D(i-1,i)=1$) while $Q_i$ fails ($D(i,i)=0$).
In this decomposition, $e_i$ quantifies the per-round departure rate gap at round $i$, representing the probability that the two processes diverge at round $i+1$. Meanwhile, $p_i$ captures the expected long-term impact of this divergence on the queue length at time $t$. Substituting this back yields the final bound $R_t \leq \sum_{i=1}^{t-1} e_i p_i$.

\para{Bounding strategy} 
We proceed to bound the regret term using Chebyshev's sum inequality (\Cref{lem:cheb}). Suppose there exist bounding sequences $\{E_i\}$ and $\{P_i\}$ such that $e_i \leq E_i$ and $p_i \leq P_i$ for all $i$. Provided that $\{E_i\}$ is non-increasing and $\{P_i\}$ is non-decreasing (with $E_i, P_i > 0$), applying the Chebyshev's sum inequality yields
\begin{align*}
    R_t \leq \sum_{i=0}^{t-1} E_i P_i \leq \frac{1}{t-1}\br{\sum_{i=0}^{t-1} E_i}\br{\sum_{i=0}^{t-1} P_i}.
\end{align*}
Finally, by demonstrating that the product of $\sum_{i=0}^{t-1} E_i$ and $\sum_{i=0}^{t-1} P_i$ is sublinear, we establish the decaying queue length regret.

\para{Construction of the monotonic sequence $E_i$} 
We first establish an upper bound on $e_i$ by constructing a monotonically decreasing sequence $\{E_i\}$.
\case 1  If random exploration is triggered at round $i$ (with probability $\lambda\eta(i-1)$), the departure rate gap is trivially bounded by $1$, i.e., $R(x_i^*, S_i^*,\theta^*) - R(x_i,S_i,\theta^*) \leq 1$. Since the exploration rate $\eta(i-1)$ decreases as $i$ increases, the expected regret contribution from this case monotonically decreases in $i$. 
\case 2 Otherwise, the agent selects $x_i$ and $S_i$ according to the optimistic rule. In this case, the Thompson Sampling analysis yields (\Cref{prop:per_round_1})
\begin{align*}
    R(x_i^*, S_i^*, \theta^*) - R(x_i,S_i, \theta^*) \leq 49\beta_{i-1} \lambda_{\min}^{-1/2}(V_{i-1}),
\end{align*}
where $\beta_l = \alpha_l \min \{\sqrt{6d\log (M t)}, ~ \sqrt{2\log(2M)} + \sqrt{6\log(Klt)}\}$ and $\lambda_{\min}(V)$ denotes the minimum eigenvalue of $V$.
Furthermore, leveraging the forced $\eta(i)$-exploration and the feature distribution assumption (\Cref{ass:iid}), we guarantee that $\mathbb{E}[\lambda_{\min}^{-1/2}(V_{i-1})]$ is monotonically decreasing in $i$.
Combining these two cases, we obtain
\begin{align*}
    e_i \leq \min \cbr{1,~ \sqrt{\underbrace{c_0 t^{-2}}_{\text{(bad)}} + \underbrace{2 \lambda \eta(i-1)}_{\case 1} + \underbrace{(49\beta_{t-1})^2  \nu(i)}_{\case 2}}}.
\end{align*}
where the term (bad) accounts for the failure of the high-probability event in the bandit analysis. $\nu(i)$ corresponds to the upper bound on $\mathbb{E}[\lambda_{\min}^{-1/2}(V_{i-1})]$, which is a decreasing function of $i$ (defined in \Cref{lem:mono_decrease}). We define $E_i$ as the right-hand side of this inequality, which forms a monotonically decreasing sequence by construction.

\para{Construction of the monotonic sequence $P_i$}

Next, we upper bound $p_i$ by constructing a monotonically increasing sequence $\{P_i\}$.
Recall that $\wtilde \psi(i,t)$ is the conditional expectation of the queue length difference $\psi(i,t)$, given that $Q_i$ fails to depart while $Q_{i-1}$ succeeds at round $i$. This event creates an one-job discrepancy in the queue states at round $i+1$.
If $Q_{i-1}$ hits queue length $0$ at any round $l\in[i+1,t]$, then the shared randomness in our coupling construction ensures that $Q_i$ also hits queue length $0$ at round $l$. From that round onward, $Q_{i-1}$ and $Q_i$ follow the same trajectory, resulting in $\psi(i,t)=0$. Therefore, the event $\{\wtilde \psi(i,t)=1\}$ implies that $Q_i$ never hits queue length $0$ over rounds $i+1,\dots,t$. Consequently, it suffices to upper bound $\wtilde\psi(i,t)$ with the probability that $Q_i$ never hits queue length $0$ over rounds $i+1,\dots,t$. 

Since $Q_i$ follows the optimal policy from round $i+1$ onward, the traffic slackness assumption (\Cref{ass:slackness}) guarantees that at each step, the agent can select an assortment with an expected departure rate exceeding the arrival rate by at least $\eps$. This induces a negative drift of $-\eps$ in the queue length. Specifically, by analyzing the probability that the queue length $Q(i, i+1)$ reduces to $0$ over the remaining $t-i-1$ rounds, we obtain the following bound: if $Q(i,i+1) \leq (t-i-1)\epsilon + 1$, then
\begin{align*}
    \wtilde \psi(i,t) \leq 2 \exp \br{-\frac{(Q(i,i+1) - (t-i-1)\eps - 1)^2}{8(t-i-1)}}
\end{align*}
where the lemma can be found in \Cref{lem:cond_exp_psi}.

It remains to bound $Q(i,i+1)$ on the right-hand side, where $Q(i,i+1)$ is the queue length after following our policy up to round $i$. Unlike the optimal policy, our policy may incur \emph{bad rounds} where the expected departure rate is not sufficiently close to the optimal, causing the queue length to grow. To address this, we carefully choose $\tau(t)$ as an upper bound on the number of such bad rounds (see \Cref{prop:small_per_round}), 
% \begin{align*}
%     \tau(t) = \max \cbr{\frac{2M(t)}{\lambda},~\frac{16\log(t)}{\lambda},~\frac{4M(t)^2}{c_1^2\lambda^2},~ \frac{256\log^2 (t)}{c_1^2\lambda^2}},
% \end{align*}
% and $M(t) = c_2 \br{\frac{d + \log(t)}{ \sigma_0^4} + \frac{8(\alpha_{t-1} + 48\beta_{t-1})^2}{ \sigma_0^2(\eps - 2\eta(\tau(t))}}$,
so that for all other rounds $l \in [\tau(t)+1, t]$, the departure rate gap is controlled as
\begin{align}
    R(x_l^*, S_l^*, \theta^*) - R(\wtilde x_l, \wtilde S_l,\theta^*) \leq \eps/2 - \eta(\tau(t)), \label{eq:drg}
\end{align}
thereby guaranteeing the negative drift and decreasing the queue length. Here, $(\wtilde x_l, \wtilde S_l)$ is a query-assortment pair with the highest optimistic departure-rate estimate. 
% \begin{remark} \label{remark:tau}
%     \citet{bae2026queuelengthregretbounds} also derive a departure rate gap bound of $\eps/2 - \eta$ similar to \Cref{eq:drg}, where $\eta = T^{-1/2}$. Their method, however, necessitates a pure-exploration phase at the beginning, calculating the required exploration amount based on the horizon $T$ and slackness $\eps$. In contrast, we employ a time-varying exploration rate $\eta(t)$ to bypass this dependence. We show that setting the threshold $\tau(t)$ as defined in \Cref{eq:tau} ensures that the accumulated random exploration is sufficient to satisfy the departure rate upper bound in \Cref{eq:drg}, thereby achieving an anytime guarantee without prior knowledge.
% \end{remark}
\Cref{eq:drg} allows us to derive a tail bound on $Q(i, i+1)$ with $\cal B(i)$ the number of bad rounds up to round $i$ (\Cref{lem:q_tail})
\begin{align*}
    \P(Q(i,i+1) \geq a\cal B(i) + b,~\cal E_g) \leq 17\eps^{-2} \exp(-b\eps/2),
\end{align*}
where $\cal E_g$ denotes the event where the bandit analysis holds. $a$ and $b$ are parameters satisfying $a\geq 1+\eps/4$ and $b\geq 0$.
Combining the upper bound on $\wtilde \psi(i,t)$ with the tail bound on $Q(i, i+1)$ yields the final result: Let $\omega = 4\tau(t)/\eps$. Then, 
\begin{align*}
    p_i \leq \min \cbr{1,~\sqrt{\frac{c_0}{t^2} + \frac{19}{\eps^2} \exp\br{-\frac{\eps^2}{32} \br{t-i-1 - \omega}}}}, 
\end{align*}
where the lemma can be found in \Cref{lem:tilde_psi}. We denote the right-hand side as $P_i$. Since the term $-(t-i-1-\omega)$ increases as $i$ increases, $P_i$ is monotonically increasing in $i$.

\para{Completing the proof}
We have constructed sequences $\{E_i\}$ and $\{P_i\}$ that upper bound $e_i$ and $p_i$, respectively, and satisfy the conditions required for Chebyshev's sum inequality. In \Cref{ssec:7}, we demonstrate that the cumulative per-round regret shows $\sum_{i=1}^{t-1} E_i = \tildeO{t^{3/4}}$, while the cumulative long-term impact shows $\sum_{i=1}^{t-1} P_i = \tildeO{1}$. Consequently, applying the inequality (which introduces a $t^{-1}$ factor) yields a decaying queue length regret of $\wtilde{\mathcal{O}}(t^{-1/4})$.

\subsection{Proof sketch of \cref{thm:regret2}}

Recall the definition of the cumulative regret as $\text{Regret}_t = \sum_{i=1}^t \E\sqbr{R(x_i, S_i^*, \theta^*) - R(x_i, S_i, \theta^*)}$. Since the optimal pair $S_i^*$ is selected from the current queue state $\mathcal{X}_i$ of the agent (i.e., $S_i^* = \argmax_{ S \in\mathcal{C}} R(x_i, S, \theta^*)$), the departure-rate comparison is performed within the same queue state. Consequently, unlike the queue length regret analysis, this definition allows us to avoid queue state misalignment, allowing us to employ standard bandit analysis techniques.

At each round $i$, the agent performs random exploration with probability $\eta(i-1)$ (indicated by $E(i-1)=1$) and otherwise follows the Thompson sampling-based optimistic rule (with probability $1-\eta(i-1)$). Since the per-round regret is bounded by $1$, the expected cumulative regret from random exploration is bounded by $\sum_{i=1}^t \eta(i-1) = \bigO{\sqrt{t}}$. For rounds using the optimistic choice, we follow the analysis for MNL-bandits, which yields a regret bound of $\tildeO{d^{3/2}\sqrt{t}}$. Combining these two components, we obtain an overall cumulative regret bound of $\tildeO{d^{3/2}\sqrt{t}}$. A detailed proof is provided in \Cref{ssec:8}.

\section{Formal regret analysis} \label{sec:regret_formal}

In this section, following \Cref{sec:regret}, we introduce the formal proof for the queue length regret \Cref{thm:regret1}, and for the cumulative regret \Cref{thm:regret2}. The outline of this section is as follows:

\begin{enumerate}
    \item \Cref{ssec:1}:~~ Coupling Process and Policy-Switching Queues
    \item \Cref{ssec:2}:~~ Queue Length Regret Decomposition
    \item \Cref{ssec:3}:~~ Events
    \item \Cref{ssec:4}:~~ Monotonic Decrease in Per-Round Regret
    \item \Cref{ssec:5}:~~ How to Set Bad Rounds $\tau(t)$
    \item \Cref{ssec:6}:~~ Monotonic Increase in the Effect of Queue Length Difference
    \item \Cref{ssec:7}:~~ Queue Length Regret Analysis (Proof of \Cref{thm:regret1})
    \item \Cref{ssec:8}:~~ Cumulative Regret Analysis (Proof of \Cref{thm:regret2})
\end{enumerate}

\subsection{Coupling process and policy-switching queues} \label{ssec:1}

We use the following definition of the policy switching queue, where $\wtilde Q(t',t'')$ denotes the length of the queue at the beginning of time step $t'' $ under our policy applied from time steps from $1$ to $t'$ and the optimal policy applied from $t'+1$ to $t''-1$. We denote such a queueing process and the corresponding policy as $\wtilde Q_{t'}$ and $\wtilde \pi_{t'}$. By definition $\wtilde Q(t'-1,t')=Q(t')$ and $\wtilde Q(0,t'')=Q^*(t'')$. Then, we can see that the queue length regret can be decomposed as the telescoping sum of the difference between two queueing processes as $R_t = \sum_{i=1}^{t-1}\E[\wtilde Q(i,t)- \wtilde Q(i-1,t)]$.

Next, we construct a coupling process for each $\{\wtilde Q_{i}\}_{i=0}^{t-1}$. We denote these as $\{Q_i\}_{i=0}^{t-1}$ and its corresponding queue length and switching policy as $\{Q(i,t)\}_{i=0}^{t-1}$ and $\{\pi_i\}_{i=0}^{t-1}$. First, for the arrival, random exploration, and parameter sampling (which will appear in Thompson sampling), we assume that all processes $\{Q_i\}_{i=0}^{t-1}$ share the same results of randomness, e.g. if in round $t$ a new job $x^{(t)}$ arrives in $Q_{i}$, then every other $Q_{i'}$ also has a same event. We conveniently reuse the same notation of $A(t)$, $x_t$, $x^{(t)}$, $E(t)$, $\{\wtilde\theta_t^{(i)}\}_{i=1}^M$ for these coupling processes.

Now,  for the departure, we use the definition of $D(i,t')$ for the random departure of $Q_i$ in round $t'\in[t]$. In each round $t'$, we draw a shared random variable $U\sim\Unif(0,1)$. Say that for some $i$ ,$t'$, $Q_i$ choose $x_t'$ and $S_t'$. Assume that $S_t'$ is sorted in order of the servers, $S_t'[j]$ defines the $j$-th item of $S_t'$. Then
\begin{align*}
    &\text{if~~} U \leq \frac{\exp(x_{t'S_{t'}[1]}\tp\theta^*)}{1 + \sum_{j'\in S_{t'}} \exp(x_{t'j'}\tp \theta^*)}, \text{ then } ~ D(i,t') = 1 \text{ and } y_{t1}=1  \\
    &\text{if~~} \sum_{k=1}^{j-1}\frac{\exp(x_{t'S[k]}\tp\theta^*)}{1 + \sum_{j'\in S_{t'}} \exp(x_{t'j'}\tp \theta^*)} < U \leq \sum_{k=1}^{j}\frac{\exp(x_{t'S[k]}\tp\theta^*)}{1 + \sum_{j'\in S_{t'}} \exp(x_{t'j'}\tp \theta^*)}\text{ for }  j\in [2,K],\\
    &\qquad \text{then } ~ D(i,t') = 1 \text{ and } y_{tj}=1  \\
    &\text{if~~} U > R(x_{t'}, S_{t'}, \theta^*), \text{ then } ~ D(i,t') = 0 \text{ and } y_{t0}=1
\end{align*}
We can see that these coupling processes preserves the marginals, therefore,
\begin{align*}
    R_t = \sum_{i=1}^{t-1}\E[\wtilde Q(i,t)- \wtilde Q(i-1,t)] = \sum_{i=1}^{t-1}\E[\udef{Q(i,t)- Q(i-1,t)}{\psi(i,t)}],
\end{align*}
and we continue to use this definition afterwards.

Now we introduce the filtration corresponding to these coupling processes of $\{\pi_i\}_{i=0}^{t-1}$: We define the arrival tuple $\v A(t)$, the departure tuple $\v D(t) = \{\v D(i, t)\}_{i=0}^{t-1}$, and the sampled parameter tuple $\v\theta(t)$ (which comes from the Thompson sampling algorithm) as
\begin{align*}
    \v A(t): = (A(t),\wtilde x^{(t)}), \quad \v D(t):= \{\v D(i, t)\}_{i=0}^{t-1} =  \{(y_t^{(i)}, (x_t^{(i)}, S_t^{(i)}))\}_{i=1}^{t-1}, \quad \v\theta(t) := (\{\wtilde \theta_t^{(i)}\}_{i=1}^M).
\end{align*}
Here, $\widetilde x^{(t)}$ is a masked feature defined as $\widetilde x^{(t)} = \widetilde x$ if $A(t)=0$ where $\widetilde x \in\R^d$ is a fixed symbol for the sign of no arrival, and $\widetilde x^{(t)} = x^{(t)}$ if $A(t)=1$. $\v D_i(t)$ denotes to the departure tuple of $\pi_i$, and $(y_t^{(i)}, (x_t^{(i)}, S_t^{(i)}))$ denotes its corresponding choice of context, assortment, and reward in round $t$. We omit the superscript of $\cdot^{(i)}$ if the situation is clear. $\{\wtilde \theta_t^{(i)}\}_{i=1}^M$ is $M$ sampled parameters in round $t$. Then, we define the filtration as
\begin{align*}
    \cal F_t := \sigma (\cal X_1, \v A(1), \v D(1), \v \theta(1), E(1),\dots, \v A(t-1), \v D(t-1), \v \theta(t-1)).
\end{align*}
Notice that $\cal F_t$ is short of $E(t-1)$. Also, for the notational convenience, if we write $\v D(i,t') = 0$ (or $\v D(i,t') = 1$), it means the observations of $D(i,t')=0$ (or $D(i,t')=1$) including $y_{t'}^{(i)}$, $x_{t'}^{(i)}$, and  $S_{t'}^{(i)}$.

Finally, for the augmented $\sigma$-algebra $\cal G$ with the filtration $\cal F$ and the random variable $X$, which is $\cal G:=\cal F \lor \sigma(X)$, we use the shorthand notation of $\cal F,X$, e.g. $\E[Z\mid \cal F, X]$ denotes the conditional expectation of $Z$ with respect to $\cal F \lor \sigma(X)$ and $\E[Z\mid \cal F, X=1]$ denotes the conditional expectation conditioning on the event $\{X=1\}$ in addition to $\cal F$.

\subsection{Queue length regret decomposition} \label{ssec:2}

Now we focus on this quantity $\psi(i,t):= Q(i,t)- Q(i-1,t)$. As the two coupled queues follow the same randomness and the same policy up to round $t-1$, their queue states at time step $t$ are identical. With this alignment, the following lemma characterize $\psi(i,t)$:
\begin{lemma} \label{lem:psi}
    We have $\psi(i,t)\in\{-1,0,1\}$ for all $i\in[0,t-1]$. Especially, if $D(i,t)=D(i-1,t)=0$ or $1$, we have $\psi(i,t)\in\{-1,0\}$, and if $D(i,t)=0$, $D(i-1,t)=1$, we have $\psi(i,t)\in\{0,1\}$ for all $i\in[0,t-1]$. 
\end{lemma}
Moreover, the expected value of $\psi(i,t)$ can be decomposed even further as follows:
\begin{lemma} \label{lem:regret_decomp}
Let $x_l^*, S_l^* = \argmax_{x\in\cal X_l, S \in \cal C} R(x,S,\theta^*)$. For all $l\in[t]$, $i\in[0,t-1]$, define
\begin{align*}
    \cal F_l^+ &:= \cal F_l \lor  \sigma(E(l-1), \v A(l), \v\theta(l)),  \\
    \wtilde \psi (i,t) &:= \E[\psi(i,t)\mid \cal F_i^+, \v D(i,i)=0, \v D(i-1,i)=1].
\end{align*}
Then, we have
\begin{align*}
    \E[\psi(i,t)] \leq \sqrt{\E\sqbr{\br{R(x_i^*, S_i^*,\theta^*) - R(x_i,S_i,\theta^*)}^2}} \sqrt{\E\sqbr{\wtilde \psi(i,t)}}.
\end{align*}    
\end{lemma}

% \subsection{Useful lemmas for Thompson sampling in MNL contextual bandits}

\subsection{Events} \label{ssec:3}

We use the following definition of confidence radius: for all $l\in[t]$
\begin{align*}
    \alpha_l &= \frac{\kappa}{2} \sqrt{d \log \br{1 + lK / (d\lambda_0)} + 4 \log l} + \kappa \sqrt{\lambda_0} \\
    \beta_l &= \alpha_l \min \cbr{\sqrt{6d\log (M t)}, ~ \sqrt{2\log(2M)} + \sqrt{6\log(Klt)}}
\end{align*}
where $M = \lceil 1 - \frac{\log (K)}{\log(1-1/(4\sqrt{e\pi}))} \rceil$. We use the following definition of the events: for all $l\in[t]$
\begin{align*}
    &\cal E_0(l) := \{A(l-1)=1, E(l-1)=1\} \\
    &\cal E_1 := \{\forall l\in[t],~ \|\what\theta_{l-1} - \theta^*\|_{V_{l-1}} \leq \alpha_{l-1}\} \\
    &\cal E_2(l) := \{\forall x\in\cal X_l, j\in[N],~ \wtilde u_{lj}(x) - x_{-j}\tp \what\theta_{l-1} \leq \beta_{l-1} \|x_{-j}\|_{V_{l-1}^{-1}}\} \\
    &\cal E_3 := \cbr{\lambda_{\min}(V_{\tau(t)}) \geq \frac{4(\alpha_{t-1} + 48\beta_{t-1})^2}{(\eps - 2\eta(\tau(t))}}
\end{align*}
By definition, $\P(\cal E_0(l))=\P(A(l)=1) \P(E(l-1)=1) =\lambda \eta(l-1)$. By \Cref{lem:alpha,lem:beta}, we have $\P(\cal E_1) = 1 - \bigO{1/t^2}$ and $\P(\cal E_2(l)) = 1 - \bigO{1/t^3}$. Denote $\cal E_2 := \bigcap_{l=1}^t\cal E_2(l)$. For $\cal E_2(l)$, taking union bound for $l\in[t]$, we have $\P(\cal E_2) = 1 - \bigO{1/t^2}$. By \Cref{prop:min_eig}, $\P(\cal E_3) \geq 1 -\bigO{1/t^2}$. Finally, for simplicity, we denote 
\begin{align}
    \cal E_g := \cal E_1 \cap \cal E_2 \cap \cal E_3, \quad \P(\cal E_g) \geq  1-c_0 t^{-2} \label{eq:good_event}
\end{align}
for some absolute constant $c_0>0$.

% {\color{red}
% For round $t$, we denote a 'bad round' phase $\tau(t)$ as
% \begin{align*}
%     &M(t) = c_2 \br{\frac{d + \log(1/\delta)}{ \sigma_0^4} + \frac{8(\alpha_t + 48\beta_t)^2}{ \sigma_0^2(\eps - 2\eta(\tau(t)))^2}} \\
%     &\tau(t) = \max\cbr{M(t)^2\frac{4}{c_1^2 \lambda^2},~ M(t)^2 \br{ \frac{c_1\lambda^2 + \sqrt{2 \log(1/\delta) - 15c_1^2 \lambda^4}}{8 c_1^2 \lambda^4 - \log(1/\delta)}}^2}.
% \end{align*}
% }

We introduce new definitions for 'bad round' phase $\tau(t)$:
\begin{equation}
\begin{aligned} \label{eq:tau}
    &\wtilde \tau(M) := \max \cbr{\frac{2M}{\lambda},~\frac{16\log(t)}{\lambda},~\frac{4M^2}{c_1^2\lambda^2},~ \frac{256\log^2 (t)}{c_1^2\lambda^2}}\\
    &\wtilde M(u) := c_2 \br{\frac{d + \log(t)}{ \sigma_0^4} + \frac{8(\alpha_{t-1} + 48\beta_{t-1})^2}{ \sigma_0^2(\eps - 2\eta(u))^2}} \\
    &\tau(t):= \min\cbr{u\in\N:~ u \geq \wtilde \tau(\wtilde M(u))} \\
    & M(t) := \wtilde M(\tau(t))
\end{aligned}
\end{equation}
Notice that these new definitions are not required in the algorithm and only appear in the analysis.

\subsection{Monotonic decrease in per-round regret} \label{ssec:4}

We show the monotonic decrease in per-round regret in the following sequence:
\begin{enumerate}
    \item Upper bound the \emph{expected} (squared) per-round regret with bonus term
    \item We show that our $\eta(l)$-exploration policy can monotonically decrease the bonus term, thereby decrease the upper bound of the expected squared per-round regret 
\end{enumerate}

Define new filtration $\cal F_t^-$ which shorts of $\v\theta(t-1)$ as follows:
\begin{align*}
    \cal F_l^- := \sigma (\cal X_1, \v A(1), \v D(1), \v \theta(1), E(1),\dots, \v A(l-1), \v D(l-1)).
\end{align*}
The following proposition shows the upper bound of the per-round regret and another upper bound with the minium eigenvalue of the design matrix $V_{l-1}$:

\begin{proposition} [Per-round regret] \label{prop:per_round_1}
    For all $l\in[t]$, define $x_l^*, S_l^* = \argmax_{x\in\cal X_l, S\in\cal C} R(x, S, \theta^*)$. On the event $\cal E_1, \cal E_2$ and $\cal E_0(l)^c$. Then, we have
    \begin{align*}
        R(x_l^*, S_l^*, \theta^*) - R(x_l,S_l, \theta^*) \leq 16\sqrt{e\pi} \beta_{l-1} \E \sqbr{\max_{j\in S(x_l, \wtilde\theta_{l-1}^{1:M})} \|x_{lj}\|_{V_{l-1}^{-1}}\given \cal F_l^-} + (\alpha_{l-1}+\beta_{l-1}) \max_{j\in S_l}  \|x_{lj}\|_{V_{l-1}^{-1}} 
    \end{align*}
    Also, we have
    \begin{align*}
        R(x_l^*, S_l^*, \theta^*) - R(x_l,S_l, \theta^*) \leq (\alpha_{l-1} + 48\beta_{l-1}) \lambda_{\min}^{-1/2}(V_{l-1}).
    \end{align*}
\end{proposition}

% The following proposition shows that the upper bound of the \emph{expected} squared per-round regret:
% \begin{proposition} \label{prop:per_round_2}
%     For all $i\in[t]$, we have,
%     \begin{align*}
%         \E\sqbr{(R(x_i^*, S_i^*,\theta^*) - R(x_i,S_i,\theta^*))^2}\leq \bigO{t^{-2}} + \lambda \eta(i-1) + \E\sqbr{ (\alpha_{i-1} + 48\beta_{i-1})^2 \lambda_{\min}^{-1}(V_{i-1}) \given \cal E_0(i)^c, \cal E'(i)}.
%     \end{align*}
% \end{proposition}

\Cref{prop:per_round_1} shows that the upper bound of the per-round regret can be decreased by increasing the minimum eigenvalue of the design matrix $V_l$. However the upper bound inequality only holds when $\cal E_0(l)^c$. Based on this, the following lemma shows that the $\eta(l)$-exploration policy can handle this problem and monotonically decrease the upper bound of the \emph{expected} squared per-round regret:

\begin{lemma} \label{lem:mono_decrease}
    We have
    \begin{align*}
        \sqrt{\E \sqbr{(R(x_l^*, S_l^*, \theta^*) - R(x_l, S_l, \theta^*))^2}} \leq \min \cbr{1,~ \sqrt{c_0 t^{-2} + 2 \lambda \eta(l-1) + (\alpha_{t-1} + 48\beta_{t-1})^2  \nu(l)}}
    \end{align*}
    where
    \begin{align*}
        \nu(l)&:= \br{\lambda_0 + \frac{ \sigma_0^2 \lambda \min\{(l-1),c_1\sqrt{l-1}\}}{4}}^{-1} \\
    &\quad + \frac{1}{\lambda_0} \br{\exp\br{-\frac{\lambda \min\{(l-1),c_1\sqrt{l-1}\} }{8}} +  d \exp\br{-\frac{ \sigma_0^2 \lambda \min\{(l-1),c_1\sqrt{l-1}\} }{16}}}.
    \end{align*}
\end{lemma}

\subsection{How to set bad rounds $\tau(t)$} \label{ssec:5}

We set $\tau(t)$ so that the uncertainty term $\|x_{lj}\|_{V_{l-1}^{-1}}$ gets small enough to guarantee 'good round' for round $l\in[\tau(t)+1,t]$. We deferred the definition of the 'good round' to a later section.

First, denote $N(l)$ as the total number of random exploration up to round $l$. Then we have the following lemma showing the high probability lower bound of the random exploration until a certain round:

\begin{lemma} \label{lem:tau_m}
    For some $M>0$, and for $\wtilde\tau(M)$ satisfying
    \begin{align*}
        &\wtilde \tau(M) \geq \max \cbr{\frac{2M}{\lambda},~\frac{16\log(t)}{\lambda},~\frac{4M^2}{c_1^2\lambda^2},~ \frac{256\log^2 (t)}{c_1^2\lambda^2}},
    \end{align*}
    we have $\P(N(\wtilde \tau (M)) \geq M) \geq 1-t^{-2}$.
\end{lemma}

The following proposition shows the high probability lower bound of the minimum eigenvalue of the design matrix in round $\tau(t)$, which implies the minimum eigenvalue afterwards will also be larger than that value:

\begin{proposition} \label{prop:min_eig}
    With probability at least $1-2t^{-2}$, we have
    \begin{align*}
        \lambda_{\min}(V_{\tau(t)}) \geq \frac{4(\alpha_{t-1} + 48\beta_{t-1})^2}{(\eps - 2\eta(\tau(t)))}.
    \end{align*}
\end{proposition}

The following proposition shows that after round $\tau(t)$, the per-round regret in round that does not explore is small:

\begin{proposition} \label{prop:small_per_round}
    Consider $l\in[\tau(t)+1,t]$. On the event $\cal E_g$ and $\cal E_0(l)^c$. Then with probability at least $1-2\delta$, 
    \begin{align*}
        R(x_{l}^*, S_{l}^*, \theta^*) - R(x_{l}, S_{l}, \theta^*) \leq \frac{\eps}{2} - \eta(\tau(t)).
    \end{align*}
\end{proposition}

\subsection{Monotonic increase in the effect of queue length difference} \label{ssec:6}

\para{Bad rounds} 

First, we introduce the following definition: for all $l\in[t]$
\begin{align*}
    \wtilde x_l, \wtilde S_l := \argmax_{x\in\cal X_l, S\in\cal C} \wtilde R(x, S)
\end{align*}
We can see that $\wtilde x_l$ and $\wtilde S_l$ are both $\cal F_l$-measurable since $\cal F_t$ is only short for $E(t-1)$, therefore $\cal X_t$ and $\{\wtilde\theta_{t-1}^{(i)}\}_{i=1}^M$ is still $\cal F_l$-measurable, which include sufficient information to compute $\wtilde x_l$ and $\wtilde S_l$. Also, if $\cal E_0(l)^c$, we have $\wtilde x_l = x_l$ and $\wtilde S_l = S_l$.

% Next, we introduce the formal definition of bad rounds:
% \begin{align*}
%     \cal B:= \cbr{\forall l\in[t]: ~ R(x_l^*, S_l^*, \theta^*) - R(\wtilde x_l, \wtilde S_l, \theta^*) \leq \frac{\eps}{2} - \eta(\tau(t))}, \quad \cal B(l) := |\cbr{\cal B \cap [l]}|
% \end{align*}
% By the definition, $\cal B(l)$ is also $\cal F_l$-measurable. We denote a 'good round' if it is not a bad round. For brevity, we use the notation of $l\notin \cal B$ to denote all good rounds.

Next, we introduce the formal definition of bad rounds and good rounds: We denote all rounds up to $\tau(t)$ as bad rounds and $\tau(t)+1$ to $t$ as good rounds. Define the set of bad rounds as $\cal B$. Then by the direct result of \Cref{prop:small_per_round}, on the event $\cal E_g$ and $\cal E_0(l)^c$, for all $l\notin\cal B$, we have
\begin{align}
    R(x_l^*, S_l^*, \theta^*) - R(\wtilde x_l, \wtilde S_l,\theta^*) =  R(x_l^*, S_l^*, \theta^*) - R(x_l, S_l,\theta^*) \leq \frac{\eps}{2} - \eta(\tau(t)), \label{eq:good_1}
\end{align}
where we use the notation of $l\notin\cal B$ to denote that $l$ is in good round.

\begin{remark} \label{remark:tau}
    \citet{bae2026queuelengthregretbounds} also derive a departure rate gap bound of $\eps/2 - \eta$ similar to \Cref{eq:good_1}, where $\eta = T^{-1/2}$. Their method, however, necessitates a pure-exploration phase at the beginning, calculating the required exploration amount based on the horizon $T$ and slackness $\eps$. In contrast, we employ a time-varying exploration rate $\eta(t)$ to bypass this dependence. We show that setting the threshold $\tau(t)$ as defined in \Cref{eq:tau} ensures that the accumulated random exploration is sufficient to satisfy the departure rate upper bound in \Cref{eq:good_1}, thereby achieving an anytime guarantee without prior knowledge.
\end{remark}

Next, we establish that for good rounds there is a negative drift as described in the following lemma:
\begin{lemma} \label{lem:drift}
    On the event $\cal E_g$, for all $l\notin\cal B$, $i\in[0,t-1]$, we have
    \begin{align*}
        \E[A(l) - D(i, l) \mid \cal F_l] \leq -\eps/2.
    \end{align*}
\end{lemma}

\para{Queue length difference under disagreement} 

In this paragraph, we give the upper bound for the expected queue length difference between two consecutive policy-switching queues, which is $\psi(i,t)$. Consider $i\in[0,t-1]$. Recall the definition of $\cal F_l^+$, which is $\cal F_i^+:= \sigma(\cal F_i \cup\{E(i-1),\v A(i), \v\theta(i)\})$. Then, the following lemma shows the upper bound of the conditional expected value of $\psi(i,t)$ is related to the queue length in $i+1$ and the remaining round $t-i+1$: Define a new event
\begin{align}
    \cal E_5(i):= \{Q(i,i+1) \leq (t-i-1)\eps + 1\}. \label{eq:event_5}
\end{align}
\begin{lemma} \label{lem:cond_exp_psi}
    For all $i\in[0,t-1]$, on the event $\cal E_5(i)$, we have,
    \begin{align*}
        \E\sqbr{\psi(i,t)\mid \cal F_i^+, \v D(i,i)=0, \v D(i-1,i)=1} \leq 2 \exp \br{-\frac{(Q(i,i+1) - (t-i-l)\eps - 1)^2}{8(t-i-1)}}.
    \end{align*}
\end{lemma}

\para{Tail bound for $Q(i,i+1)$}

The result of \Cref{lem:cond_exp_psi} shows that the queue length difference $\psi(i,t)$ under the disagreement event can be upper bounded by the exponential term of remaining rounds $(t-i-1)$ (which suits our goal to upper bound the queue length difference with the exponential ramp) and the queue length $Q(i,i+1)$. Therefore, in this paragraph, we control the value of $Q(i,i+1)$ by giving the exponential tail bound for it.

We use the definition of $\cal B(l)$ which denotes the total number of bad rounds up to round $l$. By our definition of bad rounds, for all $l$, we have
\begin{align*}
    \cal B(l):= \min\{l, \tau(t)\}
\end{align*}

\begin{lemma} \label{lem:q_tail}
    Set $\gamma\in(0,\eps/2]$, $\rho = \exp (-\gamma\eps/4)$, $\beta=\exp(\gamma)$, and some $a\geq \frac{1}{\gamma} \log\br{\frac{\beta}{\rho}} =  1 + \eps/4$ and $b\geq 0$. For all $l\in[t]$, we have
    \begin{align*}
        \P(Q(i,l) \geq a\cal B(l-1) + b,~\cal E_g) \leq \br{\rho^{l-1} \E[\exp(\gamma Q(i,1))] + \frac{1}{1-\rho}} \exp(-\gamma b).
    \end{align*}
    If we assume the initial queue starts with an empty state of $Q(1)=0$, and set $\gamma=\eps/2$, $\rho=\exp(-\eps^2/8)$, then we simplify the result as
    \begin{align*}
        \P(Q(i,l) \geq a\cal B(l-1) + b,~\cal E_g) \leq 17\eps^{-2} \exp(-\gamma b).
    \end{align*}
\end{lemma}

\para{Main lemma}

Now we are ready to introduce the main lemma:
\begin{lemma} \label{lem:tilde_psi}
    Let $\omega := 4\tau(t)/\eps$. Then, we have
    \begin{align*}
        \sqrt{\E \sqbr{\wtilde \psi(i,t)}} \leq
        \begin{cases}
            \min \cbr{1,~\sqrt{c_0t^{-2} + 19 \eps^{-2} \exp\br{-\frac{\eps^2}{32} \br{t-i-1 - \omega}}}}  & \text{if~~} i\leq t - \omega - 1 \\
            1 & \text{if~~} i> t - \omega - 1
        \end{cases}.
    \end{align*}
\end{lemma}

\subsection{Queue length regret analysis} \label{ssec:7}

We first specify the conditions on $t$ required for \Cref{thm:regret1}.

\begin{condition} \label{cond:t}
    \Cref{thm:regret1} holds for any $t$ sufficiently large satisfying:
    \begin{align}
        &t\geq \tau(t),\quad t\geq \omega \label{eq:cond1} \\
        &\eps > 2\eta(\tau(t)) \label{eq:cond2}
    \end{align}
\end{condition}

\begin{remark} 
\Cref{eq:cond1} is required to ensure that the number of bad rounds $\tau(t)$ and the threshold round $\omega=4\tau(t)/\eps$ do not exceed the current time step $t$. Since $4\tau(t)/\eps \geq \tau(t)$, it suffices to consider $t\geq 4\tau(t)/\eps$.
\Cref{eq:cond2} is required to guarantee the positivity of the denominator in the lower bound for $\lambda_{\min}(V_{\tau(t)})$ derived in \Cref{prop:min_eig}:
\begin{align*}
    \lambda_{\min}(V_{\tau(t)}) \geq \frac{4(\alpha_{t-1} + 48\beta_{t-1})^2}{(\eps - 2\eta(\tau(t)))}.
\end{align*}

We now verify the validity of these conditions for large $t$. Recall the definition of $\tau(t)$:
\begin{align*}
    \tau(t) =  \max \cbr{\frac{2M(t)}{\lambda},~\frac{16\log(t)}{\lambda},~\frac{4M(t)^2}{c_1^2\lambda^2},~ \frac{256\log^2 (t)}{c_1^2\lambda^2}},
\end{align*}
where $M(t)$ is given by
\begin{align*}
    M(t) = c_2 \br{\frac{d + \log(t)}{ \sigma_0^4} + \frac{8(\alpha_{t-1} + 48\beta_{t-1})^2}{ \sigma_0^2(\eps - 2\eta(\tau(t))}}.
\end{align*}
Since $\alpha_{t-1}=\cal O(\log^{1/2}t)$ and $\beta_{t-1} = \cal O(\log t)$, it follows that $M(t) = \cal O(\log^2t)$ and consequently $\tau(t) = \cal O(\log^4t)$.
Therefore, \Cref{eq:cond1} holds for sufficiently large $t$, as linear growth dominates polylogarithmic growth (i.e., $t > \cal O(\log^4 t)$).
Regarding \Cref{eq:cond2}, recall that $\eta(t) = \min \{1,~ c_1 (t+1)^{-1/2}\}$. Since $\tau(t) = \cal O(\log^4t)$ grows with $t$, $\eta(\tau(t))$ decays to zero. Thus, for sufficiently large $t$, the condition $\eps > 2\eta(\tau(t))$ is satisfied, ensuring \Cref{eq:cond2} holds.
Finally, we conclude that for any large $t$ satisfying \Cref{cond:t}, \Cref{thm:regret1} holds.
\end{remark}

\begin{theorem}
We have a decaying queue length regret of
\begin{align*}
    R_t = \bigO{\frac{d^5t^{-1/4}\log^5(t)}{\sigma_0^9 \eps^5} + \frac{d^{11/2}t^{-1}\log^5(t)}{\sigma_0^{12}\eps^5}}.
\end{align*}
\end{theorem}

\begin{proof}
We start with the result of the regret decomposition in \Cref{lem:regret_decomp} and prepare to apply Chebyshev sum inequality (\Cref{lem:cheb}):
\begin{align*}
    R_t = \sum_{i=1}^{t-1} \E[\psi(i,t )] = \sum_{i=1}^{t-1} \sqrt{\E\sqbr{\br{R(x_i^*, S_i^*,\theta^*) - R(x_i,S_i,\theta^*)}^2}} \sqrt{\E\sqbr{\wtilde \psi(i,t)}}
\end{align*}
We show that the upper bound of each sequence, $\cbr{\sqrt{\E\sqbr{\br{R(x_i^*, S_i^*,\theta^*) - R(x_i,S_i,\theta^*)}^2}}}_{i=1}^{t-1}$ and $\cbr{\sqrt{\E\sqbr{\wtilde \psi(i,t)}}}_{i=1}^{t-1}$, show monotonical behavior in opposite direction.

First, denote $M_i := \min \cbr{1,~ \sqrt{c_0 t^{-2} + 2 \lambda \eta(i-1) + (\alpha_{t-1} + 48\beta_{t-1})^2  \nu(i)}}$. By the direct result of \Cref{lem:mono_decrease}, $\cbr{\sqrt{\E\sqbr{\br{R(x_i^*, S_i^*,\theta^*) - R(x_i,S_i,\theta^*)}^2}}}_{i=1}^{t-1}$ can be upper bounded with $\cbr{M_i}_{i=1}^{t-1}$.

Next, denote $\Delta_i := \min \cbr{1,~\sqrt{c_0t^{-2} + 19 \eps^{-2} \exp\br{-\frac{\eps^2}{32} \br{t-i-1 - \omega}}}}$ if $i\leq t - \omega -1 $ and $\Delta_i := 1$ if $i> t- \omega - 1$. Then, by the direct result of \Cref{lem:tilde_psi}, $\cbr{\sqrt{\E\sqbr{\wtilde \psi(i,t)}}}_{i=1}^{t-1}$ can be upper bounded with $\cbr{\Delta_i}_{i=1}^{t-1}$.

Since $\cbr{M_i}_{i=1}^{t-1}$ are monotonically decreasing in $i$ and $\cbr{\Delta_i}_{i=1}^{t-1}$ are monotonically increasing in $i$, and both sequences are always positive, we can apply Chebyshev sum inequality, which gives
\begin{align}
    R_t \leq \frac{1}{t-1} \br{\sum_{i=1}^{t-1} M_i} \br{\sum_{i=1}^{t-1} \Delta_i}. \label{eq:md_sum}
\end{align}
Now, for the summation of $(M_i)$, 
\begin{align*}
    \sum_{i=1}^{t-1} M_i &\leq \sum_{i=1}^{t-1} \sqrt{c_0 t^{-2} + 2 \lambda \eta(i-1) + (\alpha_{t-1} + 48\beta_{t-1})^2  \nu(i)} \\
    &\leq \uterm{\sum_{i=1}^{t-1} \sqrt{c_0} t^{-1}}{1} + \uterm{\sum_{i=1}^{t-1} \sqrt{ 2 \lambda \eta(i-1)}}{2} + \uterm{\sum_{i=1}^{t-1} \sqrt{(\alpha_{t-1} + 48\beta_{t-1})^2 \nu(i)}}{3} 
\end{align*}
For \term 1, we have $\term 1\leq \sqrt{c_0}$. For \term 2,
\begin{align*}
    \term 2 \leq \sqrt{2\lambda c_1}\sum_{i=1}^{t-1}  i^{-1/4} \leq \sqrt{2\lambda c_1} \br{1 + \int_{1}^{t-1} x^{-1/4} dx} \leq \frac{4}{3} \sqrt{2\lambda c_1} t^{3/4}.
\end{align*}
For \term 3, we ignore the $(\alpha_{t-1} + 48\beta_{t-1})$ part and consider the $\sqrt{\nu(i)}$ part:
\begin{align*}
    \sum_{i=1}^{t-1} \sqrt{\nu(i)} &\leq \sum_{i=1}^{t-1}\br{\lambda_0 + \frac{ \sigma_0^2 \lambda \min\{(i-1),c_1\sqrt{i-1}\}}{4}}^{-1/2} \\
    &\quad + \sum_{i=1}^{t-1}\frac{1}{\sqrt{\lambda_0}} \exp\br{-\frac{\lambda \min\{(i-1),c_1\sqrt{i-1}\} }{16}} +  \sum_{i=1}^{t-1}\frac{\sqrt{d}}{\sqrt{\lambda_0}} \exp\br{-\frac{ \sigma_0^2 \lambda \min\{(i-1),c_1\sqrt{i-1}\} }{32}} \\
    &\leq \uterm{\sum_{i=1}^{t-1}\br{\lambda_0 + \frac{ \sigma_0^2 \lambda (i-1)}{4}}^{-1/2}}{4-1} + \uterm{\sum_{i=1}^{t-1}\frac{1}{\sqrt{\lambda_0}} \exp\br{-\frac{\lambda (i-1) }{16}}}{5-1} +  \uterm{\sum_{i=1}^{t-1}\frac{\sqrt{d}}{\sqrt{\lambda_0}} \exp\br{-\frac{ \sigma_0^2 \lambda (i-1) }{32}}}{6-1} \\
    &\quad + \uterm{\sum_{i=1}^{t-1}\br{\lambda_0 + \frac{ \sigma_0^2 \lambda c_1\sqrt{i-1}}{4}}^{-1/2}}{4-2} + \uterm{\sum_{i=1}^{t-1}\frac{1}{\sqrt{\lambda_0}} \exp\br{-\frac{\lambda c_1\sqrt{i-1} }{16}}}{5-2} +  \uterm{\sum_{i=1}^{t-1}\frac{\sqrt{d}}{\sqrt{\lambda_0}} \exp\br{-\frac{ \sigma_0^2 \lambda c_1\sqrt{i-1} }{32}}}{6-2}
\end{align*}
For \term{4-1}, \term{4-2}, \term{5-2}, and \term{6-2}, we upper bound the summations via an integral comparison for monotone decreasing functions. For \term{5-1} and \term{6-1}, we upper bound the finite summations by the corresponding infinite geometric series:
\begin{align*}
    &\term{4-1} \leq \frac{1}{\sqrt{\lambda_0}} + \frac{8}{\sigma_0^2 \lambda}\sqrt{\lambda_0 + \frac{\sigma_0^2 \lambda}{4}t}, \quad \term{4-2} \leq \frac{1}{\sqrt{\lambda_0}} + \frac{8}{3\sigma_0\sqrt{\lambda c_1}} t^{3/4} \\
    &\term{5-1} \leq \frac{1}{\sqrt{\lambda_0}} \frac{1}{1-\exp(-\lambda/16)}\leq\frac{1}{\sqrt{\lambda_0}}\br{1+\frac{16}{\lambda}}, \quad \term{5-2}\leq \frac{1}{\sqrt{\lambda_0}}\br{1+\frac{512}{\lambda^2 c_1^2}} \\
    &\term{6-1} \leq \frac{\sqrt{d}}{\sqrt{\lambda_0}} \frac{1}{1- \exp(-\sigma_0^2\lambda/32)}\leq \frac{\sqrt{d}}{\sqrt{\lambda_0}} \br{1+\frac{32}{\sigma_0^2 \lambda}}, \quad \term{6-2} \leq \frac{\sqrt{d}}{\sqrt{\lambda_0}}\br{1+ \frac{2048}{\sigma_0^4 \lambda^2 c_1^2}}
\end{align*}
For \term{5-1} and \term{6-1} we used the fact that $\frac{1}{1-\exp(-x)} \leq 1 + \frac{1}{x}$. Summing up results for \term 3, we have
\begin{align*}
    \term 3 &\leq (\alpha_{t-1} + 48\beta_{t-1}) \bigO{\frac{t^{3/4}}{\sigma_0} + \frac{d^{1/2}}{\sigma_0^4 }} \\
    &=\bigO{\frac{dt^{3/4}\log(t)}{\sigma_0} + \frac{d^{3/2}\log(t)}{\sigma_0^4}} \tag{$\alpha_t=\bigO{\sqrt{d\log(t)}}$, $\beta_t=\bigO{d\log(t)}$}
\end{align*}
Finally, substituting results of \term 1, \term 2, and \term 3 back to the original inequality, 
\begin{align}
    \sum_{i=1}^{t-1} M_i &= \bigO{1} + \bigO{t^{3/4}} + \bigO{\frac{dt^{3/4}\log (t)}{\sigma_0} + \frac{d^{3/2}\log(t)}{\sigma_0^4}} \nonumber \\
    &= \bigO{\frac{dt^{3/4}\log (t)}{\sigma_0} + \frac{d^{3/2}\log(t)}{\sigma_0^4}}. \label{eq:m_sum}
\end{align}

Now, for the summation of $(\Delta_i)$, recall the definition of the threshold $\omega$ from \Cref{lem:tilde_psi}, where we use the naive upper bound of $\Delta_i\leq 1$ when the remaining rounds are smaller than $\omega$, and if not, used an exponentially increasing function to upper bound $\Delta_i$. Therefore, 
\begin{align*}
    \sum_{i=1}^{t-1} \Delta_i &= \uterm{\sum_{i=1}^{t-\omega-1} \Delta_i}{6} + \uterm{\sum_{i=t-\omega}^{t-1} \Delta_i}{7} 
\end{align*}
For \term 6,
\begin{align*}
    \term 6 &\leq \sum_{i=1}^{t-\omega - 1} \sqbr{\sqrt{c_0} t^{-1} + 5 \eps^{-1}\exp \br{- \frac{\eps^2}{64} \br{t-i-1 - \omega}}} \\
    & \leq c_0 + \frac{5}{\eps} \times \frac{1}{1 - \exp\br{-\eps^2/64}} \tag{geometric series sum} \\
    &\leq c_0 + \frac{640}{\eps^3} \tag{$1-\exp(-x) \geq x/2$, $x\in(0,1]$}
\end{align*}
For \term 7, we have $\term 7 \leq \omega$. Since $M(t)=\bigO{\frac{d + \log(t)}{\sigma_0^4} + \frac{d^2 \log^2(t)}{\sigma_0^2 \eps^2}}=\bigO{\frac{d^2 \log^2(t)}{\sigma_0^4 \eps^2}}$, we have
\begin{align*}
    \term 7 &\leq \omega = \frac{4\tau(t)}{\eps} = \bigO{\frac{M(t)^2}{\eps}} = \bigO{\frac{d^4\log^4(t)}{\sigma_0^8\eps^5}}.
\end{align*}
Substituting results, we have
\begin{align}
    \sum_{i=1}^{t-1} \Delta_i &= \bigO{\frac{1}{\eps^3}} + \bigO{\frac{d^4\log^4(t)}{\sigma_0^8\eps^5}} = \bigO{\frac{d^4\log^4(t)}{\sigma_0^8\eps^5}}. \label{eq:d_sum}
\end{align}
Finally, plugging \Cref{eq:m_sum,eq:d_sum} back in to \Cref{eq:md_sum}, we have
\begin{align*}
    R_t &\leq \frac{1}{t-1} \br{\bigO{\frac{dt^{3/4}\log (t)}{\sigma_0} + \frac{d^{3/2}\log(t)}{\sigma_0^4}}} \br{\bigO{\frac{d^4\log^4(t)}{\sigma_0^8\eps^5}}} \\
    &= \bigO{\frac{d^5t^{-1/4}\log^5(t)}{\sigma_0^9 \eps^5} + \frac{d^{11/2}t^{-1}\log^5(t)}{\sigma_0^{12}\eps^5}},
\end{align*}
finishing the proof.
\end{proof}

\subsection{Cumulative regret analysis} \label{ssec:8}

\begin{remark}
We analyze a stronger regret definition. Denote the previous definition of the regret as $\text{Regret}_t^{\text{prev}}$. Let $(x_i^*, S_i^*) = \argmax_{x\in\mathcal{X}_i, S \in\mathcal{C}} R(x, S, \theta^*)$ be the optimal pair chosen from the queue state $\mathcal{X}_i$. Define
\begin{align*}
\text{Regret}_t := \sum_{i=1}^t \mathbb{E}\sqbr{R(x_i^*, S_i^*, \theta^*) - R(x_i, S_i, \theta^*)}.
\end{align*}
Since $\text{Regret}_t^{\text{prev}} \leq \text{Regret}_t$ holds by definition, establishing a bound on $\text{Regret}_t$ suffices.
\end{remark}

\begin{theorem} \label{thm:regret}
    We have a cumulative regret of
    \begin{align*}
        \text{Regret}_t = \tildeO{d^{3/2} \sqrt{t}}
    \end{align*}
\end{theorem}

\begin{proof}
We start with the definition of the cumulative regret:
\begin{align*}
    \text{Regret}_t &= \sum_{i=1}^t \E\sqbr{R(x_i^*, S_i^*, \theta^*) - R(x_i, S_i, \theta^*)}  \\
    &= \sum_{i=1}^t \E\sqbr{\ind{\cal E_g^c}(R(x_i^*, S_i^*, \theta^*) - R(x_i, S_i, \theta^*))} + \sum_{i=1}^t \E\sqbr{\ind{\cal E_g}(R(x_i^*, S_i^*, \theta^*) - R(x_i, S_i, \theta^*))} \\
    &\leq c_0 t^{-1} + \sum_{i=1}^t \E\sqbr{\ind{\cal E_g}(R(x_i^*, S_i^*, \theta^*) - R(x_i, S_i, \theta^*))} \tag{\Cref{eq:good_event}, $R(\cdot,\cdot,\cdot)-R(\cdot,\cdot,\cdot)\leq 1$}
\end{align*}
Now, for the second term on the right-hand side,
\begin{align*}
    &\sum_{i=1}^t \E\sqbr{\ind{\cal E_g}(R(x_i^*, S_i^*, \theta^*) - R(x_i, S_i, \theta^*))} \\
    &\leq \sum_{i=1}^t \E\sqbr{\ind{\cal E_g^{\leq i}}(R(x_i^*, S_i^*, \theta^*) - R(x_i, S_i, \theta^*))} \tag{$\cal E_g^{\leq i} \subseteq \cal E_g$}\\
    &= \sum_{i=1}^t \E\sqbr{\E\sqbr{\ind{\cal E_g^{\leq i}}(R(x_i^*, S_i^*, \theta^*) - R(x_i, S_i, \theta^*))\given \cal F_i, E(i-1)}} \tag{tower rule} \\
    &= \uterm{\sum_{i=1}^t \E\sqbr{\E\sqbr{\ind{\cal E_g^{\leq i}}\ind{E(i-1)=1}(R(x_i^*, S_i^*, \theta^*) - R(x_i, S_i, \theta^*))\given \cal F_i, E(i-1)}}}{1} \\
    &\quad + \uterm{\sum_{i=1}^t \E\sqbr{\E\sqbr{\ind{\cal E_g^{\leq i}}\ind{E(i-1)=0}(R(x_i^*, S_i^*, \theta^*) - R(x_i, S_i, \theta^*))\given \cal F_i, E(i-1)}}}{2} 
\end{align*}
For \term 1, we have
\begin{align*}
    \term 1 \leq \sum_{i=1}^t \P(E(i-1)=1) = \sum_{i=1}^t \eta(i-1) \leq c_1 \sum_{i=1}^t i^{-1/2} \leq c_1\br{1+ \int_{1}^t x^{-1/2} dx} \leq 2 c_1 \sqrt{t}.
\end{align*}
For \term 2, on the event $E_g^{\leq i}$ and $E(i-1)=0$, we can directly apply the result of \Cref{prop:per_round_1}, which gives
\begin{align*}
    \term 2 &\leq \uterm{\sum_{i=1}^t \E\sqbr{\ind{\cal E_g^{\leq i}}\ind{E(i-1)=0}\br{16\sqrt{e\pi} \beta_{i-1} \E \sqbr{\max_{j\in S(x_i, \wtilde\theta_{i-1}^{1:M})} \|x_{ij}\|_{V_{i-1}^{-1}}\given \cal F_i^-}}}}{3} \\
    &\quad + \uterm{\sum_{i=1}^t \E\sqbr{\ind{\cal E_g^{\leq i}}\ind{E(i-1)=0} \br{(\alpha_{i-1}+\beta_{i-1}) \max_{j\in S_i}  \|x_{ij}\|_{V_{i-1}^{-1}}}}}{4}
\end{align*}
For \term 3, let us ignore the constant part with $16\sqrt{e\pi}\beta_{t-1}$ (as $\beta_{i-1}$ is monotonically increasing in $i$) and only consider the summation of the expectation part inside the outer expectation term. Then we can proceed as
\begin{align*}
    & \sum_{i=1}^t \E \sqbr{\max_{j\in S(x_i, \wtilde\theta_{i-1}^{1:M})} \|x_{ij}\|_{V_{i-1}^{-1}}\given \cal F_i^-} \\
    &\quad = \uterm{\sum_{i=1}^t \max_{j\in S_i} \|x_{ij}\|_{V_{i-1}^{-1}}}{5} +  \uterm{\sum_{i=1}^t \br{\E \sqbr{\max_{j\in S(x_i, \wtilde\theta_{i-1}^{1:M})} \|x_{ij}\|_{V_{i-1}^{-1}}\given \cal F_i^-} - \max_{j\in S_i} \|x_{ij}\|_{V_{i-1}^{-1}}}}{6}
\end{align*}
For \term 5, we have
\begin{align*}
    \term 5 \leq \sqrt{t \sum_{i=1}^t \max_{j\in S_i} \|x_{ij}\|_{V_{i-1}^{-1}}^2} \leq \sqrt{2dt\log\br{1+\frac{tK}{d\lambda_0}}},
\end{align*}
where the first inequality follows from Cauchy-Schwarz inequality and the last inequity follows from \Cref{lem:6_oh}.

For \term{6}, we prepare to apply Azuma-Hoeffding inequality (\Cref{lem:azuma}). By construction, $Y_l := \sum_{i=1}^l \br{\E \sqbr{\max_{j\in S(x_i, \wtilde\theta_{i-1}^{1:M})} \|x_{ij}\|_{V_{i-1}^{-1}}\given \cal F_i^-} - \max_{j\in S_i} \|x_{ij}\|_{V_{i-1}^{-1}}}$ is a martingale. Next, we have
\begin{align*}
    |Y_l - Y_{l-1} | = \abs{\E \sqbr{\max_{j\in S(x_l, \wtilde\theta_{l-1}^{1:M})} \|x_{lj}\|_{V_{l-1}^{-1}}\given \cal F_l^-} - \max_{j\in S_l} \|x_{lj}\|_{V_{l-1}^{-1}}} \leq \frac{2 \|x\|_2}{\lambda_{\min}^{-1/2}(V_{l-1})} \leq \frac{2}{\sqrt{\lambda_0}}.
\end{align*}
Therefore, by setting $c_t\leftarrow \frac{2}{\sqrt{\lambda}0}$ and $a \leftarrow \sqrt{\frac{8t}{\lambda_0} \log (2t)}$, and applying Azuma-Hoeffding inequality, we have
\begin{align*}
    \term{6} \leq \sqrt{\frac{8t}{\lambda_0} \log(2t)}
\end{align*}
with probability at least $1-1/t$. Also, if the Azuma-Hoeffding inequality does not hold,  we can use the naive bound of 
\begin{align*}
    \term{6} \leq \sum_{i=1}^t \abs{\E \sqbr{\max_{j\in S(x_l, \wtilde\theta_{l-1}^{1:M})} \|x_{lj}\|_{V_{l-1}^{-1}}\given \cal F_l^-} - \max_{j\in S_l} \|x_{lj}\|_{V_{l-1}^{-1}}} \leq \frac{2t}{\sqrt{\lambda_0}}
\end{align*}
Define the event $\cal E_6$ as the Azuma-Hoeffding inequality holds. Then, substituting results back to \term 3, we have
\begin{align*}
    \term 3 &\leq 16\sqrt{e\pi}\beta_{t-1} \br{ \sqrt{2dt\log\br{1+\frac{tK}{d\lambda_0}}} + \P(\cal E_5)\sqrt{\frac{8t}{\lambda_0} \log(2t)} + \P(\cal E_5^c) \frac{2t}{\sqrt{\lambda_0}}} \\
    &\leq 16\sqrt{e\pi}\beta_{t-1} \br{ \sqrt{2dt\log\br{1+\frac{tK}{d\lambda_0}}} + \sqrt{\frac{8t}{\lambda_0} \log(2t)} +  \frac{2}{\sqrt{\lambda_0}}}
\end{align*}
For \term 4, similarly, we can apply Cauchy-Schwarz inequality and \Cref{lem:6_oh}, which gives
\begin{align*}
    \term 4 \leq (\alpha_{t-1} + \beta_{t-1}) \sum_{i=1}^t \max_{j\in S_i}  \|x_{ij}\|_{V_{i-1}^{-1}} \leq (\alpha_{t-1} + \beta_{t-1}) \sqrt{2dt\log\br{1 + \frac{tk}{d\lambda_0}}}.
\end{align*}
Finally, substituting \term 1, \term 2, \term 3, and \term 4 back to the original inequality, we have
\begin{align*}
    \text{Regret}_t &\leq c_0t^{-1} + 2c_1\sqrt{t} + 16\sqrt{e\pi}\beta_{t-1} \br{ \sqrt{2dt\log\br{1+\frac{tK}{d\lambda_0}}} + \sqrt{\frac{8t}{\lambda_0} \log(2t)} +  \frac{2}{\sqrt{\lambda_0}}} \\
    &\quad + (\alpha_{t-1} + \beta_{t-1}) \sqrt{2dt\log\br{1 + \frac{tk}{d\lambda_0}}} \\
    &= \tildeO{d^{3/2} \sqrt{t}}, \tag{$\alpha_{t-1}=\bigO{\sqrt{d\log(t)}}$, $\beta_{t-1}=\bigO{d\log(t)}$}
\end{align*}
finishing the proof.
\end{proof}

\section{Deferred proofs for \cref{ssec:2}}

\subsection{Proof of \cref{lem:psi}}

To characterize $\psi(t,T)= Q(i,t)-Q(i-1,t)$, we understand the dynamics of the coupled queues of $Q_i$ and $Q_{i-1}$. Notice that in $l\in[i+1,t]$, both $Q_i$ and $Q_{i-1}$ follows the same optimal policy. Therefore, to track down the queue length difference, we don't need to consider about the context of the job inside the queue $x$ itself; instead, we only need to consider the optimal departure rate of the corresponding job $\max_{S\in\cal C} R(x,S,\theta^*)$. For example, if there are two different job but with the same departure rate, i.e. for $x_1\neq x_2$ and $\max_{S\in\cal C}R(x_1, S, \theta^*) = \max_{S\in\cal C}R(x_2, S, \theta^*)$, we don't have to distinguish whether the optimal policy choose $x_1$ or $x_2$ since the optimal policy does not involve the learning procedure and the queue length will evolve same by the coupling process. Accordingly, we introduce the new definitions: Define $\cal X(i, t')$ for the queue state of $Q_i$ in round $t'$. Then we define the departure rate set of the queue state as
\begin{align*}
    \cal D(i, t') := \cbr{d(x), ~\forall x \in\cal X(i, t'):~ d(x)=\max_{S\in\cal C} R(x,S,\theta^*)}.
\end{align*} 
For $l\in[i+1,t]$, let us consider the following five states:
\begin{align*}
    S_{l,0} &= \{\cal D(i,l) = \cal D(i-1,l)\} \\
    S_{l,1} &= \{\cal D(i,l) \setminus \cal D(i-1,l) = \{d_l^+\},~ \cal D(i,l) \supset \cal D(i-1,l)\} \\
    S_{l,2} &= \{\cal D(i-1,l) \setminus \cal D(i,l) = \{d_l^-\},~ \cal D(i-1,l) \supset \cal D(i,l)\} \\
    S_{l,3} &= \{\cal D(i,l) \setminus \cal D(i-1,l) = \{d_l^+\},~ \cal D(i-1,l) \setminus \cal D(i,l) = \{d_l^-\},~ d_l^+ > d_l^-\} \\
    S_{l,4} &= \{\cal D(i,l) \setminus \cal D(i-1,l) = \{d_l^+\},~ \cal D(i-1,l) \setminus \cal D(i,l) = \{d_l^-\},~ d_l^+ < d_l^-\}.
\end{align*}
Now we proceed with the proof as follows:
\begin{enumerate}
    \item We show that under optimal policy, $S_{l,s'}$ only transits to $S_{l+1,s''}$ where $s',s''\in\{0,1,2,3,4\}$.
    \item We show that in round $i+1$, the state is inside $S_{l,s'}$, thereby at the end of round $t$, it will still resides inside those five states, which implies $\psi(i,t)\in\{-1,0,1\}$.
\end{enumerate}

First, we consider each case $S_{l,s'}$ for $s'\in\{0,1,2,3,4\}$ 
\begin{enumerate}
    \item[\case 1] If round $l$ is in state $S_{l,0}$, then $\cal D(i,l)=\cal D(i-1,l)$, so the optimal policy would choose the same job. As a result, round $l+1$ would be in state $S_{l+1,0}$.

    \item[\case 2] If round $l$ is in state $S_{l,1}$, it falls into the following two cases, based on whether the optimal policy chooses $d_l^+$ for $\cal D(i,l)$.
    \begin{enumerate}
        \item[\case{2-1}] If the optimal policy selects $d_i^+$ for $\cal D(i,l)$, this means that $D(i,i) \geq D(i-1,i)$. Therefore, there are three possibilities. When $D(i,i)=D(i,i)=0$, as the queues keep the same sets of for round $l+1$, we have $S_{l+1,1}$ for round $l+1$. If $D(i,i)=D(i,i-1)=1$, as the optimal policy would choose another job in $\cal D(i,i-1)$, we still have state $S_{l+1,1}$ for round $l+1$. When $D(i,i)=1$ and $D(i,i-1)=0$, round $i+1$ would be in state $S_{l+1,0}$. 

        \item[\case{2-2}] If the optimal policy does not choose $d_l^+$ from $\cal D(i,l)$, then the same for $\cal D(i,l)$ and $\cal D(i-1,l)$ will be chosen, which means that round $l+1$ would be in state $S_{l+1,1}$.
    \end{enumerate}

    \item[\case 3] If round $l$ is in state $S_{l,2}$, by the symmetry between $S_{l,1}$ and $S_{l,2}$, we can argue that we have $S_{l+1,0}$ or $S_{l+1,2}$ in round $l+1$ with a similar argument as in case 2.

    \item[\case 4] If round $l$ is in state $S_{l,4}$, it falls into the following two cases, based on whether the optimal policy chooses $d_t^+$ for $\cal D(i,l)$. 
    \begin{enumerate}
        \item[\case {4-1}] If the optimal policy selects $d_l^+$ for $Q(i,l)$, this means that $D(i,l)\geq D(i-1,l)$. Therefore, there are three possibilities. When $D(i,l)=D(i-1,l)=0$, as the queues keep the same sets of jobs for round $l+1$, we have $S_{l+1,4}$ for round $l+1$. If $D(i,l)=D(i-1,l)=1$, the optimal policy would choose another job in $\cal D(i-1,l)$. If the optimal policy chooses $d_l^-$ from $\cal D(i-1,l)$, we have state $S_{l+1,0}$ in round $l+1$. If not, round $l+1$ would be in state $S_{1,4}$. When $D(i,l)=1$ and $D(i-1,l)=0$, round $l+1$ would be in state $S_{l+1,2}$.

        \item[\case {4-2}] If the optimal policy does not choose $d_l^+$ for $\cal D(i, l)$, then it would not choose $d_l^-$ from $\cal D(i-1,l)$ either. Hence, the optimal policy chooses the same job for $\cal D(i,l)$ and $\cal D(i-1,l)$, so round $l+1$ would be in state $S_{l+1,3}$.
    \end{enumerate}
    In summary, for \case 4, we have $S_{l+1,0}$ or $S_{l+1,2}$ or $S_{l+1,3}$ in round $l+1$. 

    \item[\case 5] If round $i$ is in state $S_{l,4}$, by the symmetry between $S_{l,3}$ and $S_{l,4}$, we may argue that we have $S_{l+1,0}$ or $S_{l+1,1}$ or $S_{l+1,4}$ in round $l+1$ with a similar argument as in \case 4.
\end{enumerate}

The results above show that for $l\in[i+1,t]$, $S_{l,s'}$ only transits to $S_{l+1,s''}$ where $s',s''\in\{0,1,2,3,4\}$.

Now, we consider the state in round $i+1$. Since $Q(i,l)$ and $Q(i-1,l)$ are coupled and follow the same policy, $\cal D(i,i) = \cal D(i-1,i)$. In round $i$, $Q_i$ follows our policy and $Q_{i-1}$ follow the optimal policy, which implies $D(i,i)\leq D(i-1,i)$ and we split this in two cases:

\begin{enumerate}
    \item[\case{1'}] If $D(i,i)=0$ and $D(i-1,i)=1$, then round $i+1$ would be in state $S_{i+1,1}$. Due to our case analysis above, we have $S_{i+2,0}$ or $S_{i+2,1}$ for round $i+2$. If the state of round $i+2$ is $S_{i+2,0}$, then we have $S_{i',0}$ for each round $i'\geq i+2$, in which case $\psi(i,t)=0$. If we have $S_{i+2,1}$ for round $i+2$, we repeat the same argument as for state $i+1$. If we observe $S_{t,1}$ for round $t$, then we have $\psi(i,t)=1$. Otherwise, round $t$ would be in state $S_{t,0}$, in which case $\psi(i,t)=0$.

    \item[\case{2'}] If $D(i,i)=D(i-1,i)$, then round $i+1$ would be in state $S_{i+1,0}$ or $S_{i+1,3}$. By our case analysis above, we have $S_{i+2,0}$ or $S_{i+2,2}$ or $S_{i+2,3}$ for round $i+2$. If the state of round $i+2$ is $S_{i+2,0}$, then we have $S_{i',0}$ for each round $i'\geq i+2$, in which case $\psi(i,t)=0$. If we have $S_{i+2,2}$ for round $i+2$, we have $S_{i+3,0}$ or $S_{i+3,2}$ for round $i+3$. If $S_{i+3,0}$ is the state of round $i+3$, then as before, we deduce $\psi(i,t)=0$. If the state is $S_{i+3,2}$, we repeat the same argument as for state $i+2$. If we observe $S_{t,2}$ for round $t$, then we have $\psi(i,t)=-1$. Otherwise, round $t$ would be in state $S_{t,0}$, in which case $\psi(i,t)=0$. If we observe $S_{i+2,3}$ for round $i+2$, then we again repeat the argument as for round $i+1$ to argue that $\psi(i,t)\in\{0,-1\}$.
\end{enumerate}

This finishes the proof.

\subsection{Proof of \cref{lem:regret_decomp}}

Recall the definition of filtration given by
\begin{align*}
    \cal F_l^+ &:= \sigma(\cal F_l \cup \{E(l-1), \v A(l), \v\theta(l)\})
\end{align*}
Notice that $x_l$, $S_l$ are $\cal F_l^+$-measurable.

By the queue length regret decomposition $\R_t = \sum_{i=1}^{t-1} \E[\psi(i,t)]$, we deduce that
\begin{align*}
    R_t &= \sum_{i=1}^{t-1} \E[\E[\psi(i,t)\mid \cal F_i^+]] \\
    &= \sum_{i=1}^{t-1} \E[\P(D(i,i)=D(i-1,i) \mid \cal F_i^+)] \E[\psi(i,t)\mid \cal F_i^+, D(i,i)=D(i-1,i)] \\
    &\quad + \sum_{i=1}^{t-1} \E[\P(D(i,i)=0, D(i-1,i)=1 \mid \cal F_i^+)] \E[\psi(i,t)\mid \cal F_i^+, D(i,i)=0, D(i-1,i)=1],
\end{align*}
where the first equality holds due to the tower rule and the second equality holds since $D(i,i)\leq D(i-1,i)$ by our coupling process. For the first part of the right-hand side, it follows from \Cref{lem:psi} that
\begin{align*}
    \E[\psi(i,t)\mid \cal F_t^+, D(i,i)=D(i-1,i)]\leq 0.
\end{align*}
Next, for the second part, notice that the departure disagreement event with $D(i,i)=0$, $D(i-1,i)=1$ occurs when
\begin{align*}
    R(x_i^{(i)}, S_i^{(i)}, \theta^*) \leq U \leq R(x_i^{(i-1)}, S_i^{(i-1)}, \theta^*), \quad U\sim\Unif(0,1).
\end{align*}
Moreover, as the queue state of both $Q_i$ and $Q_{i-1}$ in round $i$ is identical to $\cal X_i$, we know that
\begin{align*}
    x_i^{(i)} = x_t,\quad S_i^{(i)} = S_i, \quad x_i^{(i-1)} = x_i^*, \quad S_i^{(i-1)} = S_t^*
\end{align*}
where $x_i^*, S_i^* = \argmax_{x\in\cal X_i, S\in\cal C} R(x, S,\theta^*)$. Therefore, we have
\begin{align*}
    \P(D(i,i)=0, D(i-1,i)=1 \mid \cal F_t^+) = R(x_i^*, S_i^*, \theta^*) - R(x_i, S_i, \theta^*).
\end{align*}
Plugging in these results to the above decomposition of $R_T$, we obtain
\begin{align*}
    R_T &\leq \sum_{i=1}^{t-1} \E\sqbr{\br{R(x_i^*, S_i^*, \theta^*) - R(x_i, S_i, \theta^*)} \E[\psi(i,t)\mid \cal F_t^+, D(i,i)=0, D(i-1,i)=1]} \\
    &\leq \sum_{i=1}^{t-1} \sqrt{\E\sqbr{\br{R(x_i^*, S_i^*, \theta^*) - R(x_i, S_i, \theta^*)}^2}} \sqrt{\E\sqbr{\E\sqbr{\psi(i,t)\mid \cal F_t^+, D(i,i)=0, D(i-1,i)=1}^2}},
\end{align*}
where the second inequality follows from the Cauchy-Schwarz inequality. For the second square root term,
\begin{align*}
    &\E\sqbr{\E\sqbr{\psi(i,t)\mid \cal F_t^+, D(i,i)=0, D(i-1,i)=1}^2} \\
    &\quad \leq \E\sqbr{\E\sqbr{\psi(i,t)^2\mid \cal F_t^+, D(i,i)=0, D(i-1,i)=1}} \tag{$\E[X\mid \cal F]^2 \leq \E[X^2\mid \cal F]$} \\
    &\quad = \E\sqbr{\E\sqbr{\psi(i,t)\mid \cal F_t^+, D(i,i)=0, D(i-1,i)=1}} \tag{\Cref{lem:psi}} \\
    &\quad = \E\sqbr{\wtilde\psi(i,t)},
\end{align*}
where the last inequality follows from the definition. This finishes the proof.

\section{Deferred proofs for \cref{ssec:4}}

\subsection{Proof of \cref{prop:per_round_1}}

The per-round regret can be decomposed as 
\begin{align*}
    &R(x_l^*, S_l^*, \theta^*) - R(x_l,S_l, \theta^*) \\
    &\quad =\uterm{R(x_l^*, S_l^*, \theta^*) - \wtilde R(x_l, S_l)}{1} + \uterm{\wtilde R(x_l, S_l) - R(x_l,S_l,\what\theta_{l-1})}{2} + \uterm{R(x_l,S_l,\what\theta_{l-1}) - R(x_l,S_l, \theta^*)}{3}
\end{align*}

Since $\cal E_1, \cal E_2$ holds and $\cal E_0(l)^c$, we can apply \Cref{lem:12_oh} for \term 1:
\begin{align*}
    \term 1 \leq  16\sqrt{e\pi} \beta_{l-1} \E \sqbr{\max_{j\in S(x_l, \wtilde\theta_{l-1}^{1:M})} \|x_{lj}\|_{V_{l-1}^{-1}}\given \cal F_l^-}.
\end{align*}
For \term 2, applying \Cref{lem:lip,lem:beta}, we have
\begin{align*}
    \term 2 &\leq \max_{j\in S_l} |\wtilde u_{lj}(x_l) - x_{lj}\tp \what\theta_{l-1}| \leq \max_{j\in S_l}  \beta_{l-1} \|x_{lj}\|_{V_{l-1}^{-1}}. 
\end{align*}
For \term 3, applying \Cref{lem:lip,lem:alpha}, we have
\begin{align*}
    \term 3 & \leq \max_{j\in S_l} |x_{lj}\tp \what\theta_{l-1} - x_{lj}\tp \theta^*| \leq \max_{j\in S_l} \alpha_{l-1} \|x_{lj}\|_{V_{l-1}^{-1}}.
\end{align*}
Substituting results, and taking the union bound, we have the desired result.

Also, notice that $16\sqrt{e \pi} \leq 47$, and for all $x\in\cal X$,
\begin{align*}
    \|x_{-j}\|_{V_{l-1}^{-1}} \leq \frac{\|x\|_2}{\lambda_{\min}^{1/2}(V_{l-1})} \leq \lambda_{\min}^{-1/2}(V_{l-1}).
\end{align*}
Therefore, 
\begin{align*}
    R(x_l^*, S_l^*, \theta^*) - R(x_l,S_l, \theta^*) \leq (\alpha_{l-1} + 48\beta_{l-1}) \lambda_{\min}^{-1/2}(V_{l-1}),
\end{align*}
finishing the proof.

\subsection{Proof of \cref{lem:mono_decrease}}

Before starting the proof, we introduce a technical result first. By our $\eta(t)$-exploration policy and \Cref{ass:iid}, we deduce the following lemma:

\begin{lemma} \label{lem:min_eig_2}
    For all $l\in[t]$, we have
    \begin{align*}
        \P\br{\lambda_{\min}(V_l) \geq \lambda_0 + \frac{ \sigma_0^2 \lambda \min\{l,c_1\sqrt{l}\}}{4}} \geq 1 - \exp\br{-\frac{\lambda \min\{l,c_1\sqrt{l}\} }{8}} -  d \exp\br{-\frac{ \sigma_0^2 \lambda \min\{l,c_1\sqrt{l}\} }{16}}.
    \end{align*}
\end{lemma}

Now we are ready to start the proof. Consider $l\in [t]$. Recall the notation of $\cal E_g := \cal E_1 \cap \cal E_2 \cap \cal E_3$. Denote $\cal F_l':=\cal F_l \lor \sigma(E(t-1))$ for simplicity. Notice that $\ind{\cal E_0(l)}$ is $\cal F_l'$-measurable. We can ignore for $\cal E_3$. We start by decomposing the expected squared per-round regret as
\begin{align*}
    &\E \sqbr{(R(x_l^*, S_l^*, \theta^*) - R(x_l, S_l, \theta^*))^2} \\
    &\quad = \E \sqbr{\E \sqbr{(R(x_l^*, S_l^*, \theta^*) - R(x_l, S_l, \theta^*))^2 \given \cal F_l'} } \tag{tower rule} \\
    &\quad = \E \sqbr{\ind{\cal E_g^c} \E \sqbr{(R(x_l^*, S_l^*, \theta^*) - R(x_l, S_l, \theta^*))^2 \given \cal F_l'} } \\
    &\qquad + \E \sqbr{\ind{\cal E_g} \E \sqbr{(R(x_l^*, S_l^*, \theta^*) - R(x_l, S_l, \theta^*))^2 \given \cal F_l'} } \\
    &\quad \leq  \P(\cal E_g^c) + \E \sqbr{\ind{\cal E_g} \E \sqbr{(R(x_l^*, S_l^*, \theta^*) - R(x_l, S_l, \theta^*))^2 \given \cal F_l'} } \tag{$R(\cdot,\cdot,\cdot) - R(\cdot,\cdot,\cdot) \leq 1$}\\
    &\quad = c_0 t^{-2} + \E \sqbr{\ind{\cal E_g} \E \sqbr{(R(x_l^*, S_l^*, \theta^*) - R(x_l, S_l, \theta^*))^2 \given \cal F_l'} } \tag{\Cref{eq:good_event}}.
\end{align*}
For the last term, we can proceed as
\begin{align*}
    &\E \sqbr{\ind{\cal E_g} \E \sqbr{(R(x_l^*, S_l^*, \theta^*) - R(x_l, S_l, \theta^*))^2 \given \cal F_l'} } \\
    &\quad = \uterm{\E \sqbr{\ind{\cal E_g} \ind{\cal E_0(l)} \E \sqbr{(R(x_l^*, S_l^*, \theta^*) - R(x_l, S_l, \theta^*))^2 \given \cal F_l'} }}{1} \\
    &\qquad + \uterm{\E \sqbr{\ind{\cal E_g} \ind{\cal E_0(l)^c} \E \sqbr{(R(x_l^*, S_l^*, \theta^*) - R(x_l, S_l, \theta^*))^2 \given \cal F_l'} }}{2} 
\end{align*}
For \term 1, we have
\begin{align*}
    \term 1 \leq \E \sqbr{\ind{\cal E_0(l)} } = \P(\cal E_0(l)) = \lambda \eta(t-1).
\end{align*}
For \term 2, under $\cal E_0^c(l)$ and $\cal E_g$, we can directly apply \Cref{prop:per_round_1}, which yields
\begin{align*}
    \term 2 &\leq \E \sqbr{ \ind{\cal E_g, \cal E_0(l)^c } \E\sqbr{ (\alpha_{l-1} + 48\beta_{l-1})^2 \lambda_{\min}^{-1}(V_{l-1}) \given \cal F_l'} } \\
    &\leq (\alpha_{t-1} + 48\beta_{t-1})^2 \E \sqbr{ \ind{\cal E_g, \cal E_0(l)^c } \E\sqbr{  \lambda_{\min}^{-1}(V_{l-1}) \given \cal F_l'}}
\end{align*}
where the last inequality follows from the monotonicity of $\alpha_l,\beta_l$.

Consider the expectation term. Denote the event $\cal E_4(l)$ such that \Cref{lem:min_eig_2} ($\lambda_{\min}(V_{l-1}) \geq \lambda_0 + \frac{ \sigma_0^2 \lambda \min\{l,c_1\sqrt{l}\}}{4}$) holds. Then, we have
\begin{align*}
    &\E \sqbr{ \ind{\cal E_g, \cal E_0(l)^c } \E\sqbr{  \lambda_{\min}^{-1}(V_{l-1}) \given \cal F_l}} \\
    &\quad = \uterm{\E \sqbr{ \ind{\cal E_g, \cal E_0(l)^c, \cal E_4(l) } \E\sqbr{  \lambda_{\min}^{-1}(V_{l-1}) \given \cal F_l}}}{3} + \uterm{\E \sqbr{ \ind{\cal E_g, \cal E_0(l)^c, \cal E_4(l)^c } \E\sqbr{  \lambda_{\min}^{-1}(V_{l-1}) \given \cal F_l}}}{4} 
\end{align*}
For \term 3, by the definition of $\cal E_4(l)$, we have
\begin{align*}
    \term 3 \leq \br{\lambda_0 + \frac{ \sigma_0^2 \lambda \min\{(l-1),c_1\sqrt{l-1}\}}{4}}^{-1}.
\end{align*}
For \term 4,
\begin{align*}
    \term 4 &\leq \P(\cal E_4(l)^c) \E\sqbr{  \lambda_{\min}^{-1}(V_{l-1}) \given \cal F_l} \\
    &\leq \br{\exp\br{-\frac{\lambda \min\{(l-1),c_1\sqrt{l-1}\} }{8}} +  d \exp\br{-\frac{ \sigma_0^2 \lambda \min\{(l-1),c_1\sqrt{l-1}\} }{16}}} \times \frac{1}{\lambda_0}
\end{align*}
where the last inequality follows from \Cref{lem:min_eig_2} and by the naive bound of $V_{t-1}\succeq \lambda_0\I$. Substituting the results back yields
\begin{align*}
    &\E \sqbr{ \ind{\cal E_g, \cal E_0(l)^c } \E\sqbr{  \lambda_{\min}^{-1}(V_{l-1}) \given \cal F_l}} \\
    &\quad \leq \br{\lambda_0 + \frac{ \sigma_0^2 \lambda \min\{(l-1),c_1\sqrt{l-1}\}}{4}}^{-1} \\
    &\qquad + \frac{1}{\lambda_0} \br{\exp\br{-\frac{\lambda \min\{(l-1),c_1\sqrt{l-1}\} }{8}} +  d \exp\br{-\frac{ \sigma_0^2 \lambda \min\{(l-1),c_1\sqrt{l-1}\} }{16}}}.
\end{align*}
Plugging this expectation term into \term 2, and substituting the result of \term 1 into the original inequality, we have the desired result as
\begin{align*}
     &\E \sqbr{(R(x_l^*, S_l^*, \theta^*) - R(x_l, S_l, \theta^*))^2} \\
     &\quad \leq c_0 t^{-2}+ 2 \lambda \eta(l-1) + (\alpha_{t-1} + 48\beta_{t-1})^2  \nu(l)
\end{align*}
where we use the definition of
\begin{align*}
    \nu(l)&:= \br{\lambda_0 + \frac{ \sigma_0^2 \lambda \min\{(l-1),c_1\sqrt{l-1}\}}{4}}^{-1} \\
    &\quad + \frac{1}{\lambda_0} \br{\exp\br{-\frac{\lambda \min\{(l-1),c_1\sqrt{l-1}\} }{8}} +  d \exp\br{-\frac{ \sigma_0^2 \lambda \min\{(l-1),c_1\sqrt{l-1}\} }{16}}}.
\end{align*}
Finally, taking $\sqrt{\cdot}$ and applying $\min\{\cdot,1\}$ to both sides finishes the proof.

\subsection{Proof of \cref{lem:min_eig_2}}

Recall the definition of $N(l)$, which is the total number of random explorations up to round $l$. We also use the definition of $\wtilde V_l$, which is the design matrix, only consists of the feature vector in random exploration round as
\begin{align*}
    \wtilde V_l := \sum_{i=1}^l \sum_{j\in S_l} \ind{A(i-1)=1, E(i-1)} x_{ij} x_{ij}\tp.
\end{align*}
For simplicity, denote
\begin{align*}
    N_1 = \E[N(l)]/2, \quad N_2 = \sigma_0^2 \lambda \min\{l,c_1\sqrt{l}\} /4.
\end{align*}
Now, we want the lower bound probability of the following event: 
\begin{align*}
    \P(N(l) \geq N_1, \lambda_{\min}(\wtilde V_l) \geq N_2) = \uterm{\P(N(l) \geq N_1)}{1} \times \uterm{\P(\lambda_{\min}(\wtilde V_l)\geq N_2 \mid N(l) \geq N_1)}{2} 
\end{align*}
For \term 1, we use the Chernoff bound (\Cref{lem:chernoff}). By setting $\delta \leftarrow 1/2, \mu \leftarrow \E[N(l)]$, we have
\begin{align}
    \P\br{N(l) \geq N_1} &= 1 - P\br{N(l) \geq N_1} \geq 1 - \exp\br{-\E[N(l)]/8} \nonumber.
\end{align}
For the lower bound of $\E[N(l)]$,
\begin{align}
    \E[N(l)] = \sum_{i=1}^{l} \lambda \eta(i-1) \geq \lambda \sum_{i=1}^l \min\{1, c_1l^{-1/2}\} = \lambda \min\{l,c_1\sqrt{l}\} \label{eq:n_lb_1}
\end{align}

Plugging \Cref{eq:n_lb_1} back to the above inequality, we have
\begin{align*}
    \P\br{N(l) \geq N_1} \geq 1 - \exp\br{-\frac{\lambda \min\{l,c_1\sqrt{l}\} }{8}}. \label{eq:ch_1}
\end{align*}

Next, for \term 2, we have
\begin{align*}
    \term 2 &\geq \P(\lambda_{\min}(\wtilde V_l) \geq N_2 \mid N(l) = N_1) \\
    &\geq \P(\lambda_{\min}(\wtilde V_l) \geq N_2 \mid N(l) = \lambda \min\{l,c_1\sqrt{l}\} /2). \tag{\Cref{eq:n_lb_1}}
\end{align*}
Now, we use the Matrix Chernoff bound (\Cref{lem:mx_chernoff}) on the right-hand side of the inequality. By setting $\delta \leftarrow 1/2$, $n\leftarrow \lambda \min\{l,c_1\sqrt{l}\} /2$,  $\sigma_0^2\leftarrow \sigma_0^2$ we have
\begin{align*}
    \P(\lambda_{\min}(\wtilde V_l) \geq N_2 \mid N(l) = \lambda \min\{l,c_1\sqrt{l}\} /2) &= 1 - \P(\lambda_{\min}(\wtilde V_l) \leq N_2 \mid N(l) = \lambda \min\{l,c_1\sqrt{l}\} /2) \\
    &\geq 1 - d \exp\br{-\frac{ \sigma_0^2 \lambda \min\{l,c_1\sqrt{l}\} }{16}}.
\end{align*}
Substituting \term 1 and \term 2 back, we have
\begin{align*}
    \P(N(l) \geq N_1, \lambda_{\min}(\wtilde V_l) \geq N_2) &\geq \br{1 - \exp\br{-\frac{\lambda \min\{l,c_1\sqrt{l}\} }{8}}} \br{1 - d \exp\br{-\frac{ \sigma_0^2 \lambda \min\{l,c_1\sqrt{l}\} }{16}}} \\
    &\geq 1 - \exp\br{-\frac{\lambda \min\{l,c_1\sqrt{l}\} }{8}} -  d \exp\br{-\frac{ \sigma_0^2 \lambda \min\{l,c_1\sqrt{l}\} }{16}}.
\end{align*}
Finally, we have $\P(\lambda_{\min}(\wtilde V_l) \geq N_2) \geq \P(N(l) \geq N_1, \lambda_{\min}(\wtilde V_l) \geq N_2)$. Also, $V_l \succeq \lambda_0\I + \wtilde V_l$, therefore
\begin{align*}
    \P(\lambda_{\min}(V_l) \geq \lambda_0 + N_2) \geq 1 - \exp\br{-\frac{\lambda \min\{l,c_1\sqrt{l}\} }{8}} -  d \exp\br{-\frac{ \sigma_0^2 \lambda \min\{l,c_1\sqrt{l}\} }{16}},
\end{align*}
finishing the proof.

\section{Deferred proofs for \cref{ssec:5}}

\subsection{Proof of \cref{lem:tau_m}} \label{sec:lem:tau_m}

Denote $N(l)$ as the total number of random exploration up to round $l$. Notice that for every round $l$, the probability for the random exploration is $\lambda\eta(l-1)$ as it only happens when $A(l-1)=1$ and $E(l-1)=1$. Then, we have
\begin{align*}
    N(l) = \sum_{i=1}^{l-1} A(i)E(i), \quad \E[N(l)] = \sum_{i=1}^{l-1} \lambda \eta(i)
\end{align*}
For the lower bound of $\E[N(l)]$,
\begin{align}
    \E[N(\wtilde \tau (M))] = \sum_{i=1}^{\wtilde \tau (M)} \lambda \eta(i-1) \geq \lambda \sum_{i=1}^{\wtilde \tau (M)} \min\{1, c_1l^{-1/2}\} = \lambda \min\{\wtilde \tau (M),c_1\sqrt{\wtilde \tau (M)}\} \label{eq:eb_2}
\end{align}
Now, we show the lower bound of the number of random exploration round until $\wtilde \tau (M)$. Since $A(l-1)E(l-1)\in\{0,1\}$, we can apply Multiplicative Chernoff bound (\Cref{lem:chernoff}) with $\delta\leftarrow 1/2$, which gives
\begin{align*}
    \P\br{N(\wtilde \tau(M)) \leq \frac{\E[N(\wtilde \tau(M))]}{2}} \leq \exp\br{- \frac{\E[N(\wtilde \tau(M))]}{8}}.
\end{align*}
In order to show
\begin{align*}
    \P\br{N(\wtilde \tau(M)) \leq M} \leq t^{-2},
\end{align*}
we are sufficient to show
\begin{align*}
    M \leq \frac{\E[N(\wtilde \tau(M))]}{2}, \quad \exp\br{- \frac{\E[N(\wtilde \tau(M))]}{8}} \leq t^{-2}.
\end{align*}
Therefore, we need
\begin{align*}
    \E[N(\wtilde \tau(M))] \geq \max \cbr{2M,~ 16\log(t)}
\end{align*}
By the lower bound of $\E[N(\wtilde \tau(M))]$ in \Cref{eq:eb_2}, if
\begin{align*}
    \wtilde \tau(M) \geq \max \cbr{\frac{2M}{\lambda},~\frac{16\log(t)}{\lambda},~\frac{4M^2}{c_1^2\lambda^2},~ \frac{256\log^2 (t)}{c_1^2\lambda^2}}
\end{align*}
holds, then it is sufficient to show the lemma holds.

\subsection{Proof of \cref{prop:min_eig}}

We prepare to apply \Cref{lem:tau_m}. Set, $M(t)$ as 
\begin{align*}
    M(t) = c_2 \br{\frac{d + \log(t)}{ \sigma_0^4} + \frac{8(\alpha_{t-1} + 48\beta_{t-1})^2}{ \sigma_0^2(\eps - 2\eta(\tau(t))}}
\end{align*}
and set $\wtilde \tau(M(t))$ satisfying the condition in \Cref{lem:tau_m} as
\begin{align*}
    &\tau(t) = \wtilde \tau (M(t)) = \max \cbr{\frac{2M(t)}{\lambda},~\frac{16\log(t)}{\lambda},~\frac{4M(t)^2}{c_1^2\lambda^2},~ \frac{256\log^2 (t)}{c_1^2\lambda^2}}
\end{align*}

\begin{remark} \label{remark:circ}
    Notice that $M(t)$ and $\tau(t)$ have circular definition as $M(t)$ requires $\tau(t)$ and $\tau(t)$ requires $M(t)$. Although this is not problematic as both $M(t)$ and $\tau(t)$ are only required in analysis, we avoid this circular definition and remove ambiguity by using the definitions in \Cref{eq:tau}. To check the existence of $\tau(t)$ and $M(t)$, first we can see that $\wtilde M(u)$ is monotonically decreasing in $u$ and $\wtilde \tau (M)$ is monotonically increasing in $M$. Therefore $\wtilde \tau(\wtilde M(u))$ is monotonically decreasing in $u$, which means there exists a minimum value of $u\in\N$ such that $u\geq \wtilde \tau(\wtilde M(u))$ which we define $\tau(t)$. Then $M(t)$ is also defined as $\wtilde M(\tau(t))$.
\end{remark}

Now, by the direct result of \Cref{lem:tau_m}, we have
\begin{align}
    \P(N(\tau(t)) \geq M(t)) \geq 1 - t^{-2}, \label{eq:taum1}
\end{align}
which means we have $M(t)$ random exploration until round $\tau(t)$ with high probability. 

Now, we prepare to apply \Cref{prop:li}, which shows high probability lower bound of the design matrix $V_l$ that consists of the i.i.d. sampled features. Recall that our random exploration is done in a round-robin manner. Therefore, we can see that $M(t)$ features inside the design matrix $V_{\tau(t)}$ are i.i.d. samples from the unknown distribution $\cal D$. Also, for the feature vector, selected under the random exploration round, say $x_{ij'}$, we have
\begin{align*}
    \lambda_{\min}\br{\E \sqbr{ x_{ij'} x_{ij'}\tp}} = \lambda_{\min}\br{\E \sqbr{\frac{1}{N} \sum_{j\in[N]} x_{ij} x_{ij}\tp}} \geq \sigma_0^2. \tag{\Cref{ass:iid}}
\end{align*}
Finally, assuming \Cref{eq:taum1} holds, and setting
\begin{align*}
    B(t) = \frac{4(\alpha_{t-1} + 48\beta_{t-1})^2}{(\eps - 2\eta(\tau(t)))^2},    
\end{align*}
and applying \Cref{prop:li} yields
\begin{align}
    \lambda_{\min}(V_{\tau(t)}) \geq B(t). \label{eq:taum2}
\end{align}
with probability at least $1-t^{-2}$. 

Summing up the results, taking union bound on \Cref{eq:taum1,eq:taum2}, we have $\lambda_{\min}(V_{\tau(t)}) \geq B(t)$ with probability at least $1-2t^{-2}$, finishing the proof.

\subsection{Proof of \cref{prop:small_per_round}}

For some $l\in[\tau(t)+1,t]$, on the event $\cal E_g$, we have
\begin{align*}
    R(x_{l}^*, S_{l}^*, \theta^*) - R(x_{l}, S_{l}, \theta^*) &\leq \br{\alpha_{l-1} + 48 \beta_{l-1}} \lambda_{\min}^{-1/2}(V_{l-1}) \\
    &\leq \br{\alpha_{l-1} + 48 \beta_{l-1}} \lambda_{\min}^{-1/2}(V_{\tau(t)}) \\
    &\leq \br{\alpha_{l-1} + 48 \beta_{l-1}} \frac{\eps - 2\eta(\tau(t))}{2(\alpha_{t-1} + 48\beta_{t-1})},
\end{align*}
where the last inequality follows from $\cal E_3$. Since $\alpha_{l}$ and $\beta_{l}$ are monotonically increasing in $l$, we have
\begin{align*}
    R(x_{l}^*, S_{l}^*, \theta^*) - R(x_{l}, S_{l}, \theta^*) \leq \br{\alpha_{t-1} + 48 \beta_{t-1}} \frac{\eps - 2\eta(\tau(t))}{2(\alpha_{t-1} + 48\beta_{t-1})} = \frac{\eps}{2} - \eta(\tau(t)),
\end{align*}
finishing the proof.

\section{Deferred proofs for \cref{ssec:6}}

\subsection{Proof of \cref{lem:drift}}

If $i\geq l$, it means the departure is followed by the optimal policy. Then by \Cref{ass:slackness}, we have
\begin{align*}
    \E[A(l) - D(i, l) \mid  \cal F_l] \leq -\eps \leq -\eps/2.
\end{align*}
If $i< l$, it means the departure is followed by our policy. As the assumption, we consider $\cal E_g$ occurs. Then, we proceed as
\begin{align*}
    &\E[A(l) - D(i, l) \mid \cal F_l]  =\uterm{E[A(l) - D(l-1, l) \mid  \cal F_l]}{1}  + \uterm{\E[D(l-1, l) - D(i, l) \mid  \cal F_l]}{2} 
\end{align*}
First, notice that $D(l-1,l)$ and $D(i,l)$ both follows our policy until round $l-1$, and $D(l-1,l)$ follows the optimal policy in round $l$, while $D(i,l)$ follows our policy in round $l$. Therefore, for \term 1, similarly, by \Cref{ass:slackness}, we have $\term 1\leq -\eps$.

For \term 2, we have
\begin{align*}
    \term 2 &= \uterm{ \P(E(l-1)=1\mid \cal F_t)\E[D(l-1, l) - D(i, l) \mid  \cal F_l, E(l-1)=1]}{3} \\
    &\qquad + \uterm{\P(E(l-1)=0\mid \cal F_t)\E[D(l-1, l) - D(i, l) \mid  \cal F_l, E(l-1)=0]}{4} .
\end{align*}
For \term 3, we have $\term 3 \leq \P(E(l-1)=1) = \eta(l-1)$. 

For \term 4, we have
\begin{align*}
    \term 4 &\leq  \E[D(l-1, l) - D(i, l) \mid  \cal F_l, E(l-1)=0] \\
    &= \E[R(x_l^*, S_l^*, \theta^*) - R(x_l,S_l, \theta^*) \mid \cal F_l, E(l-1)=0],
\end{align*}
where the equality holds since $D(l-1,l)$ and $D(i,l)$ have the same queue state in round $l$ of $\cal X_l$. Now, since $l$ is in good rounds, and on the event $\cal E_g$ and $E(l-1)=0$ (which implies $\cal E_0(l)^c$), by \Cref{eq:good_1}, we have $\term 4\leq \eps/2 - \eta(\tau(t))$. Substituting \term 3 and \term 4 back yields
\begin{align*}
    \term 2 \leq   \eta(l-1) + \frac{\eps}{2} - \eta(\tau(t)) \leq \frac{\eps}{2},
\end{align*}
where the last inequality is from the fact that $\eta(l)$ is monotonically decreasing in $l$, and $l\notin \cal B$ means that $l\in[\tau(t)+1,t]$, which means $l-1\geq\tau(t)$.

Finally, substituting results of \term 1 and \term 2 back finishes the proof.

\subsection{Proof of \cref{lem:cond_exp_psi}}

We consider the event that $Q(i,i+1) \leq (t-i-1)\eps + 1$ occurs ($\cal E_5(i)$). For simplicity, denote the event of $\v D(i,i)=0, \v D(i-1,i)=1$ as $\cal E$. 

By \Cref{lem:psi},  under $\cal E$ (when $\v D(i,i)=0, \v D(i-1,i)=1$), the value of $\psi(i,t)$ is in $\{0,1\}$. We can simply ignore the case when $\psi(i,t) = 0$. Now we consider the case when $\psi(i,t)=1$, which means the difference of queue length occurred by a disagreement event in round $i$ is preserved until round $t$. Notice that if the queue with an extra job $Q_i$ hits $0$ queue length before $t$, the queue length difference between $Q_i$ and $Q_{i-1}$ will always become $0$ thereafter by our coupling process. This implies that the probability of $Q_i$ never hitting $0$ for all round $l\in[i+1,t]$ is larger than the probability that the queue length difference is preserved until round $t$. Then, we have
\begin{align*}
    \P\br{\psi(i,t) = 1 \given \cal F_i^+, \cal E} \leq \P\br{Q(i,j) > 0,~ \forall j\in[i+1,t] \given \cal F_i^+, \cal E}.
\end{align*}
Next, notice that $Q(i,i+1)$ is a realized value under $\cal F_i^+$ and $\cal E$. By the queue dynamics, the probability of $Q_i$ never hitting length $0$ for all round $l\in[i+1,t]$ can be upper bounded by the probability that the cumulative net service $\sum_{l=i+1}^{t-1} \br{D(i,l) - A(l)}$ cannot exceed $Q(i,i+1)-1$, which is 
\begin{align*}
    \P\br{Q(i,j)>0,~\forall j\in[i+1,t] \given \cal F_i^+, \cal E} \leq \P\br{Q(i,i+1) + \sum_{l=i+1}^{t-1} \br{A(l)-D(i,l)} \geq 1 \given \cal F_i^+, \cal E}.
\end{align*}
Combining two inequalities, we have
\begin{align*}
    &\E\sqbr{\psi(i,t) \given \cal F_i^+, \cal E} \\
    &\quad = \P\br{\psi(i,t) = 1 \given \cal F_i^+, \cal E} \\
    &\quad \leq \P\br{Q(i,i+1) + \sum_{l=i+1}^{t-1} \br{A(l)-D(i,l)} \geq 1 \given \cal F_i^+, \cal E}.
\end{align*}
For the last inequality, we prepare to apply Azuma-Hoeffding inequality (\Cref{lem:azuma}). Define a martingale difference sequence for $l\in[i+1,t-1]$ as $X_l := \E[D(i,l)-A(l) \mid \cal F_l] - (D(i,l)-A(l))$ and $Y_k:= \sum_{l=i+1}^k X_l$. We have $\E[X_l \mid \cal F_l]=0$, and $|Y_k - Y_{k-1}| = |X_k| \leq 2$, therefore, by applying Azuma-Hoeffding inequality, setting $a\leftarrow (t-i-1)\eps + 1 - Q(i,i+1)$ (where the condition of $a>0$ holds since we consider on the event $\cal E_5(i)$), we have
\begin{align*}
    &\P\br{\sum_{l=i+1}^{t-1}(\E[D(i,l)-A(l) \mid \cal F_l] - (D(i,l) - A(l))) \geq (t-i-1)\eps + 1 - Q(i,i+1) \given \cal F_i^+,\cal E} \\
    &\quad \leq 2 \exp \br{-\frac{((t-i-1)\eps + 1 - Q(i,i+1))^2}{8(t-i-1)}}
\end{align*}
For $l\in[i+1,t-1]$, $Q_i$ always follows the optimal policy, thereby $\E[D(i,l)-A(l) \mid \cal F_l] \geq \eps $ by \Cref{ass:slackness}. This implies that
\begin{align*}
    &\P\br{\sum_{l=i+1}^{t-1}(\E[D(i,l)-A(l) \mid \cal F_l] - (D(i,l) - A(l))) \geq (t-i-1)\eps + 1 - Q(i,i+1) \given \cal F_i^+,\cal E}  \\
    &\quad \geq \P\br{\sum_{l=i+1}^{t-1}( - (D(i,l) - A(l))) \geq 1 - Q(i,i+1) \given \cal F_i^+,\cal E} \\
    &\quad = \P\br{Q(i,i+1) + \sum_{l=i+1}^{t-1} \br{A(l)-D(i,l)} \geq 1 \given \cal F_i^+, \cal E}.
\end{align*}
Combining results and plugging back to the original inequality, we have
\begin{align*}
    \E\sqbr{\psi(i,t) \given \cal F_i^+, \cal E} \leq 2 \exp \br{-\frac{(Q(i,i+1) - (t-i-1)\eps - 1)^2}{8(t-i-1)}}
\end{align*}
as desired.

\subsection{Proof of \cref{lem:q_tail}}

We define a prefix good event $\cal E_g^{\leq l}$ where we cut $\cal E_g$ until round $l$ to make it $\cal F_l$-measurable. Formally, 
\begin{align*}
    &\cal E_1^{\leq l} := \{\forall i\leq l,~ \|\what\theta_{i-1} - \theta^*\|_{V_{i-1}} \leq \alpha_{i-1}\},\quad \cal E_2^{\leq l} := \bigcap_{i=1}^l \cal E_2(i) \\
    &\cal E_g^{\leq l} := \cal E_1^{\leq l} \cap \cal E_2^{\leq l}.
\end{align*}
Here, we skipped $\cal E_3$ since we don't need it in proof. We can see that $\cal E_g \subseteq \cal E_g^{\leq l}$ for all $l$.

We start with the one-step bound for the moment generating function (mgf) of $Q(i,l)$. For some $\gamma \in(0,\eps/2]$,
\begin{align*}
    \E[\exp(\gamma Q(i,l+1)) \mid \cal F_l] &= \E[\exp(\gamma[Q(i,l) + A(l) - D(i,l)]^+) \mid \cal F_l] \\
    &\leq 1 + \exp(\gamma Q(i,l)) \E[\exp(\gamma (A(l) - D(i,l))) \mid \cal F_l],
\end{align*}
where the inequality follows by considering both cases, where $Q(i,l) + A(l) - D(l)\leq 0$ gives $\exp(\gamma[Q(i,l) + A(l) - D(i,l)]^+) = 1$ and $Q(i,l) + A(l) - D(l)> 0$ gives $\exp(\gamma[Q(i,l) + A(l) - D(i,l)]^+) = \exp(\gamma (Q(i,l) + A(l) - D(i,l)))$. We split into two cases for $\E[\exp(\gamma (A(l) - D(i,l))) \mid \cal F_l]$:

\case 1: If $l\notin \cal B$, on the event $\cal E_g^{\leq l}$, by \Cref{lem:drift}, we have
\begin{align*}
    \E[A(l) - D(i,l) \mid \cal F_l] \leq -\eps/2.
\end{align*}
Then, 
\begin{align*}
    \E[\exp(\gamma (A(l) - D(i,l))) \mid  \cal F_l] &=  \E[\exp(\gamma (A(l) - D(i,l)) ) \mid  \cal F_l] \\
    &\leq \exp\br{\gamma \E[A(l)-D(i,l)\mid \cal F_l] + \gamma^2 / 2} \\
    &\leq \exp(-\gamma \eps/2+\gamma^2/2)
\end{align*}
where for the last inequality, we apply Hoeffding lemma (\Cref{lem:hoeffding_2}). Moreover, by our choice of $\gamma\in(0,\eps/2]$, we have
\begin{align*}
    \E[\exp(\gamma (A(l) - D(i,l))) \mid \cal F_l] \leq \exp(-\gamma \eps/4) =: \rho.
\end{align*}

\case 2: If $l\in\cal B$, we use a naive bound of $A(l) - D(i,l)\leq 1$ as
\begin{align*}
    \E[\exp(\gamma (A(l) - D(i,l))) \mid \cal F_l] \leq \exp(\gamma) =: \beta.
\end{align*}
Combining results for both cases, we have that, on the event $\cal E_g^{\leq l}$,
\begin{align}
    \E[\exp(\gamma Q(i,l+1)) \mid  \cal F_l] \leq 1 + \exp(\gamma Q(i,l)) (\rho \ind{l\notin \cal B} + \beta \ind{l\in\cal B}). \label{eq:mgf_1}
\end{align}
To control the indicator term $(\rho \ind{l\notin \cal B} + \beta \ind{l\in\cal B})$, we define a new weighted process $V(i,l)$ as
\begin{align*}
    V(i,l) := \br{\frac{\beta}{\rho}}^{- \cal B(l-1)} \exp(\gamma Q(i,l)).
\end{align*}
Now, we can proceed as
\begin{align*}
    &\E\sqbr{V(i,l+1) \ind{\cal E_g^{\leq l+1}}\given  \cal F_l} \\
    &\quad \leq \E\sqbr{V(i,l+1) \ind{\cal E_g^{\leq l}}\given  \cal F_l} \\
    &\quad = \ind{\cal E_g^{\leq l}} \E\sqbr{\br{\frac{\beta}{\rho}}^{-\cal B(l)} \exp(\gamma Q(i,l+1)) \given  \cal F_l} \tag{$\ind{\cal E_g^{\leq l}} \in \cal F_l$}\\
    &\quad = \ind{\cal E_g^{\leq l}}  \br{\frac{\beta}{\rho}}^{-\cal B(l-1)} \br{\frac{\beta}{\rho}}^{-\ind{l\in\cal B}}\E\sqbr{ \exp(\gamma Q(i,l+1)) \given \cal F_l} \\
    & \quad \leq \ind{\cal E_g^{\leq l}}  \br{\frac{\beta}{\rho}}^{-\cal B(l-1)} \br{\frac{\beta}{\rho}}^{-\ind{l\in\cal B}} \br{1 + \exp(\gamma Q(i,l)) (\rho \ind{l\notin \cal B} + \beta \ind{l\in\cal B})}. \tag{\Cref{eq:mgf_1}}
\end{align*}
Since 
\begin{align*}
    \br{\frac{\beta}{\rho}}^{-\ind{l\in\cal B}} (\rho \ind{l\notin \cal B} + \beta \ind{l\in\cal B}) = \rho,
\end{align*}
we have
\begin{align*}
    \E\sqbr{V(i,l+1) \ind{\cal E_g^{\leq l+1}}  \given \cal F_l}  &\leq \ind{\cal E_g^{\leq l}}  \rho \br{\frac{\beta}{\rho}}^{-\cal B(l-1)} \exp(\gamma Q(i,l)) + \ind{\cal E_g^{\leq l}}  \br{\frac{\beta}{\rho}}^{-\cal B(l)} \\
    &\leq   \rho V(i,l) \ind{\cal E_g^{\leq l}}  + 1,
\end{align*}
as the last inequality holds since $\beta/\rho \geq 1$. Taking the expectation on both sides and applying the tower rule on the left-hand side gives
\begin{align*}
    \E[V(i,l+1)\ind{\cal E_g^{\leq l+1}} ] \leq \rho \E[V(i,l)\ind{\cal E_g^{\leq l}} ] + 1.
\end{align*}
Solving a linear recursion (with the fact that $\ind{\cal E_g^{\leq 1}} =1$), we have
\begin{align}
    \E[V(i,l)\ind{\cal E_g^{\leq l}} ] \leq \rho^{l-1} \E[V(i,1)] + \sum_{k=0}^{l-2} \rho^k \leq \rho^{l-1} \E[\exp(\gamma Q(i,1))] + \sum_{k=0}^\infty \rho^k \leq \rho^{l-1} \E[\exp(\gamma Q(i,1))] + \frac{1}{1-\rho}. \label{eq:v_tail}
\end{align}
Since we obtained the tail probability of $V(i,l)$, we can start the proof as
\begin{align*}
    &\P(Q(i,l) \geq a \cal B(l-1) + b,~\cal E_g)  \\
    &\quad \leq \P(Q(i,l) \geq a \cal B(l-1) + b,~\cal E_g^{\leq l})\\
    &\quad = \P(\exp(\gamma Q(i,l)) \geq \exp(\gamma(a \cal B(l-1) + b)),~\cal E_g^{\leq l}) \\
    &\quad = \P\br{V(i,l) \geq \br{\frac{\beta}{\rho}}^{-\cal B(l-1)} \exp(\gamma(a \cal B(l-1) + b)),~\cal E_g^{\leq l}} 
\end{align*}
Note that by our choice of $a\geq \frac{1}{\gamma} \log\br{\frac{\beta}{\rho}}$, we have
\begin{align*}
    \exp(\gamma a \cal B(l-1)) \geq \exp(\log(\beta/\rho) \cal B(l-1)) = (\beta/\rho)^{\cal B (l-1)}.
\end{align*}
Therefore,
\begin{align*}
    \P(Q(i,l) \geq a \cal B(l-1) + b,~\cal E_g) \leq \P(V(i,l) \geq \exp(\gamma b),~\cal E_g^{\leq l}).
\end{align*}
Under the condition of $b\geq 0$, applying Markov inequality tot he right hand side yields
\begin{align*}
    \P(V(i,l) \geq \exp(\gamma b),~\cal E_g^{\leq l}) \leq \frac{\E[V(i,l) \ind{\cal E_g^{\leq l}}]}{\exp(\gamma b)} \leq \br{\rho^{l-1} \E[\exp(\gamma Q(i,1))] + \frac{1}{1-\rho}} \exp(-\gamma b). \tag{\Cref{eq:v_tail}}
\end{align*}
Substituting this result back gives the desired result.

Lastly, if we assume the initial queue starts with an empty state of $Q(1)=0$, and set $\gamma=\eps/2$, $\rho=\exp(-\eps^2/8)$, then we simplify as
\begin{align*}
    \rho^{l-1} \E[\exp(\gamma Q(i,1))] + \frac{1}{1-\rho} \leq 1 + \frac{1}{1 - \exp(-\eps^2/8)}\leq 1 + 16/\eps^2,
\end{align*}
where the inequality follows from the fact that $1-\exp(-x) \geq x/2$ for $x\in[0,1]$. Substituting the result back yields
\begin{align*}
    \P(Q(i,l) \geq a \cal B(l-1) + b,~\cal E_g) \leq (1 + 16/\eps^2) \exp(-\gamma b) \leq 17\eps^{-2} \exp(-\gamma b),
\end{align*}
finishing the proof.

\subsection{Proof of \cref{lem:tilde_psi}}

We start as follows:
\begin{align*}
    \E\sqbr{\wtilde \psi(i,t)} & = \E\sqbr{\E\sqbr{\psi(i,t)\given \cal F_i^+, \v D(i,i)=0, \v D(i-1,i)=1}} \\
    & = \E\sqbr{\ind{\cal E_g^c} \E\sqbr{\psi(i,t)\given \cal F_i^+, \v D(i,i)=0, \v D(i-1,i)=1}} \\
    &\quad + \E\sqbr{\ind{\cal E_g} \E\sqbr{\psi(i,t)\given \cal F_i^+, \v D(i,i)=0, \v D(i-1,i)=1}} \\ 
    &\leq c_0t^{-2} + \E\sqbr{\ind{\cal E_g} \E\sqbr{\psi(i,t)\given \cal F_i^+, \v D(i,i)=0, \v D(i-1,i)=1}} \tag{\Cref{eq:good_event,lem:psi}}
\end{align*}
Now for the second term on the right-hand side, we consider the term $\E\sqbr{\psi(i,t)\given \cal F_i^+, \v D(i,i)=0, \v D(i-1,i)=1}$ under the case when the event $\cal E_g$ occurs: Recall the result of \Cref{lem:cond_exp_psi}, which is, on the event $\cal E_5(i)$,
\begin{align*}
    &\E\sqbr{\psi(i,t)\mid \cal F_i^+, \v D(i,i)=0, \v D(i-1,i)=1} \leq 2\exp \br{- \frac{(Q(i,i+1)-1 - (t-i-1)\eps)^2}{8(t-i-1)}}
\end{align*}
Our goal there is to obtain an exponential decay of the right-hand side in terms of the number of remaining rounds $t-l-1$ and avoid the dependence of $Q(i,i+1)$. Therefore, we split into two cases where $Q(i,i+1) \leq \frac{\eps(t-l-1)}{2}$ and $Q(i,i+1) > \frac{\eps(t-l-1)}{2}$ which gives
\begin{align*}
    &\E \sqbr{\psi(i,t) \given \cal F_i^+, \v D(i,i)=0, \v D(i-1,i)=1} \\
    &= \uterm{\ind{Q(i,i+1) \leq \frac{\eps(t-l-1)}{2}} \E \sqbr{\psi(i,t) \given \cal F_i^+, \v D(i,i)=0, \v D(i-1,i)=1}}{1} \\
    &\quad + \uterm{\ind{Q(i,i+1) > \frac{\eps(t-l-1)}{2}} \E \sqbr{\psi(i,t) \given \cal F_i^+, \v D(i,i)=0, \v D(i-1,i)=1}}{2} 
\end{align*}
For \term 1, we can see that $\ind{Q(i,i+1) \leq \frac{\eps(t-l-1)}{2}}$ implies $\cal E_5(i)$ (i.e., $\ind{Q(i,i+1) \leq (t-i-1)\eps + 1}$), therefore we can directly apply the result of \Cref{lem:cond_exp_psi}, which gives
\begin{align*}
    \term 1 &\leq \ind{Q(i,i+1) \leq \frac{\eps(t-i-1)}{2}} 2 \exp \br{- \frac{(Q(i,i+1)-1 - (t-i-1)\eps)^2}{8(t-i-1)}} \\
    & \leq 2 \exp \br{- \frac{(\frac{1}{2}(t-i-1)\eps + 1)^2}{8(t-i-1)}} \\
    & \leq  2 \exp \br{- \frac{\eps^2 (t-i-1)}{32}}.
\end{align*}
For \term 2, by \Cref{lem:psi}, we have $\term 2 \leq \ind{Q(i,i+1) > \frac{\eps(t-l-1)}{2}}$. Substituting results to the original inequality gives
\begin{align*}
    \E\sqbr{\wtilde \psi(i,t)} &\leq c_0t^{-2} + 2 \exp \br{- \frac{\eps^2 (t-i-1)}{32}} + \E\sqbr{\ind{\cal E_g} \ind{Q(i,i+1) > \frac{\eps(t-l-1)}{2}}} \\
    &= c_0t^{-2} + 2 \exp \br{- \frac{\eps^2 (t-i-1)}{32}} + \P\br{Q(i,i+1) > \frac{\eps(t-l-1)}{2},~\cal E_g}
\end{align*}
Now, for the last term on the right-hand side, we are going to apply the tail bound of \Cref{lem:q_tail}, by setting $l\leftarrow i+1$, which gives
\begin{align*}
    \P(Q(i,i+1) \geq a \cal B(i) + b) \leq 17 \eps^{-2} \exp(-\gamma b))
\end{align*}
under the condition of $a \geq \frac{1}{\gamma}\log\br{\frac{\beta}{\rho}}$ and $b\geq 0$. Since $\frac{1}{\gamma}\log\br{\frac{\beta}{\rho}} = 1 + \eps/4$, we can set $a=2$. Now, we need to assure $b\geq 0$ holds. In order to control this, we set a threshold value $\omega$ of the remaining round as
\begin{align*}
    \omega := \frac{4\tau(t)}{\eps} \geq \frac{2 a \cal B(t)}{\eps}.
\end{align*}
We split into 2 cases:

\case 1: If $(t-i-1) < \omega$, we do not apply \Cref{lem:q_tail}, since it means that there are not many rounds remaining to reduce the queue length difference by emptying the queue with an extra job. This leads to a naive bound of $\E[\wtilde \psi(i,t)] \leq 1$.

\case 2: If $(t-i-1) \geq \omega$, we have
\begin{align*}
    \P\br{Q(i,i+1) > \frac{\eps(t-i-1)}{2},~\cal E(g)} &= \P\br{Q(i,i+1) > a \cal B(i) + \br{\frac{\eps(t-i-1)}{2} - a \cal B(i)},~\cal E(g)} \\
    &\leq \P\br{Q(i,i+1) > a \cal B(i) + \br{\frac{\eps(t-i-1)}{2} - a \cal B(t)},~\cal E(g)} \\
    &\leq \P\br{Q(i,i+1) > a \cal B(i) + \br{\frac{\eps(t-i-1 - \omega)}{2}},~\cal E(g)} \tag{$\omega \geq 2 a \cal B(t)/\eps$}
\end{align*}
Since $\br{\frac{\eps(t-i-1 - \omega)}{2}} \geq 0$ by the assumption, setting $b\leftarrow \br{\frac{\eps(t-i-1 - \omega)}{2}}$ and applying \Cref{lem:q_tail} gives
\begin{align*}
    \P\br{Q(i,i+1) > \frac{\eps(t-i-1)}{2},~\cal E(g)} \leq 17\eps^{-2} \exp\br{-\gamma \br{\frac{\eps(t-i-1 - \omega)}{2}}}
\end{align*}
Substituting both cases gives
\begin{align*}
    \E\sqbr{\wtilde \psi(i,t)} &\leq \min\cbr{1,~ c_0t^{-2} + 2 \exp \br{- \frac{\eps^2 (t-i-1)}{32}} + 17\eps^{-2} \exp\br{-\gamma \br{\frac{\eps(t-i-1 - \omega)}{2}}}} \\
    &\leq \min \cbr{1,~ c_0t^{-2} + 19 \eps^{-2} \exp\br{-\frac{\eps^2}{32} \br{t-i-1 - \omega}}} \tag{$\gamma=\eps/2$}.
\end{align*}
Finally, taking $\sqrt{\cdot}$ on both sides finishes the proof.

\section{Auxiliary lemmas}

\begin{lemma} [Chebyshev sum inequality] \label{lem:cheb}
    If $(a_i)_{i=1}^t$ is nondecreasing and $(b_i)_{i=1}^t$ is nonincreasing, and $a_i,b_i\geq 0$, we have
    \begin{align*}
        \sum_{i=1}^t a_i b_i \leq \frac{1}{t} \br{\sum_{i=1}^t a_i} \br{\sum_{i=1}^t b_i}.
    \end{align*}
\end{lemma}

\begin{lemma} [Multiplicative Chernoff bound] \label{lem:chernoff}
    Suppose $X_1,\dots, X_n\in\{0,1\}$ are independent random variables. Let $X$ denote their sum and $\mu=\E[X]$. Then for any $0\leq \delta \leq 1$,
    \begin{align*}
        \P(X\leq (1-\delta)\mu) \leq \exp\br{-\delta^2\mu/2}.
    \end{align*}
\end{lemma}

\begin{lemma} [Matrix Chernoff bound] \label{lem:mx_chernoff}
    Let $X\in\R^d$ be a random vector with $\|X\|_2\leq 1$ and $\E[X X\tp]\succeq \sigma_0^2\I$ for some $\sigma_0>0$. Suppose $X_1,\dots, X_n$ be i.i.d. sampled vectors and define $V_n = \sum_{i=1}^n X_i X_i\tp$. Then for any $0\leq \delta < 1$,
    \begin{align*}
        \P\br{\lambda_{\min}(V_n) \leq (1-\delta) n \sigma_0^2} \leq d \br{\frac{e^{-\delta}}{(1-\delta)^{1-\delta}}}^{n\sigma_0^2} \leq d \exp\br{-\frac{\delta^2 n \sigma_0^2}{2}}
    \end{align*}
\end{lemma}

\begin{lemma} [Lemma 1 of \citet{oh2019thompson}] \label{lem:1_oh}
    Suppose $S^*(x)$ is the optimal assortment under context $x$ and true parameter $\theta^*$, i.e., $S^*(x) = \argmax_{S\in\cal C} R(x,S,\theta^*)$. Also suppose that $x_{-j}\tp\theta^* \leq x_{-j}\tp\theta'$ for all $j\in S^*(x)$. Then $R(x,S_l^*,\theta^*)\leq R(x,S_l^*,\theta')$.
\end{lemma}

\begin{lemma} [Modified version of  Lemma 2 of \citet{oh2019thompson}] \label{lem:2_oh}
    Suppose $\|\what\theta_{t-1}-\theta^*\|_{V_{t-1}} \leq \alpha_{t-1}$. Then we have
    \begin{align*}
        \P(\wtilde R(x_t,S_t) > R(x_t^*,S_t^*,\theta_t^*)\mid \cal F_t^-) \geq (4\sqrt{e\pi})^{-1}
    \end{align*}
\end{lemma}

\begin{proof}

Recall the definition of filtration $F_t^-$ which shorts of $\v\theta(t-1)$ as follows:
\begin{align*}
    \cal F_t^- := \sigma (\cal X_1, \v A(1), \v D(1), \v \theta(1), E(1),\dots, \v A(t-1), \v D(t-1)).
\end{align*}
Note that $\cal X_t$ is still $\cal F_t^-$-measurable, and $x_t$ and $\{\wtilde \theta_{t-1}\}_{i=1}^{M}$ are not $\cal F_t^-$-measurable. We use the definition of the optimal assortment $S^*(x)$ and the optimistic assortment $S(x)$ given a context $x\in\cal X_t$ as
\begin{align*}
    S^*(x) = \argmax_{S\in\cal C} R(x,S,\theta^*), \quad S(x) = \argmax_{S\in\cal C} \wtilde R(x, S).
\end{align*}
Note that $S^*(x_t^*) = S_t^*$ and $S(x_t) = S_t$. Given $\cal F_t^-$, each of Gaussian random variable $x_{-j}\tp\wtilde\theta_{t-1}^{(i)}$ has mean $x_{-j}\tp\what\theta_{t-1}$ and standard deviation $\alpha_{t-1} \|x_{-j}\|_{V_{t-1}^{-1}}$. Hence, in round $t$, for all $x\in\cal X_t$, $j\in S^*(x)$, we have
\begin{align*}
    &\P\br{\max_i x_{-j}\tp \wtilde \theta_{t-1}^{(i)} > x_{-j}\tp\theta^*\mid \cal F_t^-} \\
    &\quad = 1 - \P\br{ x_{-j}\tp\wtilde\theta_{t-1}^{(i)}\leq x_{-j}\tp\theta^*,~\forall i\in\{1,\dots,M\} \mid \cal F_t^-} \\
    &\quad = 1 - \P\br{ \frac{x_{-j}\tp\wtilde\theta_{t-1}^{(i)} - x_{-j}\tp \what\theta_{t-1}}{\alpha_{t-1} \|x_{-j}\|_{V_{t-1}^{-1}}} \leq \frac{x_{-j}\tp\theta^* - x_{-j}\tp \what\theta_{t-1} }{\alpha_{t-1} \|x_{-j}\|_{V_{t-1}^{-1}}},~\forall i\in\{1,\dots,M\} \mid \cal F_t^-} \\
    &\quad = 1 - \P\br{ Z_j \leq \frac{x_{-j}\tp\theta^* - x_{-j}\tp \what\theta_{t-1} }{\alpha_{t-1} \|x_{-j}\|_{V_{t-1}^{-1}}},~\forall i\in\{1,\dots,M\} \mid \cal F_t^-}
\end{align*}
where $Z_j$ is a standard normal random variable. Then, we can bound the right-hand side term within the probability as
\begin{align*}
    \frac{x_{-j}\tp\theta^* - x_{-j}\tp \what\theta_{t-1} }{\alpha_{t-1} \|x_{-j}\|_{V_{t-1}^{-1}}} \leq \frac{\|x_{-j}\|_{V_{t-1}^{-1}} \|\theta^* - \what\theta_{t-1}\|_{V_{t-1}}}{\alpha_{t-1} \|x_{-j}\|_{V_{t-1}^{-1}}} \leq 1,
\end{align*}
where the inequality follows from the Cauchy-Schwarz inequality and the assumption. Then, it follows that
\begin{align}
    \P\br{\max_i x_{-j}\tp \wtilde \theta_{t-1}^{(i)} > x_{-j}\tp\theta^*\mid \cal F_t^-} \geq 1 - (\P(Z\leq 1))^{M}. \label{eq:opt_1}
\end{align}

Now, we are ready to lower-bound the probability of having an expected revenue optimistic under the sampled parameter as follows:
\begin{align*}
    &\P(\wtilde R(x_t, S_t) > R(x_t^*, S_t^*,\theta^*)  \mid \cal F_t^-) \\
    &\quad = \P(\wtilde R(x_t, S(x_t))) > R(x_t^*, S^*(x_t^*),\theta^*)  \mid \cal F_t^-) \\
    &\quad \geq \P(\wtilde R(x_t^*, S(x_t^*))) > R(x_t^*, S^*(x_t^*),\theta^*)  \mid \cal F_t^-) \tag{optimistic rule}
\end{align*}
By the definition, we have $\wtilde R(x, S(x)) \geq \wtilde R(x, S^*(x))$, therefore
\begin{align*}
    &\P(\wtilde R(x_t, S_t) > R(x_t^*, S_t^*,\theta^*) \mid \cal F_t^-) \\
    &\quad \geq \P(\wtilde R(x_t^*, S^*(x_t^*)) > R(x_t^*, S^*(x_t^*),\theta^*) \mid \cal F_t^-) \\
    &\quad \geq \P(\max_i (x_{tj}^*)\tp\wtilde\theta_{t-1}^{(i)} > (x_{tj}^*)\tp\theta^*,~ \forall j\in S^*(x_t^*), \mid \cal F_t^-) \tag{\Cref{lem:1_oh}} \\
    &\quad \geq 1 - K  (\P(Z\leq 1))^{M}. \tag{\Cref{eq:opt_1}, union bound}
\end{align*}
Using the anti-concentration inequality in \Cref{lem:anti}, we have $\P(Z\leq 1)\leq 1-(4\sqrt{e\pi})^{-1}$. Hence, we can have the desired result as
\begin{align*}
    \P(\wtilde R(x_t, S_t) > R(x_t^*, S_t^*,\theta^*) \mid \cal F_t^-)  &\geq 1 - K  ( 1- (4\sqrt{e\pi})^{-1})^{M} \\
    &\geq 1 - ( 1- (4\sqrt{e\pi})^{-1}) \\
    & = (4\sqrt{e\pi})^{-1}
\end{align*}
where the second inequality follows from our choice of $M=\lceil 1 - \frac{\log (K)}{\log(1-1/(4\sqrt{e\pi}))} \rceil$ which implies $(1- (4\sqrt{e\pi})^{-1})^{M} \leq \frac{1}{K} (1- (4\sqrt{e\pi})^{-1})$.
\end{proof}

\begin{lemma} [Lemma 3 of \citet{oh2019thompson}] \label{lem:lip}
    For any two utility parameters $u_t=[u_{t1},\dots,u_{tN}]$ and $u_t'=[u_{t1}',\dots,u_{tN}']$, we have
    \begin{align*}
        \frac{\sum_{j\in S} \exp(u_{tj})}{1 + \sum_{j\in S} \exp(u_{tj})} - \frac{\sum_{j\in S} \exp(u_{tj}')}{1 + \sum_{j\in S} \exp(u_{tj}')} \leq \max_{i\in S} | u_{tj} - u_{tj}'|.
    \end{align*}
    In particular, if $u_{tj} \geq u_{tj}'$ for all $j$, then
    \begin{align*}
        \frac{\sum_{j\in S} \exp(u_{tj})}{1 + \sum_{j\in S} \exp(u_{tj})} - \frac{\sum_{j\in S} \exp(u_{tj}')}{1 + \sum_{j\in S} \exp(u_{tj}')} \leq \max_{i\in S} ( u_{tj} - u_{tj}').
    \end{align*}
\end{lemma}

\begin{lemma} [Lemma 4 of \citet{oh2019thompson}] \label{lem:alpha}
    Define $\alpha_l = \frac{\kappa}{2} \sqrt{d\log(1+\frac{lK}{d\lambda_0}) + 4\log l} + \kappa \sqrt{\lambda_0}$. Then
    \begin{align*}
        \|\what\theta_l - \theta^*\|_{V_l} \leq \alpha_l
    \end{align*}
    holds for al all $l\in[t]$ with probability $1-\bigO{1/t^2}$.
\end{lemma}

\begin{lemma} [Lemma 6 of \citet{oh2019thompson}] \label{lem:6_oh}
    Define $V_t:= \lambda_0 \I + \sum_{i=1}^t \sum_{j\in S_i} x_{ij} x_{ij}\tp$. Then, we have
    \begin{align*}
        \sum_{i=1}^t \max_{j\in S_i} \|x_{ij}\|_{V_{t-1}^{-1}} \leq \sqrt{2d t \log\br{1 + \frac{tK}{d\lambda_0}}}
    \end{align*}
\end{lemma}

\begin{lemma} [Modified version of Lemma 10 of \citet{oh2019thompson}] \label{lem:beta}
    Set $l\in[t]$. Let $\beta_l = \alpha_l \min\br{\sqrt{6d\log(Mt)}, \sqrt{2\log(2M)} + \sqrt{6\log(Klt)}}$. Then for all $x\in\cal X_l$, $j\in[N]$,
    \begin{align*}
        \wtilde u_{lj}(x) - x_{-j}\tp\what\theta_{l-1} \leq \beta_{t-1} \|x_{-j}\|_{V_{l-1}^{-1}}
    \end{align*}
    with probability $1-\bigO{1/t^3}$.
\end{lemma}

\begin{proof}
    We follow the same proof procedure introduced in Section E.1 of \citet{oh2019thompson}. Given $\cal F_l^-$, for some $x\in\cal X_l$, each of Gaussian random variable $x_{-j}\tp \wtilde\theta_{l-1}^{(i)}$ has mean $x_{-j}\tp \what\theta_{l-1}$ and standard deviation $\alpha_{l-1}\|x_{-j}\|_{V_{l-1}}^{-1}$.
    \begin{align*}
        \abs{\wtilde u_{tj}(x) - x_{-j}\tp\what\theta_{l-1}} &= \alpha_{l-1} \|x_{-j}\|_{V_{l-1}^{-1}} \frac{\abs{\max_i x_{-j}\tp \wtilde \theta_{l-1}^{(i)} - x_{-j}\tp \what\theta_{l-1}}}{\alpha_{l-1} \|x_{-j}\|_{V_{l-1}^{-1}}} \\
        &\leq \alpha_{l-1} \|x_{-j}\|_{V_{l-1}^{-1}} \max_i \abs{\frac{ x_{-j}\tp \wtilde \theta_{l-1}^{(i)} - x_{-j}\tp \what\theta_{l-1}}{\alpha_{l-1} \|x_{-j}\|_{V_{l-1}^{-1}}}} \\
        & = \alpha_{l-1} \|x_{-j}\|_{V_{l-1}^{-1}} \max_i |Z_j|
    \end{align*}
    where each $Z_i$ is a standard normal random variable. With \Cref{lem:13_oh}, we have
    \begin{align*}
        \max_i |Z_i| \leq \sqrt{2\log(2M)} + \sqrt{6\log(t)}
    \end{align*}
    with probability at least $1-t^{-3}$. Then for all $x\in\cal X_l$ and for all $j\in[N]$, 
    \begin{align}
        |\wtilde u_{lj} - x_{lj}\tp \what\theta_{l-1}| \leq \br{\sqrt{2\log(2M)} + \sqrt{6\log(Klt)}} \alpha_{l-1} \|x_{tj}\|_{V_{t-1}^{-1}} \label{eq:beta_1}
    \end{align}
    with probability at least $1-t^{-3}$, where we take union bound using the fact that $|\cal X_l|\leq l$. Now, let $m=\argmax_i x_{lj}\tp \wtilde\theta_{l-1}^{(i)}$. Then we can write
    \begin{align*}
        \abs{\wtilde u_{lj}(x) - x_{-j}\tp\what\theta_{l-1}} &= \abs{\max_i x_{-j}\tp\wtilde\theta_{l-1}^{(j)} - x_{-j}\tp\what\theta_{l-1}}\\
        &=\abs{x_{-j}\tp(\wtilde\theta_{l-1}^{(m)} - \what\theta_{l-1})} \\
        &\leq \alpha_{l-1} \|x_{-j}\|_{V_{l-1}^{-1}} \norm{\alpha_{l-1}^{-1} V_{l-1}^{1/2} (\wtilde\theta_{l-1}^{(m)} - \what\theta_{l-1})} \\
        &\leq \alpha_{l-1} \|x_{-j}\|_{V_{l-1}^{-1}} \max_i \norm{\alpha_{l-1}^{-1} V_{l-1}^{1/2} (\wtilde\theta_{l-1}^{(i)} - \what\theta_{l-1})} \\
        &= \alpha_{l-1} \|x_{-j}\|_{V_t-1}^{-1} \max_i \|\zeta_i\|
    \end{align*} 
    where each element in $\zeta_i\in\R^d$ is a univariate standard normal variable $\cal N(0,1)$. Therefore, each $\|\zeta_i\|\leq \sqrt{6d\log (t)}$ with probability at least $1-t^{-3}$. Using the union bound for all $i\in[M]$, with probability at least $1-t^{-3}$, we have 
    \begin{align}
        |\wtilde u_{lj} - x_{lj}\tp \what\theta_{l-1}| \leq \sqrt{6d\log(Mt)} \alpha_{l-1} \|x_{-j}\|_{V_{l-1}^{-1}}. \label{eq:beta_2}
    \end{align}
    Finally, taking the minimum for both cases \Cref{eq:beta_1,eq:beta_2} finishes the proof.
\end{proof}

\begin{lemma} [Lemma 13 of \citet{oh2019thompson}] \label{lem:13_oh}
    Let $Z_i\sim\cal N(0,1)$, $i\in[n]$ be a standard Gaussian random variable. Then we have
    \begin{align*}
        \P\br{\max_i |Z_i| \leq \sqrt{2\log(2n)} + \sqrt{2 \log(1/\delta)}} \geq 1- \delta
    \end{align*}
\end{lemma}

\begin{lemma} [Lemma 12 of \citet{oh2019thompson}] \label{lem:sample_regret} \label{lem:12_oh}
    Assume $E(t-1)=0$. With probability $1-\bigO{1/t^2}$, we have
    \begin{align*}
        R(x_t^*, S_t^*,\theta^*) - \wtilde R(x_t,S_t) \leq 16\sqrt{e\pi} \beta_{t-1} \E \sqbr{\max_{j\in S(x_t, \wtilde\theta_{t-1}^{1:M_t})} \|x_{tj}\|_{V_{t-1}^{-1}}\given \cal F_t^-}
    \end{align*}
\end{lemma}

\begin{proof}

The proof follows the similar steps as the proof of Lemma 12 in \citet{oh2019thompson} (see Section E.3 of \citet{oh2019thompson}). For completeness, we provide the full proof here.

First, notice that $E(t-1)=0$, which means $x_t$ and $S_t$ are selected by Thompson sampling. Next, define $\wtilde \Theta_t$ as the set of parameter samples for which the expected revenue concentrates appropriately to the expected revenue based on the ML parameter. Also, define the set of optimistic parameter samples $\wtilde \Theta_t^{\text{opt}}$ which coinciding with $\wtilde\Theta_T$ as follows:
\begin{align*}
    \wtilde\Theta_t &:= \cbr{\{\wtilde\theta_{t-1}^{(i)}\}_{i=1}^{M}:~ \wtilde R(x_t, S_t) - R(x_t,S_t,\what\theta_{t-1}) \leq \beta_{t-1} 
    \max_{j\in S_t} \|x_{tj}\|_{V_{t-1}^{-1}}} \\
    \wtilde\Theta_t^{\text{opt}}&:=\cbr{\{\wtilde\theta_{t-1}^{(i)}\}_{i=1}^{M}:~ \wtilde R(x_t,S_t) > R(x_t^*, S_t^*,\theta^*)} \cap \wtilde\Theta_t
\end{align*}
Define the event $\cal E_t$ where
\begin{align*}
    \cal E_t = \{x_{-j}\tp\what\theta_{t-1} - x_{-j}\tp\theta^* \leq \alpha_{t-1}\|x_{-j}\|_{V_{t-1}^{-1}},~ \forall x\in\cal X_t, j\} \cap \{ \wtilde u_{tj}(x) - x_{-j}\tp \what\theta_{t-1} \leq \beta_{t-1}\|x_{-j}\|_{V_{t-1}^{-1}}, ~\forall x\in\cal X_t, j\}.
\end{align*}
For any $\wtilde\theta_{t-1}^{1:M}:= \{\wtilde\theta_{t-1}^{(i)}\}_{i=1}^{M} \in \wtilde \Theta_t^{\text{opt}}$, we have
\begin{align*}
    \br{R(x_t^*, S_t^*, \theta^*) - \wtilde R(x_t, S_t)} \ind{\cal E_t} \leq \br{R(x_t^*, S_t^*, \theta^*) - \inf_{\theta_{t-1}^{1:M}\in\wtilde\Theta_t} \wtilde R(x_t,S_t,\theta_{t-1}^{1:M})} \ind{\cal E_t}
\end{align*}
where $\wtilde R(x_t,S_t,\theta_{t-1}^{1:M})$ is the optimistic expected revenue under the sampled parameters $\theta_{t-1}^{1:M}$. Then, we can bound $R(x_t^*, S_t^*, \theta_t^*) - \wtilde R(x_t,S_t)$ by the expectation over any random choice $\wtilde\theta_{t-1}^{1:M} \in\wtilde\Theta_t^{\text{opt}}$:
\begin{align*}
    R(x_t^*, S_t^*,\theta^*) - \wtilde R(x_t,S_t) &\leq \E \sqbr{\br{\wtilde R(x_t, S_t) - \inf_{\theta_{t-1}^{1:M}\in\wtilde\Theta_t}\wtilde R(x_t,S_t,\theta_{t-1}^{1:M}) } \ind{\cal E_t} \given \cal F_t^-, \wtilde\theta_{t-1}^{1:M}\in\wtilde\Theta_t^{\text{opt}}} \\
    &= \E \sqbr{\sup_{\theta_{t-1}^{1:M}\in\wtilde\Theta_t}\br{\wtilde R(x_t, S_t) - \wtilde R(x_t,S_t,\theta_{t-1}^{1:M}) } \ind{\cal E_t} \given \cal F_t^-, \wtilde\theta_{t-1}^{1:M}\in\wtilde\Theta_t^{\text{opt}}} \\
    &\leq \E \sqbr{\sup_{\theta_{t-1}^{1:M}\in\wtilde\Theta_t} \max_{j\in S_t} \abs{\wtilde u_{tj}(x_t) - x_{tj}\tp\theta_{t-1}^{(i)}} \ind{\cal E_t} \given \cal F_t^-, \wtilde\theta_{t-1}^{1:M}\in\wtilde\Theta_t^{\text{opt}}} \\
    &\leq 2\beta_{t-1} \E \sqbr{\max_{j\in S(x_t, \wtilde\theta_{t-1}^{1:M})} \|x_{tj}\|_{V_{t-1}^{-1}}\given \cal F_t^-, \wtilde\theta_{t-1}^{1:M}\in\wtilde\Theta_t^{\text{opt}}, \cal E_t} \P(\cal E_t)  
\end{align*}
where the last inequality follows from the definition of $\wtilde\Theta_t$ and $S(x_t, \wtilde\theta_{t-1}^{1:M})$ stands for the optimal assortment under the sampled parameters $\wtilde\theta_{t-1}^{1:M} = \{\wtilde\theta_{t-1}^{(i)}\}_{i=1}^{M}$.

Now, from \Cref{lem:2_oh}, we have $\P(\wtilde R(x_t,S_t)> R(x_t^*, S_t^*, \theta^*)|\cal F_t^-,\cal E_t)\geq (4\sqrt{e\pi})^{-1}$, Therefore, we have
\begin{align*}
    \P(\wtilde\theta_{t-1}^{1:M} \in \wtilde\Theta_t^{\text{opt}} \mid \cal F_t^-, \cal E_t) &= \P(\wtilde R(x_t,S_t) > R(x_t^*, S_t^*, \theta^*) \text{ and } \wtilde\theta_{t-1}^{1:M}\in\wtilde\Theta_t, \cal E_t) \\
    &\geq \P(\wtilde R(x_t, S_t) > R(x_t^*, S_t^*, \theta^*) \mid \cal F_t^-, \cal E_t) - \P(\wtilde\theta_{t-1}^{1:M} \notin \wtilde\Theta_t, \cal E_t) \\
    &\geq (4\sqrt{e\pi})^{-1} - \bigO{t^{-1}} \\
    &\geq (4\sqrt{e\pi})^{-1}/2.
\end{align*}
Note that we have
\begin{align*}
    &\E \sqbr{\max_{j\in S(x_t, \wtilde\theta_{t-1}^{1:M})} \|x_{tj}\|_{V_{t-1}^{-1}}\given \cal F_t^-, \cal E_t} \\
    &\quad \geq \E \sqbr{\max_{j\in S(x_t, \wtilde\theta_{t-1}^{1:M})} \|x_{tj}\|_{V_{t-1}^{-1}}\given \cal F_t^-, \wtilde\theta_{t-1}^{1:M}\in\wtilde\Theta_t^{\text{opt}}, \cal E_t} \P(\wtilde\theta_{t-1}^{1:M} \in \wtilde\Theta_t^{\text{opt}}\mid\cal F_t^-, \cal E_t) \\
    &\quad \geq \E \sqbr{\max_{j\in S(x_t, \wtilde\theta_{t-1}^{1:M})} \|x_{tj}\|_{V_{t-1}^{-1}}\given \cal F_t^-, \wtilde\theta_{t-1}^{1:M}\in\wtilde\Theta_t^{\text{opt}}, \cal E_t}  \cdot (4\sqrt{e\pi})^{-1}/2.
\end{align*}
Substituting results, we have
\begin{align*}
    R(x_t^*, S_t^*,\theta^*) - \wtilde R(x_t,S_t) &\leq 2\beta_{t-1} \E \sqbr{\max_{j\in S(x_t, \wtilde\theta_{t-1}^{1:M})} \|x_{tj}\|_{V_{t-1}^{-1}}\given \cal F_t^-, \wtilde\theta_{t-1}^{1:M}\in\wtilde\Theta_t^{\text{opt}}, \cal E_t} \P(\cal E_t) \\
    &\leq 4\beta_{t-1}(4\sqrt{e\pi}) \E \sqbr{\max_{j\in S(x_t, \wtilde\theta_{t-1}^{1:M})} \|x_{tj}\|_{V_{t-1}^{-1}}\given \cal F_t^-, \cal E_t} \P(\cal E_t) \\
    &\leq 16\sqrt{e\pi} \beta_{t-1} \E \sqbr{\max_{j\in S(x_t, \wtilde\theta_{t-1}^{1:M})} \|x_{tj}\|_{V_{t-1}^{-1}}\given \cal F_t^-}, 
\end{align*}
as desired.
\end{proof}

\begin{lemma} [\citet{abramowitz1948handbook}] \label{lem:anti}
For a  Gaussian random variable $Z$ with mean $\mu$ and variance $\sigma^2$, for any $z\geq 1$,
\begin{align*}
    \frac{1}{2\sqrt{\pi}z} e^{-z^2/2} \leq \P(|Z-\mu|> z\sigma) \leq \frac{1}{\sqrt{\pi}z} e^{-z^2/2}.
\end{align*}
\end{lemma}

\begin{proposition} [Proposition 1 of \citet{li2017provably}] \label{prop:li}
    Define $V_t=\sum_{i=1}^t x_i x_i\tp$, where $x_i$ is drawn i.i.d. from some unknown distribution $\nu$ with support in the unit ball, $\B^d$. Furthermore, let $\Sigma:=\E[x_i x_i\tp]$ be the second moment matrix, and $B$ and $\delta>0$ be two positive constants. Then, there exists absolute constants $c',c''>0$ such that $\lambda_{\min}(V_t)\geq B$ with probability at least $1-\delta$, as long as
    \begin{align*}
        t\geq \br{\frac{c'\sqrt{d} + c''\sqrt{\log(1/\delta)}}{\lambda_{\min}(\Sigma)}}^2 + \frac{2B}{\lambda_{\min}(\Sigma)}.
    \end{align*}
\end{proposition}

\begin{lemma} [Hoeffding inequality] \label{lem:hoeffding}
    Let $X_1, \dots, X_n$ be independent random variables such that $a_i \leq X_i \leq b_i$ almost surely. Consider $S_n = \sum_{i=1}^n X_i$. Then for all $t>0$, we have
    \begin{align*}
        \P(S_n - \E[S_n] \geq t) \leq \exp\br{- \frac{2t^2}{\sum_{i=1}^n (b_i-a_i)^2}}
    \end{align*}
\end{lemma}

\begin{lemma} [Hoeffding lemma] \label{lem:hoeffding_2}
    Let $X$ be a real-valued random variable with $a\leq X \leq b$ almost surely. Then, for all $\lambda\in\R$, we have
    \begin{align*}
        \E[\exp(\lambda X)] \leq \exp\br{\lambda \E[X] + \frac{\lambda^2(b-a)^2}{8}}.
    \end{align*}
\end{lemma}

\begin{lemma} [Azuma-Hoeffding inequality] \label{lem:azuma}
If a supermartingale $(Y_i)_{i\geq0}$ corresponding to filtration $\cal G_i$ satisfies $|Y_i-Y_{i-1}|\leq c_i$ for all $t\in[t]$, then for any $a\geq 0$, we have
\begin{align*}
    \P(Y_t - Y_0\geq a) \leq 2 \exp \br{- \frac{a^2}{2 \sum_{i=1}^t c_i^2}}.
\end{align*}
\end{lemma}

%%%%%%%%%%%%%%%%%%%%%%%%%%%%%%%%%%%%%%%%%%%%%%%%%%%%%%%%%%%%

\end{document}